\documentclass[12pt,twoside]{kidgerthesis}

\usepackage{amssymb}
\usepackage{amsmath}
\usepackage{amsthm}
\newcommand{\nicefrac}[2]{\tfrac{#1}{#2}}
\usepackage{subcaption}
\usepackage{booktabs}
\usepackage{tabularx}
\usepackage{ltablex}
\usepackage{multirow}
\usepackage{arydshln}
\usepackage{thmtools}
\usepackage{pbox}
\usepackage{dsfont}
\usepackage[ruled,vlined,noend]{algorithm2e}
\usepackage{thmtools}
\usepackage{pifont}
\usepackage{ifthen}
\usepackage{intcalc}
\usepackage[backend=biber, style=alphabetic, sorting=nyt, giveninits=true]{biblatex}
\usepackage{tikz}
\tikzset{>=latex}
\usepackage{pgfplots}
\pgfplotsset{compat=1.10}
\usepgfplotslibrary{fillbetween}
\usetikzlibrary{matrix, decorations.pathreplacing, calligraphy, arrows.meta, patterns}
\usepackage[raggedright]{titlesec}
\usepackage{imakeidx}
\usepackage{hyperref}
\hypersetup{
	linktoc=all,
	hidelinks=true,
}
\makeindex[intoc]
\usepackage{fancyhdr}
\title{On Neural Differential Equations}

\author{Patrick Kidger}
\college{Mathematical Institute}

\degree{Doctor of Philosophy}
\degreedate{Trinity 2021}

\newcommand{\dd}{\mathrm{d}}
\newcommand{\reals}{\mathbb{R}}
\newcommand{\complexes}{\mathbb{C}}
\newcommand{\naturals}{\mathbb{N}}
\newcommand{\prob}{\mathbb{P}}
\newcommand{\expect}{\mathbb{E}}
\newcommand{\cov}{\mathrm{Cov}}
\newcommand{\bigO}[1]{\mathcal{O}(#1)}
\newcommand{\eval}[2]{\left.#1\right|_{#2}}
\newcommand{\abs}[1]{\left|#1\right|}
\newcommand{\norm}[1]{\left\|#1\right\|}
\newcommand{\set}[2]{\left\{#1\,\left\vert\,#2\vphantom{#1}\right\}\right.}
\newcommand{\indicator}[1]{\mathds{1}_{#1}}
\newcommand{\BV}{\mathrm{BV}}
\newcommand{\Lip}{\mathrm{Lip}}
\newcommand{\softmax}{\mathrm{softmax}}
\newcommand{\normal}[2]{\mathcal{N}\left(#1, #2\right)}
\newcommand{\uniform}[2]{\mathrm{Uniform}[#1, #2]}
\newcommand{\sigmoid}{\mathrm{sigmoid}}
\newcommand{\sig}{\mathrm{sig}}
\newcommand{\logsig}{\mathrm{logsig}}
\newcommand{\Tau}{\mathcal{T}}
\newcommand{\adj}{\Gamma}
\newcommand{\trace}{\mathrm{tr}}
\newcommand{\diag}{\mathrm{diag}}
\newcommand{\kl}[2]{\mathrm{KL}\left(#1\middle\|#2\right)}
\newcommand{\restr}[2]{{\left.\kern-\nulldelimiterspace #1 \right|_{#2}}}
\newcommand{\eye}[1]{I_{#1 \times #1}}
\newcommand{\affine}{L_b}

\newcommand{\tssize}{\omega}
\newcommand{\floor}[1]{\left\lfloor#1\right\rfloor}
\newcommand{\realpart}{\mathrm{Re}}

\newcommand{\ts}[1]{\mathcal{T\!S}\!\left(#1\right)}
\newcommand{\tsalt}[1]{\mathcal{T\!S}\!\left(#1\right)}

\newtheorem{theorem}{Theorem}[chapter]
\newtheorem{theorem-informal}[theorem]{(Informal) Theorem}

\newtheorem{lemma}[theorem]{Lemma}
\newtheorem{corollary}[theorem]{Corollary}
\theoremstyle{plain}
\newtheorem{definition}[theorem]{Definition}
\newtheorem{remark}[theorem]{Remark}
\newtheorem{example}[theorem]{Example}
\newtheorem*{notation*}{Notation}

\begin{document}


\setcounter{secnumdepth}{3}
\setcounter{tocdepth}{2}

\maketitle

\begin{romanpages}          
\begin{centeredpreamble}{Abstract}
	The conjoining of dynamical systems and deep learning has become a topic of great interest. In particular, \textit{neural differential equations} (NDEs) demonstrate that neural networks and differential equation are two sides of the same coin. Traditional parameterised differential equations are a special case. Many popular neural network architectures, such as residual networks and recurrent networks, are discretisations.
	
	NDEs are suitable for tackling generative problems, dynamical systems, and time series (particularly in physics, finance, \ldots) and are thus of interest to both modern machine learning and traditional mathematical modelling. NDEs offer high-capacity function approximation, strong priors on model space, the ability to handle irregular data, memory efficiency, and a wealth of available theory on both sides.
	
	This doctoral thesis provides an in-depth survey of the field.
	
	Topics include: neural \textit{ordinary} differential equations (e.g. for hybrid neural/mechanistic modelling of physical systems); neural \textit{controlled} differential equations (e.g. for learning functions of irregular time series); and neural \textit{stochastic} differential equations (e.g. to produce generative models capable of representing complex stochastic dynamics, or sampling from complex high-dimensional distributions).
	
	Further topics include: numerical methods for NDEs (e.g. reversible differential equations solvers, backpropagation through differential equations, Brownian reconstruction); symbolic regression for dynamical systems (e.g. via regularised evolution); and deep implicit models (e.g. deep equilibrium models, differentiable optimisation).
	
	We anticipate this thesis will be of interest to anyone interested in the marriage of deep learning with dynamical systems, and hope it will provide a useful reference for the current state of the art.
	
%
%
%
\end{centeredpreamble}
\addcontentsline{toc}{section}{Contents}
\tableofcontents            
\begin{preamble}{Originality}
	\subsection*{Statement}
	The writing of this thesis is my original work. The material in this thesis is either (a) my original work either with or without collaborators, or (b) where relevant prior or concurrent work included for reference, so as to provide a survey of the field.
	
	\subsection*{Papers}
	This thesis contains material from the following papers on neural differential equations (organised chronologically):
	
	\newlength{\paperspacing}
	\setlength{\paperspacing}{5pt}
	
	\vspace{\paperspacing}
	
	\textbf{Neural Controlled Differential Equations for Irregular Time Series}\\
	Patrick Kidger, James Morrill, James Foster, Terry Lyons\\
	\textit{Neural Information Processing Systems}, 2020
	
	\vspace{\paperspacing}
	
	\textbf{``Hey, that's not an ODE'': Faster ODE Adjoints via Seminorms}\\
	Patrick Kidger, Ricky T. Q. Chen, Terry Lyons\\
	\textit{International Conference on Machine Learning}, 2021
	
	\vspace{\paperspacing}
	
	\textbf{Neural Rough Differential Equations for Long Time Series}\\
	James Morrill, Cristopher Salvi, Patrick Kidger, James Foster, Terry Lyons\\
	\textit{International Conference on Machine Learning}, 2021
	
	\vspace{\paperspacing}
	
	\textbf{Neural SDEs as Infinite-Dimensional GANs}\\
	Patrick Kidger, James Foster, Xuechen Li, Harald Oberhauser, Terry Lyons\\
    \textit{International Conference on Machine Learning}, 2021
    
	\vspace{\paperspacing}
	
	\textbf{Efficient and Accurate Gradients for Neural SDEs}\\
	Patrick Kidger, James Foster, Xuechen Li, Terry Lyons\\
	\textit{Neural Information Processing Systems}, 2021
	
	\vspace{\paperspacing}
	
	\textbf{Neural Controlled Differential Equations for Online Prediction Tasks}\\
	James Morrill, Patrick Kidger, Lingyi Yang, Terry Lyons\\
	\textit{arXiv:2106.11028}, 2021
	
	\subsection*{Open source software}
	A substantial component of my PhD has been the democratisation of neural differential equations via open-source software development. In particular I have authored or otherwise had a substantial hand in developing:

	\vspace{\paperspacing}
	
	\textbf{Diffrax}\\
	Ordinary, controlled, and stochastic differential equation solvers for JAX.\\
	\url{https://github.com/patrick-kidger/diffrax}
	
	\vspace{\paperspacing}
	
	\textbf{torchdiffeq}\\
	Ordinary differential equation solvers for PyTorch.\\
	\url{https://github.com/rtqichen/torchdiffeq}
	
	\vspace{\paperspacing}
	
	\textbf{torchcde}\\
	Controlled differential equation solvers for PyTorch.\\
	\url{https://github.com/patrick-kidger/torchcde}
	
	\vspace{\paperspacing}
	
	\textbf{torchsde}\\
	Stochastic differential equation solvers for PyTorch.\\
	\url{https://github.com/google-research/torchsde}
	
	\subsection*{Breakdown of contributions}
	My personal contributions to each paper break down as follows.
	
	For the `Neural Controlled Differential Equations for Irregular Time Series' paper. I did the entirety of this paper. James Morrill and James Foster had concurrently worked on similar ideas and were included as authors on the paper as a courtesy.
	
	For the `{``}Hey, that's not an ODE'': Faster ODE Adjoints via Seminorms' paper. I had the idea, theory, wrote the library implementation, and handled the neural CDE and Hamiltonian experiments. Ricky T. Q. Chen performed the experiments for the continuous normalising flows. The written text was joint work between both of us. (And whilst of course it does not appear in the final paper, Ricky T. Q. Chen handled most of the rebuttal.)
	
	For the `Neural Rough Differential Equations for Long Time Series' paper. Cristopher Salvi had the idea of using the log-ODE method to reduce a neural CDE to an ODE. I spotted the practical application to long time series. James Morrill implemented it. James Foster helped with the theory. The written text was joint work between me and James Morrill.
	
	For the `Neural SDEs as Infinite-Dimensional GANs' paper. I had the basic idea, basic theory, and and wrote all of the experimental code. James Foster provided the necessary knowledge of SDE numerics. Xuechen Li had already started writing (and released an early version of) the `torchsde' software library we used. Xuechen Li and I jointly performed subsequent development of the `torchsde' library to extend it for this paper. The more complete idea for the paper was fleshed out jointly in conversations between all three of us. The written text was joint work between all three of us. (Finally, I owe James Foster a debt of thanks: during the development of this paper, he kindly fielded endless questions from me on the topic of SDE numerics.)
	
	For the `Efficient and Accurate Gradients for Neural SDEs' paper. I had the idea and the theory for the Brownian Interval. I had the idea and the theory for gradient-penalty-free training of SDE-GANs. I wrote all the code for this paper. James Foster and I independently had the idea to look for an algebraically reversible SDE solver; the reversible Heun method we ended up using was due to just James Foster. Xuechen Li was included as an author as a courtesy, as the two neural SDE papers were originally intended to be published together as a single paper.
	
	For the `Neural Controlled Differential Equations for Online Prediction Tasks' paper. I had the idea and the abstract theory for this paper. James Morrill came up with cubic Hermite splines with backward differences, and handled the implementation. Lingyi Yang assisted with some datasets.
	
	In every case Terry Lyons was included on each paper as my supervisor.
	
	\subsection*{Previously unpublished}
	
	This thesis includes some previously unpublished material on various topics related to neural differential equations. (Usually on material that was only `half a paper' in size.) This includes material on symbolic regression, universal approximation, parameterisations of neural differential equations, and sensitivities of differential equations.
	
	\subsection*{Other}
	\paragraph{Papers} My PhD work has included several other papers \cite{bonnier2019deep, kidger2020deep, generalised-shapelets, generalised-signatures, signatory}, but as they cover other topics -- ranging from rough path theory to universal approximation -- they do not form a part of this thesis.
	
	\paragraph{Software} Likewise, my PhD work has included the development of several other software libraries \cite{fromfile, torchtyping, equinox, sympytorch}. These software libraries are for the Julia, PyTorch and JAX ecosystems, and offer a variety of tools such as improved import systems, rich type annotations for tensors, and the elevation of parameterised functions to first-class `PyTrees'.
	
	Once again these are not included in this thesis.
\end{preamble}
\begin{preamble}{Acknowledgements}
	A doctoral degree doesn't happen in a vacuum. Getting this far has meant the involvement of numerous people, all of whom I am incredibly fortunate to have in my life.
	
	First and foremost I would like to thank my parents, Penny and Alex. I am so, so lucky to have been raised in the environment that I was, with the opportunities you gave me. You have always been my personal champions.
	
	Mum -- I know having me go to Oxford was always a dream come true for you. Finishing this doctorate means finishing the journey of a lifetime, and it's one that you started me on. Everything I know about mathematics I learnt from you.
	
	Dad -- from electronics to electromagnetism, my fondest memories of childhood are all the time we spent together on the back of an envelope. I don't doubt where my love of this subject comes from. This thesis isn't quite one of those envelopes, but I hope it comes close.
	
	
	Truthfully, I have been drafting and redrafting what to say here, but what can compare to 25 years of unconditional support? I cannot put into words how blessed I feel to have you as my parents.
	
	Thank you to my sister Eleanor, who has always been there for me. Our 4am discussions on topics from philosophy to biology were time well spent. Your kindness inspires me to be a better person. Now -- go and get your own doctorate!
	
	I love you all.
	
	Thank you to all my friends for all the time we have spent together. There are two people who deserve to be highlighted in particular.
	
	To Chloe: thank you. You have been a constant presence in my PhD life, from start to finish. In times of crisis you have offered to make more shopping trips on my behalf than I can count. You have been the best friend a best friend can have.
	
	Thank you to Juliette: for friendship, food, and the south of France. (Where this document began.) Lockdown with you was unquestionably one of the best, and happiest, times of my life.
	
	Thank you to all of my academic collaborators: Ricky T. Q. Chen, Xuechen Li, Miles Cranmer, James Morrill, James Foster, Cristopher Salvi, Adeline Fermanian, Lingyi Yang, Patric Bonnier, and Imanol Perez Arribas.
	
	Across late nights, failed experiments, all-too-soon deadlines, and endless redrafting of a paper or rebuttal -- in a very real way, this work exists because of you.
	
	A particular thank you must go to David Duvenaud, Ben Hambly, James Foster, and Ben Walker, who diligently proofread this manuscript for errors. Thanks to their efforts many typographical mistakes and mathematical boo-boos were squashed. (As is traditional, any errors that remain are of course mine alone.)
	
	Last and certainly not least, thank you to my supervisor, Terry Lyons. Whenever I have needed your help, you have been generous with your time. Whenever I have needed something for my research, you have gone out of your way to help me obtain it. Your guidance over our many conversations has shaped me into the researcher I am today.
\end{preamble}
\end{romanpages}            

\pagestyle{fancy}
\chapter{Introduction}\label{chapter:motivation}
\section{Motivation}
We have two goals in writing this document. One: to satisfy the requirements of a PhD, by writing a thesis describing our original research. Two: to give an accessible survey of the new, rapidly developing, and in our opinion very exciting field of \textit{neural differential equations}. To the best of our knowledge this is the first survey to have been written on the topic.

We hope this will prove useful to the interested reader! Along the way we shall cover a wide variety of applications, both to classical mathematical modelling, and to typical machine learning problems.

\subsection{Getting started}
\paragraph{Prerequisites}
We will assume throughout that the reader is familiar with the basics of ODEs and with the basics of modern deep learning, but we will not assume an in-depth knowledge of either. On the basis that many of our readers may come from a traditional applied mathematics background without much exposure to deep learning, then Appendix \ref{appendix:deep} also provides a summary of the relevant deep learning concepts we shall assume. It also provides references for learning more about deep learning.

The material on neural SDEs will assume familiarity with SDEs.

Beyond these (relatively weak) assumptions, we will introduce concepts as we need them. Various parts of the text will touch on topics such as rough path theory, or numerical methods for differential equations. In each case we assume little-to-no familiarity on the part of the reader, and where necessary provide references for learning more about them.

The next chapter (on neural ODEs) makes an effort to explicitly spell out even `elementary' details such as the existence of solutions to ordinary differential equations, or the use of cross entropy as a loss function. Later chapters assume increasing levels of sophistication; it is recommended to read them in sequential order.

\paragraph{Code}
The reader interested in applying these techniques is strongly encouraged to write some example code.

Each chapter contains a few numerical examples -- usually on toy datasets for ease of understanding. The corresponding code is both available and well-documented: they can be found as the examples of the Diffrax software library \cite{diffrax}, which is written for the JAX framework \cite{jax2018github}.\index{Diffrax}

Indeed standard software libraries for solving and differentiating differential equations make working with NDEs essentially easy. These are discussed in Section \ref{section:numerical:software} (including both Diffrax and other options for other frameworks). These libraries are again well-documented and contain numerous examples.

\paragraph{Experiments}
The material here focuses on presenting the theory of NDEs; correspondingly our numerical examples will tend to be on toy datasets chosen for ease of understanding. Real world (and possibly very large scale) applications of these techniques may be found in the original papers, which are referenced in the text alongside each individual topic.

\subsection{What is a neural differential equation anyway?}

A \textit{neural differential equation} is a differential equation using a neural network to parameterise the vector field. The canonical example is a \textit{neural ordinary differential equation} \cite{neural-odes}:\vspace{-3pt}
\begin{equation*}
	y(0) = y_0 \qquad \frac{\dd y}{\dd t}(t) = f_\theta(t, y(t)).
\end{equation*}

Here $\theta$ represents some vector of learnt parameters, $f_\theta \colon \reals \times \reals^{d_1 \times \cdots \times d_k} \to \reals^{d_1 \times \cdots \times d_k}$ is any standard neural architecture, and $y \colon [0, T] \to \reals^{d_1 \times \cdots \times d_k}$ is the solution. For many applications $f_\theta$ will just be a simple feedforward network.

The central idea now is to use a differential equation solver as part of a learnt differentiable computation graph (the sort of computation graph ubiquitous to deep learning).

As a simple example, suppose we observe some picture $y_0 \in \reals^{3 \times 32 \times 32}$ (RGB and $32 \times 32$ pixels), and wish to classify it as a picture of a cat or as a picture of a dog.

We proceed by taking $y(0) = y_0$ as the initial condition of the neural ODE, and evolve the ODE until some time $T$. An affine transformation\footnote{Commonly referred to as a `linear' transformation in deep learning, although this is not technically correct in the mathematical sense of the word. An affine transformation\index{Affine transformation} takes the form $x \mapsto Wx + b$ with potentially nonzero bias $b$; a linear transformation is one for which $b=0$. The difference will occasionally be important to us so we endeavour to make the distinction.} $\ell_\theta \colon \reals^{3 \times 32 \times 32} \to \reals^2$ is then applied, followed by a softmax, so that the output may be interpreted as a length-2 tuple $(\prob(\text{picture is of a cat}), \prob(\text{picture is of a dog}))$.

This is summarised pictorially in Figure \ref{figure:computation-graph}. In conventional mathematical notation, this computation may be denoted
\begin{equation*}
	\softmax\left(\ell_\theta\left(y(0) + \int_0^T f_\theta(t, y(t)) \,\dd t\right)\right).
\end{equation*}

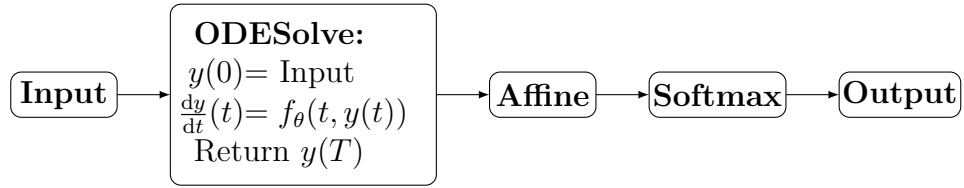
\begin{figure}
	\begin{center}
    \begin{tikzpicture}
    \draw[rounded corners] (-0.1,0.9) rectangle ++(1.4, 0.6) node[pos=.5] {\textbf{Input}};
    \draw[rounded corners] (2,0) rectangle ++(3.5, 2.4) node[pos=.5] {\begin{minipage}{8em}
    \begin{tabular}{@{}r@{}l}
        \multicolumn{2}{l}{\textbf{ODESolve:}}\\
        $y(0)$ &$=$ Input\\
        $\frac{\dd y}{\dd t}(t)$ &$= f_\theta(t, y(t))$\\
        \multicolumn{2}{l}{Return $y(T)$}
    \end{tabular}
    \end{minipage}};
    \draw[rounded corners] (6.2,0.9) rectangle ++(1.4, 0.6) node[pos=.5] {\textbf{Affine}};
    \draw[rounded corners] (8.3,0.9) rectangle ++(1.8, 0.6) node[pos=.5] {\textbf{Softmax}};
    \draw[rounded corners] (10.8,0.9) rectangle ++(1.6, 0.6) node[pos=.5] {\textbf{Output}};
    
    \draw[->] (1.3, 1.2) -- (2, 1.2);
    \draw[->] (5.5, 1.2) -- (6.2, 1.2);
    \draw[->] (7.6, 1.2) -- (8.3, 1.2);
    \draw[->] (10.1, 1.2) -- (10.8, 1.2);
    \end{tikzpicture}
\end{center}
	\caption{Computation graph for a simple neural ODE.}\label{figure:computation-graph}
\end{figure}

The parameters of the model are $\theta$. The computation graph may be backpropagated through and trained via stochastic gradient descent in the usual way. We will discuss how to backpropagate through an ODE solve in Section \ref{section:numerical:adjoint-ode}.

In total, then: there is a neural network $f_\theta$, embedded in a differential equation for $y$, embedded in a neural network (the overall computation graph).

\subsection{A familiar example}
\index{SIR model}
A potentially familiar example of a `neural' differential equation is the classic SIR model:
\begin{equation*}
	\frac{\dd}{\dd t} \begin{pmatrix}s(t)\\i(t)\\r(t)\end{pmatrix} = \begin{pmatrix}-b\,s(t)\,i(t)\\b\,s(t)\,i(t) - k\,i(t)\\k\,i(t)\end{pmatrix}.
\end{equation*}
This is used in mathematical epidemiology to describe the spread of a disease within a population.\footnote{A rather topical choice, with this thesis having being prepared during the global Covid-19 pandemic.} The quantity $s$ represents the susceptible (uninfected) portion of the population, the quantity $i$ represents the infected portion of the population, and the quantity $r$ represents the removed (recovered or deceased) portion of the population.

The vector field is theoretically derived, with parameters $b$ and $k$ describing the infectivity and the (recovery + mortality) rates respectively.

The right hand side may be regarded as a particular differentiable computation graph:

\begin{center}
\begin{tikzpicture}
\draw[dashed] (-2.2, -1.1) rectangle ++(3, 1.9) node[pos=.5] {\small\hspace{-3em} Inputs};
\draw[rounded corners=3pt] (0, 0) rectangle ++(0.7, 0.7) node[pos=.5] {$s(t)$};
\draw[rounded corners=3pt] (0, -1) rectangle ++(0.7, 0.7) node[pos=.5] {$i(t)$};

\draw[dashed] (-2.2, -3.1) rectangle ++(3, 1.9) node[pos=.5] {\small\hspace{-2.8em} Parameters};
\draw[rounded corners=3pt] (0, -2) rectangle ++(0.7, 0.7) node[pos=.5] {$b$};
\draw[rounded corners=3pt] (0, -3) rectangle ++(0.7, 0.7) node[pos=.5] {$k$};

\draw[rounded corners=3pt] (1, -0.5) rectangle ++(0.7, 0.7) node[pos=.5] {$\times$};
\draw[rounded corners=3pt] (2, -0.5) rectangle ++(0.7, 0.7) node[pos=.5] {$\times$};
\draw[rounded corners=3pt] (3, -0.5) rectangle ++(0.7, 0.7) node[pos=.5] {$-$};

\draw[-{Latex[length=1.5mm,width=1.5mm]}] (0.7, 0.3) -- (1, -0.15);
\draw[-{Latex[length=1.5mm,width=1.5mm]}] (0.7, -0.7) -- (1, -0.15);
\draw[-{Latex[length=1.5mm,width=1.5mm]}] (1.7, -0.15) -- (2, -0.15);
\draw[-{Latex[length=1.5mm,width=1.5mm]}] (0.7, -1.65) -- (2, -0.2);
\draw[-{Latex[length=1.5mm,width=1.5mm]}] (2.7, -0.15) -- (3, -0.15);
\draw[-{Latex[length=1.5mm,width=1.5mm]}] (0.7, -0.65) -- (2, -1.65);
\draw[-{Latex[length=1.5mm,width=1.5mm]}] (0.7, -2.65) -- (2, -1.65);

\draw[-{Latex[length=1.5mm,width=1.5mm]}] (2.7, -0.15) -- (3, -0.9);
\draw[-{Latex[length=1.5mm,width=1.5mm]}] (2.7, -1.65) -- (3, -0.9);

\draw[-{Latex[length=1.5mm,width=1.5mm]}] (3.7, -0.15) -- (4.5, -0.15);
\draw[-{Latex[length=1.5mm,width=1.5mm]}] (3.7, -0.9) -- (4.5, -1.15);
\draw[-{Latex[length=1.5mm,width=1.5mm]}] (2.7, -1.65) -- (4.5, -2.15);

\draw[rounded corners=3pt] (2, -2) rectangle ++(0.7, 0.7) node[pos=.5] {$\times$};

\draw[rounded corners=3pt] (3, -1.25) rectangle ++(0.7, 0.7) node[pos=.5] {$-$};

\draw[rounded corners=3pt] (4.5, -0.5) rectangle ++(2.9, 0.7) node[pos=.5] {$-bs(t)i(t)$};
\draw[rounded corners=3pt] (4.5, -1.5) rectangle ++(2.9, 0.7) node[pos=.5] {$bs(t)i(t) - ki(t)$};
\draw[rounded corners=3pt] (4.5, -2.5) rectangle ++(2.9, 0.7) node[pos=.5] {$ki(t)$};
\draw[dashed] (4.4, -2.6) rectangle ++(4.7, 2.9) node[pos=.5] {\small\hspace{7.5em} Outputs};
\end{tikzpicture}
\end{center}

The parameters may be fitted by setting up a loss between the trajectories of the model and the observed trajectories in the data, backpropagating through the model, and applying stochastic gradient descent.

This is precisely the same procedure as the more general neural ODEs we introduced earlier. At first glance, the NDE approach of `putting a neural network in a differential equation' may seem unusual, but it is actually in line with standard practice. All that has happened is to change the parameterisation of the vector field.


\subsection{Continuous-depth neural networks}\label{section:motivation:resnet-is-node}
We have just seen how neural differential equations may be approached via traditional mathematical modelling. They may also be arrived at via modern deep learning.

Recall the formulation of a residual network \cite{resnet}\index{Residual networks}:
\begin{equation}\label{eq:resnet}
	y_{j + 1} = y_j + f_\theta(j, y_j),
\end{equation}
where $f_\theta(j, \,\cdot\,)$ is the $j$-th residual block. (The parameters of all blocks are concatenated together into $\theta$.)

Now recall the neural ODE
\begin{equation*}
	\frac{\dd y}{\dd t}(t) = f_\theta(t, y(t)).
\end{equation*}
Discretising this via the explicit Euler method at times $t_j$ uniformly separated by $\Delta t$ gives
\begin{equation*}
	\frac{y(t_{j + 1}) - y(t_j)}{\Delta t} \approx \frac{\dd y}{\dd t}(t_j) = f_\theta(t_j, y(t_j)),
\end{equation*}
so that
\begin{equation*}
	y(t_{j + 1}) = y({t_j}) + \Delta t f_\theta(t_j, y(t_j)).
\end{equation*}
Absorbing the $\Delta t$ into the $f_\theta$, we recover the formulation of equation \eqref{eq:resnet}.

Having made this observation -- that neural ODEs are the continuous limit of residual networks -- we may be prompted to start making other connections.

It transpires that the key features of a GRU \cite{gru} or an LSTM \cite{lstm}, over generic recurrent networks, are updates rules that look suspiciously like discretised differential equations (Chapter \ref{chapter:neural-cde}). StyleGAN2 \cite{stylegan2} and (score based) diffusion models \cite{song2021scorebased} are simply discretised SDEs (Chapter \ref{chapter:neural-sde}). Coupling layers in invertible neural networks \cite{invertible-residual-networks} turn out to be related to reversible differential equation solvers (Chapter \ref{chapter:numerical}). And so on.

By coincidence (or, as the idea becomes more popular, by design) many of the most effective and popular deep learning architectures resemble differential equations. Perhaps we should not be surprised: differential equations have been the dominant modelling paradigm for centuries; they are not so easily toppled.

\subsection{An important distinction}\index{Physics-informed neural network}
There has been a line of work on obtaining numerical approximations to the solution $y$ of an ODE $\nicefrac{\dd y}{\dd t} = f(t, y(t))$ by representing the solution as some neural network $y = y_\theta$.

Perhaps $f$ is known, and the model $y_\theta$ is fitted by minimising a loss function of the form 
\begin{equation}\label{eq:motivation:_pinn}
\min_\theta \frac{1}{N}\sum_{i=1}^N\norm{\frac{\dd y_\theta}{\dd t}(t_i) - f(t_i, y_\theta(t_i))}
\end{equation}
for some points $t_i \in [0, T]$. As such each solution to the differential equation is obtained by solving an optimisation problem. This has strong overtones of collocation methods or finite element methods. This is a popular line of work; see for example \cite{lagaris2, lagaris1, Han8505, NEURIPS2018_d7a84628, JMLR:v19:18-046, PhysRevD.100.016002, RAISSI2019686, Fang2019, zubov2021neuralpde} amongst many others.

This is known as a physics-informed neural network (PINN). PINNs are effective when generalised to some PDEs, in particular nonlocal or high-dimensional PDEs, for which traditional solvers are computationally expensive. (Although in most regimes traditional solvers are still the more efficient choice.) \cite{zubov2021neuralpde} provide an overview.

However, we emphasise that \textit{this is a distinct notion to neural differential equations}. NDEs use neural networks to \textit{specify} differential equations. Equation \eqref{eq:motivation:_pinn} uses neural networks to \textit{obtain solutions to prespecified} differential equations. This distinction is a common point of confusion, especially as the PDE equivalent of \eqref{eq:motivation:_pinn} is sometimes referred to as a `neural partial differential equation'.

\section{The case for neural differential equations}\label{section:motivation:applications}
\subsection{Applications}
To this author's knowledge, there are four main applications for neural differential equations:

\paragraph{Physical (financial, biological, \ldots) modelling}
Mechanistic theory-driven differential equation models are already ubiquitous in classical mathematical modelling. However, such theory-driven models will at some point fail to capture the details of reality. By combining existing models with deep learning (with its high-capacity function approximators), we may close the gap between theory and observation.

\paragraph{Time series}
Messy or irregular data is ubiquitous in time series. Different channels may be observed at different frequencies, data may be missing, time series may be of variable lengths, and so on. Treating discrete data in a continuous-time regime offers a way to treat irregular data on the same footing as `regular' data.

Connections to topics such as system identification and reinforcement learning may also be made here, although they will not feature heavily in the present work.

\paragraph{Generative modelling}
Generative modelling studies how to model some target distribution $\nu$, from which typically we only have samples. The usual framework is to pick a `friendly' distribution $\mu$, and then learn a map $F$ such that (the pushforward) $F(\mu)$ approximates the target distribution $\nu$.

It transpires that effective choices for $F$ are derived from differential equations. For example with continuous normalising flows then $\mu$ may be a normal distribution (Section \ref{section:ode:cnf}); in the case of a neural SDE then $\mu$ may be (the law of) a Brownian motion (Chapter \ref{chapter:neural-sde}).
	
\paragraph{Inspiration}
Traditional `discrete' deep learning is widely applicable, and rightly so. We have already seen the parallels between differential equations and deep learning: a highly successful strategy for the development of deep learning models is simply to take the appropriate differential equation, and then discretise it.

%

\subsection{Advantages}
In summary, neural differential equations offers a best-of-both-worlds approach.

The neural network-like structure offers high-capacity function approximation and easy trainability. 

The differential equation-like structure offers strong priors on model space, memory efficiency, and theoretical understanding via a well-understood and battle-tested literature.

Relative to the classical differential equation literature, neural differential equations have essentially unprecedented modelling capacity. Relative to the modern deep learning literature, neural differential equations offer a coherent theory of `what makes a good model'.

\section{A note on history}
Practically speaking, the \textit{topic} of neural differential equations become a \textit{field} only a few years ago, starting with the explosion of interest following \cite{neural-odes}; other prominent recent work also includes \cite{E2017, haber-ruthotto}.

However, many of the basic ideas can be found in substantially older literature, often from the 1990s. For example in \cite{historical-1}, a neural ODE is trained to match the dynamics of a chemical reaction, using an MLP for the vector field. Meanwhile the basics of learning a controlled dynamical system are given in \cite{historical-2}. \cite{historical-3} consider hybridising neural ODEs with traditional theory-driven mechanistic modelling, and \cite{historical-4} use implicit integrators in conjunction with neural ODEs to learn stiff dynamical systems.

This list of examples is by no means exhaustive. 
The above references are all short and make for easy reading, so the curious reader is encouraged to look them up.
\chapter{Neural Ordinary Differential Equations}\label{chapter:neural-ode}
\section{Introduction}\label{section:neural-ode-overview}
By far the most common neural differential equation is a neural ODE \cite{neural-odes}:
\begin{equation}\label{eq:neural-ode}
	y(0) = y_0 \qquad \frac{\dd y}{\dd t}(t) = f_\theta(t, y(t)),
\end{equation}
where $y_0 \in \reals^{d_1 \times \ldots \times d_k}$ is an any-dimensional tensor, $\theta$ represents some vector of learnt parameters, and $f_\theta \colon \reals \times \reals^{d_1 \times \ldots \times d_k} \to \reals^{d_1 \times \ldots \times d_k}$ is a neural network. Typically $f_\theta$ will be some standard simple neural architecture, such as a feedforward or convolutional network.

\subsection{Existence and uniqueness}\index{Existence}\index{Uniqueness} The first question typically asked (at least by mathematicians) is about existence and uniqueness of a solution to equation \eqref{eq:neural-ode}. This is straightforward. Provided $f_\theta$ is Lipschitz -- something which is typically true of a neural network, which is usually a composition of Lipschitz functions -- then Picard's existence theorem \cite[Theorem 110C]{butcher2016numerical} applies:

\begin{theorem}[Picard's Existence Theorem]\label{theorem:picard-ode}\index{Picard's existence theorem}
	Let $f \colon [0, T] \times \reals^d \to \reals^d$ be continuous in $t$ and uniformly Lipschitz\footnote{That is, it is Lipschitz in $y$ and the Lipschitz constant is independent of $t$: there exists $C > 0$ such that for all $t, y_1, y_2$ then $\norm{f_\theta(t, y_1) - f_\theta(t, y_2)} \leq C \norm{y_1 - y_2}$.} in $y$. Let $y_0 \in \reals^d$. Then there exists a unique differentiable $y \colon [0, T] \to \reals^d$ satisfying
	\begin{equation*}
			y(0) = y_0 \qquad \frac{\dd y}{\dd t}(t) = f(t, y(t)).
	\end{equation*}
\end{theorem}

\subsection{Evaluation and training}
As compared to models that are not differential equations, there are two extra concerns that must generally be kept in mind.

First, we must be able to obtain numerical solutions to the differential equation. (An analytic solution will essentially never be available.) Second, we must be able to backpropagate through the differential equation, to obtain gradients for its parameters $\theta$.

Software for performing these tasks is now standardised (Section \ref{section:numerical:software}), so we are free to focus on the task of constructing the model architecture itself. A more in-depth look at evaluation and backpropagation is given in Chapter \ref{chapter:numerical}.

\section{Applications}
\subsection{Image classification}\label{section:image-classification}\label{section:ode:image-classification}\index{Image classification}
Image classification with CNNs is nearly everybody's first introduction to deep learning. It is a natural place to start discussing neural differential equations too.

\paragraph{Dataset} Suppose we observe some images, represented as a 3-dimensional tensor $\reals^{3 \times 32 \times 32}$, corresponding to channels (red, green, blue), height (32 pixels), and width (32 pixels) respectively. Suppose each image has a corresponding class label in $\reals^{10}$, corresponding to a one-hot encoding of what the image is a picture of: perhaps aeroplane, car, bird, cat, deer, dog, frog, horse, ship or lorry.

\paragraph{Model} Let $f_\theta \colon \reals \times \reals^{3 \times 32 \times 32} \to \reals^{3 \times 32 \times 32}$ be a convolutional neural network, and let $\ell_\theta \colon \reals^{3 \times 32 \times 32} \to \reals^{10}$ be affine.

Then we may define an image classification model as
\begin{align*}
	\phi &\colon \reals^{3 \times 32 \times 32} \to \reals^{10},\\
	\phi &\colon y_0 \mapsto \softmax(\ell_\theta(y(T))),
\end{align*}
where $y \colon [0, T] \to \reals^{3 \times 32 \times 32}$ solves
\begin{equation*}
	y(0) = y_0, \qquad \frac{\dd y}{\dd t}(t) = f_\theta(t, y(t)).
\end{equation*}

\paragraph{Loss function} 
By using an appropriate loss function (cross entropy) between this output and the true label, we may train this model so that its output is the probability that the input image is of each of these classes.

Explicitly: given a dataset of images $a_i \in \reals^{3 \times 32 \times 32}$ with corresponding labels $b_i \in \reals^{10}$, for samples $i=1,\ldots,N$, we may minimise the cross-entropy
\begin{equation*}
	-\frac{1}{N}\sum_{i = 1}^N b_i \cdot \log \phi(a_i)
\end{equation*}
by training $\theta$, where $\cdot$ denotes a dot product and $\log$ is taken elementwise.

\paragraph{This example is an example only.} In practice, for applications such as image classification there is usually little to be gained by using a continuous-time model. Traditional residual networks (that is, explicitly discretised neural ODEs) are simply easier to work with.

As such this example is an example only. We do not actually suggest using neural ODEs for this task, for which standard neural networks are likely to be superior.

\paragraph{The manifold hypothesis}\index{Manifold hypothesis} Neural ODEs interact elegantly with the manifold hypothesis (that the data lies on or near some low-dimensional manifold embedded in the higher-dimensional feature space; Appendix \ref{appendix:deep:manifold-hypothesis}). The ODE describes a flow along which to evolve the data manifold.

\subsection{Physical modelling with inductive biases}\label{section:ode:hybrid}\index{Inductive biases}
Endowing a model with any known structure of a problem is known as giving the model an \textit{inductive bias}. `Soft' biases through penalty terms are one common example. `Hard' biases through explicit architectural choices are another.

Physical problems often have known structure, and so a common theme has been to build in inductive biases by hybridising neural networks into this structure. It is this author's prediction that this will shortly become a standard technique in the toolbox of applied mathematical modelling. (If, arguably, it isn't already.)


\subsubsection{Universal differential equations}\label{section:ode:ude}\index{Universal differential equations}

\index{Lotka--Volterra model}
Consider the Lotka-Volterra model, which is a well known approach for modelling the interaction between a predator species and a prey species:
\begin{align}
	\frac{\dd x}{\dd t}(t) = \alpha x(t) - \beta x(t) y(t) \in \reals,\nonumber\\
	\frac{\dd y}{\dd t}(t) = -\gamma x(t) + \delta x(t) y(t) \in \reals.\label{eq:lotka-volterra}
\end{align}
Here, $x(t) \in \reals$ and $y(t) \in \reals$ represent the size of the population of the prey and predator species respectively, at each time $t \in [0, T]$. The right hand side is theoretically constructed, representing interactions between these species.

This theory will not usually be perfectly accurate, however. There will be some gap between the theoretical prediction and what is observed in practice. To remedy this, and letting $f_\theta, g_\theta \colon \reals^2 \to \reals$ be neural networks, we may instead consider the model
\begin{align}
	\frac{\dd x}{\dd t}(t) = \alpha x(t) - \beta x(t) y(t) + f_\theta(x(t), y(t)) \in \reals,\nonumber\\
	\frac{\dd y}{\dd t}(t) = -\gamma x(t) + \delta x(t) y(t) + g_\theta(x(t), y(t)) \in \reals,\label{eq:lotka-volterra2}
\end{align}
in which an existing theoretical model is augmented with a neural network correction term.

We broadly refer to this approach as a \textit{universal differential equation}, a term due to \cite{rackauckas2020universal}.\footnote{There is little unified terminology here. Other authors have considered essentially the same idea under other names; conversely \cite{rackauckas2020universal} additionally consider variations and extensions to SDEs, PDEs, and so on.}

\paragraph{Loss function and training} Suppose we observe data $x_{i}(t_j) \in \reals$, $y_{i}(t_j) \in \reals$, where $i=1,\ldots,N$ denote independent observations of the target process (from different initial conditions) and $j=1,\ldots,M$ correspond to different times $t_j \in [0, T]$, with $t_1 = 0$. In practice we may only have $N=1$, which may be sufficient provided $M$ is large enough.

For either \eqref{eq:lotka-volterra} or \eqref{eq:lotka-volterra2}, let $x_{x_0, y_0}(t)$ denote $x(t)$ given initial condition $x(0) = x_0$ and $y(0) = y_0$. Similarly for $y_{x_0, y_0}(t)$.

Then we may fit both \eqref{eq:lotka-volterra} and \eqref{eq:lotka-volterra2} in precisely the same way: stochastic gradient descent with respect to the loss function
\begin{equation*}
	\frac{1}{NM}\sum_{i=1}^N\sum_{j=1}^M (x_{x_i(0), y_i(0)}(t_j) - x_{i}(t_j))^2 + (y_{x_i(0), y_i(0)}(t_j) - y_{i}(t_j))^2.
\end{equation*}

In switching from \eqref{eq:lotka-volterra} to \eqref{eq:lotka-volterra2}, then no fundamental part of the modelling procedure has changed.

\begin{remark}
	The above presentation implicitly assumes that the locations of the observations $t_j$ were the same for both $x$ and $y$, and were the same for all training samples. This is just for simplicity of presentation and is not necessary in general.
\end{remark}

\paragraph{High capacity function approximation} By switching from \eqref{eq:lotka-volterra} to \eqref{eq:lotka-volterra2}, the high-capacity function approximation provided by the neural networks $f_\theta$, $g_\theta$ offers a way to close the gap between theory and practice. The neural network may be used to model the residual between the theoretical and the observed data.

The use of a neural network is an admission that \textit{there is behaviour we do not understand}: but through this augmentation, we can at least model.

These networks will frequently be very small by the standards of deep learning: \cite{ling_kurzawski_templeton_2016} consider a feedforward network of 10 layers each of width 10. \cite{rackauckas2020universal} consider feedforward networks of width 32 and a single hidden layer.

\paragraph{Use cases} This approach becomes natural whenever one is attempting to model complex poorly understood behaviour, and for which there is sufficient data that the theoretical model clearly falls short.

Derivation of closure relations is a neat example. In this case, the differential equation features a term that lacks a precise theoretical description (representing the effects over scales smaller that the numerical solver can resolve), so the strategy becomes to approximate this term with a neural network, and learn this term from data.

Turbulence modelling is a popular example of this. In a Reynolds-averaged Navier Stokes model, \cite{ling_kurzawski_templeton_2016} approximate the closure relation (the Reynolds stresses) using a neural network carefully designed to satisfy certain physical invariances. See also the substantial follow-up literature: \cite{turb1, turb2, turb3} and so on. Meanwhile as part of a climate model for the ocean, \cite{ramadhan2020capturing} model a closure relation (for turbulent vertical heat flux) using a small MLP.

\paragraph{How to train your UDE} Training \eqref{eq:lotka-volterra2} directly (via gradient descent) may not produce an interpretable model. The parameters $\alpha, \beta, \gamma, \delta$ may not necessarily correspond to their usual quantities, if the neural network has modelled some part of the behaviour as well.

One resolution is to fit \eqref{eq:lotka-volterra} first, use its parameters to initialise $\alpha, \beta, \gamma, \delta$ in \eqref{eq:lotka-volterra2}, and then train only the network parameters $\theta$. This will ensure that the neural network only fits the residual between the theoretical model and the observed data.

Another option is to regularise the norm of the neural network \cite{yin2021augmenting}, so that it is used only when necessary.

Another concern when training is that the model may become stuck in a local minimum. (Because the neural networks used with UDEs are often very small.) This may be mitigated by training on the first proportion of a time series (say the first 10\%) before training on the whole time series; more generally setting some `length schedule'\index{Length schedule} that uses an increasing fraction of the time series as training progresses.

\subsubsection{Hamiltonian neural networks}\label{section:ode:hamiltonian}\index{Hamiltonian neural networks}
Another approach is to suppose that the observed dynamics evolve according to a Hamiltonian system; a realistic assumption for many physical systems. With respect to some known canonical coordinates $q, p \in \reals^d$ and an unknown Hamiltonian function $H \colon \reals^d \times \reals^d \to \reals$, the system is assumed to evolve according to
\begin{align*}
	\frac{\dd q}{\dd t}(t) &= \frac{\partial H}{\partial p}(p(t), q(t)),\\
	\frac{\dd p}{\dd t}(t) &= -\frac{\partial H}{\partial q}(p(t), q(t)).
\end{align*}
By parameterising $H = H_\theta$ as some general neural network (for example just an MLP), this system may be learnt much like a universal differential equation -- in this case, the inductive bias is encoded through the use of a Hamiltonian-derived vector field, rather than explicit inclusion of known terms \cite{greydanus2019hamiltonian}.

\paragraph{Parameterisations of the Hamiltonian}
The Hamiltonian itself could be parameterised as an unstructured neural network, like an MLP. Alternatively one can go further, by parameterising the Hamiltonian according to kinetic and potential energy
\begin{equation*}
	H_\theta(q, p) = \frac{1}{2}p^\top M_\theta^{-1}(q)p + V_\theta(q),
\end{equation*}
where now $M_\theta$ is a learnt positive-definite mass matrix, and $V_\theta$ is a learnt potential energy \cite{symode2, Zhong2020Symplectic}.

\paragraph{Control terms}
Encoding this minimal amount of prior knowledge also makes available tools from classical dynamics. For example, we may suppose that the system responds to a control term $\beta$ according to
\begin{align*}
	\frac{\dd q}{\dd t}(t) &= \frac{\partial H}{\partial p}(p(t), q(t)),\\
	\frac{\dd p}{\dd t}(t) &= -\frac{\partial H}{\partial q}(p(t), q(t)) + g_\theta(q)\beta(q),
\end{align*}
where $g_\theta$ is some neural network. After the system has been learnt from data, then controllers may be synthesised from this description \cite{Zhong2020Symplectic}.

\subsubsection{Lagrangian neural networks}\index{Lagrangian neural networks}
One weakness of the Hamiltonian approach is that it assumes knowledge of the canonical coordinates $q, p$. In general our observed data from a dynamical system may not match up against this canonical structure.

An alternative is to instead parameterise the Lagrangian. Given positions $q \in \reals^d$ and velocities $\dot{q} = \nicefrac{\dd q}{\dd t} \in \reals^d$, a Lagrangian is parameterised as some neural network function of them both, $\mathcal{L}_\theta(q, \dot{q})$. The Euler--Lagrange equations state that a system with Lagrangian $\mathcal{L}_\theta$ evolves according to
\begin{equation*}
	\frac{\dd}{\dd t}\frac{\partial \mathcal{L}_\theta}{\partial \dot{q}} = \frac{\partial \mathcal{L}_\theta}{\partial q}.
\end{equation*}
Rearranging, we may obtain
\begin{equation*}
	\ddot{q} = \left(\frac{\partial^2 \mathcal{L}_\theta}{\partial^2\dot{q}}\right)^{-1} \left(\frac{\partial \mathcal{L}_\theta}{\partial q} - \frac{\partial^2 \mathcal{L}_\theta}{\partial q \partial \dot{q}}\dot{q}\right),
\end{equation*}
where $\nicefrac{\partial^2 \mathcal{L}_\theta}{\partial^2\dot{q}}$ is a Hessian and so $(\nicefrac{\partial^2 \mathcal{L}_\theta}{\partial^2\dot{q}})^{-1}$ denotes a matrix inverse. Once again this defines a dynamical system which may be fitted directly to data as described for universal differential equations. See \cite{cranmer2020lagrangian}.

\subsection{Continuous normalising flows}\label{section:ode:cnf}\index{Continuous normalising flows}
We now switch from supervised learning to unsupervised learning. Suppose we observe some distribution $\prob$ with a density $\pi$ over some state space $\reals^{d_1 \times \cdots \times d_k}$. We wish to learn an approximation to $\prob$.

For example we may have $\reals^{d_1 \times \cdots \times d_k} = \reals^{3 \times 32 \times 32}$, and $\prob$ may denote a probability distribution over `pictures of cats', from which we have empirical samples. By learning a generative model approximating $\pi$, we may produce synthetic pictures of cats. (An important task.)

Let $d = \prod_{m=1}^k d_m$ and for simplicity we replace $\reals^{d_1 \times \cdots \times d_k}$ with $\reals^d$.

Consider the random neural ODE defined by
\begin{equation}\label{eq:ode:cnf}
	y(0) \sim \normal{0}{\eye{d}},\qquad \frac{\dd y}{\dd t}(t) = f_\theta(t, y(t)) \text{ for $t \in [0, T]$.}
\end{equation}
We seek to train this model such that the distribution of $y(1)$ (induced by the pushforward of $y(0) \sim \normal{0}{\eye{d}}$ by $y(0) \mapsto y(1)$) is approximately $\prob$. This is called a continuous normalising flow (CNF) \cite{neural-odes, ffjord}. See Figure \ref{fig:cnf-abstract}.

\begin{figure}
	\centering
	\includegraphics[width=0.5\linewidth]{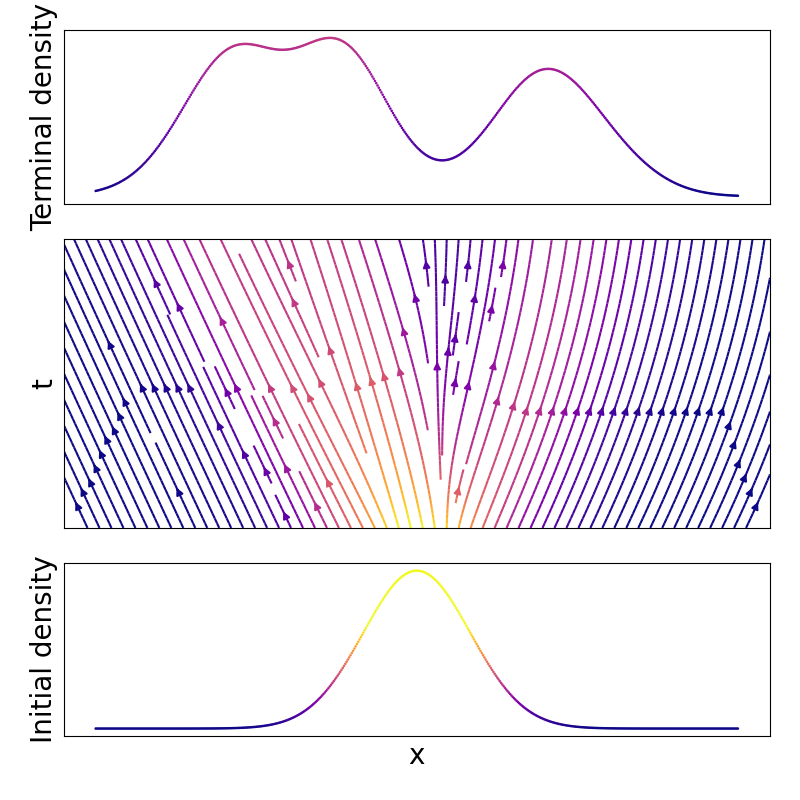}
	\caption{A continuous normalising flow continuously deforms one distribution into another distribution. The flow lines show how particles from the base distribution are perturbed until they approximate the target distribution.}\label{fig:cnf-abstract}
\end{figure}

\subsubsection{Sampling} Sampling from a trained model is straightforward: sample $y(0) \sim \normal{0}{\eye{d}}$ and then solve \eqref{eq:ode:cnf}.

\subsubsection{Instantaneous change of variables}\index{Instantaneous change of variables} We still need to train the model. We will proceed via maximum likelihood, which means that we need a tractable expression for the density of the distribution of $y(1)$.

\begin{theorem}[Instantaneous change of variables]
	Recall equation \eqref{eq:ode:cnf}. Assume $f_\theta = (f_{\theta, 1}, \ldots, f_{\theta, d})$ is Lipschitz continuous. Let
	\begin{equation*}
		p_\theta \colon [0, T] \times \reals^d \to \reals,
	\end{equation*}
	where $p_\theta(t,\,\cdot\,)$ is the density of $y(t)$ for each time $t \in [0, T]$. (In some works written informally as `$p(y(t))$'.) The subscript $\theta$ in $p_\theta$ denotes the dependence on $f_\theta$.
	
	Then $p_\theta$ evolves according to the differential equation\footnote{Actually, just an integral: $\log p_\theta$ does not appear on the right hand side.}
	\begin{equation}\label{eq:ode:instantaneous-change-of-variables}
		\frac{\dd}{\dd t}\big(t \mapsto \log p_\theta(t, y(t))\big)(t) = - \sum_{k=1}^d \frac{\partial f_{\theta, k}}{\partial y_k}(t, y(t)),
	\end{equation}
	where $y = (y_1,\ldots,y_d) \in \reals^d$.
	
\end{theorem}

The right hand side of \eqref{eq:ode:instantaneous-change-of-variables} is the divergence of $f$, or equivalently the trace of the Jacobian of $f$. The latter description draws the analogy to the change of variables formulas for normalising flows (Appendix \ref{appendix:deep:normalising-flows}).

See \cite[Appendix A]{neural-odes} for a straightforward proof.

\begin{remark}\index{Fokker--Planck equation}
	The SDE theorist will find this expression familiar. It is the Fokker--Planck equation for deterministic dynamics, subject to a random initial condition. It has been carefully written so that the right hand side is independent of the unknown $p_\theta$.
\end{remark}

\paragraph{Training}
By solving \eqref{eq:ode:instantaneous-change-of-variables} we can train a CNF via maximum likelihood. Given any terminal condition $x \in \reals^d$, let $y(t, x)$ denote the solution to the ODE
\begin{equation}\label{eq:ode:backward-cnf}
	y(T, x) = x,\qquad \frac{\dd y}{\dd t}(t, x) = f_\theta(t, y(t, x)) \text{ for $t \in [0, T]$}
\end{equation}
which will be solved backwards in time from $t=T$ to $t=0$.

Given a batch of empirical samples $y_1, \ldots, y_N \in \reals^d$, maximum likelihood states that with respect to $\theta$, we should minimise
\begin{equation*}
	-\frac{1}{N}\sum_{i=1}^N \log p_\theta(T, y_i).
\end{equation*}
Substituting in \eqref{eq:ode:instantaneous-change-of-variables}, we obtain
\begin{equation}\label{eq:ode:backward-cnf-density}
	-\frac{1}{N}\sum_{i=1}^N \log p_\theta(T, y_i) = -\frac{1}{N}\sum_{i=1}^N \left[\log p_\theta(0, y(0, y_i)) - \int_0^T \sum_{k=1}^d \frac{\partial f_{\theta, k}}{\partial y_k}(t, y(t, y_i))\,\dd t\right].
\end{equation}

This is now possible to evaluate.
\begin{enumerate}
\item Starting from some empirical sample $y_i \in \reals^d$, we may solve equation \eqref{eq:ode:backward-cnf} backwards-in-time from $t=T$ to $t=0$.
\item As the solution progresses we obtain $y(t, x)$ for $t=T$ to $t=0$. This is an input to the right hand side of \eqref{eq:ode:backward-cnf-density}. This integral may be solved as part of this backwards-in-time solve -- just concatenate the integral together with \eqref{eq:ode:backward-cnf} to form a system of differential equations.
\item Finally, evaluate $\log p_\theta(0, y(0, y_N))$ -- recalling that $p_\theta(0, \,\cdot\,)$ is taken to be a normal distribution -- and add it together with the value of the integral in order to obtain a value for \eqref{eq:ode:backward-cnf-density}.
\end{enumerate}

Having evaluated \eqref{eq:ode:backward-cnf-density}, it is backpropagated and the parameters $\theta$ updated via gradient descent.\footnote{And as the `forward pass' involved a derivative, then the backward pass will compute a second derivative; this is fine.} Note that backpropagation is a `reverse time' procedure. In summary, and as we have already performed one reversal:
\begin{itemize}
	\item Evaluating \eqref{eq:ode:backward-cnf-density} involves solving from $t=T$ to $t=0$;
	\item Backpropagating through \eqref{eq:ode:backward-cnf-density} is an operation progressing from $t=0$ to $t=T$;
	\item Additionally, note that sampling involves solving from $t=0$ to $t=T$, and is only performed at inference time.
\end{itemize}


\subsubsection{Example}\label{section:ode:cnf-example}\index{Examples!Continuous normalising flows}
As a fun example, consider a greyscale image, which we may regard as a map $f \colon [0, 1]^2 \to [0, 255]$. We may fit a continuous normalising flow to $f$, treating $f$ as the unnormalised density for a probability distribution over $\reals^2 \supseteq [0, 1]^2$. A selection of images, and some CNFs that have learnt to approximate them, are shown in Figure \ref{fig:cnf-example}.

\begin{figure}
	\includegraphics[width=\linewidth]{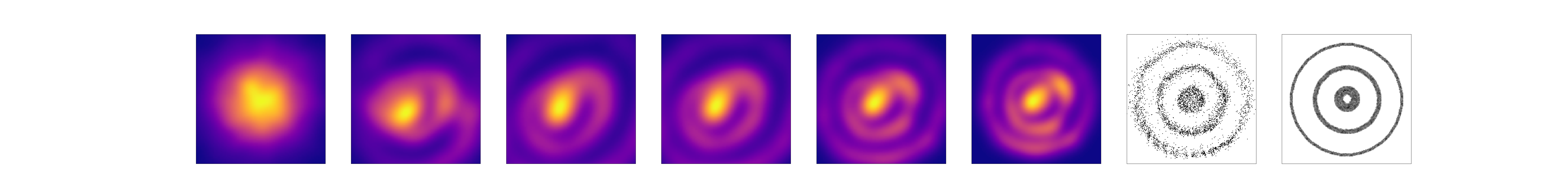}
	\includegraphics[width=\linewidth]{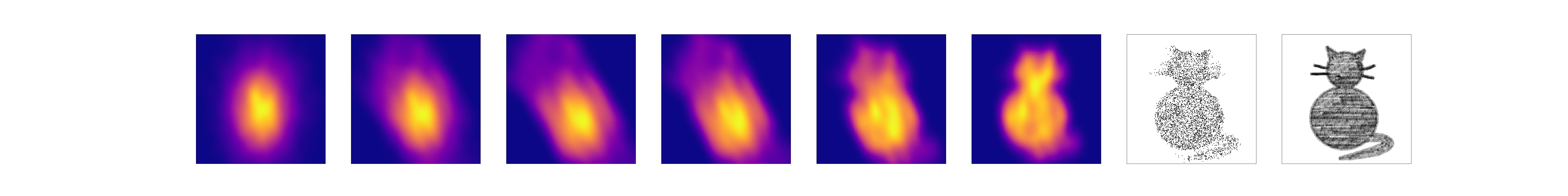}
	\includegraphics[width=\linewidth]{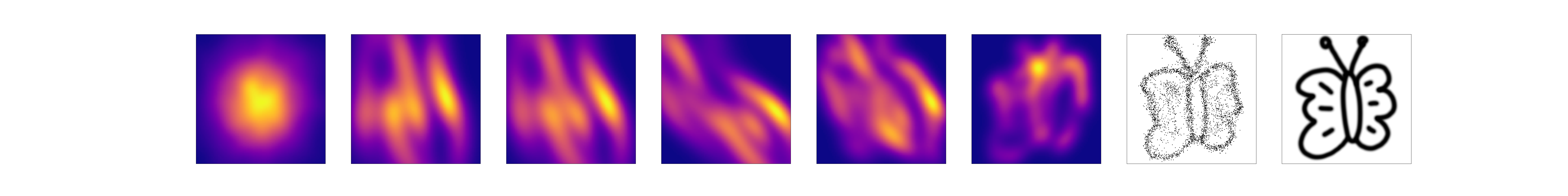}
	\caption{Continuous normalising flow example. \textbf{From top to bottom:} target, cat, butterfly. \textbf{From left to right:} the first six pictures show the evolution of the distribution of the CNF as it integrates from $t=0$ to $t=T$, transforming a normal distribution into the desired distribution. The second to last picture shows samples from the learnt CNF. The final picture shows the image used to specify the desired distribution.}\label{fig:cnf-example}
\end{figure}

We see that CNFs are capable of learning relatively complex two-dimensional distributions, including those with multiple modes (such as the different concentric rings of the target), and those with fine-scale `filaments' stretching away from the main part of the distribution (such as the whiskers or tail of the cat).

See Appendix \ref{appendix:experimental:cnf} for further details of this experiment. The code is available as an example in Diffrax \cite{diffrax}.

CNFs are a highly flexible approach to modelling probability distributions. \cite{ffjord, how-to-train-node} apply this approach to image generation. That is, the samples from the learnt distribution are images, rather than the whole distribution resembling an image as above. \cite{pointflow} represent 3D models as distributions (much like the above example representing a picture as a 2D distribution), and use this approach to generate point clouds of the model.

\subsubsection{Efficient estimation of the trace-Jacobian}\label{section:ode:hutchinson}\index{Hutchinson's trace estimator}
Note how \eqref{eq:ode:instantaneous-change-of-variables}, and thus \eqref{eq:ode:backward-cnf-density}, involve evaluating the expression $\sum_{k=1}^d \nicefrac{\partial f_{\theta, k}}{\partial y_k}(t, y)$. This is possible simply via autodifferentiation software: evaluate the neural network $f_\theta(t, y)$, and then backpropagate.

There is one foible. Autodifferentiation calculates a product of Jacobians (Appendix \ref{appendix:deep:autodiff}), which this expression is not. It may be calculated by performing $k=1,\ldots,d$ such operations, but this implies a relatively expensive $\bigO{d^2}$ cost. Each evaluation of $f_\theta \colon \reals \times \reals^d \to \reals^d$ requires at least $\bigO{d}$ work. Each subsequent autodifferentiation operation also requires $\bigO{d}$ work; so far for a total of still only $\bigO{d}$. That we must make $k=1,\ldots,d$ such calls is what raises this to $\bigO{d^2}$ cost.

We can do better.

\paragraph{Hutchinson's trace estimator}
Let $A \in \reals^{d \times d}$ be any matrix. Let $\varepsilon$ be a random variable over $\reals^d$ such that $\expect[\varepsilon] = 0 \in \reals^d$ and $\cov[\varepsilon] = \eye{d}$. (For example, a multivariate normal or Rademacher random variable.) Then
\begin{equation*}
	\trace(A) = \expect_\varepsilon[\varepsilon^\top A \varepsilon].
\end{equation*}
The Monte-Carlo approximation derived from this equation is known as Hutchinson's trace estimator \cite{hutchinson1989stochastic}.

\paragraph{The trace-Jacobian}
That the right hand side of \eqref{eq:ode:instantaneous-change-of-variables} is a trace-Jacobian now proves useful. We have that
\begin{equation*}
	\sum_{k=1}^d \frac{\partial f_{\theta, k}}{\partial y_k}(t, y) = \trace\left(\frac{\partial f_\theta}{\partial y}(t, y)\right) = \expect_\varepsilon\left[\varepsilon^\top \frac{\partial f_\theta}{\partial y}(t, y) \varepsilon\right].
\end{equation*}

Substituting into \eqref{eq:ode:backward-cnf-density} we obtain
\begin{equation}\label{eq:ode:_hutchinson-cnf}
	-\frac{1}{N}\sum_{i=1}^N \log p_\theta(T, y_i) = -\frac{1}{N}\sum_{i=1}^N \left[\log p_\theta(0, y(0, y_i)) - \expect_\varepsilon\left[\int_0^T \varepsilon^\top \frac{\partial f_\theta}{\partial y}(t, y(t, y_i))\varepsilon\,\dd t\right]\right].
\end{equation}\

In practice this expectation will often be approximated by a single Monte-Carlo sample, which as in \eqref{eq:ode:_hutchinson-cnf} is held constant for the duration of the integration. Training already involves averaging over the batch of data $i=1,\ldots,N$ and so further Monte-Carlo samples are often unnecessary.

And now for punchline: the integrand of \eqref{eq:ode:_hutchinson-cnf} may be computed in only $\bigO{d}$ work. First $\varepsilon^\top \nicefrac{\partial f_\theta}{\partial y}(t, y(t, y_i))$ may be computed as vector-Jacobian product (requiring only $\bigO{d}$ work), and then this is combined with the final $\varepsilon$ via a simple dot product (also only $\bigO{d}$ work). Overall this produces an unbiased estimate of the divergence.

\subsubsection{Comparison to normalising flows}
Recall the discussion on normalising flows from Appendix \ref{appendix:deep:normalising-flows}. In both cases, a change in log-probability densities is described in terms of the Jacobian of the transformation.

Note the difference in computational complexity. In the general normalising flow setting, the log-determinant-Jacobian costs $\bigO{d^3}$ work to evaluate (and backpropagate through). Here it has been reduced to just $\bigO{d^2}$ or $\bigO{d}$ work.

\subsection{Latent ODEs}\label{section:ode:latent-ode}\index{Latent ODEs}
The previous section considered a generative model for data from some distribution without a time-varying component. For example, a static picture of a cat, rather than samples from a dynamical system evolving in time. We now consider the case that the distribution has an intrinsic time-varying component -- for example, it may be a distribution over time series. Once again, we wish to model this distribution.

Consider the space of $d$-dimensional irregularly-sampled time series
\begin{equation*}
	\tsalt{\reals^d} = \set{((t_0, x_0), \ldots, (t_n, x_n))}{n \in \naturals, t_j \in \reals, x_j \in \reals^d, t_0 = 0, t_j < t_{j + 1}}.
\end{equation*}
For ease of presentation we suppose this is fully-observed (without missing data), but the following construction extends immediately to the partially-observed case too.

We proceed by constructing a VAE. Figure \ref{fig:latent-ode-abstract} provides a summary of the construction we about to present. This is termed a latent ODE \cite{neural-odes, latent-odes}.

\begin{figure}\centering
	\begin{tikzpicture}
\begin{scope}[shift={(2.4,0)}]
\draw[thick, ->] (0.5, 0) -- (5.5, 0);
\node at (1, 0)[circle,fill, inner sep=1.5pt] {};
\node at (1.6, 0)[circle,fill, inner sep=1.5pt] {};
\node at (3, 0)[circle,fill, inner sep=1.5pt] {};
\node at (5, 0)[circle,fill, inner sep=1.5pt] {};
\node at (1, -0.3) {$t_0$};
\node at (1.6, -0.3) {$t_1$};
\node at (3, -0.3) {$t_2$};
\node at (4, -0.3) {$\cdots$};
\node at (5, -0.3) {$t_n$};

\node at (1, 1)[circle,draw=black, fill=cyan, inner sep=1.5pt] {};
\node at (1.6, 0.7)[circle,draw=black, fill=cyan, inner sep=1.5pt] {};
\node at (3, 0.6)[circle,draw=black, fill=cyan, inner sep=1.5pt] {};
\node at (5, 0.8)[circle,draw=black, fill=cyan, inner sep=1.5pt] {};
\node at (1, 0.7) {$x_0$};
\node at (1.6, 0.4) {$x_1$};
\node at (3, 0.3) {$x_2$};
\node at (5, 0.5) {$x_n$};

\node at (4, 1.2) {\footnotesize$x_j \sim p_{\theta, y(t_j)}$};

\draw[red!85!black!80!white, thick, cap=round] (1, 2) .. controls ++(0:0.2) and ++(150:0.4) .. (2, 1.7);
\draw[red!85!black!80!white, thick, cap=round] (2, 1.7) .. controls ++(150+180:0.4) and ++ (185:0.3) .. (3, 1.6);
\draw[red!85!black!80!white, thick, cap=round, ->] (3, 1.6) .. controls ++(185+180:0.3) and ++ (165:0.7) .. (5, 1.8);
\node at (3.5, 2.5) {\begin{minipage}{0.3\textwidth}\footnotesize Decoder: ODE solution\\$y(t) = \int_0^t f_\theta(y(s))\,\dd s$\end{minipage}};

\draw[dashed, ->] (1, 1.95) -- (1, 1.1);
\draw[dashed, ->] (1.6, 1.8) -- (1.6, 0.8);
\draw[dashed, ->] (3, 1.55) -- (3, 0.7);
\draw[dashed, ->] (5, 1.75) -- (5, 0.9);
\end{scope}

\draw[thick, ->] (-0.5, 0) -- (-5.5, 0);
\node at (-1, 0)[circle,fill, inner sep=1.5pt] {};
\node at (-1.6, 0)[circle,fill, inner sep=1.5pt] {};
\node at (-3, 0)[circle,fill, inner sep=1.5pt] {};
\node at (-5, 0)[circle,fill, inner sep=1.5pt] {};
\node at (-1, -0.3) {$t_0$};
\node at (-1.6, -0.3) {$t_1$};
\node at (-3, -0.3) {$t_2$};
\node at (-4, -0.3) {$\cdots$};
\node at (-5, -0.3) {$t_n$};

\node at (-1, 1)[circle,draw=black, fill=cyan, inner sep=1.5pt] {};
\node at (-1.6, 0.7)[circle,draw=black, fill=cyan, inner sep=1.5pt] {};
\node at (-3, 0.6)[circle,draw=black, fill=cyan, inner sep=1.5pt] {};
\node at (-5, 0.8)[circle,draw=black, fill=cyan, inner sep=1.5pt] {};
\node at (-1, 0.7) {$x_0$};
\node at (-1.6, 0.4) {$x_1$};
\node at (-3, 0.3) {$x_2$};
\node at (-5, 0.5) {$x_n$};

\draw[dashed, ->] (-1, 1.1) -- (-1, 1.8);
\draw[dashed, ->] (-1.6, 0.8) -- (-1.6, 1.8);
\draw[dashed, ->] (-3, 0.7) -- (-3, 1.7);
\draw[dashed, ->] (-5, 0.9) -- (-5, 1.5);

\draw[green!80!black, thick, cap=round, ->] (-5, 1.5) -- (-5, 1.7) -- (-3, 1.7) -- (-3, 2) -- (-1.6, 2) -- (-1.6, 1.8) -- (-1, 1.8) -- (-1, 2) -- (-0.3, 2);
\node at (-3, 2.4) {\footnotesize Encoder $\nu_\theta$};

\draw[rounded corners, fill=blue!15!white] (-0.3, 1.55) rectangle (0.2, 2.45);
\node at (-0.05, 2.2) {\footnotesize$\mu$};
\node at (-0.05, 1.8) {\footnotesize$\sigma$};
\draw[thick, ->, cap=round] (0.2, 2) -- (3.4, 2);
\node at (1.8, 2.4) {\footnotesize $y_0 \sim g_\theta(\normal{\mu}{\sigma^2})$};
	\end{tikzpicture}
	\caption{Overview of latent ODE model.}\label{fig:latent-ode-abstract}
\end{figure}
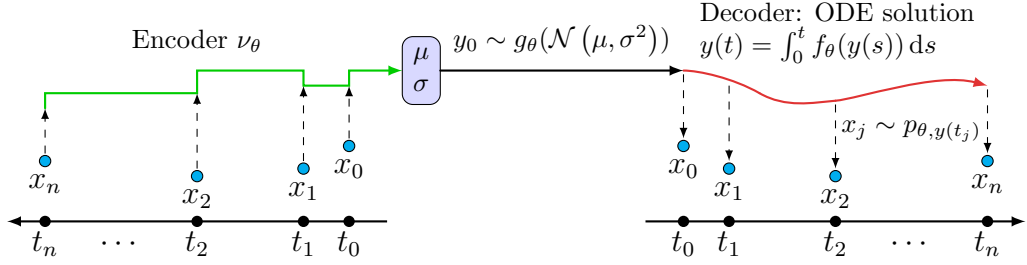

\begin{remark}
	The use of a VAE raises the question of whether other generative approaches (GANS, \ldots) may be employed. The answer is yes, and indeed Chapter \ref{chapter:neural-sde} (Neural Stochastic Differential Equations) will be almost entirely dedicated to the problem of generative time series models.
\end{remark}

\paragraph{Decoder}
Fix $d_l, d_m > 0$ as the dimensionality of two latent spaces. (Typically $d_l, d_m \gg d$.) Let 
\begin{align*}
	f_\theta &\colon \reals^{d_l} \to \reals^{d_l},\\
	g_\theta &\colon \reals^{d_m} \to \reals^{d_l}
\end{align*}
be neural networks parametrised by learnt parameters $\theta$. Let $p_{\theta, y} \colon \reals^{d_l} \to [0, \infty)$ be some probability density parameterised by $y \in \reals^{d_l}$ and learnt parameters $\theta$. (For simplicity of notation we stack all learnt parameters together into a single vector $\theta$.)

Given $z \in \reals^{d_m}$, let $y_0 = g_\theta(z) \in \reals^{d_l}$ and let $y(t, y_0)$ be the solution of the neural ODE
\begin{equation}\label{eq:ode:_latent_ode}
	y(0, y_0) = y_0,\qquad\frac{\dd y}{\dd t}(t, y_0) = f_\theta(y(t, y_0)).
\end{equation}
For each time $t$ we consider $p_{\theta, y(t, y_0)}$. The full collection of $t \mapsto p_{\theta, y(t, y_0)}$ is the output of the model.

That is, given some input $z \in \reals^{d_m}$, it is mapped into the latent space $\reals^{d_l}$, from which the ODE $y$ evolves. At each time the latent value parameterises a probability distribution.

This is the decoder of the VAE. 

\begin{example}\label{example:ode:_latent-ode}
Frequently $p_{\theta, y}$ may simply be taken to be a Gaussian with fixed variance: let $\ell_\theta \colon \reals^{d_l} \to \reals^d$ be affine and let $p_{\theta, y}$ be the density of $\normal{\ell_\theta(y)}{\eye{d_l}}$. We will discuss other choices of $p_{\theta, y}$ in a moment.
\end{example}

For any $\mathbf{x} = ((t_0, x_0), \ldots, (t_n, x_n)) \in \tsalt{\reals^d}$, let
\begin{equation*}
\rho_{\theta, y_0}(\mathbf{x}) = \prod_{j=0}^n p_{\theta, y(t_j, y_0)}(x_j),
\end{equation*}
which corresponds to the probability density of a full time series $\mathbf{x}$, rather than of just a single observation at a single point in time.

\paragraph{Encoder}
The encoder of the VAE is some $\nu_\theta \colon \tsalt{\reals^d} \to \reals^{d_m} \times (0, \infty)^{d_m}$; frequently an RNN or a neural CDE (Chapter \ref{chapter:neural-cde}).

The encoder output is the statistics of a multivariate normal distribution with diagonal covariance; we denote this by $q_{\theta, \mathbf{x}} = \normal{\mu_{\theta, \mathbf{x}}}{\diag(\sigma_{\theta, \mathbf{x}})^2}$ with $(\mu_{\theta, \mathbf{x}}, \sigma_{\theta, \mathbf{x}}) = \nu_\theta(\mathbf{x})$.

If the encoder is an RNN or neural CDE it will sometimes be run backwards-in-time over the input time series, so that the decoder starts where the encoder ends.

\paragraph{Training}

Given a batch or dataset of time series $\mathbf{x}_1, \ldots, \mathbf{x}_N \in \tsalt{\reals^d}$, then the end-to-end optimisation criterion is to minimise $\theta$ with respect to
\begin{equation*}
	\frac{1}{N}\sum_{i=1}^N \left[ \expect_{y_0 \sim q_{\theta, \mathbf{x}_i}} \left[-\log \rho_{\theta, y_0}(\mathbf{x}_i) \right] + \kl{q_{\theta, \mathbf{x}_i}}{\normal{0}{\eye{d_l}}} \right].
\end{equation*}
This is simply the standard VAE optimisation criterion, and provided $p_{\theta, y}$ is `reasonable' (for example, a Gaussian), then this expression may be evaluated and backpropagated through in the usual way. The first term ensures that the decoder learns to replicate its input samples; the second term ensures that the initial distribution in the latent space matches a known distribution, which may be sampled from at inference time.

\paragraph{Sampling} Sampling from the model is straightforward in the usual way for VAEs: sample some $z \sim \normal{0}{\eye{d_m}}$, evaluate $y_0 = g_\theta(z)$, and evaluate \eqref{eq:ode:_latent_ode} forward in time. $p_{\theta, y(t, y_0)}$ is the model output. If a point statistic is required (for example, just a sample from the model) then the mean of $p_{\theta, y(t, y_0)}$ may be returned.

\paragraph{Choice of distribution} The choice of $p_{\theta, y}$ is dependent on what behaviour is desired when sampling during inference time. If only a point statistic is required then the choice of $p_{\theta, y}$ is essentially just a choice of loss function, and simple choices like Gaussian distributions (Example \ref{example:ode:_latent-ode}) or Laplace distributions (whose log-likelihood is the $L^1$ distance) are sensible.

If the full model output $p_{\theta, y(t, y_0)}$ is of interest -- for example to perform uncertainty quantification -- then more expressive choices of $p_{\theta, y}$ may be of interest. For example, \cite{clpf} consider taking it to be a normalising flow (and additionally consider replacing the latent ODE with a latent SDE; see Chapter \ref{chapter:neural-sde}).

\subsubsection{Examples}\label{section:ode:latent-ode-example}\index{Examples!Latent ODEs}
As a simple example, consider a dataset of decaying oscillators. That is, a 2-dimensional time series consisting of (discrete observations of)
\begin{equation}\label{eq:ode:_latent-ode-example}
y(t) = \exp(At)\,y_0
\end{equation}
with $y_0, y(t) \in \reals^2$,  $A \in \reals^{2 \times 2}$, and such that the eigenvalues of $A$ are complex with negative real component. Samples look like decaying sine and cosine waves.

We take $y_0 \sim \normal{0}{\eye{2}}$, and generate sample data from \eqref{eq:ode:_latent-ode-example} at irregularly sampled timestamps over $[0, 3]$. The timestamps are not regularly spaced nor are they consistent between different batch elements.

We fit a latent ODE to this dataset. At test time, we solve the ODE over the larger interval $[0, 12]$. See Figure \ref{fig:latent-ode-example} for some samples generated from this model. We see that by the end of training, excellent samples are produced, even though they are over a time interval four times larger than the model was trained on.

\begin{figure}
	\includegraphics[width=\linewidth]{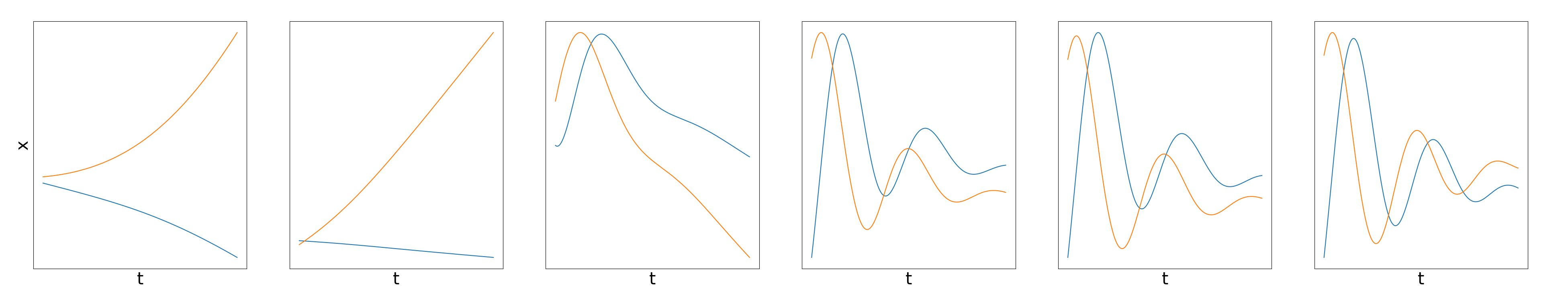}
	\caption{Plots of samples $y \colon [0, 12] \to \reals^2$ drawn from the latent ODE model. The leftmost picture is a sample from the untrained model. The rightmost picture is a sample from a fully-trained model. In between are samples from partially-trained models. Quality increases as training proceeds.}\label{fig:latent-ode-example}
\end{figure}

See Appendix \ref{appendix:experimental:latent-odes} for precise details. The code is available as an example in Diffrax \cite{diffrax}.

\paragraph{Irregular sampling}\index{Time series!Irregular}\index{Irregular sampling}
The continuous-time approach handles several irregular kinds of sampling without issue: the input data is not regularly spaced, nor are different batch elements sampled at the same times.

Meanwhile, the output is over the (continuous-time) interval $[0, 12]$, so that we are obtaining samples at all times. This is unlike the analogous RNN, which would be restricted to producing outputs only at prespecified discrete timestamps.

\paragraph{Extrapolation}
Figure \ref{fig:latent-ode-example} shows that the latent ODE has successfully reproduced this dataset. Moreover it exhibits good extrapolation qualities over an interval four times longer than the interval it was trained on.

\paragraph{Other examples}
Many other types of time series problem may be considered. For example \cite{latent-odes} apply a latent ODE to model the dynamics of a small (simulated) frog jumping into the air; \cite{model-based-semi-markov} consider applications to reinforcement learning; \cite{latent-segmenting} combine latent ODEs with changepoint detection algorithms to model switching dynamical systems.

We additionally direct the reader towards Chapter \ref{chapter:neural-sde}, in which neural SDEs will also be used to model (much more general) distributions over time series.

\subsubsection{Sequence-to-sequence models}\index{Sequence-to-sequence models}
Essentially the same construction may be used in the construction of sequence-to-sequence models, for example to perform time series forecasting. The encoder (an RNN or neural CDE, see Chapter \ref{chapter:neural-cde}) runs over the input time series; the decoder (a neural ODE or neural SDE, see Chapter \ref{chapter:neural-sde}) produces the forecasted sample.

\subsection{Residual networks}\label{section:ode:residual-networks}\index{Residual networks}
In Section \ref{section:motivation:resnet-is-node} we saw that residual networks are the explicit Euler discretisation of a neural ODE.

Correspondingly the theory of dynamical systems offers ways to derive variant residual networks with favourable properties.

\subsubsection{Rotational vector fields}\label{section:ode:rotational-vector-field}\index{Rotational vector fields}
\cite{haber-ruthotto} consider replacing the forward pass of a residual network
\begin{equation*}
	y_{j+1} = y_j + f_\theta(y_j)
\end{equation*}
with
\begin{align}
	y_{j+1} &= y_j + \sigma(K_j z_{j+1} + b_j),\nonumber\\
	z_{j+1} &= z_j - \sigma(K_j^\top y_j + b_j)\label{eq:ode:_rotational}
\end{align}
for some weights and biases $K_j, b_j$ and some choice of activation function $\sigma$. This corresponds to a semi-implicit Euler discretisation of the neural ODE\index{Semi-implicit Euler method}
\begin{equation*}
	\frac{\dd}{\dd t}\begin{pmatrix}y\\z\end{pmatrix}(t) = \sigma\left(W(t)\begin{pmatrix}y(t)\\z(t)\end{pmatrix} + b(t)\right)
\end{equation*}
where
\begin{equation*}
	W(t) = \begin{pmatrix}0 & K(t)\\-K(t)^\top & 0\end{pmatrix}.
\end{equation*}

Correspondingly, the Jacobian of the right hand side is
\begin{equation*}
	\diag\left(\sigma'\left(W(t)\begin{pmatrix}y(t)\\z(t)\end{pmatrix} + b(t)\right)\right) W(t)
\end{equation*}
Many activation functions are monotonic; if this is the case then the Jacobian is the product of a diagonal matrix with positive entries, and an antisymmetric matrix, and as such the Jacobian has pure-imaginary eigenvalues.\footnote{Proof: let $D$ be positive diagonal with square root $A$. Let $W$ be antisymmetric. Then $DW$ is similar to $AWA$, which is antisymmetric and as such has pure-imaginary eigenvalues.}

This means that the vector field is `purely rotational': eigenvalues with positive real part drive expansion; eigenvalues with negative real part cause contraction, but zero real part produces neither. Correspondingly, \eqref{eq:ode:_rotational} is largely immune to vanishing/exploding gradient issues.

\begin{remark}\label{remark:ode:rotational-vf}
The trade-off, however, is a potential reduction in expressivity. Purely rotational vector fields are volume preserving (divergence-free). Non-volume-preservation is often important for expressivity. (It is even part of the name of Real Non-Volume Preserving flows \cite{real-nvp}.)

This issue of volume preservation may be partially ameliorated by working in a higher dimensional space; see Section \ref{section:ode:augmentation-rotational} later.
\end{remark}

\subsubsection{Momentum residual networks}\label{section:ode:momentum-residual-network}\index{Momentum residual network}\index{Residual networks!Momentum}
\cite{momentum-residual-networks} consider replacing the forward pass through a residual network with
\begin{align}
	v_{j+1} &= \gamma v_j + (1 - \gamma) f_\theta(y_j)\nonumber\\
	y_{j+1} &= y_j + v_{j+1}\label{eq:ode:_momentum-residual-network}
\end{align}
for $\gamma \in (0, 1)$. ($\gamma = 0.9$ would be typical.)

\paragraph{Reversibility}
The key property of such networks is that they are \textit{reversible}: whilst \eqref{eq:ode:_momentum-residual-network} computes $(v_{j+1}, y_{j+1})$ from $(v_j, y_j)$, it also possible to reconstruct $(v_j, y_j)$ from $(v_{j+1}, y_{j+1})$ via
\begin{align}
	y_j &= y_{j+1} - v_{j+1}\nonumber\\
	v_j &= \frac{1}{\gamma}(v_{j+1} - (1 - \gamma) f_\theta(y_j)).\label{eq:ode:_momentum-residual-network-backward}
\end{align}

This dramatically improves the memory efficiency of the network, at the cost of some extra computation. When backpropagating through \eqref{eq:ode:_momentum-residual-network}, the intermediate values $y_n$ need not be stored (like they would be for the corresponding residual network). Instead, \eqref{eq:ode:_momentum-residual-network-backward} means they can be recomputed on-demand as backpropagation proceeds.

This represents a refinement of similar ideas in \cite{revnet, revnet2}, and more generally the topic of \textit{invertible neural networks} \cite{invertible-residual-networks}.

\paragraph{As a neural ODE}
Let $\varepsilon = 1/(1 - \gamma)$. Then \eqref{eq:ode:_momentum-residual-network} is given by the semi-implicit Euler method, with unit step size, applied to
\begin{equation}\label{eq:ode:_momentum-residual-network-continuous}
	\varepsilon \frac{\dd^2 y}{\dd t^2}(t) + \frac{\dd y}{\dd t}(t) = f_\theta(y(t)).
\end{equation}

\paragraph{Connection to reversible solvers}\index{Reversible solvers}
Momentum networks are reversible because the semi-implicit Euler method is reversible. Running the solver forwards in time, then backwards in time, will recover the same numerical solution. This is sometimes described as saying that there are matching truncation errors on the forward and backward solves.

Reversible solvers come strongly recommended for use with neural differential equations for the same reason as here: they allow for backpropagation that is both time and memory efficient (Section \ref{section:numerical:reversible-solvers}). As such they are of substantial interest, and moreover in general do not require the second-order structure that \eqref{eq:ode:_momentum-residual-network-continuous} exhibits.

\subsubsection{Alternative integration schemes}
Other off-the-shelf integration schemes may be substituted for the explicit Euler method.

For example \cite{yiping-lu} consider linear multistep methods and \cite{imex-net} consider IMEX methods. PolyNet \cite{polynet} considers operations of the form
\begin{equation*}
	y_{j+1} = y_j + f_\theta(y_j) + f_\theta(f_\theta(y_j)),
\end{equation*}
which if $f_\theta$ is a linear contraction, is an approximation to the implicit Euler method
\begin{align*}
	y_{j+1} &= y_n + f_\theta(y_{j+1})\\
	&= (I - f_\theta)^{-1}(y_{j})\\
	&= (I + f_\theta + f_\theta^2 + f_\theta^3 + \cdots)(y_{j+1}).
\end{align*}

We note that the advantages of switching to a different integration scheme are strongest when it exhibits particular additional properties, like symplecticity as in Section \ref{section:ode:rotational-vector-field} or reversibility as in Section \ref{section:ode:momentum-residual-network}.

\section{Choice of parameterisation}\label{section:ode:parameterisation}
So far we have touched only lightly on the parameterisation of the vector field $f_\theta$. (Although we have discussed some mathematically-inspired parameterisations, such as Hamiltonian-based parameterisations in Section \ref{section:ode:hamiltonian}.)

Should $f_\theta$ be a feedforward network, convolutional network, residual network, \ldots? Should it use batch normalisation? What kinds of activation functions are appropriate? And so on.

Good architectural choices and good choices of optimiser are often crucial for success. However (even with the following guidelines) it is not always clear what good choices are. Frequently this is still just a matter of hyperparameter optimisation -- or perhaps `try it and see what works'.

\subsection{Neural architectures}
Nearly every work uses either a feedforward or convolutional neural network for the vector field $f_\theta$. Feedforward networks are straightforward: simply concatenate $t$ and $y(t)$ together as inputs. These are what are typically used when the data is anything other than an image.

If the data $y_0$ has the (channel, height, width) structure of an image, then a suitable vector field may be obtained by using convolutional layers. Recall that the input and output of $f_\theta(t, \cdot)$ must be the same size. This typically means either using padding, or combining convolutional layers with transposed convolutional layers. Time $t$ is often appended to $y(t)$ as an additional channel.

\begin{remark}\index{Graph neural networks}
	Other parameterisations are occasionally used. For example \cite{graph-neural-ode, miles-symbolic, graph-neural-ode2, grand} consider graph neural networks, which can for example encode equivariance with respect to permutations of the input points. (Such as may be exhibited in many physical systems; for example the positions of $n$ equally-sized masses evolving under gravity.)
	
	Much of the following discussion carries through to this setting, although we will not discuss graph-structured networks and graph-structured data in detail here.
\end{remark}

\subsubsection{Activation functions}
The theory of backpropagating through ODEs does technically ask that the vector field (and thus the activation function) be continuously differentiable (Section \ref{section:numerical:adjoint-ode}), which ReLUs are not.

As such continuously differentiable activation functions like SiLU \cite{gelu, silu, swish}, softplus, or tanh are typically used.\footnote{\cite{squareplus} cooked up, and reports being fond of, the `squareplus' activation $x \mapsto \frac{1}{2} x + \frac{1}{2}\sqrt{x^2 + 4}$.}

Despite this theoretical point, ReLU activations are still often used successfully in practice.

\subsubsection{Normalisation}
Normalisation schemes, such as batch normalisation and layer normalisation \cite{batchnorm, layernorm}, are typically not used, at least within the vector field $f_\theta$. For batch normalisation, this is because the same neural network $f_\theta$ is evaluated at $y(t)$ for different $t$, and each might have different statistical properties. This is the same problem that occurs when using batch normalisation in recurrent neural networks \cite{layernorm, recurrent-batchnorm}. 

Meanwhile layer normalisation lacks a satisfying explanation for its lack of efficacy, but at least for CNFs it has been reported that this typically breaks training \cite{layer-norm-cnf}.

\subsubsection{Initialisation}\label{section:ode:initialisation}
Initialising the neural vector fields close to zero often improves training, it being easier to perturb a nearly constant $t \mapsto y(t)$ than random initial dynamics. For most neural architectures this may be accomplished by choosing the initial parameters $\theta$ close to zero.

\subsection{Non-autonomy} We have deliberately chosen to include $t$ as an input to the vector field $f_\theta$. A residual network has different layers at different depths. Analogously, neural ODE models usually exhibit higher modelling capacity by allowing $f_\theta$ to depend on the `continuous depth' parameter $t$. Such differential equations are referred to as being \textit{non-autonomous}.

This can be handled simply by concatenating $t$ and $y(t)$ together as inputs to $f_\theta$. A far more expressive choice is to additionally explicitly encode certain time dependencies.

\subsubsection{Depth discretisation: stacking}\label{section:ode:stacking}\index{Jump!In the vector field}\index{Stacking}
One straightforward and effective approach is to parameterise $f_\theta$ in piecewise fashion as several different networks, selected based on the value of $t$. For example,
\begin{equation*}
	f_\theta(t, y) = 
	\begin{cases}
		f_{\theta_1, 1}(t, y) & t \in [t_0, t_1]\\
		\qquad\vdots\\
		f_{\theta_n, n}(t, y) & t \in [t_{n - 1}, t_n]
	\end{cases}
\end{equation*}
where $\theta = (\theta_1, \ldots, \theta_n)$, and each $\theta_j$ is itself some vector of parameters.

In principle each $f_{\cdot, 1}, \ldots, f_{\cdot, n}$ could represent different architectures. Often $f_{\cdot, j}$ will all be the same neural architecture, and differ only in which parameter vector $\theta_i$ they depend upon.

Two options must be considered when using this architecture in practice, with a numerical differential equation solver: whether to use a single call to an ODE solver over $[t_0, t_n]$, or whether to solve over each $[t_i, t_{i + 1}]$ region separately, and call an ODE solver $n$ times. Both options are valid but both introduce details that one should be aware of; we defer this numerical discussion to Section \ref{section:numerical:stacked}.

\subsubsection{Spectral discretisation}\index{Spectral discretisation}
Let $\psi_j \colon [0, T] \to \reals$ be some family of (smooth) functions parameterised by $j \in \{1, \ldots, n\}$. Take the parameter vector $\theta$ to be such that $\theta = (\theta_1, \ldots, \theta_n)$ with $\theta_j \in \reals^{d_\theta}$ for some $d_\theta \in \naturals$. Now define
\begin{equation*}
	\alpha_\theta(t) = \sum_{j = 1}^n \theta_j \psi_j(t).
\end{equation*}

Then another choice of non-autonomy is given by
\begin{equation*}
	f_\theta(t, y(t)) = \widetilde{f}_{\alpha_\theta(t)}(t, y(t)),
\end{equation*}
where $\widetilde{f}$ is some fixed neural network architecture which at time $t$ uses parameters $\alpha_\theta(t) \in \reals^{d_\theta}$.

The choice of $\psi_j$ is up to us. Ideally they should be quite different to each other, for the greatest possible expressivity of the model. For example they could be chosen as Chebyshev polynomials, or as a truncated Fourier basis of sines and cosines (which is what motivates the terminology `spectral discretisation' \cite{dissecting}).

\subsubsection{Hypernetworks}\index{Hypernetworks}
Another choice is to let the parameters of the neural ODE be themselves parameterised as the solution of a neural ODE.

That is, let $\alpha \colon [0, T] \to \reals^{d_\alpha}$ be the solution of the neural ODE
\begin{equation*}
	\alpha(0) = \alpha_\theta \qquad \frac{\dd \alpha}{\dd t}(t) = g_\theta(t, \alpha(t)),
\end{equation*}
with learnt parameters $\theta$, vector field $g_\theta \colon \reals \times \reals^{d_\alpha} \to \reals^{d_\alpha}$, and learnt initial condition $\alpha_\theta$.

We then let the hidden state $y$ of our `original' neural ODE evolve according to
\begin{equation*}
	y(0) = y_0 \qquad \frac{\dd y}{\dd t}(t) = \widetilde{f}_{\alpha(t)}(t, y(t)),
\end{equation*}
where $\widetilde{f}$ is some fixed neural network architecture which at time $t$ uses parameters $\alpha(t) \in \reals^{d_\alpha}$.

In practice these two differential equations may be concatenated and solved simultaneously as a system. Overall this may be seen as just a neural ODE as originally formulated, with a particular beneficial structure to its vector field. See \cite{anode2, ode2ode}.

\subsubsection{Variant layers}\label{section:ode:variant-layers}
Other high-performing time-dependent layers may be dreamt up. For example (and inspired by \cite{ffjord}) the example CNF seen in Section \ref{section:ode:cnf-example} uses an MLP whose affine layers are replaced with layers of the form
\begin{equation*}
	(x, t) \mapsto (Ax + b) * \sigma(ct + d) + et
\end{equation*}
where $x \in \reals^{d_1}$, $t \in \reals$, $A \in \reals^{d_2 \times d_1}$, $b, c, d, e \in \reals^{d_2}$, $\sigma$ denotes the sigmoid function, and $*$ denotes elementwise multiplication.

The dependency on the time $t$ is coming in at each layer of the MLP, rather than being concatenated with $y(t)$ as just another input.

This is reminiscent of gating procedures in GRUs and LSTMs.

\subsubsection{Enforcing autonomy}
One exception to the above procedure sometimes occurs when using neural ODEs equations for time series problems, such as with a latent ODE (Section \ref{section:ode:latent-ode}). In this case, we may sometime suppose that the underlying dynamics are not time-dependent, and would instead prefer to remove $t$ as an input. (The same will often also be true of the upcoming neural CDEs and neural SDEs in Chapters \ref{chapter:neural-cde} and \ref{chapter:neural-sde}.)

\subsection{Augmentation}\label{section:ode:augmentation}\index{Augmentation}
For a moment let us focus on performing image classification with neural ODEs (Section \ref{section:ode:image-classification}); a problem chosen for its simplicity. In Section \ref{section:ode:image-classification}, the input to the model was the same size as the hidden state: both the input picture and hidden state were of shape $\reals^{3 \times 32 \times 32}$. In general however this is neither necessary nor desirable.

`Augmentation' refers to the practice of inserting an affine map between input and initial value, to increase the dimension of the hidden state. That is, given some input $x \in \reals^d$, the initial value of the ODE is taken to be $y(0) = g_\theta(x)$ for some learnt $g_\theta \colon \reals^d \to \reals^{d_l}$ with $d_l > d$, rather than simply $y(0) = x$. We have
\begin{equation*}
	y(0) = g_\theta(x), \qquad \frac{\dd y}{\dd t}(t) = f_\theta(t, y(t)).
\end{equation*}
Standard choices of $g_\theta$ are either zero augmentation: $g_\theta(x) = [x, 0]$, learnt augmentation: $g_\theta(x) = [x, \widetilde{g}_\theta(x)]$ for some learnt $\widetilde{g}_\theta$, or just an affine map: $g_\theta(x)$ is learnt and affine. The choice is usually unimportant; the increase in dimensionality is the main point. In each case, the output of the model is still obtained by applying some affine map $\ell_\theta \colon \reals^{d_l} \to \reals^{d_o}$ to $y(T)$, with $d_o \in \naturals$ the desired output dimensionality.

This improves model performance dramatically. The reason is that the continuous flow of an ODE is incapable of modifying the topology of its input -- so staying in the same space means that topological properties of the input manifold (in the sense of the manifold hypothesis; Appendix \ref{appendix:deep:manifold-hypothesis}) are necessarily preserved. This is a statement we will make precise in Section \ref{section:ode:universal-approximation}, by describing the universal approximation properties of neural ODEs.

Returning now to the general setting (beyond just image classification), we have already seen an example of augmentation: the latent ODE (Section \ref{section:ode:latent-ode}) evolved in some higher-dimensional space $\reals^{d_l}$, and used an affine map to $\reals^d$ to obtain the output.

(Conversely, note that CNFs cannot use augmentation: as with all normalising flows, it is a requirement of the construction that every operation be bijective.)

\begin{remark}\label{remark:ode:markov}\index{Markov assumption}
	Lifting into a higher-dimensional space may be regarded as a relaxation of the Markov property. For $s < t$ then the output $\ell_\theta(y(s))$ does not completely determine $\ell_\theta(y(t))$. In contrast $y(s)$ does determine $y(t)$. (Whether $y$ is the output of an unaugmented neural ODE or the latent value of an augmented neural ODE.)
	
	The Markov setting can be very beneficial if the problem is known to exhibit this structure, in particular when modelling physical systems. If the data is densely sampled then it can then be possible to avoid the ODE solve entirely: estimate $\nicefrac{\dd y}{\dd t}$ with finite differences and do direct supervised regression of $\nicefrac{\dd y}{\dd t}$ against $(t, y)$. See \cite{collocation, markov-karn} for variations on this idea. The Markov setting is also the one used for symbolic regression (Section \ref{section:misc:symbolic}).
	
	In general however the Markov setting is a restrictive assumption usually worth avoiding. The Markov/non-Markov distinction is an important one to watch out for in the NDE literature, as many works have implicitly restricted to the Markov setting without discussion.
\end{remark}

\subsubsection{Second-order-augmentation}
\cite{second-order} introduce an interesting variant on this: they take $d_l = 2d$ and structure the vector field so that the extra dimensions correspond to velocities. For example using a learnt augmentation,
\begin{equation*}
	y(0) = \begin{bmatrix}x\\g_\theta(x)\end{bmatrix}, \qquad y(t) = \begin{bmatrix}s(t)\\v(t)\end{bmatrix}, \qquad \frac{\dd y}{\dd t}(t) = \begin{bmatrix}v(t)\\f_\theta(t, s(t), v(t))\end{bmatrix}.
\end{equation*}
This may also be written as a second-order neural ODE $\nicefrac{\dd^2 s}{\dd t^2}(t) = f_\theta\left(t, s(t), \nicefrac{\dd s}{\dd t}(t)\right)$.

This is a choice that makes particular sense if using neural ODEs to model an oscillatory dynamical system.

\subsubsection{Augmenting rotational vector fields}\label{section:ode:augmentation-rotational}\index{Rotational vector fields}
Recall Remark \ref{remark:ode:rotational-vf}. Augmentation is one way to ameliorate the
lack of expressivity of rotational vector fields.

\begin{example}
Let $a = -1, b=0, c=\frac{1}{2} \in \reals$, and suppose we wish to classify $\{a, c\}$ versus $\{b\}$, by constructing a neural ODE followed by an affine layer. We will prove in Section \ref{section:ode:unaugmented-approximation} that this is actually impossible with an unaugmented neural ODE; there does not exist a flow whose terminal values linearly separate $\{a, c\}$ from $\{b\}$.

However, projecting these into $a = (-1, 0), b =(0, 0), c=(\frac{1}{2},0) \in \reals^2$, then the (volume-preserving) flow
\begin{align*}
	\frac{\dd x}{\dd t}(t) &= y\\
	\frac{\dd y}{\dd t}(t) &= -x
\end{align*}
will linearly separate $\{a, c\}$ from $\{b\}$ after any arbitrarily small amount of time.
\end{example}

\section{Approximation properties}\label{section:ode:universal-approximation}\index{Universal approximation!ODEs}
We now examine the universal approximation properties of neural ODEs, as maps from their initial value to their terminal value. See Appendix \ref{appendix:deep:universal-approximation} for an introduction to the topic of universal approximation.

\subsection{`Unaugmented' neural ODEs are not universal approximators}\label{section:ode:unaugmented-approximation}
Consider the map $y(0) \mapsto y(T)$, where $y \colon [0, T] \to \reals^d$ solves some neural ODE
\begin{equation*}
	\frac{\dd y}{\dd t}(t) = f_\theta(t, y(t)).
\end{equation*}
What functions can this approximate?

Unfortunately, the answer is `not many'. More precisely, the continuous evolution of the ODE ensures that any topological property of its input must be preserved.

Let $a = -1, b=0, c=\frac{1}{2} \in \reals$, and suppose we wish to classify $\{a, c\}$ versus $\{b\}$ by constructing a neural ODE followed by an affine layer. That is, the flow of the ODE should linearly separate $\{a, c\}$ from $\{b\}$.

This is impossible: we are asking that either $a < b$ and $c < b$, or that $a > b$ and $c > b$. Correspondingly, either the trajectories for $a$ and $b$ must cross, or the trajectories for $b$ and $c$ must cross. See Figure \ref{figure:unaugmented-crossing}.

\begin{figure}\centering
	\begin{tikzpicture}[scale=0.8]
		\draw[thick, ->] (0, 0) -- (5, 0);
		\draw[thick, ->] (0, 0) -- (0, 5);
		\node at (5.2, -0.2) {\footnotesize$x$};
		\node at (-0.2, 5.2) {\footnotesize$t$};
		
		\node at (1, 0)[circle,fill, inner sep=1.5pt] {};
		\node at (3, 0)[circle,fill, inner sep=1.5pt] {};
		\node at (4, 0)[circle,fill, inner sep=1.5pt] {};
		\node at (1, -0.35) {\footnotesize$a=-1$};
		\node at (2.5, -0.35) {\footnotesize$b=0$};
		\node at (4, -0.35) {\footnotesize$c=0.5$};
		
		\draw[black, thick, cap=round] (1, 0) .. controls ++(70:0.9) and ++(250:0.7) .. (2.8, 2.2);
		\draw[black, thick, cap=round, ->] (2.8, 2.2) .. controls ++(70:0.7) and ++(280:0.7) .. (2.8, 5);
		
		\draw[black, thick, cap=round] (3, 0) .. controls ++(70:0.6) and ++(320:0.6) .. (1.7, 2.8);
		\draw[black, thick, cap=round, ->] (1.7, 2.8) .. controls ++(140:0.6) and ++(260:0.7) .. (1, 5);
		
		\draw[black, thick, cap=round] (4, 0) .. controls ++(90:0.9) and ++(310:0.9) .. (4, 1.7);
		\draw[black, thick, cap=round] (4, 1.7) .. controls ++(310-180:0.9) and ++(260:0.9) .. (4, 3.6);
		\draw[black, thick, cap=round, ->] (4, 3.6) .. controls ++(260-180:0.9) and ++(260:0.9) .. (4, 5);
		
		\draw[dashed] (1.75, 4) -- (1.75, 6);
		\node at (3.7, 5.5) {\footnotesize Linear separation};
	\end{tikzpicture}
	\caption{ODE flows need to cross to linearly separate $\{a, c\}$ from $\{b\}$.}\label{figure:unaugmented-crossing}
\end{figure}
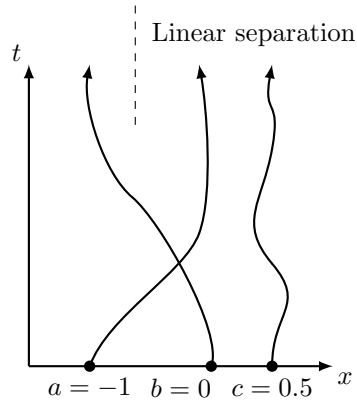

This is a contradiction, as ODE flows never cross.

\begin{remark}
This is not isolated to $d=1$. Higher dimensional counterexamples may be considered by considering analogous `nested shells', in which $\set{y \in \reals^d}{\norm{y} < r_1}$ and $\set{y \in \reals^d}{r_2 < \norm{y} < r_3}$ are classified from each other, with $0 < r_1 < r_2 < r_3$ \cite{augmented-node}.
\end{remark}

\subsection{`Augmented' Neural ODEs are universal approximators, even if their vector fields are not universal approximators}\label{section:ode:augmented-approximation}\index{Augmentation}
Fortunately, this is an issue easily remedied, through augmentation as introduced in Section \ref{section:ode:augmentation}.

\subsubsection{When the vector field is a universal approximator}
We first consider the case that the vector fields are universal approximators.

\begin{theorem}\label{theorem:ode:augmented-approximation}
Fix $d, d_l, d_o \in \naturals$ with $d_l \geq d + d_o$. For $f \in \Lip(\reals \times \reals^{d_l}; \reals^{d_l})$, $\ell_1 \in \affine(\reals^d; \reals^{d_l})$, $\ell_2 \in \affine(\reals^{d_l}; \reals^{d_o})$, let $\phi_{f, \ell_1, \ell_2} \colon \reals^d \to \reals^{d_o}$ denote the map $x \mapsto z$ with
\begin{equation*}
	y(0) = \ell_1(x),\qquad \frac{\dd y}{\dd t}(t) = f(t, y(t)) \quad\text{for $t \in [0, T]$},\qquad z = \ell_2(y(T)).
\end{equation*}

Then
\begin{equation*}
	\set{\phi_{f, \ell_1, \ell_2}}{f \in \Lip(\reals \times \reals^{d_l}; \reals^{d_l}), \ell_1 \in \affine(\reals^d; \reals^{d_l}), \ell_2 \in \affine(\reals^{d_l}; \reals^{d_o})}
\end{equation*}
is a universal approximator for $C(\reals^d; \reals^{d_o})$.
\end{theorem}
(For simplicity this theorem has assumed that the vector field may be drawn from $\Lip(\reals \times \reals^{d_l}; \reals^{d_l})$, not just some dense subset of it.)

See \cite[Theorem 7]{zhang2020approximation} for a short-and-sweet proof of Theorem \ref{theorem:ode:augmented-approximation}.

\subsubsection{When the vector field is not a universal approximator}
Perhaps surprisingly, the condition that the vector field must be a universal approximator is not a necessary condition.
\begin{restatable}{theorem}{myunivapprox}\label{theorem:ode:univ-approx-wide}
	Fix $d, d_o \in \naturals$. For $d_l \in \naturals$, $f \in C(\reals \times \reals^{d_l}; \reals^{d_l})$, $\ell_1 \in \affine(\reals^d; \reals^{d_l})$, $\ell_2 \in \affine(\reals^{d_l}; \reals^{d_o})$, let $\phi_{p, f, \ell_1, \ell_2} \colon \reals^d \to \reals^{d_o}$  be the map $x \mapsto z$ with
	\begin{equation*}
		y(0) = \ell_1(x),\qquad \frac{\dd y}{\dd t}(t) = f(t, y(t)) \quad\text{for $t \in [0, T]$},\qquad z = \ell_2(y(T))
	\end{equation*}
	for those $f$ for which the solution is unique.\footnote{The Peano existence theorem implies existence as $f$ is continuous; but as $f$ is not necessarily Lipschitz then the stronger Picard existence theorem, which gives uniqueness, does not apply.}
	
	For each $d_l \in \naturals$ there exists an $f_{d_l} \in C(\reals^{d_l}; \reals^{d_l})$, for which the above equation has a unique solution, such that
	\begin{equation*}
		\set{\phi_{d_l, f_{d_l}, \ell^1, \ell^2}}{d_l \in \naturals, \ell_1 \in \affine(\reals^d; \reals^{d_l}), \ell_2 \in \affine(\reals^{d_l}; \reals^{d_o})}
	\end{equation*}
	is a universal approximator for $C(\reals^d; \reals^{d_o})$.
\end{restatable}
See Appendix \ref{appendix:ode-universal-augmented} for the proof.

\subsubsection{Comparison}
If the vector field is a universal approximator, then the width of the latent space $d_l$ is fixed, and complexity is obtained through the vector field. In contrast, if the vector field is not a universal approximator, then the latent dimensionality $d_l$ is allowed to become arbitrarily large, and complexity is instead obtained through the affine maps.

\begin{remark}
These have direct analogues in the theory of universal approximation for neural networks.

The case for which the vector field is not a universal approximator is directly analogous to the classical universal approximation theorem, which states that sufficiently wide feedforward neural networks may be used to approximate arbitrary continuous functions \cite{pinkus1999}.

The case for which the vector field is a universal approximator is directly analogous to the `deep and narrow' universal approximation theorem, which states that sufficiently deep feedforward networks, of bounded width, may be used to approximate arbitrary continuous functions \cite{lu-width,hanin-sellke,kidger2020deep,park2021minimum}.
\end{remark}

\section{Comments}
Neural ODEs were originally considered (to the best of this author's knowledge) in early works from the 1990s, such as \cite{historical-2, historical-1, historical-4, historical-3}. A recent revival of the neural-network-as-dynamical-system was started with works such as \cite{E2017, haber-ruthotto}, and popularised (in particular in continuous time) by \cite{neural-odes}. 

Indeed \cite{neural-odes} introduced continuous-time neural ODEs for image classification (Section \ref{section:ode:image-classification}), continuous normalising flows (Section \ref{section:ode:cnf}), and latent ODEs (Section \ref{section:ode:latent-ode}). The latter two were expanded on in \cite{ffjord, latent-odes}.

Applications of neural ODEs to physical problems span multiple literatures; we can give at most a small selection of examples. Examples from machine learning include \cite{greydanus2019hamiltonian, rackauckas2020universal, bellot2021consistency}, whilst examples from engineering include \cite{ling_kurzawski_templeton_2016, Lusch2018, portwood2019turbulence, pyrolysis}. Other examples include physics \cite{physics1}, climate science \cite{ramadhan2020capturing, manifold-node1, hwang2021climate}, epidemiology \cite{pmlr-v144-wang21a}, neuroscience \cite{pmlr-v139-kim21h}, pharmacodynamics \cite{stiff-neural-ode} and so on.

Connections between neural ODEs and their discrete-time counterparts include \cite{haber-ruthotto, augmented-node, momentum-residual-networks, recurrent-depth} amongst others.

Extensions of neural ODEs to handle discontinuities, such as the velocity of a bouncing ball, include \cite{discont1, discont2, discont3}.\index{Jump!In the state}

\cite{manifold-node1, manifold-node2, manifold-node3} generalise CNFs to manifolds, and for example then use CNFs to perform density estimation over distributions on a sphere. \cite{moser-flow} offer a variation suitable for low-dimensional manifolds, that elides the ODE solve.

\cite{how-to-train-node, finlay2020learning, Onken_Wu_Fung_Li_Ruthotto_2021, moser-flow} amongst others discuss connections between CNFs and optimal transport, to select good parameterisations and regularisations for the vector field.

Good parameterisations for the vector field are often to be found by examining the code attached to any given neural ODE paper. Works discussing this topic explicitly include \cite{haber-ruthotto, augmented-node, anode2, ode2ode, dissecting, second-order}.

\cite{augmented-node} note the lack of universal approximation for `unaugmented' neural ODEs. \cite{zhang2020approximation} demonstrate universal approximation with `augmented' neural ODEs provided the vector field is a universal approximator. More subtle universal approximation results may also be found in the literature \cite{node-approx-1, node-approx-2}. The material on universal approximation when the vector is \textit{not} a universal approximator (Theorem \ref{theorem:ode:univ-approx-wide}) is new here.

A few other review articles combining ordinary dynamical systems and deep learning have recently been published, which the reader may find complements this chapter. For example \cite{pmp-proof} focus on interpreting deep learning via control theory, \cite{brunton2020ml-fluids} focus on applications to fluid mechanics, and \cite{thuerey2021pbdl} place great emphasis on performing experiments. Most such works place a strong focus on specifically hybrid neural/mechanistic modelling with neural ODEs, which is our Section \ref{section:ode:hybrid}.
\chapter{Neural Controlled Differential Equations}\label{chapter:neural-cde}
\section{Introduction}\label{section:cde:introduction}
Neural ODEs were the continuous-time limit of residual networks. We will now introduce \textit{neural controlled differential equations} as the continuous-time limit of \textit{recurrent} neural networks. The following chapter will be of particular interest for those studying RNNs or time series; also to those studying rough path theory, control theory, or reinforcement learning.

Controlled differential equations have until recently been relatively esoteric, so we do not assume familiarity with them on the part of the reader. The forthcoming section will form a `mini-chapter' offering a self-contained summary of the key ideas, applications, and \textit{raison d'{\^e}tre} for CDEs and neural CDEs.


Recall the equations for a neural ODE:
\begin{equation}\label{eq:node-in-cde-chapter}
	y(0) = y_0, \qquad y(t) = y(0) + \int_0^t f_\theta(s, y(s)) \,\dd s.
\end{equation}
An extra time-like dimension is introduced and then integrated over. The presence of this extra (artificial) dimension motivates us to consider whether this model can be extended to data already exhibiting sequential structure, such as time series.

Given some ordered data $(y_0, \ldots, y_n)$, the goal is to extend the $y(0) = y_0$ condition to one resembling `$y(0) = y_0,\, \ldots,\, y(n) = y_n$', to align the introduced time-like dimension with the natural ordering of the data. The key difficulty is that the solution of an ODE is determined by the initial condition at $y(0)$, so there is no direct mechanism for incorporating data that arrives later.

Fortunately, it turns out that the resolution of this issue -- how to incorporate incoming information into a differential equation -- is already a well-studied problem in mathematics, via \textit{controlled differential equations}.

Much of this chapter is due to \cite{kidger2020neuralcde}.

\subsection{Controlled differential equations}\label{section:cde:cde}\index{Controlled differential equations}
Let $T > 0$ and let $d_x, d_y \in \naturals$. Let $x \colon [0, T] \to \reals^{d_x}$ be a continuous function of bounded variation. Let $f \colon \reals^{d_y} \to \reals^{d_y \times d_x}$ be Lipschitz continuous. Let $y_0 \in \reals^{d_y}$.

A continuous path $y \colon [0, T] \to \reals^{d_y}$ is said to solve a \textit{controlled differential equation}, controlled or driven by $x$, if
\begin{equation}\label{eq:cde}
	y(0) = y_0,\qquad y(t) = y(0) + \int_0^t f(y(s)) \,\dd x(s)\quad\text{for $t \in (0, T]$}.
\end{equation}
Here `$\dd x(s)$' denotes a Riemann--Stieltjes integral, and `$f(y(s)) \,\dd x(s)$' refers to a matrix-vector multiplication.

\paragraph{Bounded variation and Riemann--Stieltjes integration}\index{Bounded variation}\index{Riemann--Stieltjes integration}
Beyond the ODE case of the last chapter, then CDEs depend on two new concepts: bounded variation paths, and Riemann--Stieltjes integration.

Suppose $x$ is differentiable and has bounded derivative -- a relatively weak assumption. Then $x$ will be of bounded variation, and the Riemann--Stieltjes integral may be reduced to an ordinary integral
\begin{equation}\label{eq:cde-to-ode}
	\int_0^t f(y(s)) \,\dd x(s) = \int_0^t f(y(s)) \frac{\dd x}{\dd s}(s)\,\dd s.
\end{equation}
As such whilst we will continue to treat the general case, the reader unfamiliar with these concepts should feel free to mentally substitute the above treatment throughout.

\begin{remark}
	Equation \eqref{eq:cde-to-ode} is essentially about reducing a CDE to an ODE. Correspondingly, the term `vector field' may be used to refer to either $f(y(s))$ or $f(y(s)) \nicefrac{\dd x}{\dd s}(s)$.
\end{remark}

\paragraph{CDEs are operators}
A controlled differential equation should be interpreted as a function from path-space to path-space. The input is a path $x$. The output is a path $y$. By choosing $f$ carefully, we may use a CDE to compute specific functions of its control.

\begin{example}[Value and integral of control]\label{example:cde:cde-value-integral}
	Let $f \colon \reals^2 \to \reals^{2 \times 2}$ be defined by
	\begin{equation*}
		f(y) = \begin{bmatrix} 0 & 1 \\ y_1 & 0 \end{bmatrix}
	\end{equation*}
	(where $y \in \reals^2$ is decomposed into $y = [y_1, y_2]$).

	Given any control $x \colon [0, T] \to \reals$, let $y \colon [0,T] \to \reals^2$ be the solution of the CDE driven by $t \mapsto (t, x(t))$, with vector field $f$, with initial condition $y(0) = [x(0), 0] \in \reals^2$. Then $y(t)$ will compute the value, and the first integral, of $x$.
	
	For example, consider the input signal $x(t) = \sin(t) \in \reals$. 
	
	Then
	\begin{align*}
		y(t) &= y(0) + \int_0^t f(y(s)) \, \dd \begin{bmatrix} s \\ x(s) \end{bmatrix}\\
		&= \int_0^t \begin{bmatrix} 0 & 1 \\ y_1(s) & 0 \end{bmatrix} \begin{bmatrix} 1 \\ \cos(s) \end{bmatrix} \,\dd s\\
		&= \int_0^t \begin{bmatrix} \cos(s) \\ y_1(s) \end{bmatrix} \,\dd s.
	\end{align*}
	Solving the first component, we see that
	\begin{equation*}
		y_1(t) = \int_0^t \cos(s) \,\dd s = \sin(t)
	\end{equation*}
	and so
	\begin{equation*}
		y_2(t) = \int_0^t y_1(s) \,\dd s = \int_0^t \sin(s) \,\dd s.
	\end{equation*}
	
	As advertised: $y(t)$ computes both the value $\sin(t)$ of the input signal, and its first integral $\int_0^t \sin(s) \,\dd s$.
	
	Moreover there was nothing special about the choice of $\sin(t)$, and this CDE will compute the value and first integral of any input signal.
\end{example}

We will make the equivalence more precise later on, but the connection to RNNs should be intuitive: much like CDEs, they compute some function of their time-varying input.

\paragraph{Existence and uniqueness}
The Picard existence theorem (Theorem \ref{theorem:picard-ode}) may be adapted to this setting.

\index{Existence}\index{Uniqueness}\index{Picard's existence theorem}
\begin{theorem}[{Picard existence theorem, \cite[Theorem 1.3]{levy-lyons} or \cite[Theorem 3.8]{FrizVictoir10}}]\label{theorem:cde:picard}
	Let $f \colon \reals^{d_y} \to \reals^{d_y \times d_x}$ be Lipschitz. Let $x \colon [0, T] \to \reals^{d_x}$ be of bounded variation. Let $y_0 \in \reals^{d_y}$. Then there exists a unique continuous $y \colon [0, T] \to \reals^{d_y}$ satisfying
	\begin{equation*}
		y(0) = y_0, \qquad y(t) = y(0) + \int_0^t f(y(s)) \,\dd x(s)\quad\text{for $t \in (0, T]$}.
	\end{equation*}
\end{theorem}

\begin{remark}
	The differential equation for a CDE is (by convention) autonomous, in the sense that $f$ is independent of the time $t$. If really desired then $t$ may be included by adding it to the state: replace $x$ with $[t, x]$ and $f$ with $\begin{bmatrix}1 & 0 \\ 0 & f\end{bmatrix}$. This implies we have replaced $y$ with $[t, y]$.
\end{remark}

\begin{remark}
	We might wonder about also using right hand sides of the form `$f_\theta(y(s), x(s))$'. Whilst there is nothing fundamentally wrong with this alternate approach, it is less theoretically neat. When using `$f_\theta(y(s), x(s))$' it is not possible to have $x \mapsto y$ be the identity function (see Section \ref{section:contrnn} later), whilst the `$f_\theta(y(s))\,\dd x(s)$' form has connections to integration against Brownian motion, as with stochastic differential equations.
\end{remark}

\subsection{Neural vector fields}
Suppose we observe some data in the form of a (continuous and bounded variation) path $x \colon [0, T] \to \reals^{d_x}$. This is often a little unrealistic as usually we observe discrete samples, for example in a time series. We shall fix this in a moment, when we consider applications.

Let $f_\theta \colon \reals^{d_y} \to \reals^{d_y \times d_x}$ be any (Lipschitz) neural network depending on parameters $\theta$. The value $d_y \in \naturals$ is a hyperparameter describing the size of the hidden state. Let $\zeta_\theta \colon \reals^{d_x} \to \reals^{d_y}$ be any neural network depending on the parameters $\theta$. Both $f_\theta$ and $\zeta_\theta$ will often just be parameterised as MLPs.

We define a \emph{neural controlled differential equation} \cite{kidger2020neuralcde} as the solution of the CDE
\begin{equation}\label{eq:ncde}
	y(0) = \zeta_\theta(x(0)),\qquad y(t) = y(0) + \int_{0}^t f_\theta(y(s)) \,\dd x(s)\quad\text{for $t \in (0, T]$}.
\end{equation}

The quantity $y$ is hidden state, modified in response to observations $x$. This is directly analogous to an RNN. This hidden state reflects an evolving belief about the system, updated continuously as observations $x$ are made.

Let $d_o \in \naturals$ be the desired output dimensionality of the model, and let $\ell_\theta \colon \reals^{d_y} \to \reals^{d_o}$ be a learnt affine map. Then the output of the model can be $\ell_\theta(y(t))$ if a time-evolving output is desired, or $\ell_\theta(y(T))$ if it is not, for example when performing whole-time-series classification. Once again, this parallels the construction of an RNN, for which a learnt affine readout is typically used to map from hidden state to output.

The resemblance between equations \eqref{eq:node-in-cde-chapter} and \eqref{eq:ncde} is clear. The essential difference is that equation \eqref{eq:ncde} is driven by the data process $x$, whilst equation \eqref{eq:node-in-cde-chapter} is driven only by the identity function $\reals \to \reals$. In this way, the neural CDE is naturally adapting to incoming data, as changes in $x$ change the local dynamics of the system.

\subsection{Solving CDEs}\label{section:cde:evaluate}
As with neural ODEs, we expect to numerically discretise the CDE so as to obtain an approximate solution.

A CDE may be discretised in two different ways. One option is to treat the `$\dd x(s)$` analogous to time inside a numerical differential equation solver, so that for example the explicit Euler method becomes
\begin{equation*}
	y_{j+1} = y_j + f_\theta(y_j) (x(t_{j+1}) - x(t_j)).
\end{equation*}
In practice however most software libraries do not support this (with the notable exception of Diffrax \cite{diffrax}).

Provided $x$ is differentiable -- in practice it often will be -- then the CDE may also be reduced to an ODE. Let
\begin{equation}\label{eq:g-theta-x}
	g_{\theta, x}(y, s) = f_\theta(y) \frac{\dd x}{\dd s}(s),
\end{equation}
so that for $t \in (0, T]$,
\begin{align}
	y(t) &= y(0) + \int_{0}^t f_\theta(y(s)) \,\dd x(s)\nonumber\\
	&= y(0) + \int_{0}^t f_\theta(y(s)) \frac{\dd x}{\dd s}(s) \,\dd s\nonumber\\
	&= y(0) + \int_{0}^t g_{\theta, x}(y(s), s) \,\dd s.\label{eq:ncde-to-node}
\end{align}
It is now possible to solve and train the neural CDE using the same techniques as for neural ODEs, and in particular using the same software. See Section \ref{section:numerical:software} for more discussion on software for neural differential equations.

\subsection{Application to regular time series}\label{section:cde:regular}\index{Time series}\index{Time series!Regular}
Let us now consider a concrete application to `regular' time series. That is to say, the observations are at regularly-spaced points, these points are the same for each batch element, and there is no missing data. (Extension to irregular time series will be considered in Section \ref{section:cde:irregular}.)

Let each time series be some sequence $\mathbf{x} = (x_0, \ldots, x_n)$ with each $x_j \in \reals^{d_x - 1}$. Let $x_{\mathbf{x}} \colon [0, n] \to \reals^{d_x}$ be some interpolation such that $x_{\mathbf{x}}(j) = (j, x_j)$. For example, $x_\mathbf{x}$ could be a cubic spline. Then $x_{\mathbf{x}}$ may be used to drive a neural CDE. See Figure \ref{fig:ncde-picture-regular}.

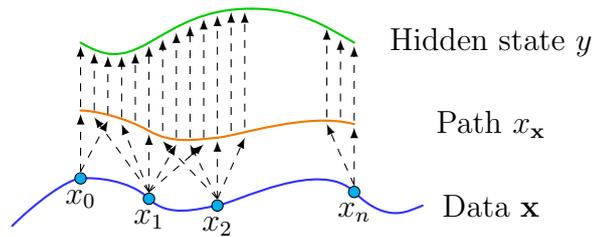
\begin{figure}[t]\centering
	\begin{tikzpicture}[scale=0.9]
		\draw[blue!80!white, thick, cap=round] (-1, 0.3) .. controls ++(45:0.5) and ++(185:0.2) .. (0, 1);
		\draw[blue!80!white, thick, cap=round] (0, 1) .. controls ++(185+180:0.2) and ++(135:0.4) .. (1, 0.7);
		\draw[blue!80!white, thick, cap=round] (1, 0.7) .. controls ++(135+180:0.4) and ++ (190:0.3) .. (2, 0.6);
		\draw[blue!80!white, thick, cap=round] (2, 0.6) .. controls ++(190+180:0.3) and ++ (135:0.7) .. (4, 0.8);
		\draw[blue!80!white, thick, cap=round] (4, 0.8) .. controls ++(135+180:0.7) and ++(200:0.2) .. (5, 0.6);
		\node at (0, 1)[circle,draw=black, fill=cyan, inner sep=1.5pt] {};
		\node at (1, 0.7)[circle,draw=black, fill=cyan, inner sep=1.5pt] {};
		\node at (2, 0.6)[circle,draw=black, fill=cyan, inner sep=1.5pt] {};
		\node at (4, 0.8)[circle,draw=black, fill=cyan, inner sep=1.5pt] {};
		\node at (0, 0.7) {$x_0$};
		\node at (1, 0.4) {$x_1$};
		\node at (2, 0.3) {$x_2$};
		\node at (4, 0.5) {$x_n$};
		\node at (6, 0.6) {Data $\mathbf{x}$};
		
		\draw[orange!95!black, thick, cap=round] (0, 2) .. controls ++(0:0.2) and ++(150:0.4) .. (1, 1.7);
		\draw[orange!95!black, thick, cap=round] (1, 1.7) .. controls ++(150+180:0.4) and ++ (185:0.3) .. (2, 1.6);
		\draw[orange!95!black, thick, cap=round] (2, 1.6) .. controls ++(185+180:0.3) and ++ (165:0.7) .. (4, 1.8);
		\node at (6, 1.8) {Path $x_\mathbf{x}$};
		
		\draw[dashed, ->] (0, 1.1) -- (0, 1.95);
		\draw[dashed, ->] (1, 0.8) -- (0.6, 1.8);
		\draw[dashed, ->] (2, 0.7) -- (2, 1.55);
		\draw[dashed, ->] (4, 0.9) -- (4, 1.75);
		
		\draw[dashed, ->] (0, 1.1) -- (0.4, 1.85);
		\draw[dashed, ->] (1, 0.8) -- (0.2, 1.9);
		\draw[dashed, ->] (1, 0.8) -- (1, 1.65);
		\draw[dashed, ->] (1, 0.8) -- (1.4, 1.5);
		\draw[dashed, ->] (1, 0.8) -- (1.8, 1.5);
		\draw[dashed, ->] (2, 0.7) -- (1.2, 1.55);
		\draw[dashed, ->] (2, 0.7) -- (1.6, 1.5);
		\draw[dashed, ->] (2, 0.7) -- (2.4, 1.6);
		\draw[dashed, ->] (4, 0.9) -- (3.6, 1.8);
		
		\draw[green!80!black, thick, cap=round] (0, 3) .. controls ++(-30:0.4) and ++(210:0.5) .. (1, 3);
		\draw[green!80!black, thick, cap=round] (1, 3) .. controls ++(210+180:0.5) and ++(180:0.7) .. (2.5, 3.5);
		\draw[green!80!black, thick, cap=round] (2.5, 3.5) .. controls ++(0:0.7) and ++(150:0.5) .. (4, 3);
		
		\draw[dashed, ->] (0, 2.05) -- (0, 2.95);
		\draw[dashed, ->] (0.2,2) -- (0.2,2.85);
		\draw[dashed, ->] (0.4,1.95) -- (0.4,2.8);
		\draw[dashed, ->] (0.6,1.9) -- (0.6,2.8);
		\draw[dashed, ->] (0.8,1.85) -- (0.8,2.85);
		\draw[dashed, ->] (1,1.75) -- (1,2.95);
		\draw[dashed, ->] (1.2, 1.7) -- (1.2, 3.05);
		\draw[dashed, ->] (1.4, 1.6) -- (1.4, 3.15);
		\draw[dashed, ->] (1.6, 1.6) -- (1.6, 3.25);
		\draw[dashed, ->] (1.8, 1.6) -- (1.8, 3.35);
		\draw[dashed, ->] (2, 1.65) -- (2, 3.4);
		\draw[dashed, ->] (2.2, 1.7) -- (2.2, 3.45);
		\draw[dashed, ->] (2.4, 1.75) -- (2.4, 3.45);
		\draw[dashed, ->] (3.6, 1.9) -- (3.6, 3.2);
		\draw[dashed, ->] (3.8, 1.9) -- (3.8, 3.05);
		\draw[dashed, ->] (4, 1.85) -- (4, 2.95);
		\node at (6,3) {Hidden state $y$};
	\end{tikzpicture}
	\caption{The hidden state of the neural CDE evolves continuously, driven by observational data.}\label{fig:ncde-picture-regular}
\end{figure}

\begin{remark}\label{remark:cde:approx-not-important}
	$x_\mathbf{x}$ is sometimes interpreted as an approximation to some underlying process that $\mathbf{x}$ has been sampled from. This is true, but not really relevant. Rather, $x_\mathbf{x}$ is just a continuous-time representation of the input data. If we had used $-x_\mathbf{x}$ instead then this would have represented the information contained in $\mathbf{x}$ just as well, despite being neither an interpolation nor an approximation.
\end{remark}

We discuss choices of interpolation scheme in more detail in Section \ref{section:cde:interpolation}.

\subsubsection{Spiral classification}\label{section:cde:spiral-experiment}\index{Examples!Neural CDEs}
As a toy example, we construct a two-dimensional dataset consisting of time series of the $(x, y)$-position of spirals, and train a neural CDE to perform binary classification of clockwise against anticlockwise. We consider data both with and without corruption by additive Gaussian noise.

The hidden state $y$ of the CDE evolves in $\reals^{d_l}$ with $d_l = 8$, and the prediction of the model at time $t$ is given by $\sigma(\ell_\theta(y(t))) \in (0, 1)$, where $\ell_\theta \colon \reals^{d_l} \to \reals$ is a learnt affine readout and $\sigma$ is the sigmoid function.  The model is trained with binary cross entropy on $\sigma(\ell_\theta(y(T)))$.

The final output of the model is given by $\sigma(\ell_\theta(y(T)))$, but we may examine the evolving $t \mapsto \sigma(\ell_\theta(y(t)))$ for interest. See Figure \ref{fig:cde-spiral-experiment}. The prediction updates as the input sequence is fed into the model, converging towards a steady state of the correct classification. (On this simple problem the model achieves perfect accuracy.)

Precise experimental details may be found in Appendix \ref{appendix:cde:spiral-experiment}. The code is available as an example in Diffrax \cite{diffrax}.

\begin{remark}
	The presence of noise does not necessitate any changes to this approach. If really desired the data could for example be smoothed with a filter, but in principle this is not necessary. The interpolation is just a continuous-time representation of the noisy data, which the model consumes as input.
\end{remark}

\begin{figure}[t]\centering
	\includegraphics[width=0.5\linewidth]{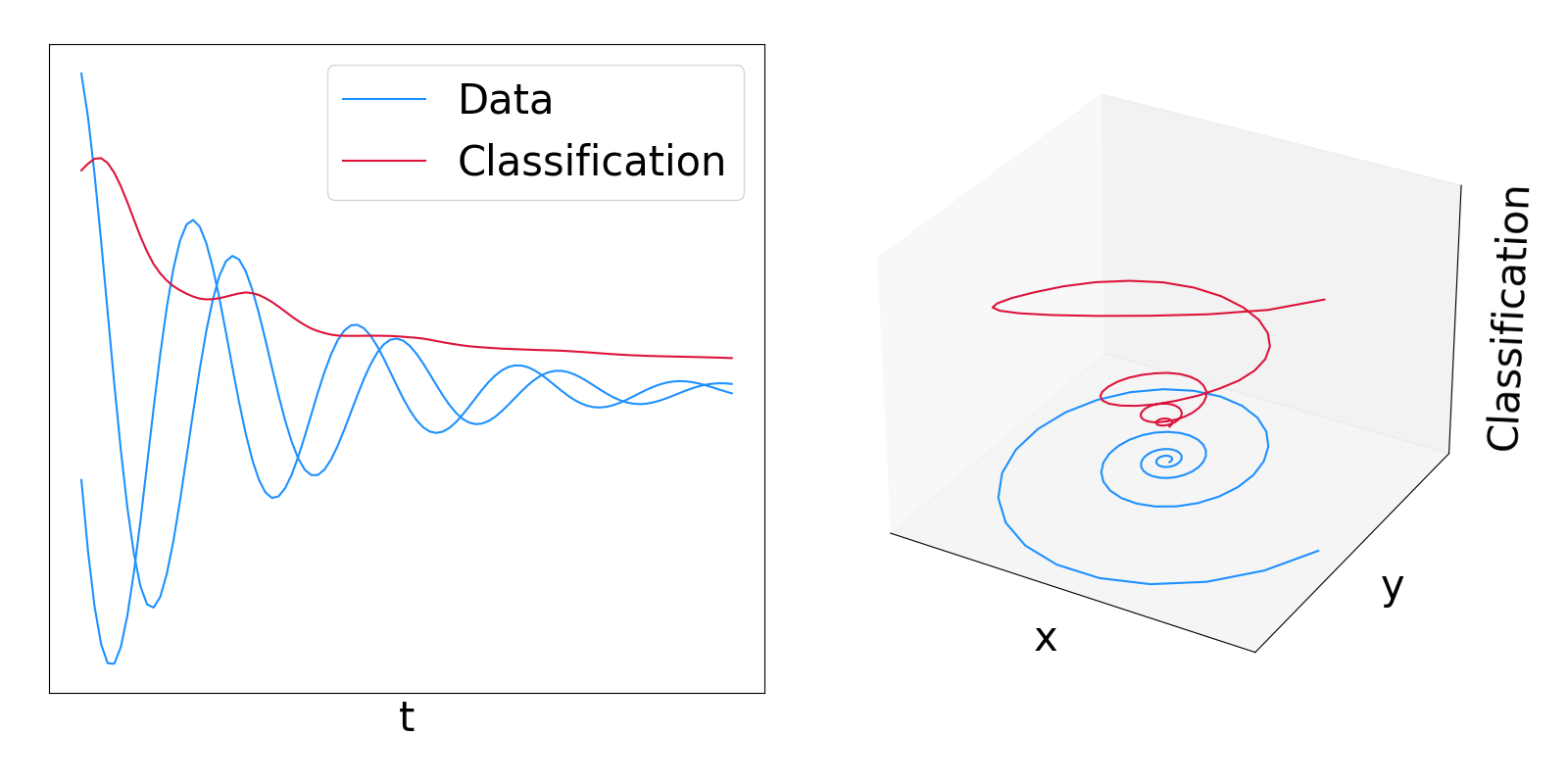}
	\includegraphics[width=0.5\linewidth]{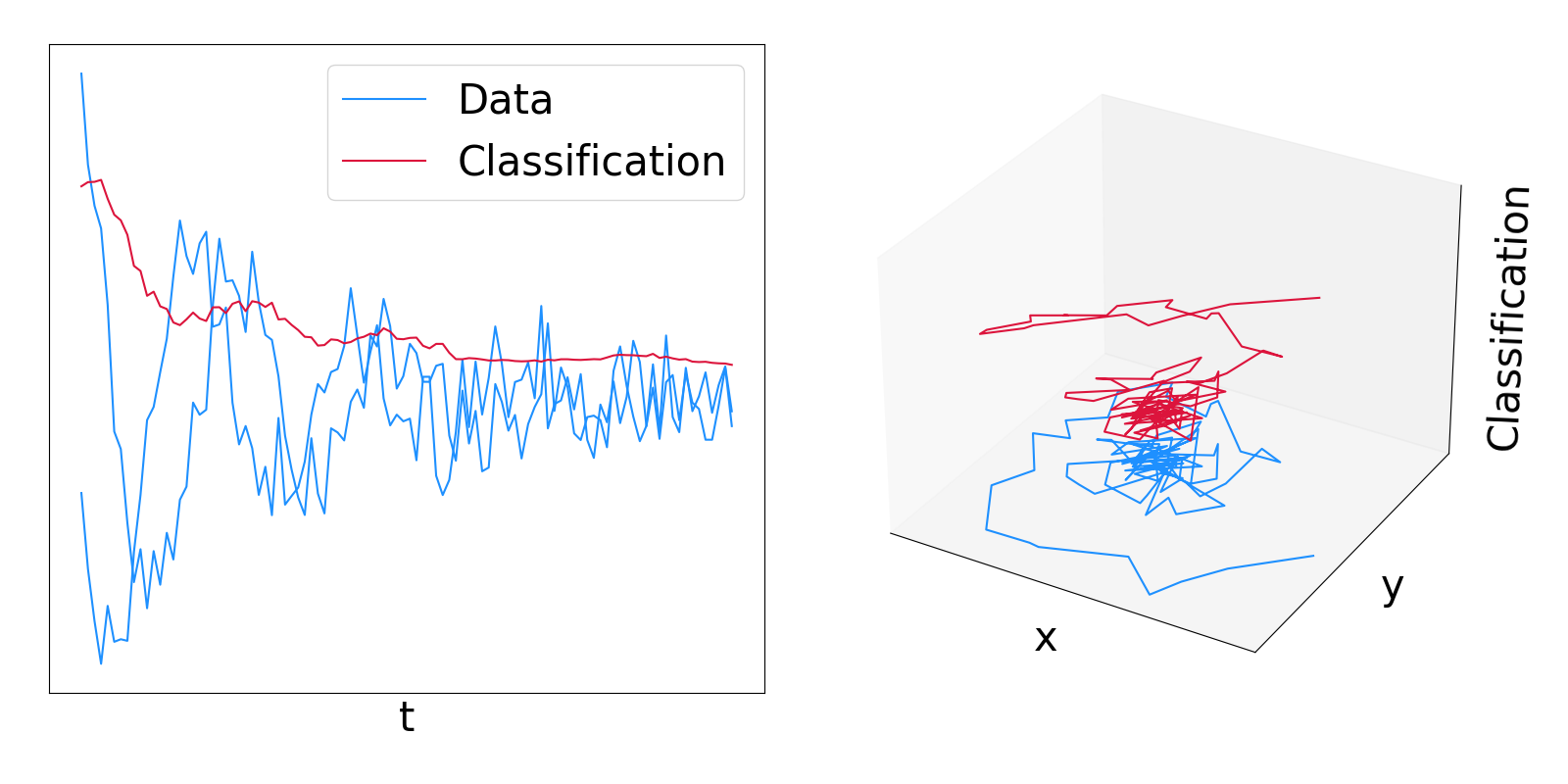}
	\caption{Data, and evolving prediction, of the neural CDE. \textbf{Top:} Data without noise. \textbf{Bottom:} Data is corrupted with Gaussian noise prior to training. \textbf{Left:} The $(x,y)$ data, and the prediction of the neural CDE, are shown evolving over time. The prediction converges towards a steady state (of zero; a classification of a clockwise spiral) as sufficient input data becomes available. \textbf{Right:} The $(x,y)$ position is shown at the bottom of the figure. Above it is shown the current prediction of the neural CDE.}\label{fig:cde-spiral-experiment}
\end{figure}

\subsubsection{The inclusion of time}\label{section:cde:inclusion-of-time}
There is only one foible with this construction, which is that $x_{\mathbf{x}}(j) = (j, x_j)$ and not simply $x_{\mathbf{x}}(j) = x_j$. This one detail is important for expressivity of the model.

\begin{example}
Suppose the function computed by the neural CDE should be the length of the input time series. If $x_{\mathbf{x}}(j) = (j, x_j)$ then this is straightforward. Take $d_y = 1$ so that $y(t) \in \reals$, and let $f_\theta(y) = [1, 0, \ldots, 0] \in \reals^{d_x}$ be constant. Let the initial value network $\zeta_\theta(x) = 0$ for all inputs. Then
\begin{align*}
	y(n) &= y(0) + \int_0^{n} f_\theta(y(s))\,\dd x(s)\\
	&= \int_0^{n} \begin{bmatrix}1 & 0 & \cdots & 0 \end{bmatrix} \begin{bmatrix} 1 \\ * \\ \vdots \\ *\end{bmatrix} \,\dd s\\
	&= \int_0^{n} \dd s\\
	&= n,
\end{align*}
where the `$* \cdots *$' refers to whatever the derivative of the interpolation of $\mathbf{x}$ is.

If this extra `time' variable is missed out, and simply $x_{\mathbf{x}}(j) = x_j$, then computing the length is impossible. For example suppose $x_j = 0$ for all $j$, and correspondingly any reasonable interpolation scheme will have $x_{\mathbf{x}}(t) = 0$ for all $t$. Then $\nicefrac{\dd x}{\dd t}(t) = 0$ as well, and $y(t) = y(0)$ regardless of the choice of $f_\theta$. And so $y(t)$ cannot calculate the length of $\mathbf{x}$ for this particular choice of $\mathbf{x}$.
\end{example}

(Note how the Example \ref{example:cde:cde-value-integral}, earlier, also included time as an additional channel.)

\subsection{Discussion}
Neural CDEs offer several advantages, both conceptually and practically.

\subsubsection{Universal approximation}\index{Universal approximation!CDEs}
Provided this formulation is followed carefully -- and this extra time-like variable is included -- then the neural CDE will be a universal approximator.

\begin{restatable}{theorem-informal}{informalunivapprox}\label{theorem-informal:informalunivapprox}
	An affine map on the terminal value of a neural CDE is a universal approximator from $\{\text{\emph{sequences in }} \reals^{d_x}\}$ to $\reals$.
\end{restatable}

We will discuss this further in Section \ref{section:cde:universal-approximation}, and provide a formal statement and proof in Appendix \ref{appendix:cde:universal-approximation}.

\subsubsection{Continuous-time updates}\label{section:cde:continuous-time-computation}
Neural CDEs update their hidden state in continuous time. In many contexts this is far more natural than the discrete-time updates typical of an RNN.

\paragraph{Irregular data} Suppose the data arrives at irregular times. (We'll discuss this use case in much more detail in the next section.) If the times are close together then the hidden state of an RNN or neural CDE will often need only a small update. If the times are far apart then the belief about the system may need a very large update.

An RNN, however, devotes equal processing power to both of these use cases. (A single update step.) This may be inefficient if the times were close together, and insufficient if they were far apart.

In contrast a neural CDE updates continuously. The amount of computational work scales with the gap between observations, and this is likely close to the `natural timescale' at which we should update our belief about the system.\footnote{Could we instead repeatedly give the last piece of input data to an RNN, whilst waiting a long time for another observation? Yes, and in doing so have just reinvented a particular discretisation of a neural CDE.}

\paragraph{Decoupled data and computation}
Put precisely, continuous-time updates decouple data and computation; the latter is no longer tied to the former. This is particularly true if solving a neural CDE with an adaptive step size numerical ODE/CDE solver. Such a solver automatically detects the complexity of the dynamics and takes appropriately-sized numerical steps.

\paragraph{Special cases of neural CDEs}
In light of this, there have now been several proposals in which the hidden state of an RNN is updated in continuous time between observations; popular examples are GRU-D or ODE-RNNs \cite{Che2018, latent-odes, gru-ode-bayes}. These are special cases or discretisations of neural CDEs. (Exercise for the reader!)

\subsubsection{Memory efficient backpropagation}\label{section:cde:training}\index{Optimise-then-discretise!CDEs}
Given an RNN, for which evaluating and backpropagating a single step consumes $H$ memory, then backpropagating an RNN evaluated on a time series of length $T$ will consume $\bigO{HT}$ memory.

In contrast, neural CDEs can reduce this to only $\bigO{H + T}$ memory. This consists of $\bigO{H}$ to backpropagate through each step individually, and $\bigO{T}$ to hold the underlying data $x$ in memory. That each step can be backpropagated through individually is due to the use of `optimise-then-discretise' backpropagation. We will discuss this style of backpropagation alongside our other numerical discussions, in Section \ref{section:numerical:adjoint-cde-sde}.

\subsection{Summary}
The goal of this section, Section \ref{section:cde:introduction}, has been to summarise the main ideas behind CDEs and neural CDEs. Now that these are in place, we will be ready to move on to some more serious applications of neural CDEs.

We conclude this section with some thoughts on the connections of CDEs to other fields of study.

\subsubsection{Rough path theory}\label{section:cde:rough-path-theory}\index{Rough!Path theory}
The theory of CDEs may be extended to highly irregular driving paths $x$, which are not even of bounded variation. This is known as \textit{rough path theory}, and correspondingly such CDEs become rebranded as \textit{rough differential equations} (RDEs).

There is broadly speaking a hierarchy from ODEs to CDEs to RDEs to SDEs: CDEs introduce the notion of control; RDEs additionally consider when the control is rough; SDEs additionally consider when the control is stochastic (usually Brownian motion).\footnote{For example this hierarchy is demonstrated by numerical SDE solvers, which typically operate by drawing a sample of the Brownian motion and then solving the SDE pathwise.
}

We will use rough path theory in a few contexts: when applying neural CDEs to long time series (Section \ref{section:cde:long-time-series}), in the proof of universal approximation for neural CDEs (Section \ref{section:cde:universal-approximation}), and in a later chapter to construct the `optimise-then-discretise' equation for neural SDEs (Section \ref{section:adjoint-nsde}).

These all tend towards the theoretical end of things, and we emphasise that a familiarity is neither expected nor required to read this thesis, or to work with the techniques discussed.

\begin{remark}
For those with the right background (graduate-level analysis), then rough path theory gives an excellent framework for understanding neural differential equations. It offers a pathwise theory, and a general framework through which ODEs/CDEs/SDEs may all be unified; for example Diffrax \cite{diffrax} uses its principles to construct a unified system of numerical differential equation solvers. The first few pages of \cite{roughstochasticNF} give a brief introduction to the essential ideas of rough path theory, \cite{levy-lyons} is a typical introductory text, and \cite{FrizVictoir10} is the canonical textbook.
\end{remark}

\subsubsection{Control theory}\index{Control theory}\label{section:cde:control-theory}
Despite their similar names, and treatment of similar problems, controlled differential equations and control theory are typically treated as separate fields.

The difference is to some extent philosophical. In control theory, the system $f$ is typically specified\footnote{Perhaps incompletely via observations, necessitating the additional step of performing system identification.}, and the task is to find a control $x$ producing the desired response $y$. Meanwhile with (neural) CDEs, this is flipped around: the control $x$ is typically specified, and we shall attempt to find a system $f$ that produces a desired response $y$.

This is not a distinction we particularly wish to enforce, though -- there is still substantial overlap.

\section{Applications}
Neural CDEs have a number of applications, usually to time series. We will see applications to difficult time series (such as irregular or long time series), and will later briefly touch on connections to reinforcement learning. In addition we have previously remarked that RNNs and neural CDEs are linked, and we will also make this connection explicit.

\subsection{Irregular time series}\label{section:cde:irregular}\index{Time series!Irregular}\index{Irregular sampling}
Suppose we observe some irregular time series of the form $\mathbf{x} = ((t_0, x_0), \ldots, (t_n, x_n))$, with each $t_j \in \reals$ the timestamp of the observation
\begin{equation*}
x_j = (x_{j, 1}, \ldots, x_{j, d_x - 1}) \in (\reals \cup \{*\})^{d_x - 1}.
\end{equation*}
Here $*$ denotes the possibility of missing data, and $t_0 < \cdots < t_n$. The length $n$ is not assumed to be consistent between different time series.

Let $T > 0$ and $0 = s_0 < s_1 < \cdots < s_n = T$. Let $x_{\mathbf{x}} \colon [0, T] \to \reals^{d_x}$ be some interpolation such that $x_{\mathbf{x}}(s_j) = (t_j, x_j)$ (with the equality being defined up to those elements of $x_j$ which are not missing). For example, we could take $s_j = t_j$ (or $s_j = j$ as in the previous section) and $x_\mathbf{x}$ to be a cubic spline with knots at $s_0, \ldots, s_n$.

\newcommand{\mainpicturedata}{
	\draw[thick, ->] (-1, 0) -- (5, 0);
	\node at (0, 0)[circle,fill, inner sep=1.5pt] {};
	\node at (0.6, 0)[circle,fill, inner sep=1.5pt] {};
	\node at (2, 0)[circle,fill, inner sep=1.5pt] {};
	\node at (4, 0)[circle,fill, inner sep=1.5pt] {};
	\node at (0, -0.3) {$t_0$};
	\node at (0.6, -0.3) {$t_1$};
	\node at (2, -0.3) {$t_2$};
	\node at (3, -0.3) {$\cdots$};
	\node at (4, -0.3) {$t_n$};
	\node at (5.7, 0) {Time};
	
	\draw[blue!80!white, thick, cap=round] (-1, 0.3) .. controls ++(45:0.5) and ++(185:0.2) .. (0, 1);
	\draw[blue!80!white, thick, cap=round] (0, 1) .. controls ++(185+180:0.2) and ++(135:0.4) .. (0.6, 0.7);
	\draw[blue!80!white, thick, cap=round] (0.6, 0.7) .. controls ++(135+180:0.4) and ++ (190:0.3) .. (2, 0.6);
	\draw[blue!80!white, thick, cap=round] (2, 0.6) .. controls ++(190+180:0.3) and ++ (135:0.7) .. (4, 0.8);
	\draw[blue!80!white, thick, cap=round] (4, 0.8) .. controls ++(135+180:0.7) and ++(200:0.2) .. (5, 0.6);
	\node at (0, 1)[circle,draw=black, fill=cyan, inner sep=1.5pt] {};
	\node at (0.6, 0.7)[circle,draw=black, fill=cyan, inner sep=1.5pt] {};
	\node at (2, 0.6)[circle,draw=black, fill=cyan, inner sep=1.5pt] {};
	\node at (4, 0.8)[circle,draw=black, fill=cyan, inner sep=1.5pt] {};
	\node at (0, 0.7) {$x_0$};
	\node at (0.6, 0.4) {$x_1$};
	\node at (2, 0.3) {$x_2$};
	\node at (4, 0.5) {$x_n$};
	\node at (5.7, 0.6) {Data $\mathbf{x}$};
}

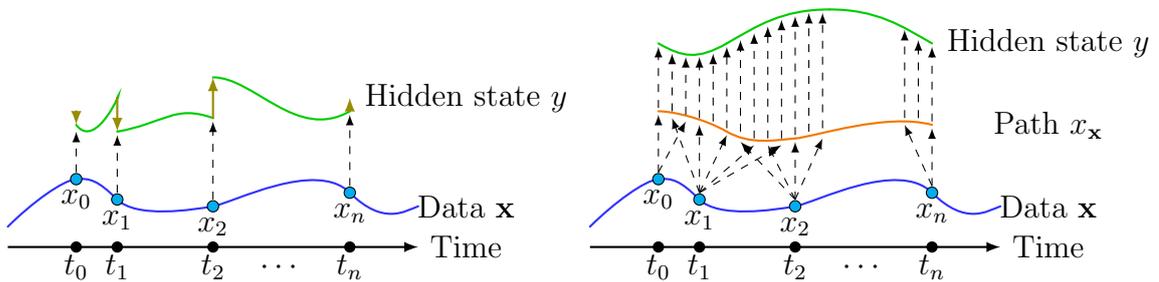
\begin{figure}[b]
	\begin{subfigure}[b]{0.49\textwidth}\centering
		\begin{tikzpicture}[scale=0.9]
			\mainpicturedata
			
			\draw[olive, thick, cap=round, ->] (0, 2) -- (0, 1.8);
			\draw[green!80!black, thick, cap=round] (0, 1.8) .. controls (0.3, 1.4) and (0.7, 2.4) .. (0.6,2.2);
			\draw[olive, thick, cap=round, ->] (0.6,2.2) -- (0.6,1.7);
			\draw[green!80!black, thick, cap=round] (0.6,1.7) .. controls ++(10:0.7) and ++(160:0.6) .. (2,1.9);
			\draw[olive, thick, cap=round, ->] (2,1.9) -- (2,2.5);
			\draw[green!80!black, thick, cap=round] (2,2.5) .. controls ++(0:0.6) and ++(210:0.9) .. (4,2);
			\draw[olive, thick, cap=round, ->] (4,2) -- (4, 2.2);
			
			\draw[dashed, ->] (0, 1.1) -- (0, 1.7);
			\draw[dashed, ->] (0.6, 0.8) -- (0.6, 1.65);
			\draw[dashed, ->] (2, 0.7) -- (2, 1.85);
			\draw[dashed, ->] (4, 0.9) -- (4, 1.95);
			
			\node at (5.7, 2.2) {Hidden state $y$};
		\end{tikzpicture}
	\end{subfigure}
	\hspace{-0.2em}
	\begin{subfigure}[b]{0.49\textwidth}\centering
		\begin{tikzpicture}[scale=0.9]
			\mainpicturedata
			
			\draw[orange!95!black, thick, cap=round] (0, 2) .. controls ++(0:0.2) and ++(150:0.4) .. (1, 1.7);
			\draw[orange!95!black, thick, cap=round] (1, 1.7) .. controls ++(150+180:0.4) and ++ (185:0.3) .. (2, 1.6);
			\draw[orange!95!black, thick, cap=round] (2, 1.6) .. controls ++(185+180:0.3) and ++ (165:0.7) .. (4, 1.8);
			\node at (5.7, 1.8) {Path $x_\mathbf{x}$};
			
			\draw[dashed, ->] (0, 1.1) -- (0, 1.95);
			\draw[dashed, ->] (0.6, 0.8) -- (0.6, 1.8);
			\draw[dashed, ->] (2, 0.7) -- (2, 1.55);
			\draw[dashed, ->] (4, 0.9) -- (4, 1.75);
			
			\draw[dashed, ->] (0, 1.1) -- (0.4, 1.85);
			\draw[dashed, ->] (0.6, 0.8) -- (0.2, 1.9);
			\draw[dashed, ->] (0.6, 0.8) -- (1, 1.65);
			\draw[dashed, ->] (0.6, 0.8) -- (1.4, 1.5);
			\draw[dashed, ->] (0.6, 0.8) -- (1.8, 1.5);
			\draw[dashed, ->] (2, 0.7) -- (1.2, 1.55);
			\draw[dashed, ->] (2, 0.7) -- (1.6, 1.5);
			\draw[dashed, ->] (2, 0.7) -- (2.4, 1.6);
			\draw[dashed, ->] (4, 0.9) -- (3.6, 1.8);
			
			\draw[green!80!black, thick, cap=round] (0, 3) .. controls ++(-30:0.4) and ++(210:0.5) .. (1, 3);
			\draw[green!80!black, thick, cap=round] (1, 3) .. controls ++(210+180:0.5) and ++(180:0.7) .. (2.5, 3.5);
			\draw[green!80!black, thick, cap=round] (2.5, 3.5) .. controls ++(0:0.7) and ++(150:0.5) .. (4, 3);
			
			\draw[dashed, ->] (0, 2.05) -- (0, 2.95);
			\draw[dashed, ->] (0.2,2) -- (0.2,2.85);
			\draw[dashed, ->] (0.4,1.95) -- (0.4,2.8);
			\draw[dashed, ->] (0.6,1.9) -- (0.6,2.8);
			\draw[dashed, ->] (0.8,1.85) -- (0.8,2.85);
			\draw[dashed, ->] (1,1.75) -- (1,2.95);
			\draw[dashed, ->] (1.2, 1.7) -- (1.2, 3.05);
			\draw[dashed, ->] (1.4, 1.6) -- (1.4, 3.15);
			\draw[dashed, ->] (1.6, 1.6) -- (1.6, 3.25);
			\draw[dashed, ->] (1.8, 1.6) -- (1.8, 3.35);
			\draw[dashed, ->] (2, 1.65) -- (2, 3.4);
			\draw[dashed, ->] (2.2, 1.7) -- (2.2, 3.45);
			\draw[dashed, ->] (2.4, 1.75) -- (2.4, 3.45);
			\draw[dashed, ->] (3.6, 1.9) -- (3.6, 3.2);
			\draw[dashed, ->] (3.8, 1.9) -- (3.8, 3.05);
			\draw[dashed, ->] (4, 1.85) -- (4, 2.95);
			\node at (5.7,3) {Hidden state $y$};
		\end{tikzpicture}
	\end{subfigure}
	\caption{Some data process is observed at times $t_0, \ldots, t_n$ to give observations $x_0, \ldots, x_n$. It is otherwise unobserved. \textbf{Left:} RNNs modify the hidden state at each observation; common variants \cite{Che2018, latent-odes, gru-ode-bayes} continuously evolve the hidden state between observations. \textbf{Right: } In contrast, the hidden state of the neural CDE model has continuous dependence on the observed data.}\label{figure:ncde-picture}
\end{figure}

Then $x_{\mathbf{x}}$ may be used to drive a neural CDE \cite{kidger2020neuralcde}; see Figure \ref{figure:ncde-picture}.

\begin{remark}\label{remark:cde:no-impute}
	We stress that this interpolation is \textit{not} imputing missing data. It is simply constructing a continuous-time representation of the input data. We will discuss how to appropriately handle missing data in a moment.
\end{remark}

The choice of interpolation scheme, including the choice of $s_j$, is one we will defer until Section \ref{section:cde:interpolation}. A few different choices may be made, depending on the type of problem.

\subsubsection{Missingness as a channel}\label{section:cde:missingness1}
It has been observed that the frequency of observations may carry information \cite{Che2018}. For example, doctors may take more frequent measurements of patients they believe to be at greater risk. Some previous work has for example sought to incorporate this information by learning an intensity function \cite{shukla2018interpolationprediction, latent-odes, neural-odes}.

A simple (non-learnt) procedure is just to concatenate the index $j$ as an additional channel. That is, construct the path $x_\mathbf{x} \colon [0, T] \to \reals^{d_x + 1}$ such that $x_\mathbf{x}(s_j) = (t_j, x_j, j)$ instead of just $(t_j, x_j)$. The extra channel of $x_\mathbf{x}$ then gives the cumulative number of observations over time.

As the derivative of $x_\mathbf{x}$ is what is then used when evaluating the neural CDE model, as in equation \eqref{eq:ncde-to-node}, then it is the current observational rate that then determines the vector field.

\subsubsection{Partially observed data}\label{section:cde:missingness2}
When some data is missing, then the frequency of observations in each individual channel may carry information. The previous procedure may now be straightforwardly extended, by having a separate observational channel for each original channel.

Explicitly, take $x_\mathbf{x}(s_j) = (t_j, x_j, c_j(\mathbf{x}))$, where $c_j(\mathbf{x}) = (c_{j,1}, \ldots, c_{j, d_x}) \in \reals^{d_x}$, where $c_{j, k} = \sum_{m=0}^j \indicator{x_{m, k} \neq *}$ counts the number of observations in the $k$th channel by time $t_j$. This means that $x_\mathbf{x} \colon [0, T] \to \reals^{2 d_x - 1}$.

Adding observational masks is standard practice when working with informatively missing data \cite{Che2018}; this is the appropriate continuous-time analogy.

\subsubsection{Batching irregular data and choice of $s_j$}\label{section:cde:batching}
In the context of CDEs, the data consists of multiple $x_\mathbf{x} \colon [s_0, s_n] \to \reals^{d_x}$ that need to be batched together. In principle each interval $[s_0, s_n]$ may be different for each batch element, for example if we chose $s_j = t_j$ and the data is irregularly sampled.


\paragraph{Batchable differential equation solvers}\index{Batchable differential equation solvers}

Some (very few) differential equation software libraries allow batching over different regions of integration. In this case the problem is straightforward: simply use the capabilities of the library. For example this is the case with Diffrax \cite{diffrax}.

\paragraph{Other differential equation solvers}

Most differential equation software libraries do not intrinsically support batching over different regions of integration. For example this is the case with \texttt{torchdiffeq} and \texttt{torchcde} \cite{torchdiffeq, torchcde}.

Fortunately, the structure of neural CDEs mean this is not a serious hurdle, as we may choose specifically $s_j = j$. This ensures that each region of integration begins at the same value, namely zero, and we need merely integrate forwards in time for sufficiently long that every batch element has been integrated. (See also Section \ref{section:cde:interpolation-points} for more discussion on the choice of $s_j$.)

As an additional benefit, the fact that all $s_j$ are the same for each batch element can be used to simplify the storage of batches of multiple control paths $x_\mathbf{x} \colon [0, T] \to \reals^{d_x}$. (Instead of juggling different collections of intervals $[s_j, s{j+1}]$ for each batch element.)

\begin{remark}
	This was actually a mistake we made in \cite{kidger2020neuralcde} -- the above procedure was not done, in favour of an alternate (storage-inefficient) scheme that involved taking the union over the times $t_j$ needed to handle each batch of data.
\end{remark}

\begin{remark}
	One minor quirk in this case arises when using an adaptive step size solver. A subtle dependency between batch elements is introduced, as the step size will be determined by the behaviour across the whole (batched) system. This is usually not a major issue, and this is simply tolerated. (This quirk is not unique to neural CDEs, and is true whenever a batch of neural differential equations are solved with an unbatched adaptive differential equation solver.)
\end{remark}

\subsubsection{Example}
We now refer back to the regularly-spaced example of Section \ref{section:cde:regular}. Had the observations been irregularly spaced, they could have been handled in the manner just described. Morally speaking neural CDEs make little difference between regular and irregular time series; once a continuous path $x_\mathbf{x}$ is obtained then both are treated in exactly the same way.

\subsection{RNNs are discretised neural CDEs}\label{section:cde:rnn-cde}
We will now make the connection between RNNs and CDEs explicit; see also \cite{kidger2020neuralcde}.

\subsubsection{CDEs as RNNs}
Consider the CDE
\begin{equation*}
	y(t) = y(0) + \int_0^t f(y(s)) \,\dd x(s)\quad\text{for $t \in (0, T]$}.
\end{equation*}
Discretising this with Euler's method produces either
\begin{equation*}
	y(t_{j+1}) = y(t_j) + f(y(t_j))(x(t_{j+1}) - x(t_j))
\end{equation*}
or
\begin{equation*}
	y(t_{j+1}) = y(t_j) + f(y(t_j)) \frac{\dd x}{\dd t}(t_j) (t_{j+1} - t_j)
\end{equation*}
depending on whether the CDE is converted into an ODE first.

In either case, this is an RNN-like structure: suppose $f = f_\theta$ is some neural network and $x$ is some input data.

\subsubsection{RNNs as CDEs}
Conversely consider an RNN of the form
\begin{equation*}
	y_{j + 1} = h_\theta(y_j, x_j).
\end{equation*}
This is an explicit Euler discretisation with unit timestep of
\begin{equation}\label{eq:cde:_discrete-rnn-as-cde}
	y(t) = y(0) + \int_{0}^t h_\theta(y(s), x(s)) - y(s)\,\dd s.
\end{equation}
Equations of this form, in which the integrand is some function of $y(s)$ and $x(s)$, are special cases of neural CDEs. (We'll discuss this in Section \ref{section:contrnn}.)

\subsubsection{RNN variants}
The `$\mathop - y(s)$' term in \eqref{eq:cde:_discrete-rnn-as-cde} feels a little out of place. It would not appear for RNNs of the form
\begin{equation*}
	y_{j + 1} = y_j + h_\theta(y_j, x_j).
\end{equation*}
RNNs of this form resemble residual networks, and indeed this parameterisation clearly provides a better differential-equation-like structure.

\begin{example}\label{example:cde:gru-ode}
	A GRU is of the above form. Recall that a GRU is defined by
	\begin{align*}
		i_j &= \sigmoid(W_{1} x_j + W_{2} h_j + b_1),\\
		r_j &= \sigmoid(W_{3} x_j + W_{4} h_j + b_2),\\
		n_j &= \tanh(W_5 x_j + b_3 + r_j * (W_6 h_j + b_4)),\\
		h_{j+1} &= n_j + i_j * (h_j - n_j),
	\end{align*}
	for input time series $x_j$ evolving hidden state $h_j$, and suitably shaped weight matrices $W_1, W_2, W_3, W_4, W_5, W_6$ and bias vectors $b_1, b_2, b_3, b_4$. Here $*$ denotes elementwise multiplication.
	
	This is an explicit Euler discretisation of
	\begin{align*}
		i(t) &= \sigmoid(W_{1} x(t) + W_{2} h(t) + b_1)\\
		r(t) &= \sigmoid(W_{3} x(t) + W_{4} h(t) + b_2),\\
		n(t) &= \tanh(W_5 x(t) + b_3 + r(t) * (W_6 h(t) + b_4)),\\
		\frac{\dd h}{\dd t}(t) &= (1 - i(t)) * (n(t) - h(t)).
	\end{align*}
\end{example}

\begin{remark}\label{remark:cde:gru-ode}
	Note the $-h(t)$ that appears on the right hand side of the continuous-time GRU. This corresponds to exponential decay of the hidden state of the GRU, just as with the differential equation $\nicefrac{\dd y}{\dd t}(t) = -y(t)$ for exponential decay \cite{decaying-gru}. This explains the classic fact that GRUs/LSTMs struggle to learn long-term time dependencies.
\end{remark}


\subsection{Long time series and rough differential equations}\label{section:cde:long-time-series}\index{Time series!Long}\index{Rough!Differential equations}
Neural CDEs, as with RNNs, begin to break down for very long time series. Loss/accuracy worsens, and training time becomes prohibitive due to the sheer number of operations required to evaluate a single pass of the model.

We will now see how this may be remedied. The key idea is to take very large integration steps -- much larger than the sampling rate of the data -- whilst incorporating sub-step information through additional terms in the numerical solver, through what is known as the \textit{log-ODE method}.

A CDE treated in this way is termed a \textit{rough differential equation}, in the sense of rough path theory. Correspondingly we refer to this approach as \textit{neural rough differential equations}. (Or less snappily, `the log-ODE method applied to neural CDEs'.) This was introduced in \cite{morrill2021neuralrough}.

\begin{remark}
	For the reader familiar with numerical SDEs, this inclusion of `sub-step information' is directly analogous to the difference between the Euler--Maruyama method and Milstein's method \cite{kloedenplaten1992}. For the reader familiar with the Magnus expansion \cite{magnus2008expansion}, then the log-ODE method is a generalisation to nonlinear differential equations.
\end{remark}

More so than the rest of this chapter, this will rely on advanced theoretical tools from rough path theory. As such this material is deferred to Appendix \ref{appendix:neural-rde} to avoid breaking the flow.

\subsection{Training neural SDEs}\label{section:cde:training-nsde}
We will see in Chapter \ref{chapter:neural-sde} that SDEs, as generative models, may be trained as GANs. However samples from SDEs are continuous-time paths, which necessitate a discriminator that admits a continuous-time path as an input -- such as a neural CDE.

The main ideas for this were introduced in \cite{kidger2021sde1, kidger2021sde2}, and this will be discussed alongside neural SDEs in Chapter \ref{chapter:neural-sde}.

\section{Theoretical properties}

\subsection{Universal approximation}\label{section:cde:universal-approximation}
In the CDE literature, it is a well-known theorem that they represent general functions on streams. We think \cite[Theorem 4.2]{perezarribas2018}, see also \cite[Proposition A.6]{bonnier2019deep}, give the clearest statement of this result. This may be applied to show that neural CDEs are universal approximators, which we summarise in the following informal statement.

\informalunivapprox*

The essential idea is that a suitably large CDE can compute a truncated basis for the space of continuous functions of its input. The final affine map may then take some affine combination of these, and in doing so approximate any continuous function.

Theorem \ref{theorem:datawise-universal-approximation} in Appendix \ref{appendix:cde:universal-approximation} gives a formal statement and a proof, which generalises the original presentation in \cite[Appendix B]{kidger2020neuralcde}.

This property is not necessary for good empirical performance (GRUs frequently achieve good performance without being universal approximators \cite{unbounded-counting}), but it is reassuring to know that it may be accomplished in principle.

\subsection{Comparison to alternative ODE models}\label{section:contrnn}
If unfamiliar with CDEs, then it may seem natural to replace $f_\theta(y(s))\frac{\dd x}{\dd s}(s)$ with some $h_\theta(y(s), x(s))$ that is directly applied to, and potentially nonlinear in, $x(s)$. Indeed, special cases of this have been suggested before, in particular to derive the `GRU-ODE' analogous to a GRU \cite{gru, gru-ode-bayes, gruode} (Example \ref{example:cde:gru-ode}).

However, it turns out that something is lost by doing so, which we summarise in the following statement.

\begin{theorem-informal}
	Any equation of the form $y(t) = y(0) + \int_{0}^t h_\theta(y(s), x(s)) \,\dd s$ may be represented exactly by a neural CDE of the form $y(t) = y(0) + \int_{0}^t f_\theta(y(s)) \,\dd x(s)$. However the converse statement is not true.
\end{theorem-informal}

The essential idea is that a neural CDE can easily represent the identity function between paths, whilst the alternative is incapable of doing so. Theorem \ref{theorem:continuous-rnn} in Appendix \ref{appendix:cde:nonlinear-g} provides the formal statement and proof, which originally appeared in \cite[Appendix C]{kidger2020neuralcde}.

(This does not preclude using the $h_\theta(y(s), x(s))$ form if this happens to work on any given problem, of course.)

\subsection{Invariances}\label{section:invariances}
CDEs exhibit two possible invariances. In general these invariances are often undesirable and are removed as follows.

\subsubsection{Translation invariance and initial value networks}\index{Invariances!Translation}
The integral defining the evolution of a neural CDE depends upon the control $x$ only through its derivative $\nicefrac{\dd x}{\dd t}$. If this were the only way that $x$ was input to the model, then $y$ would be invariant to translations of $x$.

It is for this reason that the initial hidden state $y(0)$ depends on $x(0)$ through the initial value network $\zeta_\theta$, so as to ensure sensitivity to translations. Other alternatives may also be admitted: a channel whose first derivative includes translation-sensitive information could be appended, for example by replacing $x$ with $\widetilde{x}$ where $\widetilde{x}(t) = (x(t), tx(0))$. 

\subsubsection{Reparameterisation invariance}\label{section:reparam-invariance}\index{Invariances!Reparameterisation}
CDEs exhibit a \textit{reparameterisation invariance} property.\footnote{In fact they also exhibit a \emph{tree-like invariance} property \cite{hambly2010uniqueness}, which is a slight generalisation.}
\begin{restatable}{proposition}{reparaminvariance}
	Let $\psi \colon [0, S] \to [0, T]$ be differentiable, increasing, and such that $\psi(0) = 0$ and $\psi(S) = T$. Let $y$ solve a CDE driven by a path $x$. Then $y \circ \psi$ solves the same CDE driven by $x \circ \psi$, and in particular their terminal values are the same: $(y \circ \psi)(S) = y(T)$.
\end{restatable}

This means that a CDE is blind to the speed at which $x$ is traversed. Typically, the speed at which input data arrives is important (for example consider data indicating that a patient's health is declining over years -- or over minutes), so this means that the speed at which events occur must be explicitly encoded as a channel in $x$. Indeed, this is precisely what is done in Section \ref{section:cde:inclusion-of-time} by including time as a channel.

See Appendix \ref{appendix:cde:reparam-invariance} for a proof, which is straightforward change of variables.

\section{Choice of parameterisation}\label{section:cde:parameterisation}
So far we have discussed `the mathematics'. Now we must discuss `the engineering'. We must still choose optimisers, learning rates, model architectures and so on. As is often the case with deep learning, these choices can make or break the efficacy of the model.

\subsection{Neural architectures and gating procedures}\label{section:cde:gating}

The initial value network $\zeta_\theta$ is typically parameterised as an MLP.

The vector field $f_\theta$ is typically parameterised as $f_\theta = \tanh\mathop\circ\mathrm{MLP}_\theta$, where $\tanh$ is applied elementwise, and $\mathrm{MLP}_\theta$ denotes an MLP with weights and biases $\theta$ and as in Section \ref{section:ode:parameterisation} uses a continuously differentiable activation function.

If the $\tanh$ were removed then the MLP could produce arbitrarily large and unconstrained outputs. In contrast the inclusion of a squashing function such as $\tanh$ \textit{constrains the rate of change of the hidden state}. As $f_\theta$ is iteratively evaluated multiple times over the differential equation, then large outputs from $f_\theta$ can easily result in the model exploding, with large and untrainable losses.

This is precisely analogous to RNNs, where one of the key features of GRUs and LSTMs are gating procedures which control the rate of change of the hidden state. (Their other key feature is a differential equation like structure.)

Slightly more complex variations on the same theme may be considered. For example let $\phi_\theta, \psi_\theta \colon \reals^{d_y} \to \reals^{d_y \times d_x}$ be neural networks; for efficiency often the same MLP just with different final affine layers. Then we may define $f_\theta(y) = \sigma(\phi_\theta(y))) * \tanh(\psi_\theta(y))$, where $*$ denotes an elementwise product and $\sigma$ is the sigmoid function applied elementwise. 

\subsection{State-control-vector field interactions}
If the vector field $f_\theta \colon \reals^{d_y} \to \reals^{d_y \times d_x}$ is a feedforward neural network, with final hidden layer of size $d_h \in \naturals$, then the number of scalars for the final affine transformation is of size $\bigO{d_h d_y d_x}$, which can easily be very large.

As such this layer is often the greatest computational bottleneck of a neural CDE. But it is possible that not every three-way interaction between the hidden state $y$ (of size $d_y$), control $x$ (of size $d_x$), and penultimate layer in the vector field (the layer of size $d_h$) actually needs to be modelled.

Anecdotally, making the layer sparse (down to a density of about 1\%) often still produces good results, whilst rank-one representations of the final matrix $f_\theta(y)$, as an outer product of transformations $\reals^{d_y} \to \reals^{d_y}$ and $\reals^{d_y} \to \reals^{d_x}$ seem to produce bad results. One can easily imagine many other kinds of reduced-parameter parameterisations, and this is a topic that merits further investigation.

\subsection{Multi-layer neural CDEs}
Let $y_1 \colon [0, T] \to \reals^{d_{y_1}}$ solve a neural CDE driven by $x \colon [0, T] \to \reals^{d_x}$, with system $f_{\theta, 1} \colon \reals^{d_{y_1}} \to \reals^{d_{y_1} \times d_x}$:
\begin{equation*}
	y_1(t) = y_1(0) + \int_0^t f_{\theta, 1}(y_1(s)) \,\dd x(s)\quad\text{for $t \in (0, T]$}.
\end{equation*}
We may now repeat this procedure: let $y_2 \colon [0, T] \to \reals^{d_{y_2}}$ solve a neural CDE driven by $y_1$, with system $f_{\theta, 2} \colon \reals^{d_{y_2}} \to \reals^{d_{y_2} \times d_{y_1}}$:
\begin{equation*}
	y_2(t) = y_2(0) + \int_0^t f_{\theta, 2}(y_2(s)) \,\dd y_1(s)\quad\text{for $t \in (0, T]$}.
\end{equation*}
This idea of stacking may of course be repeated arbitrarily many times.

The joint system may be solved together as a single CDE driven by $x$. For example in the two-layer case, we obtain
\begin{equation*}
	\begin{bmatrix}y_1(t) \\ y_2(t)\end{bmatrix} = \begin{bmatrix}y_1(0) \\ y_2(0)\end{bmatrix} + \int_0^t \begin{bmatrix}f_{\theta, 1}(y_1(s)) \\ f_{\theta, 2}(y_2(s))f_{\theta, 1}(y_1(s))\end{bmatrix} \,\dd x(s)\quad\text{for $t \in (0, T]$}.
\end{equation*}

The main disadvantage of this approach is that the output dimension of $f_{\theta, 2}$ is large ($d_{y_1} \times d_{y_2}$), and therefore computationally expensive.

This offers a sensible way to increase the model capacity of a neural CDE; see for example \cite{multi-layer-cde}. This is precisely analogous to multi-layer RNNs, in which the hidden state of one RNN is used as the input to another.

\section{Interpolation schemes}\label{section:cde:interpolation}\index{Interpolation}
Suppose we observe some (potentially irregular) time series, each of the form $\mathbf{x} = ((t_0, x_0), \ldots, (t_n, x_n))$, as in Section \ref{section:cde:regular} or Section \ref{section:cde:irregular}. Each $t_j \in \reals$ is the timestamp of observation
\begin{equation*}
x_j = (x_{j, 1}, \ldots, x_{j, d_x - 1}) \in (\reals \cup \{*\})^{d_x - 1},
\end{equation*}
where $*$ denotes the possibility of missing data, $t_0 < \cdots < t_n$. The length $n$ is not assumed to be consistent between different time series.

We are interested in picking $T > 0$, $0 = s_0 < \cdots < s_n = T$, and constructing interpolations $x_\mathbf{x} \colon [0, T] \to \reals^{d_x}$ such that $x_\mathbf{x}(s_j) = (t_j, x_j)$.

Note that we could also include $c_j(\mathbf{x})$ channels, which in general are needed when handling missing data; see Sections \ref{section:cde:missingness1} and \ref{section:cde:missingness2}. These are handled in precisely the same way as the $x_j$ channels (just construct interpolations $x_\mathbf{x} \colon [0, T] \to \reals^{2 d_x - 1}$ such that $x_\mathbf{x}(s_j) = (t_j, x_j, c_j(\mathbf{x}))$) so for simplicity of notation we leave them out here.

\begin{remark}
	That neural CDEs need interpolation schemes sometimes attracts skepticism. Are we imputing missing data? Are we constructing a continuous-time approximation to some underlying data process? Why, morally speaking, should we need to construct an interpolation scheme when the actual data we have observed is discrete?
	
	The answers are: `no', `yes (but it's not important)', and `to process the data at its natural timescale', respectively. Each of these points has been discussed earlier in the text, in Remark \ref{remark:cde:no-impute}, Remark \ref{remark:cde:approx-not-important}, and Section \ref{section:cde:continuous-time-computation} respectively.
\end{remark}

We begin with some theoretical conditions that we would like ideal interpolation schemes to satisfy, and then present some sensible choices. The best choice of interpolation scheme will depend on the problem at hand.

Much of the following section is drawn from \cite{morrill2021cdeonline}.

\subsection{Theoretical conditions}
There are two main theoretical conditions, namely \textit{measurability} and \textit{smoothness}. 

\subsubsection{Measurability}\index{Interpolation!Measurable}

Any given time series problem needs outputs at particular times. For example we may wish to only have an output after having observed an entire time series. Alternatively we may wish to produce a continuously-evolving output, updated as more data arrives over time.

We formalise this distinction in terms of \textit{measurability},\footnote{In \cite{morrill2021cdeonline} the terminology of `online' is used instead.} and will require an interpolation scheme that supports the desired behaviour.

We describe problems, and interpolation schemes, as exhibiting one of three different kinds of measurability. Figure \ref{fig:measurable} provides a visual summary.

\begin{figure}
	\includegraphics[width=\linewidth]{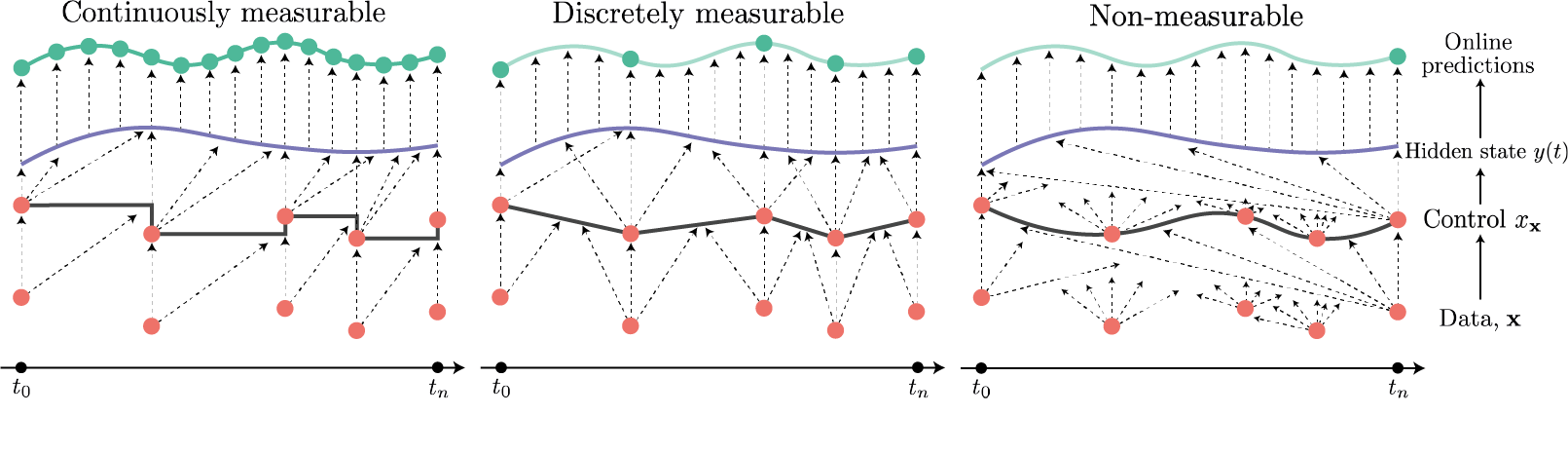}
	\caption{Summary of measurability definitions. Arrows indicate what data can influence. A green dot indicates that a measurable prediction can be made at that point. \textbf{Left:} a continuously measurable model in which no information is passed backward in time, resulting in a measurable solution at all points in time. \textbf{Middle:} a discretely measurable model in which information can be passed backwards-in-time, but no further than the preceding observation. \textbf{Right:} a non-measurable scheme in which information is passed backward in time further than the preceding observation.}\label{fig:measurable}
\end{figure}

\paragraph{Continuously measurable} We say that an interpolation scheme is \textit{continuously measurable} if $x_\mathbf{x}(s)$ depends only on those $(t_j, x_j)$ with $s_j \leq s$. (Recalling that $x_\mathbf{x}(s_j) = (t_j, x_j)$.) That is to say, only observations in the past or present may be used to define the interpolation scheme.

This is probably the most intuitively natural definition of measurability, but it is relatively difficult to construct interpolation schemes satisfying it. (We will present only a single such scheme.) Practically speaking many use cases will find a weaker notion of measurability acceptable.

This kind of measurability is needed if data is arriving over time at inference time (sometimes described as the problem being \textit{online}), and either:
\begin{itemize}
	\item model predictions are needed between observations;
	\item model predictions are needed prior to the final observation in a time series, and there is missing data.
\end{itemize}

\paragraph{Discretely measurable} We say that an interpolation scheme is \textit{discretely measurable} if $x_\mathbf{x}(s)$ depends only on those $(t_j, x_j)$ with $s_{j-1} \leq s$. That is to say, we may look up to one observation into the future when defining the interpolation scheme.

For example, this is the case with linear interpolation. $s \mapsto (1-s)a + sb$ with $s \in [0, 1]$ depends on $b$ for all $s$, even though $b$ will only be attained at $s=1$.

This kind of measurability is needed if data is arriving over time at inference time (sometimes described as the problem being \textit{online}), and model predictions are only needed at observations, and there is no missing data. Simply integrate up to some $s_j$, wait until $(t_{j+1}, x_{j+1})$ are observed, and then interpolate and integrate over $[s_j, s_{j+1}]$.

\paragraph{Non-measurable} Finally we say that a scheme is not measurable if $x_\mathbf{x}(s)$ may depend only on any and all $(t_j, x_j)$. For example, this is the case with natural cubic splines.

Such schemes are appropriate if, at inference time, the whole time series will be available prior to evaluating the model.

\subsubsection{Smoothness}\index{Interpolation!Smooth}
The second desirable theoretical property for an interpolation scheme is smoothness.

We will not define smoothness in a mathematically rigorous way. Rather, we point out that `smoother' interpolation schemes will result in easier-to-integrate dynamics, which will improve the computational efficiency of the model.

The main dichotomy here is whether the dynamics are globally smooth (for example a cubic spline) or piecewise smooth (for example linear interpolation). The former is preferable but the latter can be tolerated.

\begin{remark}
	See Section \ref{section:numerical:stacked} for how to correctly integrate piecewise smooth dynamics when using adaptive step size solvers.
	
	See also Section \ref{section:numerical:regularising-derivatives}, which (mainly in the neural ODE setting) seeks to regularise higher-order derivatives in order to promote easy-to-integrate dynamics.
\end{remark}


\subsection{Choice of interpolation points}\label{section:cde:interpolation-points}
The choice of $T>0$ and interpolation points $s_j$ (such that $x_\mathbf{x}(s_j) = (t_j, x_j)$) is really about determining the desired behaviour of the numerical solver.

In continuous time, this choice is arbitrary, by the reparameterisation property of Section \ref{section:reparam-invariance}. The value of the integral is invariant to the choices of $T$ and $s_j$.

Practically speaking, this does not completely carry through to the numerical discretisation. For example consider using $s_j = j$ and a fixed-step numerical solver with unit step size. Then a single numerical step is made between each observation, a fixed number of vector field evaluations would be made, and the neural CDE reduces to an RNN.

Overall, the spacing between $s_j$ corresponds roughly to the amount of computational work that should be done between those observations. (Either precisely, when using a fixed solver, or approximately, when using an adaptive solver.) As per Section \ref{section:cde:continuous-time-computation}, a desirable choice is for the amount of computational work to scale with the `natural timescale' at which the data varies. For many datasets this means a reasonable choice is $s_j  = t_j$.

\subsection{Particular interpolation schemes}

There are a few main interpolation schemes of interest. In every case, each channel is interpolated separately. If there is missing data then it is interpolated over; for example if $x_{j, 1}$ and $x_{j+2, 1}$ are observed but $x_{j+1, 1}$ is missing then we apply the following procedures over $[s_j, s_{j+2}]$ rather than $[s_j, s_{j+1}]$.

\subsubsection{Hermite cubic splines with backward differences}\index{Hermite cubic splines with backward differences}
The first choice of interest are \textit{Hermite cubic splines with backward differences}. Each interval $[s_j, s_{j+1})$ is treated independently, and the interpolation over this interval is chosen to satisfy
\begin{align*}
	x_\mathbf{x}(s_j) &= (t_j, x_j),\\
	x_\mathbf{x}(s_{j+1}) &= (t_{j+1}, x_{j+1}),\\
	\frac{\dd x_\mathbf{x}}{\dd t}(s_j) &= \left(\frac{t_j - t_{j-1}}{s_j - s_{j-1}}, \frac{x_j - x_{j-1}}{s_j - s_{j-1}}\right),\\
	\frac{\dd x_\mathbf{x}}{\dd t}(s_{j+1}) &= \left(\frac{t_{j+1} - t_j}{s_{j+1} - s_j}, \frac{x_{j+1} - x_j}{s_{j+1} - s_j}\right).
\end{align*}

\paragraph{Measurability} This scheme is discretely measurable.

\paragraph{Smoothness} Such splines are by construction continuously differentiable, so adaptive step size solvers will find the resulting dynamics easy to integrate.

\paragraph{Choice of $s_j$} Typically $s_j = t_j$ is a reasonable choice.

\subsubsection{Linear interpolation}\index{Interpolation!Linear}
One simple option is just linear interpolation:
\begin{equation*}
	x_\mathbf{x}(s) = (t_j, x_j) + (t_{j+1} - t_j, x_{j+1} - x_j)\frac{s - s_j}{s_{j+1} - s_j}\quad\text{ for $s \in [s_j, s_{j+1})$}.
\end{equation*}
(Mapping over each entry of the $(t, x)$ tuple.)

\paragraph{Measurability} This scheme is discretely measurable.

\paragraph{Smoothness} This scheme is only piecewise continuously differentiable. If reducing the neural CDE to an ODE by taking $\int f_\theta(y(s)) \,\dd x(s) = \int f_\theta(y(s)) \nicefrac{\dd x}{\dd s}(s) \,\dd s$ as in equation \eqref{eq:ncde-to-node}, then the vector field $f_\theta(y(s)) \nicefrac{\dd x}{\dd s}(s)$ will be piecewise constant.

This makes linear interpolation a poor choice if using an adaptive step size numerical solver, which will struggle with these jumps. However if using a fixed step size solver it can be just as good as Hermite cubic splines with backward differences, whilst being slightly cheaper to compute.

\paragraph{Choice of $s_j$} Typically $s_j = t_j$ is a reasonable choice.

\subsubsection{Rectilinear interpolation}\index{Interpolation!Rectilinear}

%
%

Define $\bar{x}_j = (\bar{x}_{j, 1}, \ldots, \bar{x}_{j, d_x - 1}) \in \reals^{d_x - 1}$ as being the fill-forward\footnote{Explicitly: let $\bar{x}_{j, k} = x_{\bar{j}(j, k), k}$, where $\bar{j}(j, k) = \max\set{m \leq j}{x_{m, k} \neq *}$, recalling that $*$ denotes missing data. If this set is empty then define $\bar{x}_{j, k} = 0$.} of $x_j$.

Now additionally select points $r_j \in [0, T]$, for $j \in \{1, \ldots, n\}$, such that
\begin{equation*}
s_0 < r_1 < s_1 < r_2 < \cdots < s_{n-1} < r_{n} < s_n.
\end{equation*}
Rectilinear interpolation is defined by taking $x_\mathbf{x}(s_j) = (t_j, \bar{x}_j)$, $x_\mathbf{x}(r_j) = (t_{j + 1}, \bar{x}_j)$, and linearly interpolating between $s_0, r_1, s_1, r_2, \ldots, r_n, s_n$.

\paragraph{Measurability} This now satisfies the measurability condition we require. At inference time we can wait at each $s_j$ for the next data point to arrive (regardless of whether it is only partially observed / has missing data), interpolate over $[s_j, s_{j + 1}]$ in the manner above, and then solve the CDE over $[s_j, s_{j + 1}]$.

\paragraph{Smoothness} Rectilinear interpolation is only piecewise differentiable. As with linear interpolation, when reduced to an ODE then the vector field has jumps. Fixed step size solvers are one resolution to this problem. Alternatively an adaptive step size numerical solver can be made aware of these jumps so as to treat them in the appropriate way, see Section \ref{section:numerical:stacked}.

\paragraph{Choice of $s_j$ and $r_j$} Typically $s_j = t_j + j$ and $r_j = t_j + j - 1$ is a reasonable choice.

\subsubsection{Natural cubic splines}\index{Natural cubic splines}
One simple approach is to use natural cubic splines. (Indeed this was used in the original neural CDE paper \cite{kidger2020neuralcde}.)

\paragraph{Measurability} Unfortunately, natural cubic splines are not measurable. This limits their applicability.

\paragraph{Smoothness} However, natural cubic splines are at least very smooth.

\paragraph{Choice of $s_j$} Typically $s_j = t_j$ is a reasonable choice.

\paragraph{Model performance} For reasons unknown, neural CDEs that use natural cubic splines produce slightly worse results than those using other interpolation schemes \cite{morrill2021cdeonline}.

\subsubsection{Overall}
Most problems will find either Hermite cubic splines with backward differences or rectilinear interpolation to be of most interest. Use Hermite cubic splines with backward differences if possible, due to their smoothness and relatively good measurability. If their measurability properties are insufficient, then use rectilinear interpolation.

\section{Comments}
CDEs are a classic piece of mathematics, emerging essentially as a meaningful special case of the more general rough differential equations introduced by \cite{lyons1998}. The first few pages of \cite{roughstochasticNF} give an excellent brief introduction to the essential ideas. \cite{levy-lyons} is our recommended introductory text. \cite{FrizVictoir10} is the canonical reference text.

(As a fascinating historical note, the essential ideas behind CDEs may actually be traced back to Newton \cite[Prob. 1, Prob. 2]{newton-fluxion}. Newton considers evolving systems in multiple variables, with position (fluent) and derivative (fluxion). Given a relation for either fluent or fluxion, then the other is then solved for. Obtaining the relation of fluents from a relation of fluxions is precisely what we would term `solving a CDE'. We thank Terry Lyons for this observation.)

Various texts use different pieces of terminology to refer to the concept of a CDE. Much of the rough path literature uses the terms CDE and RDE interchangeably. Meanwhile \cite{FrizVictoir10} prefer to use CDE and ODE interchangeably. Some texts use the term `controlled ordinary differential equations'.

Specifically \textit{neural} CDEs were first introduced in \cite{kidger2020neuralcde}, where they were applied to both regular and irregular time series, and most of the relevant theoretical properties discussed. The follow-up work \cite{morrill2021cdeonline} investigated the choice of interpolation scheme, and promoted the use of the alternative ones just presented (\cite{kidger2020neuralcde} used only natural cubic splines). The discussion here is an extension of the one introduced there, and is partly new here. Rectilinear interpolation is due to \cite{levin2013}; the pieces of the interpolation for which the time channel is constant are sometimes referred to as `virtual time'.

Many of the computational concerns discussed (batching, smoothness and so on) arose from experimentation using both \cite{torchcde} and \cite{diffrax}. The discussion on batching (Section \ref{section:cde:batching}) is new here, and is relevant beyond just CDEs -- much of the literature has assumed that a differential equation solver must operate over the same time interval for each batch element, and had to work around this limitation (see for example the `time-varying CNF' of \cite{spatiotemporal}).

We recall the link between CDEs and control theory (Section \ref{section:cde:control-theory}). Meanwhile control theory has well-known links to reinforcement learning (RL). RL applications either explicitly including or of essentially similar character to neural CDEs therefore include \cite{alvarez2020dynode, pmlr-v136-killian20a, pmlr-v139-lutter21a} amongst others.

The neural RDE formulation applying neural CDEs to long time series was introduced in \cite{morrill2021neuralrough}. The same rough path theoretic ideas also appear in \cite{rnn-kernel}, to frame RNNs as a kernel method.

Applications to training neural SDEs (really the focus of our next chapter) were introduced in \cite{kidger2021sde1, kidger2021sde2}.

Much of the discussion on good architectural choices, gating, sparsity and so on, is new here.
\chapter{Neural Stochastic Differential Equations}\label{chapter:neural-sde}
\section{Introduction}
\subsection{Stochastic differential equations} Stochastic differential equations have seen widespread use for modelling real-world random phenomena, such as particle systems \cite{langevinbook, langevinbook2, sdesformoldynamics}, financial markets \cite{blackscholes, cir, ratebook}, population dynamics \cite{stocLotkaVolterra, populationgrowth} and genetics \cite{wrightfisher}. They are a natural extension of ordinary differential equations (ODEs) for modelling systems that evolve in continuous time subject to uncertainty.

The dynamics of an SDE consist of a deterministic term and a stochastic term:
\begin{equation}\label{eq:sde}
	\dd y(t) = \mu(t, y(t)) \,\dd t + \sigma(t, y(t)) \circ \dd w(t),
\end{equation}
where
\begin{align*}
	\mu &\colon [0, T] \times \reals^{d_y}\to \reals^{d_y},\\
	\sigma &\colon [0, T]\times \reals^{d_y} \to \reals^{d_y \times d_w}
\end{align*}
are suitably regular functions, $w \colon [0, T] \to \reals^{d_w}$ is a $d_w$-dimensional Brownian motion, and $y \colon [0, T] \to \reals^{d_y}$ is the resulting $d_y$-dimensional continuous stochastic process.

The strong solution $y$ is guaranteed to exist and be unique given mild conditions: that $\mu$, $\sigma$ are Lipschitz, and that $\expect\left[y(0)^2\right] < \infty$.\index{Existence}\index{Uniqueness}

We refer the reader to \cite{revuz-yor} for a rigorous account of stochastic integration.

\paragraph{It{\^o} versus Stratonovich}\index{It{\^o}}\index{Stratonovich}\label{section:sde:introduction}
The notation ``$\circ$'' in the noise refers to the SDE being understood in the sense of Stratonovich integration. This is as an alternative to the standard notion of It{\^o} integration.

The reader unfamiliar with Stratonovich integration should generally feel free to ignore this subtlety. Stratonovich SDEs will sometimes be slightly more efficient to backpropagate through (Remark \ref{remark:numerical:stratonovich-over-ito}, later). However, any It{\^o} SDE may be converted to a Stratonovich SDE, and vice versa, so as we will shortly introduce learnt (neural) vector fields then modelling-wise the choice is arbitrary.

\paragraph{Theoretical construction of SDEs}
SDEs have typically been constructed theoretically, and are usually relatively simple.

One frequent and straightforward technique is to fix a constant matrix $\sigma$, and add ``$\sigma \circ \dd w(t)$'' to a pre-existing ODE model.\footnote{In passing we remark that It{\^o} and Stratonovich are identical in this case as the noise is additive so the corresponding It{\^o}--Stratonovich correction term is zero. We could equally well have written ``$\sigma \,\dd w(t)$''.}

As another example, the Black--Scholes equation, widely used to model asset prices in financial markets, has only two scalar parameters: a fixed drift and a fixed diffusion \cite{blackscholes}.

\paragraph{Calibrating SDEs}
Once an SDE model has been chosen, then model parameters must be calibrated\footnote{Fit, trained.} from real-world data.

Since SDEs produce random sample paths, the parameters are typically chosen so that the average behaviour of the SDE matches some statistic(s). A classical approach to calibrating SDEs to observed data $y_\text{true}$ is to pick some prespecified functions of interest $F_1, \ldots, F_N$, and then ask that $\expect_{y}\left[F_i(y)\right] \approx \expect_{y_\text{true}}\left[F_i(y_\text{true})\right]$ for all $i$. For example this may be done by optimising
\begin{equation}\label{eq:sde:calibration}
	\min_\theta \max_{i=1,\ldots,N} \big|\expect_{y}\left[F_i(y)\right] - \expect_{y_\text{true}}\left[F_i(y_\text{true})\right]\big|
\end{equation}
where the model $y$ depends implicitly on parameters $\theta$.

This ensures that the model and the data behave the same with respect to the functions $F_i$. The functions $F_i$ are known as either `witness functions' or `payoff functions' depending on the field \cite{mmd-gan, josef-sde}. If the SDE is simple enough -- for example the analytically tractable Black--Scholes model -- then equation \eqref{eq:sde:calibration} can often be computed explicitly \cite{blackscholes}.

\subsection{Generative and recurrent structure}
SDEs feature inherent randomness. In modern machine learning parlance SDEs are generative models.

\paragraph{Comparison to random RNNs}
As usual, a numerically discretised neural (stochastic) differential equation has a correspondence in the deep learning literature. As with neural CDEs, the appropriate analogy is an RNN. In this case its input is random noise -- Brownian motion -- and its output is a generated sample.

Consider the autonomous one-dimensional It{\^o} SDE
\begin{equation*}
	\dd y(t) = \mu(y(t))\,\dd t + \sigma(y(t)) \,\dd w(t),
\end{equation*}
with $y(t), \mu(y(t)), \sigma(y(t)), w(t) \in \reals$. Then its numerical Euler--Maruyama discretisation is
\begin{align*}
	y_{j + 1} = y_j + \mu(y_j) \Delta t + \sigma(y_j) \Delta w_j,
\end{align*}
where $\Delta t$ is some fixed time step and $\Delta w_j \sim \normal{0}{\Delta t}$. This numerical discretisation is clearly just an RNN of a particular form.


\paragraph{Generative time series models}
Each sample $y$ from an SDE
\begin{equation*}
	\dd y(t) = \mu(t, y(t)) \,\dd t + \sigma(t, y(t)) \circ \dd w(t),
\end{equation*}
is a continuous-time path $y \colon [0, T] \to \reals^{d_y}$. As such, we may treat neural SDEs as generative time series models.

(Generative) time series models are of classical interest, with forecasting models such as Holt--Winters \cite{holt-winters1, holt-winters2}, ARMA \cite{arma}, ARCH \cite{arch}, GARCH \cite{garch} and so on.

It has also attracted much recent interest with, besides neural SDEs, the development of ODE-based models like latent ODEs (Section \ref{section:ode:latent-ode})\footnote{And related ideas such as ODE$^2$VAE \cite{ode2vae} or Neural ODE Processes \cite{norcliffe2021neural-ode-process}}; discrete-time models like Time Series GAN \cite{ts-gan}; non-ODE continuous-time models like CTFPs \cite{ctfp, clpf} and Copula Processes \cite{copula-processes}.

\newcommand{\Brownian}[9]{
\pgfmathsetseed{#9}
\draw[#4] (#5, #6)
\foreach \x in {1,...,#1}
{   -- ++(#2,#8 + rand*#3)
}
node[right] {#7};
}

\newcommand{\SDE}[9]{
\pgfmathsetseed{#9}
\draw[#4] (#5, #6)
\foreach \x in {1,...,#1}
{   -- ++(#2,#8*\x - #8*0.5*#1 + rand*#3)
}
node[right] {#7};
}

\begin{figure}\centering
\begin{tikzpicture}[scale=0.9]\small
    \node at (-1.5,0) {};
    
    \draw[->, thick] (0, 0) -- (4, 0);
    \draw[->, thick] (10, 0) -- (14, 0);
    
    \Brownian{100}{0.04}{0.1}{color=blue!80!black, text=black, line width=1pt}{0}{0.7}{}{0.005}{56791}
    \Brownian{100}{0.04}{0.1}{color=blue!35!white, line width=1pt}{0}{0.6}{}{0.005}{56789}
    \Brownian{100}{0.04}{0.1}{color=blue!35!white, line width=1pt}{0}{0.5}{}{0.005}{56790}
    \Brownian{100}{0.04}{0.1}{color=blue!35!white, line width=1pt}{0}{1.1}{}{0.005}{45678}
    
    \SDE{100}{0.04}{0.1}{color=red!80!black, text=black, line width=1pt}{0}{3}{}{0.0008}{7890}
    \SDE{100}{0.04}{0.1}{color=red!50!white, line width=1pt}{0}{3.4}{}{0.0008}{7891}
    \SDE{100}{0.04}{0.1}{color=red!50!white, line width=1pt}{0}{3.6}{}{0.0008}{7892}
    \SDE{100}{0.04}{0.1}{color=red!50!white, line width=1pt}{0}{2.8}{}{0.0008}{7893}
    
    \node[anchor=west] at (3.8, 3.3) {\begin{tabular}{l}\textbf{SDE}\end{tabular}};
    \node[anchor=west] at (3.8, 2) {\begin{tabular}{l} Continuously \\ inject noise\end{tabular}};
    \node[anchor=west] at (3.8, 0.6) {\begin{tabular}{l}\textbf{Brownian}\\\textbf{Motion}\end{tabular}};
    
    \draw[decorate, decoration={brace, amplitude=13pt}] (0, 4) -- (4, 4);
    
    \draw (0.3,4.5) rectangle (3.7,5.3) node[midway] {Fixed statistics};
    
    \node at (2, 5.7) {\textit{Classical approach}};
    
    \begin{scope}[shift={(2,0)}]
    \SDE{100}{0.04}{0.07}{color=yellow!70!red!60!black, text=black, line width=1pt}{8}{1.2}{}{0.0008}{7894}
    \SDE{100}{0.04}{0.07}{color=yellow!80!red!90!black, line width=1pt}{8}{1.6}{}{0.0008}{7895}
    \SDE{100}{0.04}{0.07}{color=yellow!80!red!90!black, line width=1pt}{8}{1.8}{}{0.0008}{7896}
    \SDE{100}{0.04}{0.07}{color=yellow!80!red!90!black, line width=1pt}{8}{0.7}{}{0.0006}{7897}
    
    \begin{scope}[shift={(10, 3)}] 
    \draw[color=green!60!black, line width=1pt] plot [domain=-2:2, samples=100, smooth] (\x, {0.2 + 0.8 * tanh(\x*2) + 0.05*rand});
    
    \draw[color=green!90!black, line width=1pt] plot [domain=-2:2, samples=100, smooth] (\x, {0.2 + 0.5 * tanh(\x*1.5 + 0.5) + 0.05*rand});
    
    \draw[color=green!90!black, line width=1pt] plot [domain=-2:2, samples=100, smooth] (\x, {0.1 + max(0,0.8-3*\x*\x-\x) + 0.8 * tanh(\x*1.5 - 0.5) + 0.05*rand});
    
    \draw[color=green!90!black, line width=1pt] plot [domain=-2:2, samples=100, smooth] (\x, {-0.1 - 0.5 * tanh(\x*2) + 0.05*rand});
    
    \node[anchor=west] at (1.8, 1) {\begin{tabular}{l}Final value\end{tabular}};
    \node[anchor=west] at (1.8, 0.2) {\begin{tabular}{l}\textbf{CDE}\end{tabular}};
    \node[anchor=west] at (1.8, -1) {\begin{tabular}{l} Continuously \\ perform control\end{tabular}};
    \node[anchor=west] at (1.8, -2.3) {\begin{tabular}{l}\textbf{Data or SDE}\end{tabular}};
    
    \draw[dashed, line width=0.8pt, ->] (2, -0.8) -- (2, 1.5);
    
    \draw (-1.3,1.5) rectangle (2.1,2.3) node[midway] {Learnt statistic};
    
    \node at (0, 2.7) {\textit{Generalised (GAN) approach}};
    \end{scope}
    \draw[dashed, line width=0.8pt, ->] (8,2) -- (8,2.2);
    \draw[dashed, line width=0.8pt, ->] (8.4,1.6) -- (8.4,2.2);
    \draw[dashed, line width=0.8pt, ->] (8.8,1.24) -- (8.8,2.2);
    \draw[dashed, line width=0.8pt, ->] (9.2,1.03) -- (9.2,2.2);
    \draw[dashed, line width=0.8pt, ->] (9.6,0.9) -- (9.6,2.45);
    \draw[dashed, line width=0.8pt, ->] (10,0.89) -- (10,2.66);
    \draw[dashed, line width=0.8pt, ->] (10.4,0.96) -- (10.4,2.4);
    \draw[dashed, line width=0.8pt, ->] (10.8,1.02) -- (10.8,2.23);
    \draw[dashed, line width=0.8pt, ->] (11.2,1.15) -- (11.2,2.22);
    \draw[dashed, line width=0.8pt, ->] (11.6,1.45) -- (11.6,2.22);
    \draw[dashed, line width=0.8pt, ->] (12,1.59) -- (12,2.28);
    \end{scope}
    
    \draw[dashed, line width=0.8pt, ->] (4.1, 3.7) -- (7, 3.7) -- (7, 1) -- (9.9, 1);
    
    \draw[dashed, line width=0.8pt, ->] (0,1.22) -- (0,2.4);
    \draw[dashed, line width=0.8pt, ->] (0.4,1.1) -- (0.4,2.2);
    \draw[dashed, line width=0.8pt, ->] (0.8,1.09) -- (0.8,2.05);
    \draw[dashed, line width=0.8pt, ->] (1.2,1.12) -- (1.2,2);
    \draw[dashed, line width=0.8pt, ->] (1.6,1.03) -- (1.6,1.9);
    \draw[dashed, line width=0.8pt, ->] (2,0.94) -- (2,1.85);
    \draw[dashed, line width=0.8pt, ->] (2.4,1.19) -- (2.4,1.95);
    \draw[dashed, line width=0.8pt, ->] (2.8,1.05) -- (2.8,2.1);
    \draw[dashed, line width=0.8pt, ->] (3.2,1.13) -- (3.2,2.42);
    \draw[dashed, line width=0.8pt, ->] (3.6,1.15) -- (3.6,2.65);
    \draw[dashed, line width=0.8pt, ->] (4,1.18) -- (4,2.85);
\end{tikzpicture}
\caption{Brownian motion is continuously injected as noise into an SDE to generate time series. The classical approach fits the SDE to prespecified statistics. One (important) way of handling neural SDEs is to generalise from prespecified statistics to a learnt statistic, namely the discriminator of a (Wasserstein) GANs.}\label{fig:sde-gan-abstract}
\end{figure}
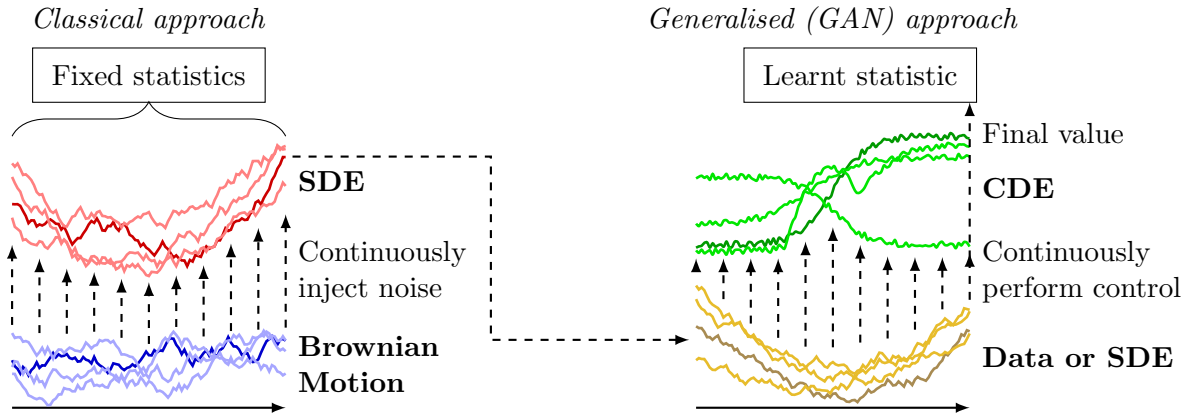

See Figure \ref{fig:sde-gan-abstract} for an abstract summary of many of the essential ideas: that SDEs are generative time series models, how SDEs may classically be calibrated, and one of the ways in which we will later generalise this approach to neural networks, via \textit{SDE-GANs} (Section \ref{section:sde:sde-gan}).

\paragraph{`Static' generative models}
We may also consider just the terminal value $y(T)$ of an SDE
\begin{equation*}
	\dd y(t) = \mu(t, y(t)) \,\dd t + \sigma(t, y(t)) \circ \dd w(t).
\end{equation*}
This is a sample drawn from some distribution over $\reals^{d_y}$. As such we may also treat neural SDEs as `static' generative models -- that is to say, not over a time series.

This immediately draws natural connections to a variety of topics. This is the same basic set-up as a continuous normalising flow (Section \ref{section:ode:cnf}), except that the randomness is injected via a Brownian motion $w$ rather than a random initial condition $y(0)$.

It is also the same starting point used in score-based generative modelling\index{Score-based generative modelling}, in which a neural drift and fixed additive diffusion is used, with the initial-to-terminal map calculating a transition between two distributions \cite{song2021scorebased, diffusion-bridge}.

We shall focus mainly on the time-series case discussed in the previous heading. At time of writing, the connections between neural SDEs as presented here, and CNFs and score-based modelling, are largely unexplored.

\paragraph{Comparison to neural CDEs}
We have now described both neural CDEs and neural SDEs as `continuous time RNNs'. It is worth being precise about the distinction.

(Neural) CDEs model \textit{functions of time series}, or equivalently \textit{functions of paths}. The path is an input and the output is, for example, a classification result determining whether the input path is a clockwise or anticlockwise spiral.

(Neural) SDEs model \textit{distributions on time series}, or equivalently \textit{distributions on paths}. Rather than modelling some function of the path, it is is the paths themselves that are being modelled.

In this respect the terminology of differential equations is slightly more precise than the terminology of neural networks, which uses `RNN' to describe both concepts.

\section{Construction}\label{section:sde:construction}
The following constructions are primarily from \cite{kidger2021sde1}.

Let $T>0$ be a fixed time horizon and consider a path-valued random variable $x_\text{true} \colon [0, T] \to \reals^{d_x}$, with $d_x \in \naturals$ the dimensionality of the data. $x_{\text{true}}$ is what we wish to model, and is the random variable we assume we have observed samples from. For example, this may correspond to the evolution of stock prices over time.

(Typically we actually observe $x_\text{true}$ only at some discretised time stamps; not over a full continuous-time path. For ease of presentation we neglect this detail for now and will return to it later.)

Let $w \colon [0, T] \to \reals^{d_w}$ be a $d_w$-dimensional Brownian motion, and let $v \sim \normal{0}{\eye{d_v}}$ be drawn from a $d_v$-dimensional standard multivariate normal. The values $d_w, d_v \in \naturals$ are hyperparameters describing the size of the noise. Let
\begin{align}
	\zeta_\theta &\colon \reals^{d_v} \to \reals^{d_y},\nonumber\\
	\mu_\theta &\colon [0, T] \times \reals^{d_y} \to \reals^{d_y},\nonumber\\
	\sigma_\theta &\colon [0, T] \times \reals^{d_y} \to \reals^{d_y \times d_w},\nonumber\\
	&\alpha_\theta \in \reals^{d_x \times d_y},\nonumber\\
	&\beta_\theta \in \reals^{d_x},\label{eq:sde:nsde-setup}
\end{align}
where $\zeta_\theta$, $\mu_\theta$ and $\sigma_\theta$ are neural networks. Collectively $\zeta_\theta$, $\mu_\theta$, $\sigma_\theta$, $\alpha_\theta$ and $\beta_\theta$ are parameterised by $\theta$. The dimension $d_y$ is a hyperparameter describing the size of the hidden state.

Then a \textit{neural stochastic differential equation} is a model of the form
\begin{align}
	y(0) &= \zeta_\theta(v),\nonumber\\
	\dd y(t) &= \mu_\theta(t, y(t)) \,\dd t + \sigma_\theta(t, y(t)) \circ \dd w(t),\nonumber\\
	x(t) &= \alpha_\theta y(t) + \beta_\theta,\label{eq:sde:nsde}
\end{align}
for $t \in [0, T]$, with $y \colon [0, T] \to \reals^{d_y}$ the (strong) solution to the SDE.

The objective will be to train $\theta$ so that the distribution of the model $x$ is approximately equal to the distribution of the data $x_\text{true}$. (For some notion of `approximate'.)

\paragraph{Architecture}\index{Markov assumption}
Equation \eqref{eq:sde:nsde} has a certain minimum amount of structure. First, the solution $y$ represents hidden state. If it were the output, then future evolution would satisfy a Markov property which need not be true in general. This is the reason for the additional readout operation to $x$. 

Second, there must be an additional source of noise for the initial condition, passed through a nonlinear $\zeta_\theta$, as $x(0) = \alpha_\theta \zeta_\theta(v) + \beta_\theta$ does not depend on the Brownian noise $w$. This will be a learnt approximation to the initial condition of the SDE.

$\zeta_\theta, \mu_\theta$, and $\sigma_\theta$ may be taken to be any standard network architectures, such as feedforward networks.

\paragraph{RNNs as discretised SDEs}
This minimal amount of structure parallels that of RNNs. The solution $y$ corresponds to the hidden state of an RNN.

\paragraph{Sampling} Given a trained model, we sample from it by sampling some initial noise $v$ and some Brownian motion $w$, and then solving equation \eqref{eq:sde:nsde} with a numerical SDE solver.

\paragraph{Comparison to the Fokker--Planck equation}\index{Fokker--Planck equation} The distribution of an SDE, as learnt by a neural SDE, contains more information than the distribution obtained by learning a corresponding Fokker--Planck equation. The solution to a Fokker--Planck equation gives (the time evolution of) the probability density of a solution \emph{at fixed times}. It does not encode information about the time evolution of individual sample paths. This is exemplified by stationary processes, whose sample paths may be nonconstant but whose distribution does not change over time.

\section{Training criteria}
Equation \eqref{eq:sde:nsde} produces a random variable $x \colon [0, T] \to \reals^{d_x}$ implicitly depending on parameters $\theta$. This model must still be fit to data. This may be done by optimising a distance between the probability distributions (laws) for $x$ and $x_\text{true}$.


There are two main options: fitting a Wasserstein distance, or fitting a KL divergence. These correspond to \textit{SDE-GANs} and \textit{latent SDEs} respectively.

\subsection{SDE-GANs}\label{section:sde:sde-gan}\index{SDE-GANs}

Let $\prob_x$ denote the law of the model $x$. Likewise let $\prob_{x_\text{true}}$ denote the (empirical) law of the data $x_\text{true}$. Let $W(\prob_x, \prob_{x_\text{true}})$ denote the 1-Wasserstein distance between them. We may train the model by optimising
\begin{equation*}
	\min_\theta W(\prob_x, \prob_{x_\text{true}}),
\end{equation*}
where $\prob_x$ depends implicitly on the learnt parameters $\theta$.

We will do so in the usual way for Wasserstein GANs, by constructing a discriminator and training adversarially \cite{wgan}.

Each sample from the generator is a continuous path $x \colon [0, T] \to \reals^{d_x}$; these are infinite dimensional and the discriminator must accept such paths as inputs. Fortunately there is a natural choice: parameterise the discriminator as a neural CDE, as in Chapter \ref{chapter:neural-cde}.

This approach is due to \cite{kidger2021sde1}.

\paragraph{Architecture}
Let
\begin{align*}
\xi_\phi &\colon \reals^{d_x} \to \reals^{d_h},\\
f_\phi &\colon [0, T] \times \reals^{d_h} \to \reals^{d_h},\\
g_\phi &\colon [0, T] \times \reals^{d_h} \to \reals^{d_h \times d_x},\\
&m_\phi \in \reals^{d_h},
\end{align*}
where $\xi_\phi$, $f_\phi$ and $g_\phi$ are (Lipschitz) neural networks. Collectively they are parameterised by $\phi$. The value $d_h \in \naturals$ is a hyperparameter describing the size of the hidden state.

Recalling that $x$ is the generated sample, we take the discriminator to be a CDE
\begin{align}
	h(0) &= \xi_\phi(x(0)),\nonumber\\
	\dd h(t) &= f_\phi(t, h(t)) \,\dd t + g_\phi(t, h(t)) \circ \dd x(t),\nonumber\\
	D &= m_\phi \cdot h(T),\label{eq:sde:discriminator}
\end{align}
for $t \in [0, T]$, with $h \colon [0, T] \to \reals^{d_h}$ the (strong) solution to this CDE, and where $\cdot$ denotes the dot product.

The solution to the CDE exists given mild conditions, namely Lipschitz $f_\phi$ and $g_\phi$; simply concatenate \eqref{eq:sde:nsde} and \eqref{eq:sde:discriminator} together and treat the joint system as an SDE.

The value $D \in \reals$, which is a function of the terminal hidden state $h(T)$, is the discriminator's score for real versus fake; correspondingly we define the overall action of the discriminator via $F_\phi(x) = D$. This is a deterministic function of the generated sample $x$.

\paragraph{Summary of equations}
See Figure \ref{fig:sde-gan-equation-summary} for a summary of equations, combining together both generator and discriminator.

\newcommand{\fillcolour}{yellow!40!white}
\newcommand{\backfillcolour}{green!10!white}

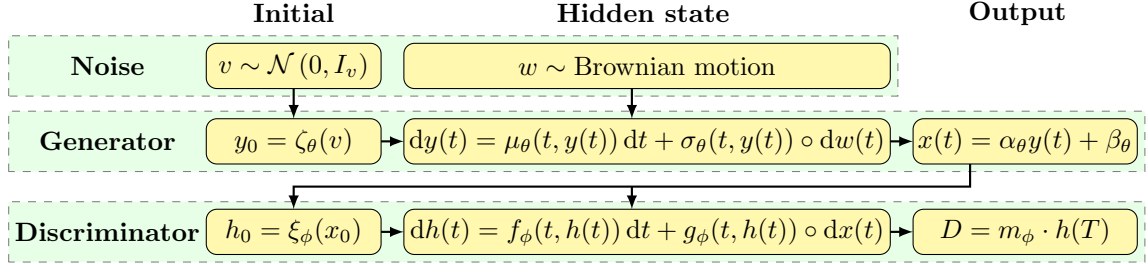
\begin{figure}\centering
\begin{tikzpicture}\footnotesize
    \draw[gray, dashed, fill=\backfillcolour] (-2.6, -0.1) rectangle ++(14.95, 0.8);
    \draw[gray, dashed, fill=\backfillcolour] (-2.6, 1.1) rectangle ++(14.95, 0.8);
    \draw[gray, dashed, fill=\backfillcolour] (-2.6, 2.1) rectangle ++(11.7, 0.8);

    \draw[rounded corners, fill=\fillcolour] (0, 0) rectangle node {$h_0 = \xi_\phi(x_0)$} ++(2.3, 0.6);
    
    \draw[rounded corners, fill=\fillcolour] (0, 1.2) rectangle node {$y_0 = \zeta_\theta(v)$} ++(2.3, 0.6);
    \draw[rounded corners, fill=\fillcolour] (0, 2.2) rectangle node {$v \sim \normal{0}{I_v}$} ++(2.3, 0.6);
    
    \draw[rounded corners, fill=\fillcolour] (2.6, 2.2) rectangle node {$w \sim\null$Brownian motion} ++(6.4, 0.6);
    \draw[rounded corners, fill=\fillcolour] (2.6, 1.2) rectangle node {$\dd y(t) = \mu_\theta(t, y(t)) \,\dd t + \sigma_\theta(t, y(t)) \circ \dd w(t)$} ++(6.4, 0.6);
    \draw[rounded corners, fill=\fillcolour] (2.6, 0) rectangle node {$\dd h(t) = f_\phi(t, h(t)) \,\dd t + g_\phi(t, h(t)) \circ \dd x(t)$} ++(6.4, 0.6);
    
    \draw[rounded corners, fill=\fillcolour] (9.3, 0) rectangle node {$D = m_\phi \cdot h(T)$} ++(2.95, 0.6);
    
    \draw[rounded corners, fill=\fillcolour] (9.3, 1.2) rectangle node {$x(t) = \alpha_\theta y(t) + \beta_\theta$} ++(2.95, 0.6);
    
    \draw[->, thick] (2.3, 1.5) -- (2.6, 1.5);
    \draw[->, thick] (2.3, 0.3) -- (2.6, 0.3);
    
    \draw[->, thick] (9, 1.5) -- (9.3, 1.5);
    \draw[->, thick] (9, 0.3) -- (9.3, 0.3);
    
    \draw[->, thick] (10.05, 1.2) -- (10.05, 0.9) -- (1.15, 0.9) -- (1.15, 0.6);

    \draw[->, thick] (5.6, 2.2) -- (5.6, 1.8);
    
    \draw[->, thick] (5.6, 0.9) -- (5.6, 0.6);
        
    \draw[->, thick] (1.15, 2.2) -- (1.15, 1.8);
        
    \draw (-1.3, 2.5) node[label=center:\textbf{Noise}] {};
    \draw (-1.3, 1.5) node[label=center:\textbf{Generator}] {};
    \draw (-1.3, 0.3) node[label=center:\textbf{Discriminator}] {};
    
    \draw (1.15, 3.2) node[label=center:\textbf{Initial}] {};
    \draw (5.75, 3.2) node[label=center:\textbf{Hidden state}] {};
    \draw (10.675, 3.2) node[label=center:\textbf{Output}] {};
\end{tikzpicture}
\caption{Summary of equations for an SDE-GAN.}\label{fig:sde-gan-equation-summary}
\end{figure}

\paragraph{Training loss}
The training loss is the usual one for Wasserstein GANs \cite{gan, wgan}, namely optimisation with respect to
\begin{equation*}
	\min_\theta \max_\phi \left( \expect_{x}\left[F_\phi(x)\right] - \expect_{x_\text{true}}\left[F_\phi(x_\text{true})\right]\right).
\end{equation*}
Training is performed via stochastic gradient descent techniques as usual.

This generalises the classical approach to calibration seen in equation \eqref{eq:sde:calibration}. Instead of optimising over some fixed collection of payoff functions $\{F_i\}_{i=1}^N$, we optimise over some infinite collection of discriminators $\{F_\phi\}_\phi$.

\begin{remark}
	In Chapter \ref{chapter:neural-cde}, we emphasised the need to include time as a channel in the control of a neural CDE. This corresponds to the inclusion of a `drift' term in equation \eqref{eq:sde:discriminator}. Equivalently we could replace $x$ with $t \mapsto (t, x(t))$ and use the same notation as Chapter \ref{chapter:neural-cde}.
\end{remark}

\subsubsection{Lipschitz regularisation}\index{Lipschitz regularisation}
Wasserstein GANs need a Lipschitz discriminator. A variety of methods have been proposed in the GAN literature, such as weight clipping \cite{wgan}, gradient penalty \cite{improved-wgan}, or spectral normalisation \cite{miyato2018spectral}. The recurrent nature of the SDE setting means that a little care is needed to employ these successfully -- see Section \ref{section:sde:lipschitz}.

\subsubsection{Discretised observations}\label{section:sde:sparse}\index{Interpolation}
Observations of $x_\text{true}$ are typically a discrete time series, rather than a true continuous-time path. This is not a serious hurdle. Simply evaluate \eqref{eq:sde:discriminator} on an interpolation $x_\text{true}$ of the observed data. The effect of this is as follows.


\paragraph{Dense data regime}
Suppose we observe samples from $x_\text{true}$ `densely' -- that is, with little gap between successive values in time. (Of approximately no more than the step size of the numerical solver.) Then interpolation produces a distribution in path space; the one desired to be modelled. Simple linear interpolation will be sufficient, but due to the dense sampling of the data this is a choice that is largely unimportant.

Technically speaking, as (linear) interpolation will produce a path of bounded variation, then \eqref{eq:sde:discriminator} will be defined as a Riemann--Stieltjes integral.

\paragraph{Sparse data regime}
Now suppose data is not observed densely, and may even have substantial time gaps between observations. In this case, we fall back to the neural CDE approach: sample the generated paths at some collection of time points, and interpolate both the generated sample \textit{and} the true data. (Before passing them to the discriminator defined as a Riemann--Stieltjes integral in both cases.)

This is the familiar setting for applying neural CDE to time series, as set up in Chapter \ref{chapter:neural-cde}. The interpolation scheme has simply become part of the discriminator, and no modelling or discriminatory power is lost.

\subsubsection{Single SDE solve}\index{Optimise-then-discretise!SDEs}
If working in the dense data regime, then \eqref{eq:sde:nsde} and \eqref{eq:sde:discriminator} may be concatenated together into a single SDE solve. This is of relevance if training using optimise-then-discretise, or with a reversible solver. Both of these are topics we will discuss in Chapter \ref{chapter:numerical}; the reader unfamiliar with these concepts should feel free to skip this heading for now.

The state is the combined $[y, h]$, the initial condition is the combined
\begin{equation*}
	[\zeta_\theta(v),\, \xi_\phi(\alpha_\theta \zeta_\theta(v) + \beta_\theta)],
\end{equation*}
the drift is the combined
\begin{equation*}
	[\mu_\theta(t, y(t)),\, f_\phi(t, h(t)) + g_\phi(t, h(t)) \alpha_\theta \mu_\theta(t, y(t))],
\end{equation*}
and the diffusion is the combined
\begin{equation*}
	[\sigma_\theta(t, y(t)),\, g_\phi(t, h(t)) \alpha_\theta \sigma_\theta(t, y(t))].
\end{equation*}
Then $h(T)$ is extracted from the final hidden state, and $m_\phi$ applied, to produce the discriminator's score for that sample.

Training in this way improves memory efficiency, as the SDE solution $y \colon [0, T] \to \reals^{d_y}$ and the output $x \colon [0, T] \to \reals^{d_x}$ are not recorded during training. The asymptotics improve from $\bigO{H + T}$, as in Section \ref{section:cde:training}, to just $\bigO{H}$, where $H$ is the memory cost of evaluating and backpropagating the vector fields once.

\subsection{Latent SDEs}\label{section:sde:latent-sde}\index{Latent SDEs}
We will now consider training not with respect to the Wasserstein distance, but with respect to the KL divergence. This approach is due to \cite{scalable-sde}.

Let
\begin{align}
	\xi_\phi &\colon \reals^{d_x} \to \reals^{d_y} \times (0, \infty)^{d_y},\nonumber\\
	\nu_\phi &\colon [0, T] \times \reals^{d_y} \times \{[0, T] \to \reals^{d_x}\} \to \reals^{d_y},\label{eq:latent-xi}
\end{align}
be Lipschitz neural networks parameterised by $\phi$. The notation $\{[0, T] \to \reals^{d_x}\}$ denotes the space of all functions $[0, T] \to \reals^{d_x}$.

\begin{remark}
We do not discuss the regularity of the functions in $\{[0, T] \to \reals^{d_x}\}$, as this input to $\nu_\theta$ will actually be a sample of $x_\text{true}$, and in practice we will have discrete observations.

$\nu_\phi$ is commonly parameterised as $\nu_\phi(t, y, x) = \nu_{\phi, 1}(t, y, \nu_{\phi, 2}(\restr{x}{[t, T]}))$, where $\nu_{\phi, 1}$ is an MLP and $\nu_{\phi, 2}$ is either a reverse-time RNN/NCDE or the evaluation function $\nu_{\phi, 2}(\restr{x}{[t, T]}) = x(t)$.
\end{remark}

Let $(m, s) = \xi_\phi(x_{\text{true}}(0))$, let $\widehat{v} \sim \normal{m}{\diag(s)^2}$, and let
\begin{equation*}
	\widehat{y}(0) = \zeta_\theta(\widehat{v}),\qquad\dd \widehat{y}(t) = \nu_\phi(t, \widehat{y}(t), x_\text{true}) \,\dd t + \sigma_\theta(t, \widehat{y}(t)) \circ \dd w(t),\qquad \widehat{x}(t) = \alpha_\theta \widehat{y}(t) + \beta_\theta.
\end{equation*}
Note that $w$ is the same Brownian motion as used in \eqref{eq:sde:nsde}. Similarly $\alpha_\theta \in \reals^{d_x \times d_y}$, $\beta_\theta \in \reals^{d_x}$ and $\sigma_\theta$ are the same objects defined in \eqref{eq:sde:nsde-setup}.

In doing so, we have constructed another SDE using the same diffusion as the main generative model, but with a different initial condition and drift. There is a standard formula for the KL divergence between two SDEs with the same diffusion, which in this case is given by
\begin{equation}\label{eq:sde:_kl-latent-sde}
	\kl{\widehat{y}}{y} = \expect_{w} \int_0^T \frac{1}{2} \norm{(\sigma_\theta(t, \widehat{y}(t)))^{-1}(\mu_\theta(t, \widehat{y}(t)) - \nu_\phi(t, \widehat{y}(t), x_{\text{true}}))}^2_2 \,\dd t,
\end{equation}
where $(\sigma_\theta(t, \widehat{y}(t)))^{-1}$ is the Moore--Penrose pseudoinverse of $\sigma_\theta(t, \widehat{y}(t))$. Note that although $y$ does not appear explicitly on the right hand side, $y$ defines (and is defined by) the $\mu_\theta$ and $\sigma_\theta$ which do appear.

\begin{remark}
	Equation \eqref{eq:sde:_kl-latent-sde} may be identified as an integral over the KL divergence between two Gaussians.
\end{remark}

This opens up a possible training procedure. This `auxiliary' SDE, which depends on samples of the observed data $x_\text{true}$, may be used to autoencode the data. Once the data is represented as an SDE, we may remove the dependence on $x_\text{true}$ by minimising a KL divergence between our original generative model and the auxiliary model.

Explicitly, this corresponds to training according to
\begin{equation*}
	\min_{\theta, \phi}\expect_{x_\text{true}}\left[(\widehat{x}(0) - x_{\text{true}}(0))^2 + \kl{\widehat{v}}{v} + \expect_{w} \int_0^T (\widehat{x}(t) - x_{\text{true}}(t))^2 \,\dd t + \kl{\widehat{y}}{y}\right].
\end{equation*}

\begin{remark}
	Note that this training procedure only involves solving the auxiliary SDE, never the original SDE. The main generative model is trained without ever being evaluated.
\end{remark}

\paragraph{As a variational autoencoder}
\cite{scalable-sde} interpret this procedure as a variational autoencoder, with a learnt prior, whose latent space is an entire stochastic process. (And indeed the above formula may be derived as an evidence lower-bound.) For this reason \cite{scalable-sde} refer to the auxiliary SDE as a posterior SDE.

Interpreted in this way, the first two terms are a VAE for generating $x(0)$, with latent $v$. Meanwhile the third term and fourth term are a VAE for generating $x$, by autoencoding $x_\text{true}$ to $\widehat{x}$, and then fitting $y$ to $\widehat{y}$.

\paragraph{Single SDE solve}
Equation \eqref{eq:sde:_kl-latent-sde} is an integral, and so may be estimated by concatenating it alongside the SDE solve.

\paragraph{Alternate probability densities}
The first and third terms of \eqref{eq:sde:_kl-latent-sde} are the $L^2$ loss, which corresponds to maximising the log-likelihood of $x_{\text{true}}(t)$ with respect to a fixed-variance Gaussian whose mean is $\widehat{x}(t)$:
\begin{equation*}
	\normal{\widehat{x}(t)}{\frac{1}{\sqrt{2}}\eye{d_y}}.
\end{equation*}

However other probability densities are also admissible. As such the above presentation is chosen for simplicity, and compatibility with the presentation of the generative model in Section \ref{section:sde:construction}. The affine map corresponding to $\alpha_\theta, \beta_\theta$ is being used to produce the mean of a fixed-variance Gaussian, but it may be replaced by any other procedure for producing the parameters of some probability distribution, and the log-likelihood optimised as normal.

\subsection{Comparisons and combinations}
The difference between SDE-GANs and latent SDEs is essentially the standard GAN/VAE split. SDE-GANs are more finicky to train, but exhibit substantially higher modelling capacity. Conversely, latent SDEs are easy to train, but often produce worse final models; in particular it is a common feature of latent SDEs that their diffusion will be too small.

It is possible to combine both latent SDEs and SDE-GANs together. (And indeed GAN/VAE hybrids have been proposed in the main deep learning literature too \cite{vaegan1, vaegan2, vaegan3}.) This is a way to offset the weakness of each approach with the strengths of the other. An example of this is given in Section \ref{section:sde:example}, applied to modelling a Lorenz system.

\section{Choice of parameterisation}
As usual with deep learning, the theoretical construction is only half of the work needed to produce a workable model, and the `engineering details' -- of finding good hyperparameters, optimisers, and so on -- still remain.

At time of writing, finding good choices is still largely an open problem for neural SDEs. Much inspiration can likely be drawn from the mainstream generative modelling literature, which has spent the past few years investigating this topic in depth: see for example negative momentum \cite{pmlr-v89-gidel19a}, complex momentum \cite{lorraine2021complex}, stochastic weight averaging (Ces{\`a}ro means) \cite{swa2, swa}, progressive growing \cite{karras2018progressive}, Lipschitz regularisation \cite{improved-wgan, miyato2018spectral}, architectural choices \cite{gans-equal, stylegan2} and so on.


\subsection{Choice of optimiser}
\subsubsection{SDE-GANs}
SDE-GANs can be relatively unstable to train.

\paragraph{Adadelta} Empirically, Adadelta \cite{adadelta}, or the similar RMSprop, seems to outperform either SGD or Adam when training SDE-GANs. In part this is because Adadelta lacks momentum; a lack of momentum is beneficial as the optimisation criterion for a GAN is a moving target.

Adam with $\beta_1 = 0$, where $\beta_1$ is its momentum hyperparameter, also seems to be outperformed by Adadelta \cite{adadelta-better-than-adam}.

\paragraph{Learning rate} The initial networks $\zeta_\theta$ and $\xi_\phi$ often work best with a larger learning rate than is used for the rest of the model. (For example a factor of 10 would be typical.) This helps to offset the fact that the initial distribution (of $x(0)$) often gets relatively weak supervision compared to the time-varying component (of $t \mapsto x(t)$).

\paragraph{Stochastic weight averaging} Using the Ces{\`a}ro mean of both the generator and discriminator weights, averaged over training, can improve performance in the final model \cite{swa2, swa}. This averages out the oscillatory training behaviour for the min-max objective used in GAN training.


\subsubsection{Latent SDEs}
Latent SDEs are relatively easy to train. Given their VAE-like structure, standard optimisers like Adam \cite{kingma2015} work without difficulty.

Once again it is still usually worth increasing the learning rate for $\zeta_\theta$ and $\xi_\phi$.

\subsection{Choice of architecture}
\subsubsection{Generator}
Recall that $\mu_\theta$ and $\sigma_\theta$ were the drift and diffusion of the SDE, defined in \eqref{eq:sde:nsde-setup}. 

$\mu_\theta$ and $\sigma_\theta$ are typically taken to be MLPs. Numerical SDE solvers will usually demand that the vector fields be sufficiently smooth (for example, bounded with continuous bounded first and second derivatives), so the activation function is often taken to be smooth, like softplus or SiLU.

\paragraph{Final nonlinearities}
It is common to add a final tanh nonlinearity to $\mu_\theta$ and $\sigma_\theta$. This is for the same reason as neural CDEs: to prevent an unconstrained rate of change in the hidden state and the model potentially exploding (especially at initialisation). If this constrains the rate of change too strongly, then this may be managed by parameterising $\mu_\theta$ and $\sigma_\theta$ as
\begin{equation*}
(t, y) \mapsto \gamma \tanh(\mathrm{MLP}_\theta(t, y)),
\end{equation*}
where $\gamma \in \reals$ is a learnt scalar (part of $\theta$).

\paragraph{Initialisation}
As with ODEs (Section \ref{section:ode:initialisation}), training dynamics may be improved by initialising $\mu_\theta$ and $\sigma_\theta$ close to zero.

%

\paragraph{Choice of driving noise}\index{Jump!Process}
The construction of this chapter has taken the driving noise $w$ to be a Brownian motion. This choice is not necessary; for example fractional Brownian motion or L{\'e}vy processes could also be used, together with or instead of the Brownian motion $w$.

A choice of particular interest are counting processes $t \mapsto N(t)$ (for example the cumulative sum of a Poisson process) so that the resulting SDE is a jump process
\begin{equation*}
	\dd y(t) = \mu_\theta(t, y(t)) \,\dd t + \sigma_\theta(t, y(t)) \circ \dd w(t) + \lambda(t, y(t-)) \,\dd N(t),
\end{equation*}
where the notation `$t-$' is used to emphasise that the vector field depends upon the value immediately prior to the jump.

The optimisation criteria can get slightly more involved in these cases: whilst the SDE-GAN approach translates over without any changes, at time of writing the latent SDE approach has not yet been explored. See also \cite{njsde} who develop a direct likelihood-based approach to optimise diffusionless drift/jump processes of the form
\begin{equation*}
	\dd y(t) = \mu_\theta(t, y(t)) \,\dd t + \lambda(t, y(t-)) \,\dd N(t),
\end{equation*}

\paragraph{Diffusions for latent SDEs}\label{section:sde:diffusion-param}
When training a latent SDE, then the KL divergence of equation \eqref{eq:sde:_kl-latent-sde}, used in the latent SDE, multiplies by the (pseudo)inverse of $\sigma_\theta$. This is expensive to compute for general matrices.

One effective simplification is to take $d_w = d_y$ and parameterise the diffusion $\sigma_\theta$ as a diagonal matrix. This is cheap to compute the inverse of: take the reciprocal of each diagonal element.

For numerical stability it is additionally often desirable to then bound these diagonal elements away from zero: use $z \mapsto \sigmoid(z) + 10^{-4}$ as a final nonlinearity for $\sigma_\theta$, or alternatively clamp any values in the range $[-10^{-6}, 10^{-6}]$ to the edges of that range.

\paragraph{Approximation properties}\index{Universal approximation!SDEs}
Provided $\mu_\theta$ and $\sigma_\theta$ are drawn from suitable (universal approximating) classes of functions, then it is clear that \eqref{eq:sde:nsde} is more than capable of approximating any Markov SDE, by the universal approximation theorem for neural networks \cite{pinkus1999, kidger2020deep} and standard approximation results for SDEs.

What is less clear is its ability to model non-Markov SDEs. Certainly this is possible to some extent, due to the explicit use of hidden state. (Indeed this is the reason hidden state is introduced in the first place.) At time of writing a formal result has not been derived.

\subsubsection{Discriminator}\label{section:sde:discriminator-parameterisation}
When training an SDE-GAN, then additional networks $\xi_\phi$, $f_\phi$, $g_\phi$ are introduced. These should be parameterised in accordance with neural CDEs (Section \ref{section:cde:parameterisation}).

The initial distribution, learnt by $\zeta_\theta$, can often be improved by providing it additional supervision during training. Redefine $D = m_\phi \cdot h(0) + m_\phi \cdot h(T)$ or $D = \kappa_\phi(x(0)) + m_\phi \cdot h(T)$ instead of just $D = m_\phi \cdot h(T)$ in equation \eqref{eq:sde:discriminator}, where $\kappa_\phi$ is some neural network.

As with any Wasserstein GAN, the discriminator should be Lipschitz. This is the focus of our next section.

\subsection{Lipschitz regularisation}\label{section:sde:lipschitz}\index{Lipschitz regularisation}
This section is specific to SDE-GANs. SDE-GANs, as with any Wasserstein GAN, need a Lipschitz discriminator.

A variety of methods for enforcing Lipschitzness have been proposed in the general GAN literature, such as weight clipping \cite{wgan}, gradient penalty \cite{improved-wgan}, or spectral normalisation \cite{miyato2018spectral}. However a little care must be taken when applying these to the discriminator of an SDE-GAN.

Much of the following discussion originated in \cite{kidger2021sde2}.

\subsubsection{Exponential Lipschitz constant}\label{section:sde:hard-constraint}
Given vector fields with Lipschitz constant $\lambda$, then the recurrent structure of the discriminator means that the Lipschitz constant of the overall discriminator will be $\bigO{\lambda^T}$. This is a key consideration in performing Lipschtiz regularisation, and unfortunately, the aforementioned techniques cannot simply be applied `off the shelf'.

\paragraph{Lipschitz constant one}
The first option will be to somehow ensure that the vector fields $f_\phi$ and $g_\phi$ of the discriminator are not only Lipschitz, but have Lipschitz constant at most one. Ensuring $\lambda \approx 1$ with $\lambda \leq 1$ will enforce that the overall discriminator is Lipschitz, with a Lipschitz constant of approximately one, as well.

This will be the approach we take in Section \ref{section:sde:careful-clipping}.

\paragraph{Hard constraint}
The exponential size of $\bigO{\lambda^T}$ means that $\lambda$ only slightly greater than one is still insufficient for stable training. This is why we specify `$\lambda \approx 1$ with $\lambda \leq 1$' and not merely `$\lambda \approx 1$'. Moreover, it rules out enforcing $\lambda \approx 1$ via soft constraints like spectral normalisation.

\paragraph{Whole-discriminator regularisation}
The second option is to regularise the Lipschitz constant of whole discriminator, without regard for its recurrent structure. This will be the approach we take in Section \ref{section:sde:gp}.

\subsubsection{Careful clipping}\label{section:sde:careful-clipping}\index{Careful clipping}
Let us (within this subsection) now assume that our discriminator vector fields $f_\phi$, $g_\phi$ are MLPs. This is also a common choice made in practice.

%

\paragraph{Careful clipping}
Consider each linear operation from $\reals^a \to \reals^b$ as a matrix in $A \in \reals^{a \times b}$. After each gradient update, clip its entries to the region $[-1/b, 1/b]$. Given $z \in \reals^a$ then this enforces $\norm{Az}_\infty \leq \norm{z}_\infty$.

\paragraph{LipSwish activation function}\index{LipSwish}
Next we must pick an activation function with Lipschitz constant at most one. It should additionally be at least twice continuously differentiable to ensure convergence of a numerical SDE solver. In particular this rules out the ReLU.

There remain several admissible choices. We tend to use the LipSwish activation function introduced by \cite{lipswish}, defined as $\rho(z) = 0.909\,z\sigma(z)$, where $\sigma$ denotes the sigmoid function. This has Lipschitz constant one (due to the carefully-chosen $0.909$ scaling factor), and is smooth. Moreover the SiLU activation function from which it is derived has been reported as an empirically strong choice \cite{gelu, silu, swish}.

\paragraph{Overall}
The overall vector fields $f_\phi$, $g_\phi$ of the discriminator consist of linear operations (which are constrained by clipping), adding biases (an operation with Lipschitz constant one), and activation functions (taken to be LipSwish). Thus the Lipschitz constant of the overall vector field is at most one, as desired.

\subsubsection{Gradient penalty}\label{section:sde:gp}\index{Gradient penalty}
Another option is to directly regularise the Lipschitz constant of the entire discriminator, via gradient penalty. Add
\begin{equation}\label{eq:sde:_gradient-penalty}
	\expect_{\widehat{x}} \left[ \left( \norm{\frac{\partial F_\phi}{\partial x}(\widehat{x})} - 1 \right)^2 \right]
\end{equation}
as a regularisation term to the training loss, where $\widehat{x}$ is sampled according to $\widehat{x} = \alpha x + (1 - \alpha) x_\text{true}$ with $x \sim \prob_x$ and $x_\text{true} \sim \prob_{x_\text{true}}$ and $\alpha \sim \uniform{0}{1}$.

This approach works, which is more than can be said for other na{\"i}ve approaches. However, compared to the careful clipping of Section \ref{section:sde:careful-clipping}, this approach mostly comes with disadvantages.

\paragraph{Disadvantages}
Because \eqref{eq:sde:_gradient-penalty} involves calculating a gradient, then optimising it involves calculating a second derivative -- a `double backward'.

This is of relevance if training using optimise-then-discretise, which is a topic we will discuss in Chapter \ref{chapter:numerical}. (The reader unfamiliar with this concept should feel free to skip this heading for now.)\index{Optimise-then-discretise!SDEs}

If training proceeds using optimise-then-discretise, then as a single backward constructs an `adjoint SDE', a double backward constructs an `adjoint-of-adjoint SDE'. This starts to imply substantial errors in the numerical discretisation, and this can be sufficient to degrade or destroy training.

Another negative is the additional computational cost implied by computing, and autodifferentiating, \eqref{eq:sde:_gradient-penalty}. This can easily result in a training procedure that takes about 50\% longer than the careful clipping approach.


\begin{remark}
	We sidestep questions of how the derivative in \eqref{eq:sde:_gradient-penalty} is defined -- given that $\widehat{x}$ is path valued -- by defining it with respect to the numerically discretised solution of $\widehat{x}$. In practice gradient penalty is not the preferred option, due to the disadvantages already discussed, so this is not an issue we will seek to tackle formally.
\end{remark}

\section{Examples}\index{Examples!Neural SDEs}\label{section:sde:example}
\paragraph{Brownian motion}\index{Brownian!Motion}
As a simplest-possible first example, consider a dataset of samples of (univariate) Brownian motion, with initial condition $\uniform{-1}{1}$. Each element of the dataset is a time series of observations along a single Brownian sample path. We train a small SDE-GAN to match the distribution of the initial condition and the distribution of the time-evolving samples; see Figure \ref{fig:sde:brownian}.

This example may seem almost trivially simple, and yet it highlights a class of time series that would be almost impossible to learn with a latent ODE (Section \ref{section:ode:latent-ode}). A Brownian motion represents pure diffusion, whilst a latent ODE is pure drift.

\begin{figure}
	\begin{minipage}{0.49\linewidth}
		\includegraphics[width=\linewidth]{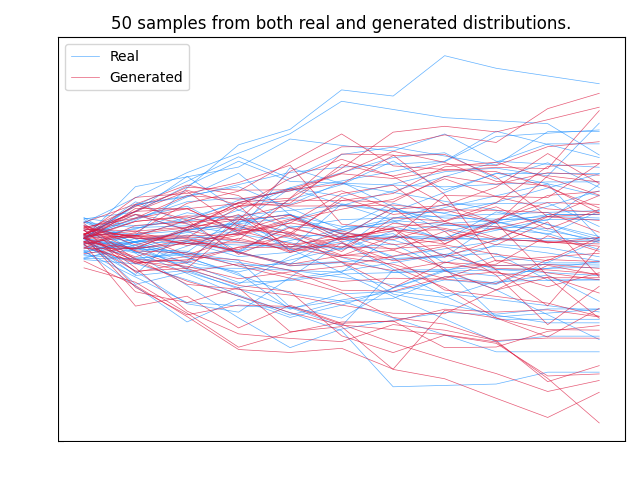}
		\caption{Coarsely-spaced $(t, y)$ samples of a Brownian motion, and an SDE-GAN trained to match its distribution.}\label{fig:sde:brownian}
	\end{minipage}\hfill\begin{minipage}{0.49\linewidth}
		\includegraphics[width=\linewidth]{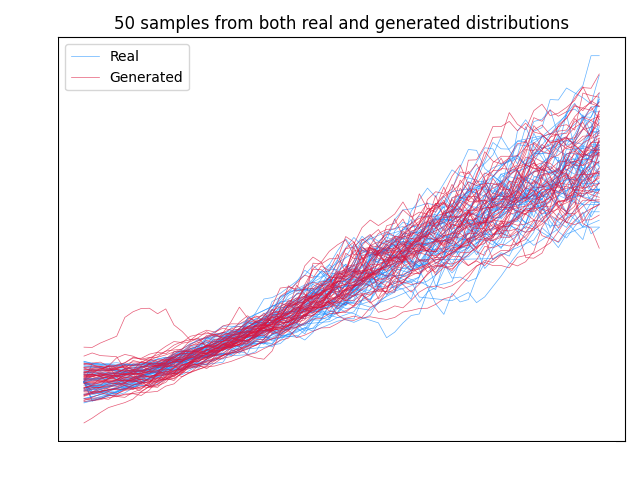}
		\caption{Finely-spaced $(t, y)$ samples of a time-dependent Ornstein--Uhlenbeck process, and an SDE-GAN trained to match its distribution.}\label{fig:sde:ou}
	\end{minipage}
\end{figure}

\paragraph{Time-dependent Ornstein--Uhlenbeck process}\index{Orenstein--Uhlenbeck process}
Next we consider training an SDE-GAN to recover the distribution of
\begin{equation*}
	y(0) \sim \uniform{-1}{1},\qquad \dd y(t) = (a t - b y(t)) \,\dd t + ct \,\dd w(t),\qquad\text{for $t \in [0, 63]$}.
\end{equation*}
We take in particular $a=0.02$, $b=0.1$, $c=0.013$.

This example introduces explicit time dependency; in particular a time-dependent diffusion.

That this has both nontrivial drift and diffusion makes it an example of a process that is easy to learn via SDEs, but would be difficult to learn with models such as a latent ODE (which is pure drift; Section \ref{section:ode:latent-ode}) or a CTFP (which is almost-pure diffusion; \cite{ctfp}).

See Figure \ref{fig:sde:ou}.

\paragraph{Damped harmonic oscillator}
Next we consider a dataset of samples from a two-dimensional damped harmonic oscillator.
\begin{equation*}
	y_1(0), y_2(0) \sim \uniform{-1}{1},\qquad \dd \begin{bmatrix} y_1(t) \\ y_2(t) \end{bmatrix} = \begin{bmatrix} -0.01 & 0.13 \\ -0.1 & -0.01 \end{bmatrix} \begin{bmatrix} y_1(t) \\ y_2(t) \end{bmatrix} \, \dd t,
\end{equation*}
for $t \in [0, 100]$.

This example is multidimensional, pure-drift, and solved over a long time interval.

In this case we train a latent SDE to recover the distribution. See Figure \ref{fig:sde:harmonic}, which shows a single sample, in the $(y_1, y_2)$-plane, from both the true and generated dataset. (So that time evolves as the trajectory spirals inwards.)

It is by coincidence that the generated sample begins so close to, and partway along, the true sample. (They are not necessarily meant to overlap.) The generated sample does an excellent job at matching the drift, even extrapolating past the end of the true sample. The only issue is that the diffusion is still too high -- indeed the true diffusion is zero -- demonstrating that some additional training may still be required. Nonetheless this demonstrates how neural SDEs subsume neural ODEs as a special case, practically as well as theoretically.
 
\paragraph{Lorenz attractor}\index{Lorenz attractor}
We consider a dataset of samples from the Lorenz attractor
\begin{align*}
	&y \sim \normal{0}{\eye{3}},\\
	\dd y_1(t) &= a_1 (y_2(t) - y_1(t)) \,\dd t + b_1 y_1(t) \dd w(t),\\
	\dd y_2(t) &= (a_2 y_1(t) - y_1(t)y_3(t) )\,\dd t + b_2 y_2(t) \dd w(t),\\
	\dd y_3(t) &= (y_1(t) y_2(t) - a_3 y_3(t)) \,\dd t + b_3 y_3(t) \dd w(t),
\end{align*}
for $t \in [0, 2]$. We take specifically $a_1 = 10$, $a_2 = 28$, $a_3 = \frac{8}{3}$, $b_1 = 0.1$, $b_2 = 0.28$, $b_3 = 0.3$.

This example is multidimensional, chaotic, and has state-dependent diffusion.\footnote{In passing, note that this is an It{\^o} SDE. As discussed in Section \ref{section:sde:introduction}, it is no issue that we are about to learn it with a Stratonovich neural SDE.}

We train a combined latent SDE / SDE-GAN on this dataset. They are combined simply by interchanging separate training steps: one as a latent SDE, followed by one as an SDE-GAN. See Figure \ref{fig:sde:lorenz}. The model has correctly learnt the distribution of this chaotic multidimensional time series.

\begin{figure}
\begin{minipage}{0.49\linewidth}
	\includegraphics[width=\linewidth]{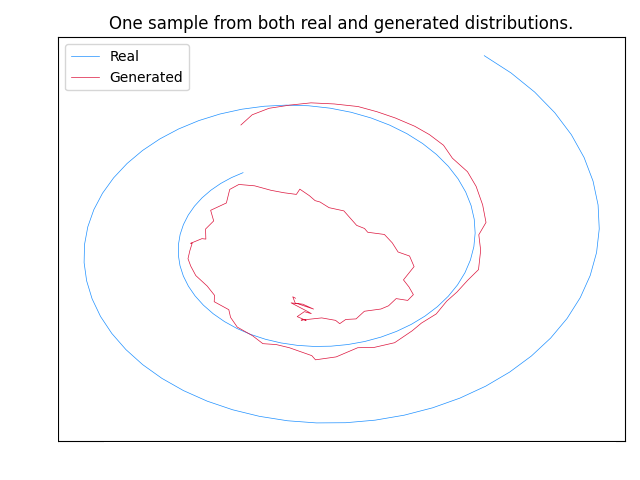}
	\caption{A single $(y_1, y_2)$-plane sample of a damped harmonic oscillator, and a latent SDE trained to match its distribution.}\label{fig:sde:harmonic}
\end{minipage}\hfill\begin{minipage}{0.49\linewidth}
	\includegraphics[width=\linewidth]{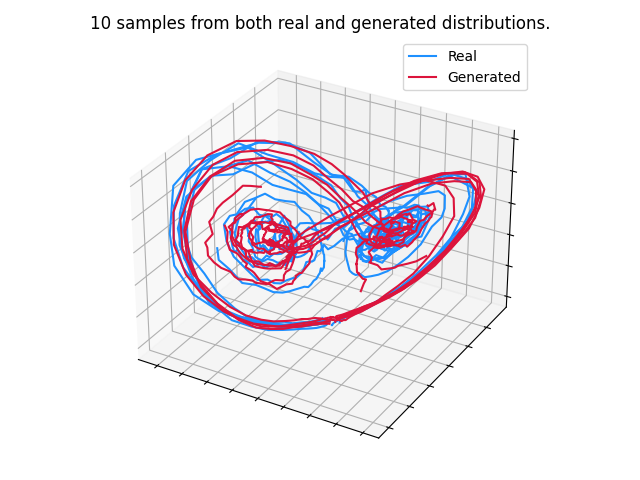}
	\caption{Evolving $(y_1, y_2, y_3)$ samples from a Lorenz attractor, and a `latent SDE-GAN' trained to match its distribution}\label{fig:sde:lorenz}
\end{minipage}
\end{figure}

\paragraph{Further details}
See Appendix \ref{appendix:experimental:sde} for precise details on the experiments considered here. The code is available as an example in Diffrax \cite{diffrax}.

\paragraph{Irregular sampling}\index{Irregular sampling}
Both the Brownian motion and the Ornstein--Uhlenbeck example were irregularly sampled with missing data. The process was observed at each integer (in the time domain) with only $70\%$ probability, and unobserved otherwise.

The continuous-time approach discussed in this chapter means that this irregularity requires no special treatment. Moreover the output of each model evolves in continuous time and may be observed at any location.

\paragraph{Other examples}
Other (real-world) time series problems may be considered.

\cite{scalable-sde} give an example training latent SDEs to perform short-term forecasting on a 50-dimensional motion capture dataset.

\cite{kidger2021sde1} consider a dataset of 14.6 million observations of Google/Alphabet stock prices, and train an SDE-GAN to replicate the evolution of the midpoint and spread as it evolves over a minute.

\cite{kidger2021sde2} train both latent SDEs and SDE-GANs, and give an example modelling the air quality over Beijing.


\section{Comments}
Several authors have independently introduced notions of neural SDEs.

\cite{scalable-sde, kidger2021sde1} were the main works to derive the material presented here, whilst our presentation is derived from the follow-up \cite{kidger2021sde2}.

We have focused on using the Wasserstein distance or KL divergence to match model against data. In principle the classical calibration approach, using fixed statistics, may be employed in conjunction with neural vector fields, and this is now essentially the formulation of an MMD (Appendix \ref{appendix:deep:mmd}). Some care should be taken as to the choice of feature map. For example some authors have used only the mean and variance of the marginal distributions at each time $t$, and this fails to distinguish $t \mapsto w(t)$ from $t \mapsto \sqrt{t}w(1)$. A good choice of feature map is the signature transform \cite{signatory}; for example this is done in \cite[Section 4]{kidger2021sde1} and \cite[Section 4.1]{bonnier2019deep}. (There also exists a corresponding signature kernel \cite{sigkernel1, sigkernel2}.)

\cite{nsde-mmd, finance-nsde} consider variations on the formulation given here, but optimise a distance only between finite-dimensional marginal distributions, rather than optimising the continuous-time model. \cite{josef-sde} consider another variation on this formulation, by adding known structure to the discriminator, corresponding to prespecified payoff functions of interest. \cite{sam-sde} consider specifically Markov neural SDEs and optimise via maximum likelihood (more-or-less equivalent to optimising the KL criterion considered here) as part of a larger framework for market models; indeed many of the above references target financial applications.

Meanwhile \cite{nsde-basic, nsde-generative} obtain neural SDEs as a continuous limit of deep latent Gaussian models, and largely focus on the theoretical construction.

The connections between score-based generative modelling\index{Score-based generative modelling} and neural SDEs as presented here has not yet been explored in detail. We recommend the first few pages of \cite{diffusion-bridge} for an introduction to score-based generative modelling. \cite{maoutsa2020interacting, zhang2021diffusion} emphasise connections to continuous normalising flows. \cite{pmlr-v139-shi21b} give an application to molecular conformation, \cite{ho2021cascaded, dhariwal2021diffusion} give large-scale applications to image generation, and \cite{sdedit} give an application to image editing. \cite{var-diffusion1, var-diffusion2, var-diffusion3} give variational/likelihood-based perspectives.

The mainstream deep learning literature frequently uses stochasticity as a regulariser. A neural SDE may likewise be treated as a regularised neural ODE or CDE \cite{nsde-normalisation, stochasticnode, roughstochasticNF}. This will also be discussed as part of our numerical treatment of differential equations in Section \ref{section:numerical:additive-noise}.

\cite{kong-sde} give one application of neural SDEs not discussed here, by using the stochasticity as part of procedure to distinguish between epistemic and aleatoric uncertainty.
\chapter{Numerical Solutions of Neural Differential Equations}\label{chapter:numerical}
\section{Backpropagation through ODES}\label{section:numerical:adjoint-ode}
Training a neural differential equation usually means backpropagating through the differential equation solve. There are actually several ways to do this.

For clarity of exposition we begin by studying ODEs only, and will return to backpropagation through CDEs and SDEs in the next section.

We shall see three main ways of differentiating through an ODE.
\begin{itemize}
	\item Discretise-then-optimise -- memory inefficient, but accurate and fast;
	\item Optimise-then-discretise -- memory efficient, but approximate and a little slow;
	\item Reversible ODE solvers -- memory efficient and accurate, but a little slow.
\end{itemize}

Generally speaking discretise-then-optimise is the preferred approach. If this is not possible, typically due to memory constraints, then reversible ODE solvers are the next best option. Finally, if this is not suitable then optimise-then-discretise methods may be used, but these are typically the least-favoured approach.

All of these choices may typically be found in major differential equation software libraries (Section \ref{section:numerical:software}), so that the choice of backpropagation is usually an easy thing to change.

\subsection{Discretise-then-optimise}\label{section:discretise-then-optimise}\label{section:numerical:discretise-then-optimise}\index{Discretise-then-optimise} The first option is simply to backpropagate through the internal operations of the differential equation solver.

A differential equation solver internally performs the usual arithmetic operations of addition, multiplication, and so on, each of which is differentiable. Given that a solve operation is a composition of differentiable operations, it is also differentiable.

This is known as `discretise-then-optimise'. The derivatives are computed with respect to the discretised version of the differential equation that the solver computed, and not with respect to the idealised continuous-time equation.

\subsubsection{Advantages}

\paragraph{Accuracy of gradients} The computed gradients will be accurate for the discrete model that is actually being used. This is in contrast to some of the techniques we shall see later, which compute only approximate gradients.

\paragraph{Speed} This is often the quickest way to backpropagate. One reason for this is that the full computation graph is known prior to performing the backpropagation, and so the underlying autodifferentiation library may better exploit parallelism.

\paragraph{Ease of implementation} The implementation of discretise-then-optimise is generally straightforward: provided the differential equation solver is written in an autodifferentiable framework (such as PyTorch or JAX), then gradients may automatically be computed in the usual way for these frameworks.

\subsubsection{Disadvantages}

\paragraph{Memory inefficiency} This approach is memory-inefficient, as every internal operation of the solver must be recorded. If the memory cost of recording the operations of a single differential equation step is $H$, and recalling that $T$ is the time horizon, then this approach consumes $\bigO{HT}$ memory.

This is in contrast to the techniques we shall see later, which reduce this to only $\bigO{H}$.

\begin{remark}
	In some sense it's a little unfair to state that discretise-then-optimise is memory-inefficient. It's simply performing backpropagation as normal, as with any other neural network model, and we do not usually refer to those as memory-inefficient. It is simply that the other options we see later can reduce memory costs to essentially negligible amounts.
\end{remark}

\paragraph{Difficulty of implementation} In contrast to the `ease of implementation' just discussed -- if the differential equation solver is provided without having been written in an autodifferentiable framework, then this approach is essentially impossible to implement.

\subsubsection{Checkpointing}\label{section:numerical:dto-checkpointing}\index{Checkpointing}
It is possible to finesse the problem of memory inefficiency through checkpointing. That is, record the value of the forward pass at certain points during the solve, and use these to reconstruct values during the backward pass. This is a general technique in deep learning \cite{treeverse}. \cite{anode} discuss this in the specific context of neural ODEs.

\subsection{Optimise-then-discretise}\index{Optimise-then-discretise!ODEs}
We now move on to the optimise-then-discretise approach. This instead works by differentiating the idealised continuous-time model. Doing so produces a backwards-in-time differential equation, which is then solved numerically. (References include \cite{pontryagin, hager, sundials-cvodes, neural-odes} but this technique is widespread.)

\begin{restatable}{theorem}{odeadjoint}\label{theorem:ode-adjoint}
	Let $y_0 \in \reals^d$ and $\theta \in \reals^m$. Let $f_\theta \colon [0, T] \times \reals^d \to \reals^d$ be continuous in $t$, uniformly Lipschitz in $y$, and continuously differentiable in $y$. Let $y \colon [0, T] \to \reals^d$ be the unique solution to
	\begin{equation*}
		y(0) = y_0, \qquad \frac{\dd y}{\dd t}(t) = f_\theta(t, y(t)).
	\end{equation*}
	Let $L = L(y(T))$ be some (for simplicity scalar) function of the terminal value $y(T)$.
	
	Then $\nicefrac{\dd L}{\dd y(t)} = a_y(t)$ and $\nicefrac{\dd L}{\dd \theta} = a_\theta(0)$, where $a_y \colon [0, T] \to \reals^d$ and $a_\theta \colon [0, T] \to \reals^m$ solve the system of differential equations
	\begin{align}
		a_y(T) &= \frac{\dd L}{\dd y(T)},\hspace{1.6em} &\frac{\dd a_y}{\dd t}(t) = -a_y(t)^\top \frac{\partial f_\theta}{\partial y}(t, y(t)),\nonumber\\
		a_\theta(T) &= 0,\hspace{1.6em} &\frac{\dd a_\theta}{\dd t}(t) = -a_y(t)^\top \frac{\partial f_\theta}{\partial \theta}(t, y(t)).\label{eq:ode-adjoint}
	\end{align}
\end{restatable}

\begin{remark}
	The vector-matrix products in equation \eqref{eq:ode-adjoint} are, more specifically, vector-Jacobian products. These can be computed efficiently via autodifferentiation; see Appendix \ref{appendix:deep:autodiff}.
\end{remark}

\begin{remark}\label{remark:numerical:not-an-ode}
	Note that the $a_y(t)$ in the second equation (rather than $a_\theta(t)$) is not a typographical error. The vector fields are independent of $a_\theta$. This may be exploited to speed up backpropagation through neural ODEs; this is a topic we shall return to in Section \ref{section:numerical:not-an-ode}.
\end{remark}

Equations \eqref{eq:ode-adjoint} are known as \textit{the continuous adjoint equations}.

These give a way to backpropagate through an ODE solve. Consider for example Figure \ref{figure:computation-graph}. The gradient $\nicefrac{\dd L}{\dd y(T)}$ is calculated by backpropagating through the softmax and affine layer in the usual way. Then the system of equation \eqref{eq:ode-adjoint} is solved backwards in time from $t=T$ to $t=0$, to compute the parameter gradients $\nicefrac{\dd L}{\dd \theta} = a_\theta(0)$.\footnote{And had there also been any preceding operations, then backpropagation could then continue as usual from the computed $\nicefrac{\dd L}{\dd \theta} = a_\theta(0)$, $\nicefrac{\dd L}{\dd y_0} = a_y(0)$.}

Note that equation \eqref{eq:ode-adjoint} requires knowing the solution $y$ as an input. The usual approach is to find this by augmenting equations \eqref{eq:ode-adjoint} with the original neural ODE solved backwards-in-time, starting from the numerical approximation to $y(T)$ computed on the forward pass. In integral notation, this is
\begin{equation}\label{eq:backward-ode}
	y(t) = y(T) + \int_T^t f_\theta(s, y(s)) \,\dd s.
\end{equation}

This is known as the `continuous adjoint method`, or as the `optimise-then-discretise' approach. The derivatives are calculated with respect to the idealised continuous-time model, and then the adjoint equations of \eqref{eq:ode-adjoint} must themselves then be discretised.

\begin{remark}
The continuous adjoint method is also commonly referred to as simply `the adjoint method', especially in the modern neural differential equation literature.

This is an unfortunate ambiguity of terminology. Across several prominent works, `the adjoint method' has been used to refer to both the the discretise-then-optimise approach \cite{smoking-adjoints}, and the optimise-then-discretise approach \cite{neural-odes}. Moreover it has sometimes been used to refer to something else entirely: \cite[Definition 8.2]{hairer} use it to refer to the inverse map of a numerical integration step, so that for example an implicit Euler step is the `adjoint' of an explicit Euler step.\footnote{This relationship between numerical integration steps is something we will actually need later. We refer to it as \textit{analytic reversibility}, see Section \ref{section:numerical:analytic-reversibility}.}

In this text we strive to be unambiguous by always making clear which adjoint method we are referring to, and would strongly discourage simply writing `the adjoint method' without further qualification.
\end{remark}

\subsubsection{Proof: `continuous-time backpropagation'}\label{section:numerical:ode-adjoint-sketchproof}
Proving Theorem \ref{theorem:ode-adjoint} is straightforward, and it is informative to compare this to backpropagation.

Consider two points $s, t \in [0, T]$ with $s < t$, and consider solving the ODE from $s$ to $t$, and then from $t$ to the terminal time $T$. Then by the chain rule,
\begin{equation*}
	\frac{\dd L}{\dd y(s)} = \frac{\dd L}{\dd y(t)}\frac{\dd y(t)}{\dd y(s)}.
\end{equation*}
For notation's sake let $a(t)^\top = \frac{\dd L}{\dd y(t)}$. Now the left hand side is independent of $t$, so differentiate with respect to $t$ and set $t = s$:
\begin{align*}
	0 &= \eval{\frac{\dd}{\dd t} \left(a(t)^\top \frac{\dd y(t)}{\dd y(s)}\right)}{t = s}\\
	&= \eval{\frac{\dd}{\dd t} \left(a(t)^\top\right)\frac{\dd y(t)}{\dd y(s)}}{t = s} + \eval{a(t)^\top \frac{\dd}{\dd t} \left(\frac{\dd y(t)}{\dd y(s)}\right)}{t = s}\\
	&= \frac{\dd a}{\dd s}(s) + \eval{a(t)^\top \frac{\dd}{\dd y(s)} \left(\frac{\dd y(t)}{\dd t}\right)}{t = s}\\
	&= \frac{\dd a}{\dd s}(s) + \eval{a(t)^\top \frac{\dd}{\dd y(s)} \left(f_\theta(t, y(t))\right)}{t = s}\\
	&= \frac{\dd a}{\dd s}(s) + a(s)^\top \frac{\partial f_\theta}{\partial y}(s, y(s)).
\end{align*}

This is now precisely the first adjoint equation of equation \eqref{eq:ode-adjoint}. The second adjoint equation can be derived by replacing $y$ with the $[y, \theta]$ and $f_\theta$ with $[f_\cdot, 0]$ (that is to say treating the parameters $\theta$ as additional state, subject to zero vector field), and applying the same argument as before.

In this way we see that the adjoint equations are essentially `continuous time backpropagation'.\footnote{The mathematically precise reader may still be skeptical: we have not justified the change of limits, nor that the solution to the adjoint differential equation actually exists. See Appendix \ref{appendix:ode-adjoint} for these technical points.}

\subsubsection{Advantages}

\paragraph{Memory efficiency} The continuous adjoint method has one clear advantage: memory efficiency. The forward computations for $y$ need not be stored, as $y$ is recomputed on the backward pass. So whilst differentiating through the internal operations has memory cost $\bigO{HT}$, the continuous adjoint method has a memory cost of only $\bigO{H}$, independent of the time horizon $T$.

\paragraph{Ease of implementation} Another advantage is a practical one: the differential equation solver need not be written using autodifferentiation software, as required for the discretise-then-optimise approach. The differential equation solver can instead be treated as a black box; for example the solver may have been written for some other purpose as part of some other software package.

\subsubsection{Disadvantages}

\paragraph{Computational cost} The continuous adjoint method incurs additional computations necessary to recalculate $y(t)$ on the backward pass; this implies a slightly slower and more computationally expensive procedure.

\paragraph{Truncation errors} The second disadvantage is numerical discretisation error: there will be a difference in the value computed for $y(t)$ on the forward pass (computed starting from the initial condition $y(0)$), and the value computed for $y(t)$ on the backward pass (computed starting from the numerical approximation to the terminal condition $y(T)$ obtained on the forward pass).

Furthermore and in addition to the recomputation of $y(t)$, the continuous adjoint equations for $a_y(t)$ and $a_\theta(t)$ must themselves be solved numerically, and in doing so incur some additional numerical error.

The result is that gradients calculated via the continuous adjoint method will not be as accurate as those computed by backpropagating through the solver. (Which are the gold standard, corresponding to the model actually used.) This means that training may be slower, final model performance may be impacted, and in the worst case training may fail altogether. \cite{anode} give a description of the possible failure modes of the continuous adjoint method, and \cite{onken2020discretize} perform a thorough empirical investigation comparing optimise-then-discretise against discretise-then-optimise, in favour the of the latter.

\begin{example}\label{example:numerical:forward-backward-errors}
Consider solving the system $\nicefrac{\dd y}{\dd t}(t) = \lambda y(t)$ with a numerical ODE solver, where $y(t) \in \reals$, $y(0) = y_0$, and suppose $\lambda < 0$. Most differential equation solvers -- those with a nontrivial region of stability \cite[Definition 2.1]{hairer-ii} -- will handle this without trouble, as errors will decay exponentially. However when this is instead solved backwards-in-time from $y(T)$, as in equation \eqref{eq:backward-ode}, then small errors are instead magnified exponentially. Moreover if $\lambda > 0$ then the same problem arises simply by interchanging the forward and backward passes in the above discussion.
\end{example}

\paragraph{However \ldots}

Despite these dire warnings, continuous adjoint methods often (but not always) still work in practice, without needing any special care. The continuous adjoint method's suitability for any given problem is typically determined empirically -- `Does training seem to be working?' -- and difficulties are frequently not a concern for practitioners.

\subsubsection{Interpolated adjoints}\index{Interpolated adjoints}
It is possible to finesse the problem of numerical errors in the continuous adjoint method.

Record the values of $y(t)$ at specific locations on the forward pass (but not the internal operations of the solver used to obtain these $y(t)$). Interpolate these recorded values to form an approximation to $y(t)$ for all $t$. Then solve the continuous adjoint equations on the backward pass \textit{without} additionally recomputing $y(t)$, by instead using the interpolated approximation at whatever values of $t$ are required during the backward solve.

Interpolated adjoints are actually the default backpropagation method used in \cite{sundials-cvodes} and \cite{rackauckas2020universal}, for example. \cite{stiff-neural-ode} report finding that optimise-then-discretise adjoints failed on a stiff differential equation, but that both interpolated adjoints and discretise-then-optimise succeeded.

\paragraph{Stability and stiffness} The use of interpolated adjoints implies that the differential equations solved -- for $y(t)$ forward in time, and $a_y(t)$ backward in time -- now exhibit very similar behaviour, in particular with respect to stability and stiffness.

The local behaviour of a differential equation is understood through the eigenvalues of the Jacobian of the vector field\footnote{An eigenvalue with positive real part describes a mode of a system that is `locally expansive': points diverge from each other exponentially fast. Likewise negative real part describes a mode of a system that is `locally contractive': points draw closer together exponentially fast. The imaginary part of an eigenvalue corresponds to the `local rotation' of a system -- see also Section \ref{section:ode:rotational-vector-field}.} \cite[Section 112]{butcher2016numerical}, \cite[Section I.13, Equation (13.2)]{hairer}.

For $y(t)$, this is $\nicefrac{\partial f_\theta}{\partial y}(t, y(t))$. Meanwhile for $a_y(t)$ and $a_\theta(t)$, this is
\begin{align*}
\frac{\partial}{\partial a_y}\left(-a_y(t) \cdot \frac{\partial f_\theta}{\partial y}(t, y(t))\right) &= -\frac{\partial f_\theta}{\partial y}(t, y(t)),\\
\frac{\partial}{\partial a_\theta}\left(-a_y(t) \cdot \frac{\partial f_\theta}{\partial \theta}(t, y(t))\right) &= 0.
\end{align*}

The (local) behaviour of $a_\theta(t)$ is trivial as the vector field has zero Jacobian. Meanwhile recalling that the equation for $y(t)$ is solved forward-in-time and the equation for $a_y(t)$ is solved backward in time, we see that their Jacobians are identical.

\begin{remark}
Morally speaking this is expected: differentiation obtains a local linear approximation -- that is to say a Jacobian -- to a function. In this case the function is the overall act of solving an ODE. Meanwhile, the present analysis on local behaviour of a system is about constructing a local linear approximation -- a Jacobian  -- to the vector field.
\end{remark}

Note that this discussion is not true of standard optimise-then-discretise, for which the equation for $y$ is solved in both the forward and backward directions, which exhibit opposite behaviour to each other.

\paragraph{Memory efficiency} Recording the values of $y(t)$ incurs a small memory cost, but not as much as recording every internal operation of the solver as in the discretise-then-optimise approach.

\paragraph{Choice of interpolation} There are several sensible ways to record and interpolate $y(t)$. \cite{sundials-cvodes} use cubic Hermite interpolation whilst \cite{interpolation-adjoint} use barycentric Lagrange interpolation, in both cases recording $y(t_1), \ldots, y(t_n)$ for some prespecified values $t_1, \ldots, t_n$.

%

\subsubsection{Checkpointing}\label{section:numerical:otd-checkpointing}\index{Checkpointing}
Another way to finesse the problem of numerical errors is to use checkpointing as in Section \ref{section:numerical:dto-checkpointing}. Each time we hit a checkpoint recorded on the forward pass, then we effectively reset any accumulated truncation error in $y$ acquired on the backward pass. The trade-off being that this increases the memory usage from $\bigO{H}$ to $\bigO{H + C}$, where $C$ is the number of checkpoints. (Although in practice this is often a relatively modest amount.)

\subsection{Reversible ODE solvers}\index{Reversible solvers}
Reversible ODE solvers offer a best-of-both-worlds approach compared to discretise-then-optimise and optimise-then-discretise. Reversible solvers offer both memory efficiency and accuracy of the computed gradients. Like optimise-then-discretise, they do require a small amount of extra computational work, to recompute the forward solution during backpropagation.

This is a topic still in its infancy. At present, two general reversible ODE solvers are known: the reversible Heun method, and the asynchronous leapfrog method (ALF). In addition, if the differential equation has the right structure, then symplectic solvers are typically also reversible.

The main drawback of reversible ODE solvers is that, at time of writing, all such solvers are low-order and exhibit poor stability properties. This is often not a problem if using pure-neural-network vector fields, but if the differential equation has known structure as in Section \ref{section:ode:hybrid} then in principle this may result in a poor-quality solution.

We defer a full discussion of reversible solvers to sit alongside our discussion of other numerical solvers, in Section \ref{section:numerical:reversible-solvers}.

\subsection{Forward sensitivity}\index{Forward sensitivity}
Whilst not technically \textit{back}propagation, for completeness we mention that it is also possible to compute forward sensitivities of an ODE. (Or indeed a CDE or SDE.)

Once again this can be done either in discretise-then-optimise fashion or in optimise-then-discretise fashion. Discretise-then-optimise is accomplished by computing the forward sensitivity of the solver's computation graph.

Optimise-then-discretise is accomplished by simply differentiating equation \eqref{eq:numerical:_forward-diff-pre} with respect to $y_0$ and $\theta$, through which we obtain the following theorem.
\begin{theorem}\label{theorem:numerical:ode-forward-sensitivity}
	Let $y_0 \in \reals^d$, $\theta \in \reals^m$, $f_\theta \colon [0, T] \times \reals^d \to \reals^d$ be  continuous in $t$, uniformly Lipschitz in $y$, and continuously differentiable in $y$. Let $y \colon [0, T] \to \reals^d$ be the unique solution to
	\begin{equation}\label{eq:numerical:_forward-diff-pre}
		y(0) = y_0, \qquad \frac{\dd y}{\dd t}(t) = f_\theta(t, y(t)).
	\end{equation}
	
	Then $\nicefrac{\dd y(t)}{\dd y_0} = J_y(t)$ and $\nicefrac{\dd y(t)}{\dd \theta} = J_\theta(t)$, where $J_y \colon [0, T] \to \reals^{d \times d}$ and $J_\theta \colon [0, T] \to \reals^{d \times m}$ solve the system of differential equations
	\begin{align*}
		J_y(0) &= \eye{d},\hspace{1.6em} &\frac{\dd J_y}{\dd t}(t) &= \frac{\partial f_\theta}{\partial y}(t, y(t)) J_y(t),\\
		J_\theta(0) &= 0,\hspace{1.6em} &\frac{\dd J_\theta}{\dd t}(t) &= \frac{\partial f_\theta}{\partial y}(t, y(t)) J_\theta(t) + \frac{\partial f_\theta}{\partial \theta}(t, y(t)).
	\end{align*}
	The right hand side consists of Jacobian-vector products, which can be computed efficiently via autodifferentiation.
\end{theorem}

As usual, forward sensitivity is typically less efficient than reverse-mode autodifferentiation when considering machine learning problems with many parameters, so this approach is infrequently used. \cite{rackauckas2018comparison} report finding it useful (only) on problems with very few ($<$100) parameters.

\section{Backpropagation through CDEs and SDEs}\label{section:numerical:adjoint-cde-sde}
We now study backpropagation through differential equations more generally.

\subsection{Discretise-then-optimise}\index{Discretise-then-optimise}
This is exactly the same as in the ODE case (Section \ref{section:numerical:discretise-then-optimise}) -- simply differentiate through the internal operations of the controlled/stochastic differential equation solvers, typically by using solvers written in an autodifferentiable framework.

\subsection{Optimise-then-discretise for CDEs}\label{section:adjoint-ncde}\index{Optimise-then-discretise!CDEs}
There are two approaches to constructing continuous adjoint methods for CDEs. One is to reduce the CDE to an ODE as in Chapter \ref{chapter:neural-cde}, and then applying the continuous adjoint method for ODEs. For example this is what is done in the \texttt{torchcde} library \cite{torchcde}.

Alternatively a backwards-in-time CDE may be constructed, and then numerically solved in whatever manner is desired, by reduction to an ODE or otherwise. The corresponding theorem is as follows.

\begin{restatable}{theorem}{cdeadjoint}\label{theorem:numerical:cde-adjoint}
	Let $f \colon \reals^{d_y} \to \reals^{d_y \times d_x}$ be both Lipschitz and continuously differentiable. Let $x \colon [0, T] \to \reals^{d_x}$ be continuous and of bounded variation. Let $L \colon \reals^{d_y} \to \reals$ be differentiable (and scalar just for simplicity). Let $y_0 \in \reals^{d_y}$ and let $y \colon [0, T] \to \reals^{d_y}$ solve
	\begin{equation}\label{eq:numerical:_cde-adjoint-pre}
		y(0) = y_0,\qquad y(t) = y(0) + \int_0^t f(y(s))\,\dd x(s).
	\end{equation}
	Then the adjoint process $a(t) = \nicefrac{\dd L(y(T))}{\dd y(t)}$ satisfies the backwards-in-time linear CDE
	\begin{equation}\label{eq:numerical:_cde-adjoint}
		a(t) = a(T) + \int_T^t -a(s)^\top\frac{\partial f}{\partial y}(y(s)) \,\dd x(s),
	\end{equation}
	starting from the terminal condition $a(T) = \nicefrac{\dd L(y(T))}{\dd y(T)}$, and	where the right hand side denotes a vector-Jacobian product.
\end{restatable}

For simplicity we have avoided explicitly encoding the dependence on the parameterisation $\theta$. This case may be recovered by replacing $y$, $y_0$, $f_\theta$ with $[y, \theta]$, $[y_0, \theta]$, and $f(y, \theta) = [f_\theta(y), 0]$.

After having solved equation \eqref{eq:numerical:_cde-adjoint} backward-in-time, $a(0) = \nicefrac{\dd L(y(T))}{\dd y(0)}$ are the desired gradients. As with the ODE case, $y$ need not be recorded on the forward pass and may instead by recomputed on the backward pass, by stacking \eqref{eq:numerical:_cde-adjoint-pre} and \eqref{eq:numerical:_cde-adjoint} together and solving as a joint system backwards-in-time.

See Appendix \ref{appendix:cde-adjoint} for a proof.

\subsection{Optimise-then-discretise for SDEs}\label{section:adjoint-nsde}\index{Optimise-then-discretise!SDEs}
\begin{theorem-informal}\label{theorem-informal:numerical:sde-adjoint}
	Let $\mu \colon \reals \times \reals^{d_y} \to \reals^{d_y}$ and $\sigma \colon \reals \times \reals^{d_y} \to \reals^{d_y \times d_w}$ be sufficiently regular. Let $L \colon \reals^{d_y} \to \reals$ be differentiable (and scalar just for simplicity). Let $y_0 \in \reals^{d_y}$ and let $y \colon [0, T] \to \reals^{d_y}$ solve the Stratonovich SDE
	\begin{equation}\label{eq:numerical:_sde-adjoint-pre}
		y(0)=y_0,\qquad\dd y(t) = \mu(t, y(t))\,\dd t + \sigma(t, y(t)) \circ\dd w(t).
	\end{equation}
	Then the adjoint process $a(t) = \nicefrac{\dd L(y(T))}{\dd y(t)} \in \reals^{d_y}$ is a (strong) solution to the backwards-in-time linear Stratonovich SDE
	\begin{equation}\label{eq:numerical:_sde-adjoint}
		\dd a_{k_1}(t) = -a_{k_2}(t)\frac{\partial \mu_{k_2}}{\partial y_{k_1}}(t, y(t))\,\dd t - a_{k_2}(t) \frac{\partial \sigma_{k_2, k_3}}{\partial y_{k_1}}(t, y(t)) \,\circ\dd w_{k_3}(t),
	\end{equation}
	using Einstein notation over the indices $k_1, k_2, k_3$, starting from the terminal condition $a(T) = \nicefrac{\dd L(y(T))}{\dd y(T)}$.
	
	In particular $w$ is the same Brownian motion as used in the forward pass.
\end{theorem-informal}

As in the previous subsection, we have avoided explicitly encoding the dependence on the parameterisation $\theta$, the desired gradients are the computed value $a(0)$, and \eqref{eq:numerical:_sde-adjoint-pre}--\eqref{eq:numerical:_sde-adjoint} may be stacked together to recover $y(t)$ during the backpropagation. The right hand side of \eqref{eq:numerical:_sde-adjoint} consists of vector-Jacobian products, which may be calculated using autodifferentiation.

Note that the nondifferentiability of Brownian sample paths is unrelated to being able to compute derivatives $\nicefrac{\dd L(y(T))}{\dd y(0)}$. (It just means that derivatives like $\nicefrac{\dd y}{\dd t}$ do not exist.)

See Appendix \ref{appendix:sde-adjoint} for a precise statement and a proof of Theorem \ref{theorem-informal:numerical:sde-adjoint} via rough path theory.

\paragraph{Rough path theory}\index{Rough!Path theory}
Theorem \ref{theorem-informal:numerical:sde-adjoint} is stated informally, because putting a precise meaning on the solution of \eqref{eq:numerical:_sde-adjoint} is a little outside the usual framework for SDEs. In particular \eqref{eq:numerical:_sde-adjoint} fails to exhibit measurability with respect to the natural filtration of $w$.

Rough path theory provides an elegant (and intuitive) solution, by allowing solutions to \eqref{eq:numerical:_sde-adjoint-pre} and \eqref{eq:numerical:_sde-adjoint} to be defined \textit{pathwise}. We simply fix a single sample of $w$, and evaluate the forward pass via \eqref{eq:numerical:_sde-adjoint-pre} and the backward pass via \eqref{eq:numerical:_sde-adjoint} -- in every respect just like the ODE case.

\begin{remark}
When backpropagating through an SDE solve via discretise-then-optimise, then a Brownian motion is sampled on the forward pass of the numerical solver, its random samples are fixed as part of the computation graph, and then this computation graph is backpropagated. (Indeed just like any neural generative model; the noise sampled on the forward pass is the same noise on the backward pass.) As such `discretise-then-optimise' is somehow intrinsically also pathwise.
\end{remark}

\begin{remark}\label{remark:numerical:stratonovich-over-ito}
Note the use of Stratonovich integration. This is naturally `time reversible', unlike It{\^o} integration. If \eqref{eq:numerical:_sde-adjoint-pre} was an It{\^o} SDE then the equivalent of \eqref{eq:numerical:_sde-adjoint} is substantially more thorny to work with: it would be derived by applying the It{\^o}-Stratonovich correction term to convert \eqref{eq:numerical:_sde-adjoint-pre} into a Stratonovich integral, applying Theorem \ref{theorem-informal:numerical:sde-adjoint}, and then applying the Stratonovich-It{\^o} correction term to \eqref{eq:numerical:_sde-adjoint}.

In a practical implementation then this double-correction implies substantial computational overhead, so it is preferable to use Stratonovich SDEs instead of It{\^o} SDEs when training via optimise-then-discretise methods.
\end{remark}

\subsection{Reversible differential equation solvers}\index{Reversible solvers}
CDEs may be reduced to ODEs as discussed in Chapter \ref{chapter:neural-cde}, and correspondingly any reversible ODE solver may be applied. Meanwhile SDEs have a single known reversible solver, namely the reversible Heun method. See Section \ref{section:numerical:reversible-solvers}.

\section{Numerical solvers}
\newcommand{\aviciouslie}{Neural networks represent unstructured vector fields. This means that many of the more specialised differential equation solvers (developed for any particular equation) do not apply, and we must rely on `general' solvers.}
\subsection{Off-the-shelf numerical solvers}\label{section:numerical:off-the-shelf}
\aviciouslie

There is a rich literature of such numerical differential equation solvers. We will largely focus on explicit Runge--Kutta solvers, in particular for ODEs and CDEs, which are a popular family of numerical solvers. Other reasonable choices exist -- for example linear multistep methods -- but it is not our purpose to restate the numerical differential equation literature.

\subsubsection{General principles}
There are some principles, specific to neural differential equations over differential equations in general, that help guide the choice of numerical solver.

\paragraph{Implicit solvers} Implicit solvers, for example the implicit Euler method $y_{j+1} = y_j + \Delta t f(t_{j + 1}, y_{j+1})$, are rarely used.\index{Implicit!solvers}

They are computationally expensive: implicit solvers solve a linear or nonlinear system at every step, often through a fixed point iteration. Neural differential equations are a regime in which the vector field evaluations are expensive, and many vector field evaluations are already being made (over a batch, and over the course of training). Reducing computational cost is of substantial interest.

A major use-case for implicit solvers is solving stiff differential equations.\footnote{Somewhat tautologically, as stiff differential equations are broadly categorised as `equations for which explicit solvers fail'.} However stiffness is often not a problem for neural differential equations -- if stiffness and an explicit solver produce a poor solution, then the loss between model and data may be large. As this is the criteria we explicitly train to avoid (to achieve a small loss), then the issue is avoided.

\begin{remark}
The above description is typical for `machine learning neural differential equations' such as CNFs or neural CDEs -- Section \ref{section:ode:cnf} and Chapter \ref{chapter:neural-cde} respectively -- but it is not a universal rule. For example if the vector field incorporates known structure (Section \ref{section:ode:hybrid}), or the data has multiple different timescales, then stiffness may be unavoidable and an implicit solver may become a reasonable choice. See for example \cite{stiff-neural-ode}.
\end{remark}

\paragraph{Adaptive versus fixed step solvers} Both fixed step size and adaptive step size solvers are often reasonable choices for neural differential equations.\index{Adaptive solvers}\index{Fixed solvers}

Given a time horizon $T > 0$, then fixed step size solvers choose some step locations $0 = t_0 < t_1 < \cdots < t_n = T$ in advance, usually with $\Delta t = t_{j+1} - t_j$ independent of $j$.

Adaptive step solvers vary the size of the next step $t_{j+1} - t_j$, so that the (local) error made during the solve is approximately equal to some tolerance. For example embedded Runge--Kutta methods are of this type. This implies a variable computational cost, typically increasing over the course of training as model complexity increases \cite[Figure 3(d)]{neural-odes}, \cite[Figure 3(c)]{how-to-train-node}.\footnote{Loosely speaking, neural networks tend to increase in complexity over the course of training \cite{sgd-increasing-complexity}, \cite[Section 5]{ntk}. This manifests as the training and validation losses following the classic bias-variance curves during training.}

\begin{example}
Consider solving a neural CDE with densely-sampled and slowly-varying time series data as input. The slow variation of the input data means that processing every piece of it -- as we may do with an RNN -- is likely overkill. The adaptivity of a solver may automatically detect the slow timescale at which the differential equation is driven, and produce integration steps of the appropriate size, larger than the discretisation of the data.

Moreover, RNN training often breaks down as the length of a time series increases. If this length has been achieved by sampling the same signal more and more densely, then it is a bit perverse that this extra information should cause our model to fail to train. Philosophically speaking, it is reassuring to be able to overcome this issue using adaptive solvers.
\end{example}

\paragraph{Baked-in discretisations}
\label{section:numerical:baked-in-discretisation}\index{Baked-in discretisations}

If only a single solver (in particular a low-order solver or a solver with a fixed step size) is used during model training, then this choice of discretisation may become an intrinsic part of the model. The neural vector fields will have been trained to work best at this discretisation, and may fail with other discretisations \cite{ott2021resnet, continuous-net}.

For many applications this need not be a problem. In the introduction to this thesis (Section \ref{section:motivation:applications}), `inspiration for a discretised model' was described as a good use case for neural differential equations, and baking in the numerical discretisation is simply a subtle instance of this.

\paragraph{Step size and error tolerance}\index{Adaptive solvers}
Step size (for fixed step size solvers) or error tolerance (for adaptive step size solvers) will often be very large compared to that seen in the usual numerical differential equation literature. For example when using a neural SDE to generate a time series sampled at some points $t_0 < \cdots < t_n$, then we may elect to take a single numerical step over each interval $[t_j, t_{j+1}]$, even when $t_{j+1} - t_j$ is relatively large.

Once again this may be thought of as treating the continuous model as an ideal, and then deliberately fitting a discretised model. Large step sizes are often motivated by a desire to reduce computational cost, and thus training time.

\begin{example}\index{Large step size regime}
	In this `large step size regime', do note that taking smaller steps may produce slightly more expressive models. When looking to increase or decrease the modelling capacity of a neural differential equation, both step sizes and vector field complexity are options that may be adjusted.
	
	As an example, consider fitting a neural SDE as in Chapter \ref{chapter:neural-sde}. If taking a single numerical step using the Euler--Maruyama method, then the conditional distribution of $y(t_{j+1})|y(t_j)$ will only be Gaussian, which is relatively inexpressive.
\end{example}

\paragraph{Order of solver}
Low-order solvers are often reasonable choices, especially when explicitly baking in the discretisation in the large step size regime. For example Euler's method is first-order convergent; midpoint or Heun's method are second-order convergent.

If aiming to fit the idealised continuous model (for example when training via optimise-then-discretise), then higher-order solvers such as Dormand--Prince are often preferred. (At least when they are available, that is to say for ODEs and CDEs but not SDEs.)

\subsubsection{ODEs and CDEs}\label{section:numerical:typical-solvers}\index{Euler's method}\index{Midpoint method}\index{Heun's method}\index{RK4}\index{Dormand--Prince}\index{Tsitouras}
Bringing the above points together: when solving ODEs, or CDEs reduced to ODEs, then standard low-order solvers are the explicit Euler method (first order), the midpoint method (second order), or Heun's method (second order). Standard higher-order\footnote{Relatively speaking.} methods are RK4 (fourth order), Dormand--Prince (fifth order), or Tsit5 (fifth order).

\cite{rackauckas-diffeq-runge-kutta} offer an extensive comparison of explicit Runge--Kutta methods.

\subsubsection{SDEs}\index{Euler--Maruyama method}\index{Heun's method}\index{Milstein's method}
Standard choices of numerical solver for neural SDEs include the Euler--Maruyama method or Heun's method, for It{\^o} or Stratonovich SDEs respectively.\footnote{Broadly speaking SDEs solvers are distinguished by whether they converge to the It{\^o} or Stratonovich solution.} If the problem has commutative noise\footnote{This is a condition that is satisfied in several common special cases: if the Brownian motion is scalar valued; if the diffusion matrix is independent of the SDE solution; if the diffusion matrix is diagonal.} then Milstein's method may be applied for both It{\^o} and Stratonovich SDEs.

\subsection{Reversible solvers}\label{section:numerical:reversible-solvers}\index{Reversible solvers}
We indicated in Sections \ref{section:numerical:adjoint-ode} and \ref{section:numerical:adjoint-cde-sde} that besides discretise-then-optimise and optimise-then-discretise, there is a third option, namely reversible solvers.

Consider a differential equation solver iteratively computing $(y_j, \alpha_j) \mapsto (y_{j+1}, \alpha_{j+1})$, where $\{y_j\}_{j=0}^n$ is the numerical approximation to the solution of some differential equation -- whether it be an ODE, CDE, or SDE -- and $\alpha_j$ denotes any extra state that the differential equation solver wishes to keep around.

By definition, $(y_{j+1}, \alpha_{j+1})$ may be computed from $(y_j, \alpha_j)$. We say that a solver is \textit{reversible} if $(y_j, \alpha_j)$ may be computed from $(y_{j+1}, \alpha_{j+1})$. (Note that we have not yet made precise what is meant by `computed'.)

\paragraph{Backpropagation with reversible solvers}
\begin{algorithm}[t]
	\SetKwInput{kwInput}{Input}
	\SetKwInput{kwOutput}{Output}
	\kwInput{$t_{j+1}, y_{j+1}, \alpha_{j+1}, \Delta t, x,\frac{\partial L(y_n)}{\partial y_{j + 1}},\frac{\partial L(y_n)}{\partial \alpha_{j + 1}}$}
	\vspace{0.5em}
	
	$y_j, \alpha_j = \mathrm{Reverse}(t_{j+1}, y_{j+1}, \alpha_{j+1}, \Delta t, x)$\\[2pt]
	
	$y_{j+1}, \alpha_{j+1} = \mathrm{Forward}(t_j, y_j, \alpha_j, \Delta t, x)$\\[2pt]
	
	$\frac{\partial L(y_n)}{\partial (y_j, \alpha_j)} = \frac{\partial L(y_n)}{\partial (y_{j + 1}, \alpha_{j+1})} \cdot\frac{\partial (y_{j+1}, \alpha_{j+1})}{\partial (y_j, \alpha_j)}$\qquad\# vector-Jacobian product\\[2pt]
	
	\kwOutput{$t_j, y_j, \alpha_j,\frac{\partial L(y_n)}{\partial y_j},\frac{\partial L(y_n)}{\partial \alpha_j}$}
	\caption{Backward pass through a reversible solver.\\$t$ denotes time, $y$ denotes the numerical solution, $\alpha$ denotes any additional state the solver keeps around, and $\Delta t$ denotes a step size. $x$ denotes a possible control input, which may be unused if the equation is an ODE and may be a Brownian motion if the equation is an SDE.}\label{alg:numerical:reversible-backward}
\end{algorithm}

Backpropagation through a reversible solver is shown in Algorithm \ref{alg:numerical:reversible-backward}. The `local' forward is needed to construct a computational graph, through which the vector-Jacobian product is calculated.

\paragraph{Computational cost}
Assuming that a forward and reverse step cost the same, the total computational cost of evaluating and backpropagating through a step of a reversible solver is three forward operations and one backward operation: one forward operation on the forward pass, two forward operations on the backward pass, and a single backward operation on the backward pass. As a combined forward/backward operation costs at most four forward operations \cite[Equation (4.21)]{griewank2008evaluating}, then the overall computational cost is approximately that of six forward operations.

This cost should be contrasted with discretise-then-optimise and optimise-then-discretise. Discretise-then-optimise involves simply a forward operation on the forward pass, and a backward operation on the backward pass, for an overall computational cost of approximately four forward operations. Optimise-then-discretise involves a forward operation on the forward pass, a single forward operation on the backward pass, and a single backward operation on the backward pass, for an overall computational cost of approximately five forward operations.

\begin{remark}
It is sometimes possible to elide the local forward in Algorithm \ref{alg:numerical:reversible-backward}. Given suitable structure in the solver, it may be possible to reuse the computational graph of the reverse step, and in doing so save the cost of a single forward operation. See for example the asynchronous leapfrog method coming up in Section \ref{section:numerical:alf}.
\end{remark}

\paragraph{Precise gradients}
As the same numerical solution $\{y_j\}_{j=0}^n$ is recovered on both the forward and backward passes, the computed gradients are precisely the discretise-then-optimise gradients of the numerical discretisation of the forward pass.

As it is the discretised model that is fit to data, discretise-then-optimise represents the gold standard -- the `true' gradients of the model -- so this property is desirable.

\subsubsection{Analytic and algebraic reversibility}\label{section:numerical:analytic-reversibility}
We now make precise what it means to `compute' $(y_n, \alpha_n)$ from $(y_{n+1}, \alpha_{n+1})$.
\paragraph{Analytic reversibility}\index{Analytic reversibility}
In some sense, essentially every solver is reversible. For example, the explicit Euler method
\begin{equation*}
	y_{j+1} = y_j + f(y_j) \Delta t
\end{equation*}
may be reversed via the (backwards-in-time) implicit Euler method
\begin{equation*}
	y_j = y_{j+1} - f(y_j) \Delta t,
\end{equation*}
for those $\Delta t$ small enough that the contraction mapping theorem ensures this nonlinear equation has a unique solution.

Unfortunately this requires solving a fixed-point iteration, and computing $y_j$ from $y_{j+1}$ is both approximate and computationally expensive.  We refer to reversibility of this type as \textit{analytic reversibility}.

\begin{example}
	For example \cite{invertible-residual-networks} consider residual networks as the explicit Euler discretisation of a neural ODE, and exactly as above, use the implicit Euler method to invert the operation of each layer during backpropagation.
\end{example}

\paragraph{Algebraic reversibility}\index{Algebraic reversibility}
Substantially more preferable is what we shall refer to as \textit{algebraic reversibility}. These are solvers for which the solution of $(y_j, \alpha_j)$ may be written as a closed-form expression with respect to $(y_{j+1}, \alpha_{j+1})$. As such they are much more computationally reasonable.

We additionally refer to an algebraically reversible solver as \textit{symmetric} if the reverse computation is of the same form as the forward step, simply performed backward-in-time. (If `$(y_{j+1}, \alpha_{j+1}) = \mathrm{step}(y_j, \alpha_j, \Delta t)$' implies `$(y_{j}, \alpha_{j}) = \mathrm{step}(y_{j+1}, \alpha_{j+1}, -\Delta t)$'.)


\subsubsection{Reversible Heun method}\label{section:numerical:reversible-heun}\index{Reversible Heun}
The reversible Heun method, introduced in \cite{kidger2021sde2}, is a symmetric algebraically reversible ODE, CDE, or SDE solver. We will consider solving the SDE
\begin{equation*}
	y(0) = y_0,\quad\dd y(t) = \mu(t, y(t))\,\dd t + \sigma(t, y(t))\circ\dd w(t),
\end{equation*}
over $[0, T]$, for which we will obtain the numerical solution $\{y_j\}_{j=0}^n$. If solving an ODE simply set $\sigma = 0$. If solving a CDE then either reduce it to an ODE, or \textit{mutatis mutandis} replace $w$ with some control $x$.

To the best of this authors' knowledge, and at time of writing, this is the first and only general (non-symplectic) algebraically reversible SDE solver.

\paragraph{Initialisation} The solver tracks several extra pieces of state, $\widehat{y}_j$, $\mu_j$, $\sigma_j$, which are initialised at $\widehat{y}_0 = y_0$, $\mu_0 = \mu(0, y_0)$, $\sigma_0 = \sigma(0, y_0)$, and which have solution-like, drift-like and diffusion-like interpretations respectively. We additionally take $w$ to be a single sample of Brownian motion, which must be the same on both the forward and backward passes.

\paragraph{Stepping} The forward iteration then proceeds by iterating Algorithm \ref{alg:numerical:reversible-heun}. Note the similarity to Heun's method.

\begin{figure}
\begin{minipage}{0.47\linewidth}
	\begin{algorithm}[H]
		\SetKwInput{kwInput}{Input}
		\SetKwInput{kwOutput}{Output}
		\kwInput{$t_j, y_j, \widehat{y}_j, \mu_j, \sigma_j, \Delta t, w$}
		
		\begin{align*}
			t_{j+1} &= t_j + \Delta t\\
			\Delta w_j &= w(t_{j+1}) - w(t_j)\\
			\widehat{y}_{j+1} &= 2 y_j - \widehat{y}_j + \mu_j \Delta t + \sigma_j \Delta w_j\\
			\mu_{j+1} &= \mu(t_{j+1}, \widehat{y}_{j+1})\\
			\sigma_{j+1} &= \sigma(t_{j+1}, \widehat{y}_{j+1})\\
			y_{j+1} &= y_j + \frac{1}{2}(\mu_j + \mu_{j+1})\Delta t\\
			&\quad \mathop+ \frac{1}{2}(\sigma_j + \sigma_{j+1}) \Delta w_j
		\end{align*}
		
		\kwOutput{$t_{j+1}, y_{j+1}, \widehat{y}_{j+1}, \mu_{j+1}, \sigma_{j+1}$}
		\caption{Forward pass for the reversible Heun method.}\label{alg:numerical:reversible-heun}
	\end{algorithm}
\end{minipage}\hfill
\begin{minipage}{0.52\linewidth}
	\begin{algorithm}[H]
		\SetKwInput{kwInput}{Input}
		\SetKwInput{kwOutput}{Output}
		\kwInput{$t_{j+1}, y_{j+1}, \widehat{y}_{j+1}, \mu_{j+1}, \sigma_{j+1}, \Delta t, w$}
		
		\begin{align*}
			t_{j} &= t_{j+1} - \Delta t\\
			\Delta w_j &= w(t_{j+1}) - w(t_j)\\
			\widehat{y}_{j} &= 2 y_{j+1} - \widehat{y}_{j+1} - \mu_{j+1} \Delta t - \sigma_{j+1} \Delta w_j\\
			\mu_{j} &= \mu(t_{j}, \widehat{y}_{j})\\
			\sigma_{j} &= \sigma(t_{j}, \widehat{y}_{j})\\
			y_{j} &= y_{j+1} - \frac{1}{2}(\mu_{j+1} + \mu_{j})\Delta t\\
			&\quad \mathop- \frac{1}{2}(\sigma_{j+1} + \sigma_{j}) \Delta w_j
		\end{align*}
		
		\kwOutput{$t_{j}, y_{j}, \widehat{y}_{j}, \mu_{j}, \sigma_{j}$}
		\caption{Reverse pass for the reversible Heun method.}\label{alg:numerical:reversed-reversible-heun}
	\end{algorithm}
\end{minipage}
\end{figure}

\paragraph{Convergence}
Convergence results are as follows.

\begin{restatable}{theorem}{reversibleheunconvergence}
	The reversible Heun method, when applied to ODEs, is a second-order method.
\end{restatable}
\begin{theorem}
	The reversible Heun method, when applied to SDEs, exhibits strong convergence of order $\nicefrac{1}{2}$. If the noise is additive then this increases to order $1$.
\end{theorem}

A proof of the ODE case is given in Appendix \ref{appendix:numerical:ode-reversible-heun}. A proof of the SDE case is given in \cite[Appendix D]{kidger2021sde2}.

\paragraph{Stability} One drawback of the reversible Heun method is its unimpressive stability properties.

\begin{restatable}{theorem}{reversibleheunstability}
	The region of stability for the reversible Heun method (for ODEs) is the complex interval $[-i, i]$.
\end{restatable}

A proof is given in Appendix \ref{appendix:numerical:ode-reversible-heun}.

\paragraph{Adaptive step sizing} The step size $\Delta t$ may either be fixed in advance (to make this a fixed step size solver) or it may be adapted over the course of the integration.

When solving ODEs, then the reversible Heun method may be treated in the same way as embedded Runge--Kutta method by returning the error estimate $y_\text{error} = (\mu_{j+1} - \mu_j) \Delta t / 2$. This may now be used to adapt step sizing in the usual way (Section \ref{section:numerical:adaptive}).

\paragraph{Reversibility} For completeness, the reverse pass through the reversible Heun method is given in Algorithm \ref{alg:numerical:reversed-reversible-heun}. Note the similarity to the reversible Heun method, as the reversible Heun method is not just algebraically reversible but also symmetric.

\paragraph{Use cases} If training via discretise-then-optimise is not an option, then the reversible Heun is an excellent choice of solver for any of neural ODEs, CDEs, or SDEs. If using pure-neural-network vector fields then its low order and lack of stability need not always be a concern, especially if the discretisation is baked-in as in Section \ref{section:numerical:baked-in-discretisation}.

It is only not recommended if solving neural differential equations with built-in structure as in Section \ref{section:ode:ude}, for which the low order and lack of stability may be concerns.

\paragraph{Computational efficiency}
Unlike the traditional Heun method, the reversible Heun method makes only a single evaluation per step. This can mean that it is more computationally efficient.

\subsubsection{Asynchronous leapfrog method}\label{section:numerical:alf}\index{Asynchronous leapfrog method}
The asynchronous leapfrog method is a symmetric algebraically reversible ODE/CDE (but not SDE) solver, introduced in \cite{alf} and popularised by \cite{zhuang2021mali}. We will consider solving the ODE
\begin{equation*}
	y(0) = y_0,\qquad \frac{\dd y}{\dd t}(t) = f(t, y(t)),
\end{equation*}
over $[0, T]$, for which we will obtain the numerical solution $\{y_j\}_{j=0}^n$.

\paragraph{Initialisation} The solver tracks a single extra piece of extra state $v_j$ in addition to the numerical solution $y_j$. This extra state has a velocity-like interpretation and is initialised as $v_0 = f(y_0)$.

\paragraph{Stepping} The stepping procedure is then given by iterating Algorithm \ref{alg:numerical:alf}. Note the similarity to the midpoint method.

\begin{figure}\centering\begin{minipage}{0.6\linewidth}
\begin{algorithm}[H]
	\SetKwInput{kwInput}{Input}
	\SetKwInput{kwOutput}{Output}
	\kwInput{$t_j, y_j, v_j, \Delta t$}
	
	\begin{align*}
	\widehat{t}_j &= t_j + \Delta t / 2\\
	\widehat{y}_j &= y_j + v_j \Delta t / 2\\
	\widehat{v}_j &= f(\widehat{t}_j, \widehat{y}_j)\\
	\\
	t_{j+1} &= t_j + \Delta t\\
	y_{j+1} &= y_j + \widehat{v}_j \Delta t\\
	v_{j+1} &= 2\widehat{v}_j - v_j
	\end{align*}
	
	\kwOutput{$t_{j+1}, y_{j+1}, v_{j+1}$}
	\caption{Forward pass through the asynchronous leapfrog method.}\label{alg:numerical:alf}
\end{algorithm}
\end{minipage}\end{figure}

\paragraph{Convergence}
The following convergence result may be shown.
\begin{theorem}
	The asynchronous leapfrog method is a second-order method. Specifically, the local truncation error in $y$ is $\bigO{\Delta t^3}$, whilst the local truncation error in $v$ is $\bigO{\Delta t^2}$.
\end{theorem}
See \cite[Theorem 3.1]{zhuang2021mali}.

\paragraph{Stability} As with the reversible Heun method, one drawback of the asynchronous leapfrog method are its unimpressive stability properties.

\begin{theorem}
	The region of stability for the asynchronous leapfrog method is the complex interval $[-i, i]$.
\end{theorem}

A proof is given in \cite[Appendix A.4]{zhuang2021mali}.

\paragraph{Adaptive step sizing} If adaptively setting the step size, then the error estimate $y_\text{error} = (v_{j+1} - v_j) / 2$ may be used (in the same way as an embedded Runge--Kutta method).

\paragraph{Efficient reverse pass}
The general reversibility algorithm given in Algorithm \ref{alg:numerical:reversible-backward} involves both $\mathrm{Reverse}$ and $\mathrm{Forward}$ operations. For the asynchronous leapfrog method, the extra $\mathrm{Forward}$ operation may be elided. This is accomplished by instead differentiating the $\mathrm{Reverse}$ pass, and appropriately adjusting the surrounding calculation.

This is because the quantity $f(\widehat{t}_j, \widehat{y}_j)$ is computed during both $\mathrm{Reverse}$ and $\mathrm{Forward}$. As $f$ is generally both complicated and user-supplied, this is the piece we are most interested in autodifferentiating. We may calculate by hand the appropriate derivatives for the surrounding structure of the algorithm.

In doing so, we obtain Algorithm \ref{alg:numerical:reverse-alf}.

\begin{figure}\centering\begin{minipage}{0.8\linewidth}
\begin{algorithm}[H]
	\SetKwInput{kwInput}{Input}
	\SetKwInput{kwOutput}{Output}
	\kwInput{$t_{j+1}, y_{j+1}, \mu_{j+1}, \Delta t, \frac{\partial L(y_n)}{\partial y_{j + 1}},\frac{\partial L(y_n)}{\partial \mu_{j + 1}}$}
	
	\begin{align*}
		\widehat{t}_j &= t_{j+1} -\Delta t/2\\
		\widehat{y}_j &= y_{j+1} - \mu_{j+1} \Delta t/2\\
		\widehat{\mu}_j &= \mu(\widehat{t}_j, \widehat{y}_j)\\
		\quad\\
		t_j &= t_{j+1} - \Delta t\\
		y_j &= y_{j+1} - \widehat{\mu}_j \Delta t\\
		\mu_j &= 2\widehat{\mu}_j - \mu_{j+1}\\
		\quad\\
		\frac{\partial L(y_n)}{\partial y_j} &= \left(2 \frac{\partial L(y_n)}{\partial \mu_{j + 1}} + \frac{\partial L(y_n)}{\partial z_{j + 1}} \Delta t\right)^\top \frac{\partial \widehat{\mu}_j}{\partial \widehat{z}_j} + \frac{\partial L(y_n)}{\partial z_{j+1}}\\
		\frac{\partial L(y_n)}{\partial \mu_j} &= \frac{1}{2}\frac{\partial L(y_n)}{\partial z_j}\Delta t - \frac{\partial L(y_n)}{\partial \mu_{j + 1}}
	\end{align*}
	
	\kwOutput{$t_j, y_j, \mu_j, \frac{\partial L(y_n)}{\partial y_j},\frac{\partial L(Z_T)}{\partial \mu_j}$}
	\caption{Efficient backward pass through the asynchronous leapfrog method.}\label{alg:numerical:reverse-alf}
\end{algorithm}
\end{minipage}\end{figure}

\paragraph{Use cases} The asynchronous leapfrog method is useful in essentially the same cases as the reversible Heun method, except that it applies only for ODEs.

\begin{remark}
The asynchronous leapfrog method does not seem to extend to SDEs. There is a clearly analogous procedure, tracking a diffusion-like quantity in addition to a drift-like quantity. However it does not seem obviously possible to demonstrate theoretical convergence of this solver, and neural SDEs trained using it perform relatively poorly empirically.
\end{remark}

\subsubsection{Symplectic solvers}\index{Symplectic methods}
Many pre-existing symplectic solvers are already algebraically reversible. There are a great many symplectic solvers; we highlight only a few interesting ones here.

\paragraph{Semi-implicit Euler method}\index{Semi-implicit Euler method}
Given a pair of differential equations
\begin{align*}
	y(0) &= y_0,\qquad \frac{\dd y}{\dd t}(t) = f(t, v(t)),\\
	v(0) &= v_0,\qquad \frac{\dd v}{\dd t}(t) = g(t, y(t)),
\end{align*}
the semi-implicit Euler method is defined by
\begin{align*}
	y_{j+1} &= y_j + f(t_j, v_{j+1}) \Delta t,\\
	v_{j+1} &= v_j + g(t_j, y_j) \Delta t.
\end{align*}

A common special case is $f(t, v) = v$ so that $v$ is the velocity of $y$, and $y$ is understood as the solution of a second-order system.

This is notable for its popularity in deep learning papers; it is the solver used with the rotational vector fields and momentum residual networks of Section \ref{section:ode:residual-networks}.

\paragraph{Leapfrog/midpoint}\label{section:numerical:leapfrog-midpoint}\index{Leapfrog/midpoint method}
Consider the integrator
\begin{align*}
	y_{j+1} &= y_{j-1} + f(t_j, y_j) \Delta t,\\
	y_{j+2} &= y_{j} + f(t_{j+1}, y_{j+1}) \Delta t
\end{align*}
for solving 
\begin{equation*}
	y(0) = y_0,\qquad \frac{\dd y}{\dd t}(t) = f(t, y(t)),
\end{equation*}
over $[0, T]$, where $\{y_j\}_{j=0}^n$ is the numerical solution.

We refer to this as the `leapfrog/midpoint integrator' in accordance with the title of \cite{shampine-leapfrog}, but other texts will call it simply `leapfrog' (in ambiguity with the integrator for second order systems of the same name), or the `explicit midpoint method' (in ambiguity with the Runge--Kutta method of the same name).

This is both algebraically reversible and symmetric, and is applicable to general first-order systems.

\begin{remark}
As a linear multi-step method it does not admit an immediate way to adapt step sizes during integration, nor does it have very good stability properties \cite{shampine-leapfrog}. Fixing the first problem produces the asynchronous leapfrog integrator of Section \ref{section:numerical:alf} \cite{alf}.
\end{remark}

\subsection{Solving vector fields with jumps}\label{section:numerical:stacked}\index{Jump!In the vector field}
One common scenario is that the vector field of a neural ODE has a piecewise structure with respect to time. That is we are solving
\begin{equation*}
	\frac{\dd y}{\dd t}(t) = f_\theta(t, y(t)),
\end{equation*}
where
\begin{equation*}
	f_\theta(t, y) = 
	\begin{cases}
		f_{\theta, j}(t, y) & t \in [t_0, t_1]\\
		\qquad\vdots\\
		f_{\theta, n}(t, y) & t \in [t_{n - 1}, t_n]
	\end{cases}.
\end{equation*}
For example this occurs when solving a stack of neural ODEs as in Section \ref{section:ode:stacking}, or when solving a neural CDE, reduced to an ODE, using linear or rectilinear interpolation (Section \ref{section:cde:interpolation}).

In this case we have two options: make $n$ separate calls to an ODE solver, over each $[t_j, t_{j + 1}]$, or to make a single call to an ODE solver, over the whole $[t_0, t_n]$.

Both options are fine, but in some cases each requires a small amount of caution.

\paragraph{Separate calls}
If making $n$ separate calls to an ODE solver, and training the neural ODE either via optimise-then-discretise or via reversible ODE solvers (Section \ref{section:numerical:reversible-solvers}), then practically speaking the memory cost will be $n$ times larger than if we had made a single ODE solve: each ODE solve will store $y(t_{j+1})$ at the end of its solve, for the sake of the later backpropagation through the ODE solve.

This may be desirable -- essentially implementing checkpointing as in Section \ref{section:numerical:otd-checkpointing}. Alternatively it may be undesirable due to the increased memory cost.

\paragraph{Single call}
If making a single call to the ODE solver, and using an adaptive step size ODE solver, then the solver should be informed about the location of the jumps. Otherwise, the error control in the step size controller will detect a large error every time a threshold is crossed, slow down to resolve it, and then speed back up again.

It is substantially more efficient to simply step directly to the discontinuity; failing to do so can result in an order-of-magnitude slow-down. The software libraries we recommend in Section \ref{section:numerical:software} support this as an option.

\subsection{Hypersolvers}\index{Hypersolvers}
\paragraph{Known versus unknown structure}
When motivating the use of standard off-the-shelf solvers like Euler or Dormand--Prince, we wrote:

\textit{\aviciouslie}

This was, in fact, a white lie.

Neural differential equations \textit{do} exhibit structure -- they exhibit whatever the structure of the problem being modelled is. The problem is simply that this structure is specified by a black-box neural vector field, and is not understood.

A running theme throughout machine learning, and thus also this work, has been that we may substitute theoretical understanding with data -- and that given sufficient data, we may close the gap between a theoretical model and the behaviour observed in practice. We may apply the same principle here.

\paragraph{Learnt error corrections}
For $q \in \naturals$, consider some $q$-th order ODE solver\footnote{We will treat only ODEs; the extensions to CDEs and SDEs are natural but so far unexplored.} with update rule $\psi$. That is to say, for some time points $\{t_j\}_{j=0}^n$ (for simplicity with constant step size $\Delta t = t_{j+1} - t_j$), the numerical solution $y_j \approx y(t_j)$ is obtained by iterating
\begin{equation*}
	y_{j+1} = y_j + \psi(t_j, y_j) \Delta t.
\end{equation*}
For example $\psi(t, y) = f_\theta(t, y)$ for Euler's method.

As a $q$-th order solver, the local truncation error is of order $q + 1$:
\begin{equation*}
	\norm{y(t_{j+1}) - y(t_j) - \psi(t_j, y(t_j)) \Delta t} = \bigO{\Delta t^{q+1}}.
\end{equation*}

A \textit{hypersolver} \cite{hypersolver} is now defined by a learnt correction
\begin{equation}\label{eq:numerical:hypersolver}
	y_{j+1} = y_j + \psi(t_j, y_j) \Delta t + g_\omega(t_j, y_j, \Delta t) \Delta t^{q + 1}
\end{equation}
with $g_\omega \colon \reals \times \reals^d \times \reals \to \reals^d$ some neural network depending on learnt parameters $\omega$.

\begin{example}
For example, recall Heun's method
\begin{align*}
	\widehat{y}_{j+1} &= y_j + f_\theta(t_j, y_j) \Delta t\\
	y_{j+1} &= y_j + \frac{1}{2} (f_\theta(t_j, y_j) + f_\theta(t_{j+1}, \widehat{y}_{j+1})) \Delta t.
\end{align*}
Then HyperHeun is defined by
\begin{align*}
	\widehat{y}_{j+1} &= y_j + f_\theta(t_j, y_j) \Delta t\\
	y_{j+1} &= y_j + \frac{1}{2} (f_\theta(t_j, y_j) + f_\theta(t_{j+1}, \widehat{y}_{j+1})) \Delta t + g_\omega(t_j, y_j, \Delta t) \Delta t^3.
\end{align*}
(This implicitly features a $\bigO{\Delta t^2}$ term due to the $\widehat{y}_{j+1}$.)
\end{example}

\paragraph{Training}
Training is performed by assuming access to the true solution of the neural ODE. In practice this may be approximately obtained by using a traditional numerical solver with high order, small step sizes, or tight error tolerances.

Let
\begin{equation*}
	R(t, y_\text{prev}, y_\text{next}) = \frac{1}{\Delta t^{q+1}} (y_\text{next} - y_\text{prev} - \psi(t, y_\text{prev}) \Delta t).
\end{equation*}
Then a hypersolver may then trained by minimising either
\begin{equation}\label{eq:numerical:hypersolver-training-one}
	\frac{1}{n}\sum_{j=0}^{n-1}\norm{R(t_j, y(t_j), y(t_{j+1})) - g_\omega(t_j, y(t_j), \Delta t)}
\end{equation}
or
\begin{equation}\label{eq:numerical:hypersolver-training-two}
	\frac{1}{n}\sum_{j=0}^{n-1}\norm{R(t_j, y_j, y(t_{j+1})) - g_\omega(t_j, y_j, \Delta t)},
\end{equation}
where $\{y_j\}_{j=1}^n$ is the numerical solution obtained by iterating \eqref{eq:numerical:hypersolver}. (The former is analogous to training an RNN using teacher forcing; the latter to training an RNN without it.)

\paragraph{Applications}
The primary interest in hypersolvers is to obtain solutions that are both fast and accurate. As speed is of interest, then typically the base update rule $\psi$ is very simple (Euler or Heun), whilst $g_\omega$ may only be a single-layer MLP. This is often sufficient to obtain excellent results. For example \cite{hypersolver} report striking results in which 80 Dormand--Prince steps may be replaced by only 2 HyperHeun steps, without degrading accuracy (in that example, on a continuous normalising flow).

However, hypersolvers are generally only useful for speeding up inference, not training. The neural ODE changes during training, and so \eqref{eq:numerical:hypersolver-training-one}--\eqref{eq:numerical:hypersolver-training-two} become a moving target.

There is related work on training neural networks as differential equation solvers \cite{QIN2019620}. This is a related but distinct notion to training a neural network as the solution of the differential equation itself \cite{lagaris1, lagaris2, Han8505, NEURIPS2018_d7a84628, JMLR:v19:18-046, PhysRevD.100.016002, RAISSI2019686, Fang2019}.

\section{Tips and tricks}
\subsection{Regularisation}
\subsubsection{Weight decay}
Adding weight decay to the parameters of a neural vector field may help to improve model performance, just as in traditional deep learning.

For many neural networks, the scale of its output is roughly proportional to the scale of its weights. That is to say, $\norm{f_\theta}/\norm{\theta}$ may be approximately constant over different values of $\theta$. As such an additional implication of weight decay is that the vector field may be closer to zero, and thus numerically easier to integrate.

\subsubsection{Temporal regularisation}
For applications of neural differential equations to `non time series' problems, one form of regularisation is to select the region of integration randomly. For example when training a continuous normalising flow, then instead of taking a fixed a region of integration $[\tau, T]$, the endpoints $\tau$, $T$ may be sampled from some distribution; perhaps $\tau = 0$ remains fixed whilst $T \sim \uniform{0.9}{1.1}$. This is computationally cheap whilst encouraging models which are robust to small perturbations \cite{steer}.

\subsubsection{Additive noise}\label{section:numerical:additive-noise}
Mainstream deep learning often uses stochasticity as a regulariser, such as dropout. Correspondingly, and for neural ODEs and neural CDEs specifically, then including some small additive noise after each step (so that the model becomes an SDE) is another computationally cheap option that encourages more robust models \cite{stochasticnode, additive-noise}.

In this case the added noise should be fixed. If it is learnt then the training process will shrink it to zero and no regularisation will be applied.

\subsubsection{Regularising higher-order derivatives}\label{section:numerical:regularising-derivatives}
Consider the usual set-up for a neural ODE in which $y$ solves $\nicefrac{\dd y}{\dd t}(t) = f_\theta(t, y(t))$.

Let $q \in \naturals$ and consider some $q$-th order numerical ODE solver. Over any given numerical step $[t_j, t_{j+1}]$, such solvers operate by locally approximating the solution $y$ by some $q$-th order polynomial. Correspondingly the error made over some step is determined by the $q+1$-th total derivative $\nicefrac{\dd^{q+1}}{\dd t^{q+1}}(y(t)) = \nicefrac{\dd^{q}}{\dd t^{q}}(f_\theta(t, y(t)))$.

We may seek to minimise numerical errors, and promote easy-to-integrate dynamics, by regularising
\begin{equation}\label{eq:numerical:_regularising-higher-order}
	\int_0^T \norm{\frac{\dd^{q}}{\dd t^{q}}\left(f_\theta(t, y(t))\right)}^2_2 \,\dd t.
\end{equation}
This was investigated in \cite{easy}.

(As this is only a regularisation term -- this will be made small, not precisely zero -- then we may also consider regularising lower-order derivatives instead.)

\paragraph{Taylor-mode autodifferentiation}
In principle evaluating \eqref{eq:numerical:_regularising-higher-order} may be done via autodifferentiation, although a little care is needed to get the correct derivatives: each derivative on the right hand side involves taking a derivative of $y$ with respect to $t$, so that $\nicefrac{\dd y}{\dd t} = f_\theta$ will start to appear multiple times \cite[Section 310]{butcher2016numerical}.

However doing so via na{\"i}ve autodifferentiation will be unnecessarily expensive. A Jacobian-vector product typically costs $2.5$ times the cost of the corresponding forward evaluation, so nesting $K$ such evaluations will result in a $2.5^K = \bigO{\exp(K)}$ cost. That this is a higher-order derivative implies certain structure that may be exploited to reduce the cost to only $\bigO{K^2}$; see \cite[Appendix A]{easy} or \cite[Chapter 13]{griewank2008evaluating}.

\paragraph{Continuous normalising flows}
A special case arises -- see \cite{how-to-train-node} -- when regularising low-order derivatives of continuous normalising flows (Section \ref{section:ode:cnf}). When evaluating \eqref{eq:numerical:_regularising-higher-order} with $q=1$, then
\begin{equation*}
	\norm{\frac{\dd}{\dd t}(f_\theta(t, y(t)))}_2^2 = \norm{\frac{\partial f_\theta}{\partial t}(t, y(t)) + \frac{\partial f_\theta}{\partial y}(t, y(t)) f_\theta(t, y(t))}_2^2
\end{equation*}
and so we may accomplish similar goals by regularising
\begin{equation*}
	\int_0^T \norm{\frac{\partial f_\theta}{\partial y}(t, y(t))}_2^2 \,\dd t\qquad\text{and}\qquad
	\int_0^T \norm{f_\theta(t, y(t))}_2^2 \,\dd t.
\end{equation*}
Because this is a CNF, we are already computing derivatives of $f$. This means that the Jacobian (left) expression may be computed very cheaply, without additional calls to autodifferentiation. (This is also the reason that the $\nicefrac{\partial f_\theta}{\partial t}$ term is neglected, as computing that would require an additional autodifferentiation operation.)

If using exact Jacobian computations then the Jacobian expression may be evaluated directly using the already-computed Jacobian. If using Hutchinson's trace estimator, then letting $A = \frac{\partial f_\theta}{\partial y}(t, y(t))$ the following expression applies:
\begin{equation*}
	\norm{A}_2^2 = \trace\left(A A^\top\right) = \expect_{\varepsilon \sim \normal{0}{\eye{d}}}\, \varepsilon A A^\top \varepsilon = \expect_{\varepsilon \sim \normal{0}{\eye{d}}} \norm{\varepsilon A}_2^2.
\end{equation*}

\subsection{Exploiting the structure of adaptive step size controllers}\label{section:numerical:adaptive}
For simplicity we will now focus on explicit embedded Runge-Kutta methods as a class of numerical ODE solvers. As per Section \ref{section:numerical:typical-solvers} these include many of the typical solvers used for neural differential equations, like Heun's method or Dormand--Prince.

\begin{remark}
	The discussions of this section will often actually apply to other solvers and differential equation types too. For example both the reversible Heun method and asynchronous leapfrog method (Section \ref{section:numerical:reversible-solvers}) are very similar to Runge--Kutta methods.
\end{remark}

Such solvers may be decomposed into two main components: an update rule (defined by a Butcher tableau \cite{hairer}), and a step size controller for updating the step size. (Which may simply be to use a constant step size.)

The update rule is typically the better-advertised component of a solver. Here we will instead focus on how the step size controller may be used or modified to our advantage.

We begin with a brief exposition of how step sizes are adjusted; see \cite{rackauckas-diffeq-stepping}, \cite[Section II.4]{hairer}, \cite[Section 271]{butcher2016numerical} for reference.

\paragraph{Set-up}
We begin with the usual setup. Let $y_0 \in \reals^d$, $\theta \in \reals^m$. Let $f_\theta \colon [0, T] \times \reals^d \to \reals^d$ be uniformly Lipschitz and continuously differentiable, and let $y \colon [0, T] \to \reals^d$ solve
\begin{equation}\label{eq:numerical:not-an-ode-forward}
	y(0) = y_0, \qquad \frac{\dd y}{\dd t}(t) = f_\theta(t, y(t)).
\end{equation}

Let $y_j \approx y(t_j)$ be some numerical approximation to the solution of \eqref{eq:numerical:not-an-ode-forward}. Over a step size $t_{j+1} - t_j$, then a numerical ODE solver may propose some $y_{j+1}^\text{candidate} \approx y(t_{j+1})$, along with a local error estimate $y_{j+1}^\text{error} \in \reals^d$ of the numerical error made in each channel during that step.

\paragraph{Scale and error ratios}
Given some prespecified absolute tolerance $\text{ATOL}$ (for example $10^{-6}$) and relative tolerance $\text{RTOL}$ (for example $10^{-3}$) and (semi)norm $\norm{\,\cdot\,} \colon \reals^d \to [0, \infty)$ (for example $\norm{y} = \sqrt{\tfrac{1}{d}\sum_{k=1}^d y_k^2}$ the RMS norm), then an estimate of the \textit{scale} of the equation is given by
\begin{equation*}
	\text{SCALE} = \text{ATOL} + \text{RTOL} \,\max(y_j, y_{j+1}^\text{candidate}) \in \reals^d
\end{equation*}
with an elementwise maximum. The \textit{error ratio} $r$ is then computed as
\begin{equation*}
	r = \norm{\frac{y_{j+1}^\text{error}}{\text{SCALE}}} \in \reals
\end{equation*}
with an elementwise division.

Note the dependence on the choice of norm $\norm{\,\cdot\,}$. In particular this determines the relative importance of each channel.

\paragraph{Accepting/rejecting steps}
If $r \leq 1$ then the error is deemed acceptable, the step is accepted and $y_{j+1} = y_{j+1}^\text{candidate}$ is taken. If $r > 1$ then the error is deemed too large, the step is rejected and the procedure is repeated with a smaller step size.

\paragraph{Step size changes}
Regardless of whether the step is accepted or rejected, then the next step size ($t_{j+2} - t_{j+1}$ or $t_{j+1} - t_{j}$ if the step was accepted or rejected respectively) is selected based on the size of $\text{SCALE}$.

For example the step size may be updated by a multiplicative factor
\begin{equation}\label{eq:numerical:step-size-factor}
	\max\left(\min\left(\frac{\text{SAFETY}}{\text{SCALE}^{1/\text{ORDER}}}, \text{IFACTOR}\right), \text{DFACTOR}\right)
\end{equation}
where $\text{ORDER}$ refers to the order of the solver (2 for Heun, 5 for Dormand--Prince and so on), and $\text{SAFETY}$, $\text{IFACTOR}$, $\text{DFACTOR}$ are hyperparameters. Typical values would be $\text{SAFETY} = 0.9$, $\text{IFACTOR} = 10$, $\text{DFACTOR} = 0.2$.

This is the `textbook' step size controller, which is memoryless (each multiplicative factor is dependent only on the previous step). Other step size controllers, often with memory, may also be considered \cite{rackauckas-diffeq-stepping}, \cite[Section 271]{butcher2016numerical}.

This will be all the necessary background material we need on step size controllers.

\subsubsection{Not-an-ODE and adjoint seminorms}\label{section:numerical:not-an-ode}\index{Adjoint seminorms}\index{Not-an-ODE}
Consider specifically when training neural ODEs via optimise-then-discretise, in which a backward-in-time adjoint ODE is constructed. The particular structure of the continuous adjoint equations actually means that the usual choice of norm for computing the error ratio, such as the RMS norm, is unnecessarily stringent: steps are unnecessarily rejected, and step sizes are too small \cite{kidger2020hey}.

By replacing it with a more appropriate (semi)norm, then on the backward pass:
\begin{enumerate}
	\item Fewer steps are rejected overall;
	\item Fewer steps are accepted overall;
	\item Fewer steps are rejected, as a proportion of the overall number of steps.
\end{enumerate}
That fewer steps are both accepted and rejected corresponds to generally larger step sizes being used. Moreover, this occurs without adversely impacting model performance. 

\paragraph{Continuous adjoint equations}\index{Optimise-then-discretise!ODEs}
For convenience we begin by recalling the set-up for backpropagating via optimise-then-discretise.

Let $L = L(y(T))$ be some (for simplicity scalar) function of the terminal value $y(T)$, so that the continuous adjoint equations (Theorem \ref{theorem:ode-adjoint}) correspond to $a_y \colon [0, T] \to \reals^d$ and $a_\theta \colon [0, T] \to \reals^m$ solving
\begin{align}
	a_y(T) &= \frac{\dd L}{\dd y(T)},\hspace{1.6em} &\frac{\dd a_y}{\dd t}(t) = -a_y(t)^\top \frac{\partial f_\theta}{\partial y}(t, y(t)),\nonumber\\
	a_\theta(T) &= 0,\hspace{1.6em} &\frac{\dd a_\theta}{\dd t}(t) = -a_y(t)^\top \frac{\partial f_\theta}{\partial \theta}(t, y(t)),\label{eq:numerical:not-an-ode-adjoint}
\end{align}
which are solved backward-in-time from a terminal condition.

\paragraph{Integral, not an ODE}
The continuous adjoint equations exhibit certain structure: their vector fields are independent of $a_\theta$, and correspondingly the second equation in \eqref{eq:numerical:not-an-ode-adjoint} is merely an integral: not an ODE. (See also Remark \ref{remark:numerical:not-an-ode}.)

As such, whilst it is convenient to evaluate the $a_\theta$ component of \eqref{eq:numerical:not-an-ode-adjoint} as part of the backward-in-time ODE solve, the ODE solver makes the false assumption that small errors in $a_\theta$ may propagate to create larger errors later.

\paragraph{Adjoint seminorms}
When numerically solving \eqref{eq:numerical:not-an-ode-adjoint} backward-in-time, the easy solution is to pick a choice of $\norm{\,\cdot\,}$ that scales down the influence of the $a_\theta$ channels. A simple such choice is to take $\norm{\,\cdot\,}$ as a seminorm, such as $\norm{(y, a_y, a_\theta)} = \sqrt{\frac{1}{2d}\sum_{k=1}^d (y_k^2 + a_{y, k}^2)}$ the RMS norm over the $y$ and $a_y$ components, and independent of the $a_\theta$ component. (Recall that the $y$ component is often solved backward-in-time alongside \eqref{eq:numerical:not-an-ode-adjoint}.)

\paragraph{Does this reduce the accuracy of the parameter gradients?}
One obvious concern is that we are ultimately interested in the parameter gradients $a_\theta(0)$, in order to train a model. In this respect, this approach seems counter-intuitive. Empirically this does not appear to negatively affecting training, however -- we explain this by noting that as the $y$ and $a_y$ channels truly are ODEs, they are likely to be the dominant source of error overall.

\paragraph{Results}
This can dramatically reduce the computational cost of training. \cite{kidger2020hey} reduce the cost of the backward pass through neural CDEs and Hamiltonian neural networks (Chapter \ref{chapter:neural-cde}, Section \ref{section:ode:hybrid}) by 40\%--62\%, and (less dramatically) through CNFs (Section \ref{section:ode:cnf}) by 5\%.

\paragraph{Quadrature}
Other methods for evaluating $a_\theta$ may also be admitted -- for example, whilst it is less convenient than simply using an already-existing ODE solver, $a_\theta$ could also be evaluated using a quadrature rule \cite[Section 2.5]{cvode}.

\subsubsection{Non-backpropagation through adaptive step size controllers}
Consider backpropagation via discretise-then-optimise. Technically speaking, we should expect to backpropagate through the entire computational graph, including through updates to step sizes, and through rejected steps.

Even rejected steps will in principle have a small effect on the backpropagated gradients. Every step (accepted or rejected) is used as an input to $\text{SCALE}$, which is used to compute the multiplicative factor by which a step size is updated (equation \eqref{eq:numerical:step-size-factor}), which determines the timestep values $\{t_j\}_{j=1}^n$, which in general may be used as an input to the neural vector field $f_\theta$.

In practice this is not always desirable. Backpropagating through rejected steps implies additional computational work \cite{aca}, and anecdotally we have observed that backpropagating through equation \eqref{eq:numerical:step-size-factor} will sometimes introduce gradient pathologies that hinder training.

For this reason it is very common not to backpropagate through step size selection -- when differentiating the computational graph we treat the result of equation \eqref{eq:numerical:step-size-factor} as a constant.\footnote{In PyTorch this means applying \texttt{detach} to the output of equation \eqref{eq:numerical:step-size-factor}; in JAX or TensorFlow this means applying \texttt{stop\_gradient}.} Neither the \texttt{torchdiffeq} nor Diffrax software libraries backpropagate through step size selection, for example \cite{torchdiffeq, diffrax}.

\subsubsection{Regularising error estimates}
\cite{pal2021opening} seek to encourage easy-to-integrate dynamics by adding
\begin{equation*}
	\sum_j y_j^\text{error} \abs{(t_{j+1} - t_j)}
\end{equation*}
as a regularisation term when solving a neural ODE. (\cite{pal2021opening} also consider a variant of this, by regularising a term used for detecting stiffness of the differential equation.)

This is computationally almost free, as all $y_j^\text{error}$ will already have been computed.

This technique relies on optimising the neural ODE via discretise-then-optimise (or with a bit of work, a reversible solver). If using optimise-then-discretise then the computational graph for computing $y_j^\text{error}$ is not saved for later backpropagation.

\section{Numerical simulation of Brownian motion}\label{section:brownian-interval}\index{Brownian!Motion}
Numerically solving an SDE requires sampling a Brownian motion $w \colon [0, T] \to \reals^{d_w}$.

\paragraph{Brownian bridges}\index{Brownian!Bridge}
Mathematically, sampling Brownian motion is straightforward. A fixed-step numerical solver may simply sample independent Gaussian random variables during its time stepping. An adaptive solver (which may reject steps) may use L{\'e}vy’s Brownian bridge formula \cite{revuz-yor} to generate the appropriate correlations: for any $s < t < u$,
\begin{equation}\label{eq:bbridge}
	w(t)|(w(s),w(u)) \sim \normal{w(s) + \frac{t - s}{u - s}(w(u) - w(s))}{\frac{(u - t)(t - s)}{u - s}\eye{d_w}}
\end{equation}
and this quantity is (conditionally) independent of $w(v)$ for $v < s$ or $v > u$.

\paragraph{Brownian reconstruction}\index{Brownian!Reconstruction}
However, there are computational difficulties. The main one is that during backpropagation, the same Brownian sample as the forward pass must be used, and if using optimise-then-discretise (Section \ref{section:adjoint-nsde}) may potentially be queried at locations other than were sampled on the forward pass.

In addition, we need to efficiently track the value of the Brownian motion at each end of any interval we will later need to condition on. (To apply the Brownian bridge formula, for example when rejecting steps.)

\paragraph{Brownian sampling}
We will now see three approaches to handling this: the Brownian Path, the Virtual Brownian Tree, and the Brownian Interval.\footnote{These choices of terminology are not completely standard; we adopt the names used in the \texttt{torchsde} library \cite{torchsde}.} The Brownian Path and Virtual Brownian Tree are included as `warm-ups' for pedagogical purposes; in practice the Brownian Interval will usually be the go-to choice.

\subsection{Brownian Path}\index{Brownian!Path}
One approach is simply to store every sample, and apply equation \eqref{eq:bbridge} when appropriate. There are some questions about the optimal data structure to store these values in, for efficient querying later -- in practice the tree-like structure we will later introduce for the Brownian Interval is often a good choice -- but otherwise there is little to discuss here.

This approach is simple and usually gets the job done. During an SDE solve, querying takes $\bigO{1}$ time (assuming a suitable data structure, see the Brownian Interval later). The main downside is the consumption of $\bigO{d_wT}$ memory.

\subsection{Virtual Brownian Tree}\index{Brownian!Tree}
The memory cost of the previous approach can sometimes be large enough to be a concern. This is especially true when taking many small steps to solve the SDE, or when using the continuous adjoint method or reversible SDE solvers (Sections \ref{section:adjoint-nsde} and \ref{section:numerical:reversible-solvers}) for which the Brownian motion samples represent a higher proportion of the overall memory usage.

As such, \cite{scalable-sde}, motivated by \cite{gaines}, introduce the `Virtual Brownian Tree'.

\paragraph{Splittable PRNGs}\label{section:numerical:splittable-prng}\index{Splittable PRNGs}
The first key ingredient is `splittable' pseudo-random number generator (PRNG) seeds \cite{prng1, prng2}.

Given an $m$-bit random seed $\rho \in \{0, 1\}^m$, splitting is an operation that produces some $n$ new $m$-bit random seeds $\rho_1, \ldots, \rho_n \in \{0, 1\}^m$, as a deterministic function of $\rho$, for which $\rho, \rho_1, \ldots, \rho_n$ produce statistically independent streams of random numbers when used as the seed for a PRNG.

Given any rooted tree, we can associate a random seed with every node in the tree in the following way.

Let $(V, E, *)$ be a rooted tree, where $V$ denotes some vertex set,
\begin{equation*}
E \subseteq \set{\{x, y\}}{x, y \in V, x \neq y}
\end{equation*}
denotes some edge set (connected and without cycles), and $* \in V$ denotes the root. For any $x \in V$, let $\adj(x) = \set{y \in V}{\{x, y\} \in E}$ denote the set of vertices adjacent to $x$.

Let $\rho \in \{0, 1\}^m$ be an $m$-bit seed, which we associate with the root $*$. Split $\rho$ into $\rho_1, \ldots, \rho_{|\adj(*)|}$ random seeds and pair each one with a corresponding element $v_i \in \adj(*)$. Recursively split each $\rho_i$ and pair the resulting seeds with the elements of $\adj(v_i)$, and so on, recursing this procedure throughout the tree.

By fixing a rooted tree $(V, E, *)$ and a root seed $\rho$, we may deterministically create a PRNG at every node in the tree. Provided we remember only the tree structure and the root seed $s$, we can later rematerialise every PRNG sequence, for every node of the tree, without holding the samples in memory.

\begin{example}
	A function call graph is an example of such a rooted tree. We begin by calling a function. This in turn may call other functions, which in turn may call other functions, and so on -- which we consider a tree, rather than a DAG, by treating multiple calls to the same function separately. Splittable PRNGs may be used to deterministically generate pseudorandomness at any point in this call graph, for example when writing pure functions. This is actually the procedure used throughout the JAX software library \cite{jax2018github} whenever generating random samples is required.
\end{example}

\paragraph{Generating Brownian samples}
Let $\varepsilon > 0$ be some fixed (small) tolerance. Consider the collection of dyadic points $V = \set{Tj2^{-k}}{j, k \in \naturals, T2^{-k + 1} > \varepsilon}$. These form a tree-like structure: each $Tj2^{-k}$ has $T\left\lceil j / 2 \right\rceil 2^{-k + 1}$ as its parent.

By recording only some root-level seed $\sigma \in \{0, 1\}^m$, and associating seeds with the elements of this tree as in the previous heading, then a Brownian sample $w(v)$ is completely determined for all $v \in V$: see Algorithm \ref{alg:numerical:btree}.

\begin{algorithm}[h]
	\SetKwInput{kwInput}{Input}
	\SetKwProg{Def}{def}{:}{}
	\SetAlgoVlined
	\kwInput{Time horizon $T > 0$, sample time $\tau \in [0, T]$, seed $\rho$, error tolerance $\varepsilon > 0$}
	\KwResult{Approximation to $w(\tau)$}
	\quad\\
	
	$\rho$, $\widehat{\rho}$ = \texttt{split\_seed}($\rho$)\\
	$w_T \sim \normal{0}{T\eye{d_w}}$ sampled with seed $\widehat{\rho}$\\
	$s = 0$\\
	$t = T / 2$\\
	$u = T$\\
	$w_s = 0$\\
	$w_t \sim \texttt{bridge}(0, T/2, T, 0, w_T)$ sampled with seed $\rho$\\
	$w_u = W_T$\\
	\quad\\
	
	\While{$\abs{\tau - t} > \varepsilon$}{
		$\rho_1, \rho_2$ = \texttt{split\_seed}($\rho$)\\
		\eIf{$\tau > t$}{
			$s = t$\\
			$w_s = w_t$\\
			$\rho = \rho_1$
		}{
			$u = t$\\
			$w_u = w_t$\\
			$\rho = \rho_2$
		}
		$t = (s + u) / 2$\\
		$w_t \sim \texttt{bridge}(s, t, u, w_s, w_u)$ sampled with seed $\rho$\\
	}
	\quad\\
	
	return $w_t$
	
	\caption{Sampling the Virtual Brownian Tree. \texttt{split\_seed} denotes splitting a seed into two, and \texttt{bridge} denotes the Brownian bridge of equation \eqref{eq:bbridge}.}\label{alg:numerical:btree}
\end{algorithm}

For $x \in [0, T]$ let $[x]_V$ denote the member of $V$ closest to $x$. To approximately sample a Brownian increment $w(t) - w(s)$ (when an SDE solver steps from $s$ to $t$), we first discretise $s$ and $t$ to $[s]_V$ and $[t]_V$, sample $w([s]_V)$ and $w([t]_V)$ as in Algorithm \ref{alg:numerical:btree}, and then return $w([t]_V) - w([s]_V)$. The main downsides are that this takes $\bigO{\log(1/\varepsilon)}$ time, and produces only approximate samples. However, it has the advantage that this requires only $\bigO{1}$ memory.

Whilst sampling Brownian motion is easy, the key point of this construction is how it additionally allows for \textit{reconstructing} the same Brownian motion sample, without holding the individual samples in memory.

\subsection{Brownian Interval}\index{Brownian!Interval}\label{section:numerical:brownian-interval}
We are now ready to present the Brownian Interval, which improves upon the Brownian Tree with exact sampling and $\bigO{1}$ query times.

\subsubsection{Overview}

\paragraph{Sampling intervals}
Let $w(s, t)$ denote $w(t) - w(s) \in \reals^w$.

We begin by shifting from a point-evaluation approach, in which each query to the Brownian object produces some $w(t)$, to an interval-evaluation approach, in which each query to the Brownian object generates some $w(s, t)$.

Rewriting the Brownian bridge equation \eqref{eq:bbridge} gives
\begin{equation}\label{eq:numerical:binterval-bbridge}
	w(s, t)|w(s, u) \sim \normal{\frac{t - s}{u - s}w(s, u)}{\frac{(u - t)(t - s)}{u - s}\eye{d_w}}.
\end{equation}
The complement $w(t, u)|w(s, u)$ is calculated as $w(t, u)|w(s, u) = w(s, u) - w(s, t)|w(s, u)$.

\paragraph{Binary tree of (interval, seed) pairs}
Similar to the binary tree of (point, seed) pairs used in the Brownian Tree, we will now have a binary tree of (interval, seed) pairs. Each parent interval will be the disjoint union of its child intervals.

The tree starts as a stump consisting of the global interval $[0, T]$ and an $m$-bit random seed $\rho$. New leaf nodes are created as queries over intervals are made. For example, making a first query at $[s, t] \subseteq [0, T]$ (an operation that will return $w(s, t)$) produces the binary tree shown in Figure \ref{fig:binterval1}; making a subsequent query at $[u, v]$ with $u < s < v < t$ produces Figure \ref{fig:binterval2}. Using a splittable PRNG as in Section \ref{section:numerical:splittable-prng}, each child node has a random seed deterministically produced from the seed of its parent. Unlike the Virtual Brownian Tree, which has a fixed (dyadic) tree construction, the tree used in the Brownian interval is query-dependent.

\begin{figure}\centering
	\begin{subfigure}[b]{0.49\linewidth}\centering
		\begin{tikzpicture}[scale=0.7, every node/.style={inner sep=0.5mm,outer sep=0.5mm}]
			\draw (0, 0) -- (-1.2, -1);
			\draw (0, 0) -- (1.35, -1);
			\draw (1.35, -1) -- (0.5, -2);
			\draw (1.35, -1) -- (2.2, -2);
			
			\draw (0, 0) node[fill=white] {$[0, T]$};
			\draw (-1.2, -1) node[fill=white] {$[0, s]$};
			\draw (1.2, -1) node[fill=white] {$[s, T]$};
			\draw (0.5, -2) node[fill=white] {$[s, t]$};
			\draw (2.2, -2) node[fill=white] {$[t, T]$};
		\end{tikzpicture}
		\caption{}\label{fig:binterval1}
	\end{subfigure}
	\begin{subfigure}[b]{0.49\linewidth}\centering
		\begin{tikzpicture}[scale=0.7, every node/.style={inner sep=0.5mm,outer sep=0.5mm}]
			\draw (0, 0) -- (-1.2, -1);
			\draw (0, 0) -- (1.35, -1);
			
			\draw (1.35, -1) -- (0.5, -2);
			\draw (1.35, -1) -- (2.2, -2);
			
			\draw (-1.2, -1) -- (-2.2, -2);
			\draw (-1.2, -1) -- (-0.5, -2);
			
			\draw (0.5, -2) -- (-0.3, -3);
			\draw (0.5, -2) -- (1.3, -3);
			
			\draw (0, 0) node[fill=white] {$[0, T]$};
			
			\draw (-1.2, -1) node[fill=white] {$[0, s]$};
			\draw (1.35, -1) node[fill=white] {$[s, T]$};
			
			\draw (0.7, -2) node[fill=white] {$[s, t]$};
			\draw (2.4, -2) node[fill=white] {$[t, T]$};
			
			\draw (-2.2, -2) node[fill=white] {$[0, u]$};
			\draw (-0.5, -2) node[fill=white] {$[u, s]$};
			
			\draw (-0.1, -3) node[fill=white] {$[s, v]$};
			\draw (1.5, -3) node[fill=white] {$[v, t]$};
		\end{tikzpicture}
		\caption{}\label{fig:binterval2}
	\end{subfigure}
	\caption{Binary tree of intervals. Only the intervals, without the corresponding seeds, are shown.}
\end{figure}
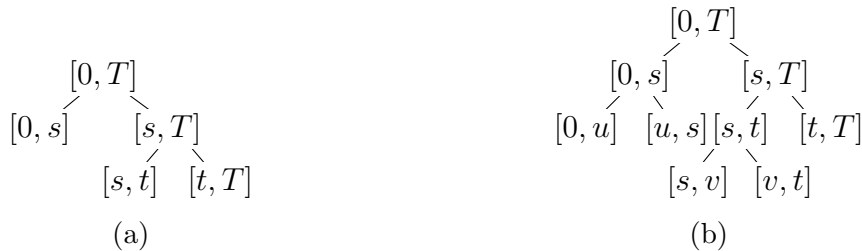

The tree thus completely encodes the conditional statistics of a Brownian motion, conditional on all previous queries: $w(s, t), w(t, u)$ are completely specified by $s$, $t$, $u$, $w(s, u)$, equation \eqref{eq:numerical:binterval-bbridge}, and the random seed associated with $[s, u]$.

\paragraph{Generating Brownian samples}

In principle we may now calculate $w(s,t)$ for any $s < t$. The query over $[s, t]$ adds extra nodes to the tree (if necessary; $[s, t]$ may have been queried before), so that the conditional statistics of the query, with respect to all previous queries, are captured. As in Figure \ref{fig:binterval2}, this may decompose $[s, t]$ into some disjoint union of subintervals. We then calculate $w(s,t)$ by applying equation \eqref{eq:numerical:binterval-bbridge} to each subinterval.

This calculation does require the Brownian increment $w(s, u)$ over the parent interval $[s, u]$. In principle this is calculated recursively in the same way, working our way up the tree. (As with the Virtual Brownian Tree.) However, this may be improved by adding a least recently used (LRU) cache to the computed increments $w(s, t)$.

Queries are exact because the tree aligns with the query points. Queries are fast because of the LRU cache: in SDE solvers, subsequent queries are likely to be close to (and thus conditional on) previous queries. The average-case (modal) time complexity is thus $\bigO{1}$. Even in the event of cache misses all the way up the tree, the worst-case time complexity will only be $\bigO{\log(1/h)}$ in the average step size $h$ of the SDE solver. The (GPU) memory cost is essentially the size of the LRU cache, which is constant and thus $\bigO{1}$.

The trade-off here is that we must store the tree structure itself, which grows each time a query is made. For an SDE solve on $[0, T]$ then this will consume $\bigO{T}$ CPU memory. In practice however this is unlikely to be a limitation: GPU memory is usually the limiting factor, in comparison to which CPU memory is essentially infinite.

\subsubsection{Algorithmic definitions and further discussion}
Precise algorithmic definitions and (substantial) further discussion on the Brownian Interval is deferred to Appendix \ref{appendix:proofs:binterval} to avoid breaking the flow of the presentation.

\begin{remark}
	To what extent may either the point-based approach of the Virtual Brownian Tree, or the interval-based approach of the Brownian Interval, be interchanged?
	
	An interval-based approach to the Virtual Brownian Tree is possible, but would likely be inefficient. For small approximation tolerance $\varepsilon$ and step sizes possibly much larger than $\varepsilon$, then a query for some interval $[s, t]$ would require logarithmically many dyadic intervals (each of length $Tj2^{-k}$ for some $j, k$), to construct $[[s]_V, [t]_V]$. This is as opposed to just making two point-based queries. (The Brownian Interval largely avoids this issue with its query-dependent trees.)
	
	A point-based approach to the Brownian Interval is possible (despite the name), but the interval-based approach has several upsides.
	\begin{itemize}
		\item Elegance. We directly query for the sample $w(s,t)$ actually used in the SDE solver.
		\item Efficiency. Making only a single query for $w(s,t)$ as opposed to making two queries for $w(s)$ and $w(t)$.
		\item L{\'e}vy area approximation. Some numerical SDE solvers sample additional randomness beyond just the point evaluations $w(t)$: typically these samples are of higher-order integrals $\int_s^t w_{k_1}(u) \,\dd w_{k_2}(u)$, computed over intervals $[s, t]$. For example stochastic Runge--Kutta methods may require space-time L{\'e}vy area samples, and the log-ODE method for SDEs\footnote{This is the same log-ODE method as seen in Appendix \ref{appendix:neural-rde}. The additional L{\'e}vy area terms used here correspond to the logsignature terms used there.} uses full L{\'e}vy area samples. As they are defined over intervals $[s, t]$, these quantities are intrinsically interval-based values, requiring an interval-based Brownian motion construction. The details of this are a topic beyond our scope here; James Foster's doctoral thesis \cite{fosterthesis} introduce the requisite formulas analogous to \eqref{eq:numerical:binterval-bbridge}, and the necessary extensions to the Brownian Interval are implemented in \normalfont{\texttt{torchsde}} \cite{torchsde}.
	\end{itemize}
\end{remark}

\section{Software}\label{section:numerical:software}\index{Software}
Software packages for the numerical solving and training of neural differential equations are now relatively standardised. They handle most of the details described over the course of this chapter, and correspondingly the user is free to focus on the modelling details that have been the focus on the other chapters in this thesis.

At time of writing, there are a selection of options.
\begin{itemize}
\item In the JAX ecosystem \cite{jax2018github} there is Diffrax.\index{Diffrax}
\begin{itemize}
	\item \url{https://github.com/patrick-kidger/diffrax}
\end{itemize}
\item In the PyTorch ecosystem \cite{pytorch} there is the \texttt{torchdiffeq}, \texttt{torchcde}, and \texttt{torchsde} family of libraries. (And additionally \texttt{torchdyn} as a higher-level wrapper providing some common models.)
\begin{itemize}
	\item \url{https://github.com/rtqichen/torchdiffeq}\index{torchdiffeq}
	\item \url{https://github.com/patrick-kidger/torchcde}\index{torchcde}
	\item \url{https://github.com/google-research/torchsde}\index{torchsde}
	\item \url{https://github.com/DiffEqML/torchdyn}\index{torchdyn}
\end{itemize}
\item In the Julia \cite{Julia-2017} ecosystem there is DifferentialEquations.jl.\index{DifferentialEquations.jl}
\begin{itemize}
	\item \url{https://github.com/SciML/DifferentialEquations.jl}
\end{itemize}
\end{itemize}

Every package we recommend is open source, offers a stable API, is relatively feature-complete, and comes with comprehensive documentation and examples -- including code examples for many of the techniques discussed in this thesis.

Whilst exact functionality differs slightly by package, one can expect most of:
\begin{enumerate}
	\item Explicit and implicit solvers;
	\item Fixed and adaptive step size solvers;
	\item Differentiation via both optimise-then-discretise and discretise-then-optimise;
	\item Reversible differential equation solvers;
	\item Event handling;
	\item Callbacks;
	\item Handling of jumps in the vector field;
	\item For neural CDEs: all interpolation schemes discussed here;
	\item For neural RDEs: logsignature pre-processing as in Appendix \ref{appendix:neural-rde};
	\item Solving of both It{\^o} and Stratonovich SDEs;
	\item Solving SDEs with varying noise types (scalar, additive, diagonal, general);
	\item Brownian Interval simulation as in Section \ref{section:brownian-interval};
	\item Levy area approximation;
	\item Gradient checkpointing;
	\item Distributed computing;
	\item CPU parallelism;
	\item GPU support.
\end{enumerate}

Any choice amongst these libraries is a reasonable one.

\begin{remark}
If the reader is free to choose, then we would recommend Diffrax. It is the newest of these libraries, and is quite exciting on a technical level, as it solves ODEs, CDEs, and SDEs in a unified way by internally lowering all of them to rough differential equations. In addition if working with irregular time series, then it is the only one amongst these libraries to offer the ability to batch over different regions of integration. We must admit to some bias -- Diffrax is the author's own project, created whilst writing this thesis.
\end{remark}

\section{Comments}
The choice of discretise-then-optimise versus optimise-then-discretise backpropagation is a classical one in the context of differential equations. \cite{anode} is the canonical reference on this topic in the context of neural ODEs. See also \cite{rackauckas2018comparison, onken2020discretize} for related comparisons.

Reversible differential equation solvers, as applied to backpropagation, are quite a new topic. \cite{alf, zhuang2021mali} introduce a reversible ODE solver, whilst \cite{kidger2021sde2} introduce the first reversible SDE solver.

The broader comparison of discretise-then-optimise against optimise-then-discretise against reversible differential equation solvers is new here. (And mistakes on this topic are frequent in the literature; we have come across several erroneous statements about favouring optimise-then-discretise over discretise-then-optimise, in contexts where the opposite is true.)

The proof of optimise-then-discretise for ODEs (relegated to Appendix \ref{appendix:ode-adjoint}), and its sketchproof (Section \ref{section:numerical:ode-adjoint-sketchproof}), are new here. To the best of our knowledge the existing literature has relied on only more complicated proofs.

The proofs of optimise-then-discretise for CDEs and SDEs (relegated to Appendices \ref{appendix:cde-adjoint} and \ref{appendix:sde-adjoint}) are new here. Once again to the best of our knowledge, the existing literature has relied on (substantially) more complicated proofs. A proof of optimise-then-discretise for SDEs appeared in \cite{scalable-sde}. A proof of optimise-then-discretise for CDEs (along with a rough path theory proof of optimise-then-discretise for SDEs) first appeared in \cite{kidger2020sdeunpublished}, although this was never published.


The discussion on the choice of numerical solver is part of the folklore of neural differential equations; our presentation here is based on our own anecdotal experience and our conversations with others.

Baked-in discretisations (both that they occur and that they are acceptable) are again part of the folklore, although they have been explicitly studied in \cite{ott2021resnet, continuous-net}.

The terminology of analytic and algebraic reversibility is new here. Some existing texts do refer to just `reversible solvers', usually in the context of symplectic solvers.

The more efficient backward step for the asynchronous leapfrog method (Algorithm \ref{alg:numerical:reverse-alf}) is new here. (\cite{zhuang2021mali} used the more general, less efficient, Algorithm \ref{alg:numerical:reversible-backward}). It is actually also possible to construct a more efficient backward step through the reversible Heun method as well, so as to elide the local forward operation. It is however more finicky to do so -- the local backward needs to occur on the reverse pass of the \textit{previous} step -- so for simplicity we omit this here.

On the stability of the asynchronous leapfrog method: \cite{zhuang2021mali} do additionally introduce a `damped asynchronous leapfrog method' with nontrivial region of stability. In practice the region of stability remains very small, and one of the main advantages of reversible solvers is the ability to use them with very large step sizes,\footnote{Whilst still getting both memory efficiency and accurate gradients; discretise-then-optimise giving only the latter and with large step sizes optimise-then-discretise giving only the former.} so the benefit of this is not clear. We note that \cite{zhuang2021mali} mangle terminology slightly by referring to a `region of A-stability' when merely `region of stability' or `region of absolute stability' would be correct. (`A-stability' is a property of the region of stability itself.)

The `not-an-ODE'/`adjoint seminorm' trick for improving backpropagation speed through neural ODEs may also be applied to forward sensitivities \cite[Section 5.5]{cvode}.

The comparison of different software libraries is new here. In fact, the Diffrax software library was written by the author for the express purpose of writing this thesis (or perhaps to procrastinate from writing this thesis). Realistically this has been a fast-moving space, and we would not be surprised if the section on software rapidly becomes outdated.
\chapter{Miscellanea}
\section{Symbolic regression}\label{section:misc:symbolic}\index{Symbolic regression}

\subsection{Introduction to symbolic regression}

Deep learning, including neural differential equations, typically produces `black-box' models. Once the model has been trained, it is a relatively opaque neural network whose mode of operation is essentially mysterious. It may be a good model, but a good model is not always the end goal. Scientific progress may be predicated upon understanding the model as well.

It is often desirable to obtain symbolic expressions -- an imprecise term which we use here to refer to some relatively shallow tree of primitive operations, for example $x \times ((y - 4.2) + z)$. These primitive operations typically include addition, multiplication, exponentiation and so on.

Symbolic regression is the process of deriving such expressions from data in an automated way. One difficulty is the lack of differentiability of the space of such expressions. Whilst any constant in the expression (such as the $4.2$ above) may be optimised differentiably, the space between expressions is usually traversed via genetic algorithms. Another difficulty is the size of this space: there are $\nicefrac{(2n)!}{(n+1)!n!}$ binary trees with $n$ vertices, and so as a rough approximation we may expect there to be a similar number of possible expressions to consider. This is a big number.

For these reasons, symbolic regression is a difficult task that often works best only on simple problems; past a certain point the complexity grows too large and the problem becomes intractable.

\subsection{Symbolic regression for dynamical systems}

\begin{example}\label{example:sindy}\index{SINDy}
	Suppose we observe paired samples of both $y(t)$ and $\nicefrac{\dd y}{\dd t}(t)$, assumed to satisfy an equation of the form
	\begin{equation*}
		\frac{\dd y}{\dd t}(t) = f(y(t)).
	\end{equation*}	
	Then SINDy \cite{sindy} seeks a symbolic expression for $f$ by selecting some features $f_i$ in advance, parameterising $f(y) = \sum_{i=1}^N \theta_i f_i(y)$, and directly regressing $\nicefrac{\dd y}{\dd t}(t)$ against $\{f_i(y(t))\}_{i=1}^N$. A sparsity penalty such as $L^1$-regularisation is applied to $\theta$ so that only a few terms are selected in the final expression.
	
	This procedure is simply standard LASSO, and the dynamical character of the problem is essentially irrelevant. SINDy is arguably the dominant technique for symbolic regression with dynamical systems; some example extensions and applications include \cite{sindy-pde, sindy-plasma, sindy-residual, sindy-stability, sindy-unified}.
	
	However, SINDy has made two strong assumptions: (a) that paired observations of both $y$ and $\nicefrac{\dd y}{\dd t}(t)$ are available, and (b) that $f$ is a shallow tree of expressions -- just a linear combination of preselected features.
\end{example}

We will now see how NDEs offer ways to remove both of the assumptions made in Example \ref{example:sindy}.

\paragraph{Removing assumption (a): no paired observations}

Suppose we observe samples $y(t)$ assumed to come from some dynamical system
\begin{equation}\label{eq:misc:snde}
	\frac{\dd y}{\dd t}(t) = f(y(t)),
\end{equation}
which for simplicity we assume is an autonomous ODE. (Although this is not necessary -- the same ideas apply equally well to non-autonomous dynamical systems, and to non-ODEs such as CDEs and SDEs.)

Given this data, we learn some $f = f_\theta$ as a neural network as described in the rest of this thesis. For example, minimising some empirical loss between the data and a numerical solution of the initial value problem, and optimising $f_\theta$ via backpropagation.

\begin{remark}
	Note that unlike SINDy, we have not assumed access to paired observations of both $y$ and $\nicefrac{\dd y}{\dd t}(t)$.
	
	SINDy sometimes works around this by approximating $\nicefrac{\dd y}{\dd t}(t)$ using finite differences. However this requires densely-packed observations, whilst the above procedure applies even when observations of $y(t)$ are sparse.
\end{remark}

\paragraph{Removing assumption (b): deep symbolic expressions}

Subsequently, we perform symbolic regression across the learnt $f_\theta$. That is, for each observed sample $y$ we evaluate $f_\theta(y)$, and symbolically regress $f_\theta(y)$ against $y$.

The symbolic regression itself may be performed in any number of ways. A reasonable choice for most tasks is regularised evolution \cite{regularised-evolution}, which is capable of learning complex trees of expressions, and traverses the space between them via genetic algorithms. Open source software libraries exist to perform this task -- at time of writing we recommend the PySR and SymbolicRegression.jl libraries \cite{pysr} for Python and Julia respectively.

More advanced techniques include \cite{deep-sym1, deep-sym2, deep-sym3, deep-sym4}. For example \textit{deep symbolic regression} introduces learnt neural network optimisers to tackle the task of searching through symbolic expressions.

(And if we really wanted, we could just take the simple approach of applying regularised linear regression against preselected features as in Example \ref{example:sindy} -- any symbolic regression technique will do.)

The result of this symbolic regression is our final result.

\begin{remark}\index{Markov assumption}
	Note the Markov assumption that is being made in equation \eqref{eq:misc:snde}: the vector field $f$ depends entirely on the observations $y$, so that there is no dependence on the past. In contrast observe that our typical set-up for NDEs has been to have the dynamical system operate in some latent space (that is, $y$ is hidden state), and then linearly project this space down to the data space. See for example Remark \ref{remark:ode:markov}, or the use of readout maps with neural CDEs and SDEs (Chapters \ref{chapter:neural-cde} and \ref{chapter:neural-sde}).
	
	Extending symbolic regression to the non-Markov setting is nontrivial. The essential difficulty is that symbolic regression in a latent space is not obviously meaningful. Supposing that the latent space is some $\reals^{d_y} \ni y(t)$, and letting $\mathcal{A} \subseteq \{\reals^{d_y} \to \reals^{d_y}\}$ be some collection of automorphisms of this space, then for any $\phi \in \mathcal{A}$, both $f$ and $\phi^{-1} \circ f \circ \phi$ define essentially the same dynamics. That is, we may only identify $f$ up to conjugacy by elements of $\mathcal{A}$. 
	
	We note that \cite{Champion22445} do consider symbolic regression in a latent space. The above problem is dealt with implicitly in an ad-hoc manner, by (a) using only simple symbolic regression techniques (LASSO as with SINDy) to constrain the complexity of the vector field; (b) relying on Lipschitz embeddings/decodings from the latent space, again to constrain complexity; (c) manually selecting the `best' element of the conjugacy class $\set{\phi^{-1} \circ f \circ \phi}{\phi \in \mathcal{A}}$ after training has completed. As such `latent symbolic regression' has seen some success, but is in many respects still an open problem. 
\end{remark}

\subsection{Example}\index{Examples!Symbolic regression}\label{section:misc:symbolic-example}
Let $T>0$ and consider the nonlinear oscillator
\begin{equation}\label{eq:misc:nonlinear-oscillator}
	\frac{\dd}{\dd t}\begin{bmatrix}x\\y\end{bmatrix}(t) = \begin{bmatrix}\frac{y(t)}{1 + y(t)}\\\frac{-x(t)}{1 + x(t)}\end{bmatrix} \quad\text{for $t \in [0, T]$},
\end{equation}
with $x(0), y(0) \sim \uniform{-0.6}{1}$. Samples from this equation resemble warped and deformed sines and cosines.

We aim to learn the symbolic form of this differential equation from data. Note the form of the vector fields, which would be very difficult to learn using SINDy.\footnote{Requiring for example the additional knowledge that the vector field is in fact a rational function \cite{rational-sindy}.}

\paragraph{Data} We fix some points $t_j \in [0, T]$, with $t_0 = 0$. For initial conditions $x(0), y(0) \sim \uniform{-0.6}{1}$ we assume access to observations of the corresponding $x(t_j), y(t_j)$. For simplicity we take the samples to be noiseless.

\paragraph{Neural regression} We train a neural ODE via the $L^2$ loss as described above or as in Chapter \ref{chapter:neural-ode}, to reconstruct $x(t_j), y(t_j)$ given $x(0), y(0)$. That is, we consider the model
\begin{equation*}
	\frac{\dd}{\dd t}\begin{bmatrix}x\\y\end{bmatrix} = f_\theta(x(t), y(t))
\end{equation*}
where $f_\theta \colon \reals^2 \to \reals^2$ is a neural network, and for each initial condition $(x(0), y(0))$ the above equation is solved as an initial value problem using a numerical ODE solver.

The result of training this neural ODE is shown in Figure \ref{fig:nonlinear-oscillator}. The model has perfectly learnt the structure of the problem.

\begin{figure}\centering
	\includegraphics[width=0.5\linewidth]{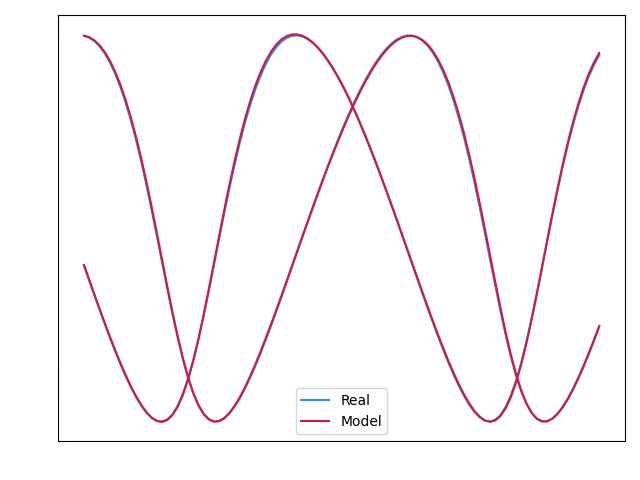}
	\caption{Sample, and trained reconstruction, of the nonlinear oscillator \eqref{eq:misc:nonlinear-oscillator}. Both dimensions $(x, y)$ of the oscillator are shown plotted against time $t$. Model and data align almost perfectly.}\label{fig:nonlinear-oscillator}.
\end{figure}

\paragraph{Symbolic regression} We now have access to a function $f_\theta \colon \reals^2 \to \reals^2$ which we may treat in isolation. The dynamic structure of the problem has been removed.

Symbolically regressing $f_\theta(x(t_j), y(t_j))$ against $(x(t_j), y(t_j))$ via regularised evolution produces the expression
\begin{equation}\label{eq:misc:symbolic-nonlinear-oscillator}
	\begin{bmatrix}x\\y\end{bmatrix} \mapsto \begin{bmatrix}\frac{y}{1.01 + y}\\\frac{-1.04 x}{1.03 + x}\end{bmatrix}.
\end{equation}

\paragraph{Neural-symbolic regression}
Finally, we treat \eqref{eq:misc:symbolic-nonlinear-oscillator} as the vector field of a `neural' differential equation and perform another round of gradient-based optimisation against the original dataset to optimise the constants in the symbolic expressions.

Rounding each constant to the nearest multiple of 0.01 then produces the desired vector field
\begin{equation}
	\begin{bmatrix}x\\y\end{bmatrix} \mapsto \begin{bmatrix}\frac{y}{1 + y}\\\frac{-x}{1 + x}\end{bmatrix}.
\end{equation}

\paragraph{Further details}
Further details may be found in Appendix \ref{appendix:experimental:symbolic}. The code is available as an example in Diffrax \cite{diffrax}.


\section{Limitations of neural differential equations}

We have spent most of this thesis discussing the numerous advantages and applications of neural differential equations. It is only fair we dedicate some space to their limitations.

\subsection{Data requirements}
Neural differential equations have one major difference to classical differential equations. Using a neural network as the vector field (or as a component of the vector field, as with UDEs), results in greatly increased model expressivity, and so correspondingly more data is needed to train the model.

As such data requirements are typically comparable to neural network based approaches. A few hundred samples represents an optimistic lower bound on the amount of data required. The toy example problems considered in this thesis use a few thousand samples. Some examples (\cite[Section 4.2]{kidger2021sde1}) use millions of samples.

This is a limitation compared to classical differential equations -- in return for which we receive more expressive models -- but compared to neural networks this is of course quite normal.

\subsection{Speed}
In particular when using higher-order differential equation solvers, which make multiple vector field evaluations, then neural differential equations can be somewhat slow to evaluate or train.

This can be mitigated by using cheaper, lower-order, solvers.\footnote{Which need not affect model efficacy. Model efficacy, and accuracy at solving the idealised differential equation, are two different things.} For example a low-order reversible solver (Section \ref{section:numerical:reversible-solvers}) can be used to obtain accurate gradients despite its low order.

Additionally, neural differential equations have a trick not available to standard neural networks: the choice of solver can be varied. For example a cheap low-order solver can be used for the bulk of training, and a more expensive higher-order solver used for fine-tuning and inference.

\subsection{Other discretised architectures}
There are successful neural network architectures not so easily explained by being discretised neural differential equations. For example neither U-Net \cite{unet} nor Transformers \cite{attention-is-all-you-need} admit obvious descriptions of this type, although Transformers do have a continuous theory of their own \cite{continuous-hopfield}.

Whilst neural differential equations are one (very successful) paradigm for constructing discrete architectures, it is apparent they are not the only one.

\section{Beyond neural differential equations: deep implicit layers}\index{Deep implicit layers}\index{Implicit!layers}
Neural differential equations are part of a larger family of models, known as \textit{deep implicit models} or \textit{deep implicit layers}.

Most layers (operations) used in machine learning models are `explicit': given an input $x$ they return an output $y$, as in
\begin{equation*}
	y = f_\theta(x),
\end{equation*}
where $f_\theta$ denotes some function depending on trainable parameters $\theta$.

In contrast, implicit layers take the form
\begin{equation}\label{eq:implicit}
	\text{Find }y\text{ such that }f_\theta(x, y) = 0.
\end{equation}
That is to say, the output is specified implicitly as one satisfying a certain condition.

This immediately opens up a host of questions -- such as uniqueness -- that we will not attempt to address in detail here. We aim only to give a high-level flavour of some of this broader family of models.

By its very nature, equation \eqref{eq:implicit} cannot usually be solved explicitly or in closed form. As such the common thread running through such models is the use of a numerical scheme to find an approximate solution to equation \eqref{eq:implicit}.

The word `implicit' thus takes on a dual meaning: not only is the solution $y$ specified implicitly, but the computational steps to locate it need not be explicitly specified either.

Backpropagation through such models is possible -- as in Section \ref{section:numerical:adjoint-ode}, there are both discretise-then-optimise and optimise-then-discretise approaches available. One option is to backpropagate through the operations of some numerical solver for \eqref{eq:implicit}. Another option is to apply the implicit function theorem, and then implicitly differentiate \eqref{eq:implicit} itself. Variations on this are discussed in \cite[Chapter 15]{griewank2008evaluating} and \cite{deq, blondel2021efficient, implicit-fixed}.

A recent line of work has begun to suggest that implicit models may consistently outperform explicit models \cite{implicit-behavioural-cloning, lu2021implicit, implicit-fixed}.

\subsection{Neural differential equations as implicit layers}

A neural ODE is an implicit model: it is specified as
\begin{equation}\label{eq:misc:node-as-implicit-one}
	\text{Find $y$ such that $y(0) = y_0$ and $\frac{\dd y}{\dd t}(t) = f_\theta(t, y(t)),$}.
\end{equation}

That the computational steps for computing it are left implicit is the very reason Chapter \ref{chapter:numerical} exists.

\subsection{Deep equilibrium models}\index{Deep equilibrium models}

`Deep Equilibrium Models' (DEQs) \cite{deq} are essentially another term for implicit modelling in general, although the term is often used to refer to the case in which $f_\theta$ takes the form of some `large' neural network architecture, such as a Transformer \cite{attention-is-all-you-need}, in which \eqref{eq:implicit} is solved via fixed-point iterations.

As before, given an input $x$ and a neural network $f_\theta$, the output $y$ is simply defined as 
\begin{equation}\label{eq:misc:deq-implicit}
	\text{Find }y\text{ such that }f_\theta(x, y) = 0.
\end{equation}
For example if $x$ and $y$ are sequences then taking $f_\theta$ to be a Transformer is a reasonable choice.

\cite{deq} consider applications to time series whilst \cite{multiscale-deq} consider applications to computer vision. Monotone Operator Equilibrium Models (monDEQ) impose additional structure to ensure that \ref{eq:misc:deq-implicit} has a unique solution \cite{mondeq}. Stability may be improved via regularisation \cite{regularised-deq}.

\subsection{Multiple shooting: DEQs meet NODEs}\index{Multiple shooting}\label{section:misc:multiple-shooting}

Let $d \in \naturals$ and let $f \in \Lip(\reals \times \reals^d; \reals^d)$. For $u < v$ and $x \in \reals^d$, let $y$ denote the solution to
\begin{equation*}
	y(0) = x,\qquad\frac{\dd y}{\dd t}(t) = f_\theta(t, y(t)) \text{ for $t \in [u, v]$}
\end{equation*}
and then let $\phi_\theta(x, u, v) = y(v)$ denote the map from initial condition to terminal condition.

Given an input $x \in \reals^d$, a time horizon $T > 0$, and time points $0 = t_0 < \cdots < t_n = T$, then \textit{multiple shooting} reframes the solution of an ODE as the solution to the implicit problem:
\begin{align*}
	\text{Find $b = (b_0, \ldots, b_n) \in \reals^d \times \cdots \times \reals^d$ such that $b_0 = x$ and $b_{j+1} = \phi_\theta(b_j, t_j, t_{j+1})$.}
\end{align*}

In many ways this is an implicit problem like any other -- for example, we may aim to solve it via a fixed-point iteration, perhaps via Newton methods. Each step of this fixed point iteration itself involves solving multiple ODEs, to evaluate each $\phi_\theta(b_j, t_j, t_{j+1})$.

The key advantage of this approach is that the solution to the multiple ODEs over each interval $[t_i, t_{i+1}]$ may be performed in parallel. Provided that only few steps of the fixed-point solver are required, then this can reduce the overall computation time. (Even though it might not necessarily reduce the overall computational work, due to parallelism.) This is generally the case provided a good initial guess for $b$ can be obtained.

\begin{example}
	For example this may be used during training. Fix a single batch of data, and train a neural ODE for several steps (of parameter optimisation, updating $\theta$) on the same batch of data.
	
	The first such step should obtain a solution via other numerical methods, for example as discussed in Chapter \ref{chapter:numerical}. This provides an initial value for each $b_j$. As the learnt parameters $\theta$ evolve only slightly over the course of each training step, these $b_j$ can be used to provide a good initial guess for subsequent parameter optimisation steps, for which the neural ODE is solved using multiple shooting.
	
	This may be used to train on the same batch of data for a few steps, before sampling a fresh batch.
\end{example}

\begin{example}
	Another example arises when using a neural ODE to model a fully-observed dynamical system. As it is a fully-observed dynamical system we may suppose it is Markov, and attempt to model the observations directly. (So that our neural ODE is `unaugmented' and does not evolve in a latent space as in Section \ref{section:ode:augmentation}.)
	
	Suppose further that each training sample consists of multiple observations $y(s_j)$ for $s_j \in [0, T]$. Then during training we may choose $t_j = s_{m_j}$ for some choice of $m_j$, and take $b_j = y(t_j) = y(s_{m_j})$.
\end{example}

\begin{example}
	A final example, this time during inference, is a classical use-case for multiple shooting: when using an ODE to provide forecasts into the future, continuously updated as new information arrives. Each forecast involves solving an ODE from the current time to some future time (for example, the time now plus ten minutes). As new data arrives we update our forecast, and the solution of the old forecast may be used to initialise each $b_j$.
\end{example}

See \cite{multiple-shooting-node} for more details.

\subsection{Differentiable optimisation}\index{Differentiable optimisation}

A final kind of implicit model is the solution to optimisation problems.

To be clear, whilst `optimisation' often refers to the training of parameters, we are here referring to a model or layer whose operation is defined as the solution to an optimisation problem. We might express this as an implicit layer as
\begin{align*}
	\min_y f_\theta(x, y)\\
	\mathrm{such~that}\quad y \in C_\theta(x)
\end{align*}
where $\theta$ denotes trainable parameters, $x$ denotes an input to the model, and $C_\theta(x)$ is some constraint set.

Equivalently, and once again setting aside concerns such as uniqueness,
\begin{align*}
	y = {\arg\min}_{\widetilde{y} \in C_\theta(x)} f_\theta(x, \widetilde{y}).
\end{align*}

Brandon Amos' doctoral thesis \cite{amos2019differentiable} gives more details on differentiable optimisation, building on \cite{pmlr-v70-amos17b, pmlr-v70-amos17a, cxvpy-diff}.

\begin{example}
For example, we might consider the optimisation problem
\begin{align}
	&\min_{y \in \reals^d} \norm{x - y}\nonumber\\
\mathrm{such~that}\quad& \theta_1 \cdot y \leq \phi_1,\ldots, \theta_n \cdot y \leq \phi_n\label{eq:misc:polytope}
\end{align}
where each $\theta_i \in \reals^d$ and each $\phi_i \in \reals$. This projects $x$ onto some convex polytope, defined as the intersection of $n$ half-spaces. In this case, $\theta_i$ and $\phi_i$ are trainable parameters. Given some paired observations $x, y$, denoting points $x$ and their projections $y$ onto some unknown polytope, then by optimising \eqref{eq:misc:polytope} we may learn an approximation to this unknown polytope.
\end{example}

\section{Comments}
The material on symbolic regression is joint work with Miles Cranmer, and is new here.

The discussion on limitations of neural differential equations is a standard part of the folklore.

The notion of deep implicit layers is a recent one in deep learning, largely popularised by \cite{implicit-layers-tutorial}.

\chapter{Conclusion}
\section{Future directions}
Having discussed the story so far -- what future directions do we anticipate?

\paragraph{Boutique versus `off the shelf'}
Most applications of neural differential equations are still `boutique', rather than `off the shelf'. The model is tailored -- with largely unstructured vector fields retrained from scratch -- for each individual use case.

This is unlike traditional differential equations, for which we have numerous well-studied models, and in each case may expect to find a wealth of literature discussing their long-term behaviour, bifurcation properties and so on. There already exists an analogous literature in modern deep learning, studying the behaviour of models such as GPT-3 \cite{gpt3}, CLIP \cite{clip} and so on.

In time the same development may take place for neural differential equations.

\paragraph{Neural ODEs}
Thousands of papers are written every year applying non-neural ODEs to topics across science, finance, economics, \ldots, and so on. Correspondingly, one significant opportunity is to apply neural ODEs to many of the tasks to which only non-neural ODEs have so far been applied.

\paragraph{Neural CDEs and SDEs}
Neural CDEs and neural SDEs are much newer. Work remains to be done on the practical machine learning details: finding expressive choices of vector field, and determining how to train these models efficiently. As with neural ODEs, another future direction is their application to practical topics, or how to hybridise them with their non-neural equivalents.

In addition, CDEs and neural CDEs have natural connections to control theory, and from that to reinforcement learning; these connections are largely unexplored.

Connections between neural SDEs, score matching diffusions, continuous normalising flows, optimal transport and Schr{\"o}dinger bridges are still in their infancy.

\paragraph{Numerical methods}
Numerical methods for neural differential equations are the single largest chapter in this thesis, and with good reason. Numerical differential equation solvers are an old topic, but recent developments such as reversible solvers and hypersolvers offer opportunities yet to be exploited. For example it would be desirable to have higher order reversible ODE solvers, or to be able to apply hypersolvers during training.

\paragraph{Symbolic regression}
Symbolic regression -- both the underlying techniques and their application to dynamical systems -- is still more alchemy than science. That is, it is more a matter of `seeing what sticks' than of applying guiding principles.

For dynamical systems, only SINDy and its variants are well-established. The development and application of more advanced techniques such as regularised evolution and deep symbolic regression remains almost entirely wide open.

\paragraph{Neural PDEs}
One topic, made conspicuous by its absence from this thesis, is the possibility of neural partial differential equations.

There have been a selection of ideas in this space. For example a convolutional network is roughly equivalent to the discretisation of a parabolic PDE. \cite{fno1, fno2, fno3} consider the `Fourier Neural Operator', which is probably the most well-developed current theory for something approaching neural PDEs. \cite{neural-spde} present some initial thoughts on neural stochastic partial differential equations. This list of references is far from exhaustive.

In practice many of the ideas in this space have yet to converge. (Perhaps unsurprisingly: there are a great many types of PDE to consider, after all.) This represents a major open direction for the field of neural differential equations.

\section{Thank you}
Finally, it remains to thank the reader for their attention. We hope we have adequately conveyed some amount of insight (and our own enthusiasm) for this new, rapidly developing, and in our opinion highly exciting field of neural differential equations.

Neural differential equations sit at the intersection of arguably the two most successful modelling paradigms ever invented. In doing so, they demonstrate that these `two' paradigms are perhaps much closer to \textit{one} paradigm than at first glance we might imagine.

\appendix
\chapter{Review of Deep Learning}\label{appendix:deep}
We expect that a reasonable proportion of our audience may come from a traditional mathematical modelling background, and may not be familiar with deep learning.

Whilst all sorts of concepts will be important at various points in the presentation, we will assume familiarity with the following elementary concepts throughout:
\begin{itemize}
	\item Common neural architectures (i.e. differentiable computation graphs):
	\begin{itemize}
		\item Feedforward networks;
		\item Convolutional networks;
		\item Recurrent networks, GRUs, LSTMs;
		\item Residual networks;
		\item Activation functions: ReLU, tanh, sigmoid, softplus;
		\item Batch normalisation;
	\end{itemize}
	\item Optimisation:
	\begin{itemize}
		\item Maximum likelihood;
		\item Stochastic gradient descent;
		\item Batching;
		\item Backpropagation;
		\item Backpropagation through time for RNNS;
		\item Weight regularisation;
		\item Dropout;
	\end{itemize}
	\item Supervised learning:
	\begin{itemize}
		\item $L^2$ loss;
		\item Softmax;
		\item Cross-entropy;
		\item Binary cross-entropy;
	\end{itemize}
	\item Unsupervised learning:
	\begin{itemize}
		\item (Wasserstein) generative adversarial networks;
		\item Variational autoencoders;
		\item Wasserstein distance;
		\item KL divergence;
	\end{itemize}
\end{itemize}
If the reader is indeed unfamiliar with deep learning, then we recommend either \cite{hands-on-book} or \cite{pytorch-book} for the necessary introductions to much of the above list. (At least at time of writing. Their discussion on the details of individual software frameworks may soon become out-of-date.)

The above list tends towards practical concerns. This thesis additionally assumes familiarity with a few other concepts, slightly more academic in nature. As these appear less frequently in introductory texts, then for readability's sake we provide an introduction to them now. The emphasis will be on brevity over completeness.

\section{Autodifferentiation}\index{Autodifferentiation}\label{appendix:deep:autodiff}
Let $f_1, \ldots, f_n$ be some collection of functions whose derivatives we know how to compute. (Referred to as `differentiable primitives'.)

Then for any composition of these functions $x \mapsto f(x) = f_{i_m}( \cdots (f_{i_1}(x)) \cdots )$, with $i_1, \ldots, i_m \in \{1, \ldots, n\}$, we also know how to compute the derivative of $f$ via the chain rule:
\begin{equation}\label{eq:deep:chain-rule}
	\frac{\dd f}{\dd x} = \frac{\dd f_{i_m}}{\dd f_{i_{m-1}}} \cdots \frac{\dd f_{i_2}}{\dd f_{i_1}}\frac{\dd f_{i_1}}{\dd x}.
\end{equation}
More generally one may consider any (topologically sorted) directed acyclic graph of compositions; we focus on the easy-to-present case.

Autodifferentiation frameworks offer an automated way to compute \eqref{eq:deep:chain-rule}, by providing differentiable primitives such as matrix multiplies, sines, cosines, ReLUs, and so on.

It remains to consider how best to evaluate \eqref{eq:deep:chain-rule}. There are two main approaches.

\paragraph{Forward-mode}
Forward-mode autodifferentiation, also known as forward sensitivity, proceeds by recursively computing
\begin{equation}\label{eq:deep:forward-mode}
	\frac{\dd f_{i_q}}{\dd x} = \frac{\dd f_{i_q}}{\dd f_{i_{q-1}}} \frac{\dd f_{i_{q-1}}}{\dd x}
\end{equation}
for $q = 2, \ldots, m$.

\paragraph{Reverse-mode}
Reverse-mode autodifferentiation, also known as backpropagation or reverse sensitivity, proceeds by recursively computing
\begin{equation}\label{eq:deep:reverse-mode}
	\frac{\dd f_{i_m}}{\dd f_{q-1}} = \frac{\dd f_{i_m}}{\dd f_{i_{q}}} \frac{\dd f_{i_{q}}}{\dd f_{q-1}}
\end{equation}
for $q = m-1, \ldots, 1$, and for convenience denoting $x = f_0$.

\paragraph{Efficiency}
The main difference is computational efficiency. Suppose $x$ is a vector and $f_{i_m}$ outputs a scalar. Suppose all intermediate layers are vectors. (This is the common case for neural networks, with a vector of parameters as input and a scalar loss as output.) Then each evaluation of \eqref{eq:deep:forward-mode} is a matrix-matrix product, whilst each evaluation of \eqref{eq:deep:reverse-mode} is only a vector-matrix product; this is substantially cheaper to compute. It is for this reason that backpropagation, not forward-mode autodifferentiation, is typically used to train neural networks.

We may more precisely characterise the above statement as follows. Let $d_{\text{input}}$ be the dimensionality of $x$ and let $d_{\text{output}}$ be the dimensionality of the output of $f_{i_m}$. Then under a reasonable model of computation, the cost of computing both $f$ and $\nicefrac{\dd f}{\dd x}$ via forward-mode autodifferentiation may be upper bounded by $2.5\, d_{\text{input}}$ times the cost of evaluating just $f$ \cite[Equation (4.17)]{griewank2008evaluating}. Computing both $f$ and $\nicefrac{\dd f}{\dd x}$ via reverse-mode autodifferentiation may be upper bounded by $4\, d_{\text{output}}$ times the cost of evaluating just the function \cite[Equation (4.21)]{griewank2008evaluating}. In each case we refer to `computing both $f$ and $\nicefrac{\dd f}{\dd x}$' as computing $\nicefrac{\dd f}{\dd x}$ typically relies on computing $f$ first.

\paragraph{Jacobian-vector and vector-Jacobian products}\index{Vector-Jacobian product}\index{Jacobian-vector product}
Consider again the case that $f_{i_m}$ outputs a scalar, so that reverse-mode autodifferentiation computes a sequence of vector-matrix products. As the matrix is a Jacobian this is referred as a a vector-Jacobian product (`vjp'). Likewise if $x$ is a scalar then forward-mode autodifferentiation computes a sequence of Jacobian-vector products (`jvp').

For these reasons, `jvp' and `vjp' are sometimes used as synonyms for forward- and reverse-mode autodifferentiation, even when $x$ or $f_{i_m}$ are not necessarily scalar.

\paragraph{Comparison}
Note that \eqref{eq:deep:forward-mode} may be evaluated \textit{during} the `forward' evaluation of $f(x) = f_{i_m}( \cdots (f_{i_1}(x))\cdots)$: just compute each $\frac{\dd f_{i_{q-1}}}{\dd f_1} \mapsto \frac{\dd f_{i_q}}{\dd f_1}$ alongside each $f_{i_{q-1}}( \cdots (f_{i_1}(x)) \cdots)\mapsto f_{i_q}( \cdots (f_{i_1}(x))\cdots)$.

In contrast\eqref{eq:deep:reverse-mode} must be evaluated \textit{after} the `forward' evaluation -- we cannot evaluate $\frac{\dd f_{i_m}}{\dd f_q}$, which is evaluated at $f_{i_{q-1}}( \cdots (f_{i_1}(x))\cdots)$, until $f_{i_{q-1}}( \cdots (f_{i_1}(x))\cdots)$ has been computed. As such all $f_{i_{q}}( \cdots (f_{i_1}(x))\cdots)$ must first be evaluated and then held in memory. The amount of space available in memory can become a concern.

The canonical reference text on autodifferentiation is \cite{griewank2008evaluating}.

\section{Normalising flows}\label{appendix:deep:normalising-flows}\index{Normalising flows}
Fix $d \in \naturals$ and let $f \colon \reals^d \to \reals^d$ be bijective and sufficiently smooth.

Let $X$ be some random variable taking values in $\reals^d$, with density $p_X \colon \reals^d \to [0, \infty)$. Let $Y = f(X)$. Then the change of variables formula gives that the density $p_Y \colon \reals^d \to [0, \infty)$ of $Y$ is
\begin{equation*}
	p_Y(y) = p_X(f^{-1}(y))\abs{\det \frac{\dd f}{\dd x}(f^{-1}(y))}^{-1}.
\end{equation*}

Now let $X$ be a multivariate normal, let $Y$ correspond to (the empirical samples of) the data, and let $f = f_\theta$ to be some flexible neural network model. Consider training $f_\theta$ by maximum likelihood, by directly optimising
\begin{equation*}
	\max_\theta \expect_{y\sim Y}\log p_Y(y) = \max_\theta \expect_{y\sim Y}\left[\log p_X(f^{-1}_\theta(y)) - \log\abs{\det \frac{\partial f}{\partial x}(f^{-1}(y))} \right].
\end{equation*}
After training we obtain a generative model capable of producing approximate samples of $Y$: simply sample $X$ then evaluate $f(X)$.

The snag is that training is computationally expensive: in general evaluating the (log-determinant-)Jacobian costs $\bigO{d^3}$. Correspondingly much of the literature has focused on finding neural architectures $f_\theta$ that (a) exhibit structure that may be exploited to cheapen the cost of the Jacobian computation, whilst (b) still being expressive enough to produce good models.

Normalising flows were introduced in \cite{normalizing-flow}. We recommend \cite{norm-flow-ref} for further introduction.

\section{Universal approximation}\index{Universal approximation}\label{appendix:deep:universal-approximation}
`Universal approximation' is what a mathematician would call `density'. That is, given some normed vector space\footnote{Or a topological space in general.} $V$ and some set $W \subseteq V$, then $W$ is said to exhibit universal approximation with respect to $V$ if for all $\varepsilon > 0$ and $v \in V$, there exists $w \in W$ such that $\norm{v - w} < \varepsilon$.

Let $K \subseteq \reals^{d_x}$ be compact. It is often desirable to demonstrate that some set of neural networks $\reals^{d_x} \to \reals^{d_y}$ exhibit universal approximation with respect to (typically) $C(K; \reals^{d_y})$. As long as we have enough data, take a large enough network, and train for long enough -- known as the `infinite data limit' -- then in principle we may hope to obtain an arbitrarily good approximation to the target function.

Most famously, the set of feedforward networks of arbitrary width is a universal approximator for the set of continuous functions. (`The' universal approximation theorem.)
\begin{definition}
	Let $\rho \colon \reals \to \reals$ be any continuous function. Then let $\mathcal{N}^\rho_d$ denote the set of feedforward neural networks with activation function $\rho$, with $d$ neurons in the input layer, one neuron in the output layer, and a single hidden layer with an arbitrary number of neurons.
\end{definition}
\begin{theorem}[{Universal Approximation Theorem \cite{pinkus1999}}]
	 Let $K \subseteq \reals^d$ be compact. Then $\mathcal{N}^\rho_d$ is dense in $C(K)$ if and only if $\rho$ is nonpolynomial.
\end{theorem}
Many introductory texts repeat weaker versions of this theorem, apparently unaware that this simpler stronger version exists.\footnote{And in fact slightly stronger (but slightly more complex) versions than the one we have stated here also exist; see \cite{leshno}.}

Other variations on this theorem can also be found. Most notably, the set of feedforward networks of arbitrary \textit{depth} is also a universal approximator for the set of continuous functions.
\begin{definition}
	Let $\rho \colon \reals \to \reals$ and $d_x, d_y, d_w \in \naturals$. Then let $\mathcal{N}\!\mathcal{N}^\rho_{d_x, d_y, d_w}$ denote the set of feedforward neural networks with $d_x$ neurons in the input layer, $d_y$ neurons in the output layer, and an arbitrary number of hidden layers of width $d_w$ with activation function $\rho$.
\end{definition}
\begin{theorem}[{Deep-and-Narrow Universal Approximation \cite{kidger2020deep}}]
	Let $\rho \colon \reals \to \reals$ be any nonaffine continuous function, which is continuously differentiable at at least one point, with nonzero derivative at that point. Let $K \subseteq \reals^{d_x}$ be compact. Then $\mathcal{N}\!\mathcal{N}^\rho_{d_x,d_y,d_x+d_y+2}$ is dense in $C(K; \reals^{d_y})$.
\end{theorem}

\section{Irregular time series}\index{Time series!Irregular}
Time series are often `messy' or `irregular'. Consider the space of $d$-dimensional irregularly-sampled time series
\begin{equation*}
	\set{((t_0, x_0), \ldots, (t_n, x_n))}{n \in \naturals, t_j \in \reals, x_j \in (\reals \cup \{*\})^d, t_j < t_{j + 1}},
\end{equation*}
where $*$ denotes the possibility of missing data.

A few practical things must be considered. (See also \cite{Che2018} for more discussion.)

\paragraph{Irregular sampling}\index{Irregular sampling}
The choice of points $t_j$ may not be the same for different time series in the dataset. The values of $t_j$ may be informative, and generally we should concatenate $(t_j, x_j)$ together before passing them to a model. (Sometimes the increments $(t_j - t_{j-1}, x_j)$ are used instead.)

\paragraph{Variable length}
The length $n$ may not be the same for different time series in the dataset. This can affect how easy it is to batch different time series together.

\paragraph{Missing data}
Each observation $x_j \in (\reals \cup \{*\})^d$ may have missing data.

Some texts suggest `imputing' missing data: that is, filling in any missing data in some sensible way prior to applying a model. Despite its popularity this is frequently the wrong thing to do: that the data was missing is information that has been lost. In general, whether the data was missing may itself be informative.

It is better to fix a vector space $V$ and an injective map $(\reals \cup \{*\})^d \to V$, and then apply this map to every $x_j$ prior to applying the model. That this is injective means no information is lost. That it maps into a vector space means that the result is in a form the model can use.

A frequent choice is $V = \reals^{2d}$ and $(z_1, \ldots, z_d) \mapsto (\phi(z_1), \ldots, \phi(z_d), \psi(z_1), \ldots, \psi(z_d))$, where
\begin{align*}
	\phi &\colon \reals \cup \{*\} \to \reals,\\
	\phi &\colon z \mapsto \begin{cases} 0 & \text{if $z = *$},\\ z & \text{if $z \in \reals$},\end{cases}
\end{align*}
\begin{align*}
	\psi &\colon \reals \cup \{*\} \to \reals,\\
	\psi &\colon z \mapsto \begin{cases} 0 & \text{if $z = *$},\\ 1 & \text{if $z \in \reals$}.\end{cases}
\end{align*}
The map $\psi$ is sometimes referred to as a `mask', `observational intensity', or similar.

\section{Miscellanea}
\paragraph{Maximum mean discrepancy}\label{appendix:deep:mmd}\index{Maximum mean discrepancy}
The maximum mean discrepancy (MMD) is a (pseudo)distance between probability distributions. Let $\mathcal{X}$ be some set and let $\phi \colon \mathcal{X} \to \reals^d$ be fixed. Let $\mathbb{P}$, $\mathbb{Q}$ be two probability distributions over $\mathcal{X}$. Then the MMD between $\mathbb{P}$ and $\mathbb{Q}$ is defined to be
\begin{equation}\label{eq:deep:mmd}
	d(\mathbb{P}, \mathbb{Q}) = \norm{\expect_{x \sim \mathbb{P}}[\phi(x)] - \expect_{x \sim \mathbb{Q}}[\phi(x)]}
\end{equation}
for any fixed choice of norm $\norm{\,\cdot\,}$ on $\reals^d$.

Like the KL divergence or the Wasserstein distance, this is a popular optimisation criterion when fitting generative models.

This may be extended from a pseudodistance (in which $d(\mathbb{P}, \mathbb{Q}) = 0$ need not imply $\mathbb{P}=\mathbb{Q}$) to a true distance by replacing $\reals^d$ with some infinite-dimensional Hilbert space.

A biased estimate of equation \eqref{eq:deep:mmd} may be obtained by estimating each expectation with $N$ Monte-Carlo samples. This takes $\bigO{N}$ work to evaluate. An unbiased estimate may be obtained by taking $\norm{\,\cdot\,} = \norm{\,\cdot\,}_2$, and squaring and expanding \eqref{eq:deep:mmd}. This produces nested expectations that take $\bigO{N^2}$ work to evaluate via Monte-Carlo.

MMDs were introduced in \cite{mmd}. The KID criterion for evaluating quality of generative image models is an example of an MMD \cite{bińkowski2018demystifying}.

\paragraph{The manifold hypothesis}\label{appendix:deep:manifold-hypothesis}\index{Manifold hypothesis}
Consider the dataset of all possible pictures of cats. (The `underlying' dataset from which in practice we observe some finite collection of samples.) For example each image may be a point in $\reals^{3 \times 32 \times 32}$, corresponding to (red, green, blue) channels and $32 \times 32$ pixels.

This is a very high-dimensional space, and it is clear that the majority of this space is of all kinds of pictures, other than of cats. Indeed most points in this space will resemble random noise. Our dataset covers only some small region of the overall space.

`The manifold hypothesis' is the informal statement that most datasets tend to behave in this way: that if you were to zoom out and squint at them, they would look a bit like a low-dimensional manifold embedded in this higher-dimensional space.

We could not find a good reference introducing the manifold hypothesis; it appears to simply be part of the folklore.

\paragraph{SiLU activation function}\index{SiLU}\index{Swish}
The SiLU activation function (also known as `swish') is defined as $x \mapsto x \sigma(x)$, where $\sigma$ is the sigmoid activation function. It is a popular activation function sometimes held to produce slightly better results than traditional options such as the ReLU \cite{gelu, silu, swish}.
\chapter{Neural Rough Differential Equations}\label{appendix:neural-rde}

This appendix follows from the introduction given in Section \ref{section:cde:long-time-series}; the material is from \cite{morrill2021neuralrough}. Here, we will apply neural CDEs to long time series, which are a regime in which both neural CDEs and RNNs tend to break down.

The key idea will be to solve a CDE by taking very large integration steps -- much larger than the sampling rate of the data -- whilst incorporating sub-step information through additional terms in the numerical solver, through what is known as the \textit{log-ODE method}.

A CDE treated in this way is termed a \textit{rough differential equation}, in the sense of rough path theory. Correspondingly we refer to this approach as either a \textit{neural rough differential equation}, or simply `the log-ODE method applied to neural CDEs'.\index{Rough!Differential equations}

\section{Background}
We begin with some necessary background on rough path theory.

\subsection{Signatures and logsignatures}\label{section:cde:sig-logsig}
\subsubsection{Signatures}\index{Signatures}
Let $x = (x_1, \ldots, x_{d_x}) \colon [a, b] \to \reals^{d_x}$ be continuous and of bounded variation. Weaker conditions may also be admitted, but this will suffice for our purposes here.

Define the iterated Riemann--Stieltjes integrals
\begin{equation}\label{eq:subsignature}
	S^{k_1, \ldots, k_m}_{a, b}(x) = \underset{a < t_1 < \cdots < t_m < b}{\int \cdots \int} \dd x_{k_1}(t_1) \cdots \dd x_{k_m}(t_m) \in \reals,
\end{equation}
and, up to some maximal index $M \in \naturals$, put all such integrals together into a single object:
\begin{equation}
	\sig^M_{a, b}(x) = \left(1,
	\{
	S^{k}_{a, b}(x)
	\}_{k = 1}^d,
	\{
	S^{k_1, k_2}_{a, b}(x)
	\}_{k_1, k_2 = 1}^d,
	\ldots,
	\{
	S^{k_1, \ldots, k_M}_{a, b}(x)
	\}_{k_1, \ldots, k_M = 1}^d
	\right).\label{eq:signature}
\end{equation}
By convention $1 \in \reals$ is also included at the start.

Then $\sig^M_{a, b}(x)$ is known as the \textit{depth-$M$ signature transform of $x$}. (Or similar variations on this theme, like `$M$-step signature of $x$'.)

\subsubsection{Signatures as Taylor expansions}
Signatures are interesting because they appear in the Taylor expansion of a controlled differential equation. Let $y$ solve a CDE with vector field $f$, driven by $x$. Then in Einstein notation over indices $k_1, k_2, k_3, k_4$,
\begin{align}
	y_{k_1}(t) &= y_{k_1}(a) + \int_a^t f_{{k_1}, {k_2}}(y(s))\,\dd x_{k_2}(s)\nonumber\\
	&= y_{k_1}(a) + \int_a^t \Big(f_{{k_1}, {k_2}}(y(a)) \nonumber\\&\hspace{5em} + \frac{\partial f_{{k_1}, {k_2}}}{\partial y_{k_3}}(y(a)) (y_{k_3}(s) - y_{k_3}(a)) + \bigO{(t - a)^2} \Big) \dd x_{k_2}(s)\nonumber\allowdisplaybreaks\\[6pt]
	&= y_{k_1}(a) + f_{{k_1}, {k_2}}(y(a)) \int_a^t \dd x_{k_2}(s) \nonumber\\&\hspace{5em} + \frac{\partial f_{{k_1}, {k_2}}}{\partial y_{k_3}}(y(a))  \int_a^t (y_{k_3}(s) - y_{k_3}(a))\,\dd x_{k_2}(s) + \bigO{(t - a)^3}\nonumber\allowdisplaybreaks\\[6pt]
	&= y_{k_1}(a) + f_{{k_1}, {k_2}}(y(a)) \int_a^t \dd x_{k_2}(s) \nonumber\\&\hspace{5em} + \frac{\partial f_{{k_1}, {k_2}}}{\partial y_{k_3}}(y(a))  \int_a^t \int_a^s f_{k_3, k_4}(y(s))\,\dd x_{k_4}(u) \,\dd x_{k_2}(s) + \bigO{(t - a)^3}\nonumber\allowdisplaybreaks\\[6pt]
	&= y_{k_1}(a) + f_{{k_1}, {k_2}}(y(a)) \int_a^t \dd x_{k_2}(s) \nonumber\\&\hspace{5em} + \frac{\partial f_{{k_1}, {k_2}}}{\partial y_{k_3}}(y(a))  f_{k_3, k_4}(y(a))\int_a^t \int_a^s \,\dd x_{k_4}(u) \,\dd x_{k_2}(s) + \bigO{(t - a)^3}\nonumber\\[6pt]
	&= y_{k_1}(a) + f_{{k_1}, {k_2}}(y(a)) S^{k_2}_{a, t}(x) + \frac{\partial f_{{k_1}, {k_2}}}{\partial y_{k_3}}(y(a))  f_{k_3, k_4}(y(a)) S^{k_4, k_2}_{a, t}(x) + \bigO{(t - a)^3}.\label{eq:signature-as-taylor-expansion}
\end{align}
The right hand side is an affine combination of terms in the signature. If higher order terms had been taken in the Taylor expansions of $f$, then higher orders in the in the signature would have appeared on the right hand side.


This property means that the signature may be used to produce a good approximation to the solution of the CDE. Intuitively, over small time scales, the signature extracts the information `most important' to solving the CDE.

\begin{remark}
	Independent of neural CDEs, there has also been a line of work investigating the use of the signature transform as a feature extractor for time series \cite{primer2016, bonnier2019deep, generalised-signatures}.
\end{remark}

\subsubsection{Logsignatures}\index{Logsignatures}
The signature has some redundancy. For example a little algebra shows that $S^{1, 2}_{a, b}(x) + S^{2, 1}_{a, b}(x) = S^1_{a, b}(x)S^2_{a, b}(x)$, so that we already know any one of these quantities given the other three.

\begin{definition}[Lyndon word]
	Let $q \in \naturals$ and consider some set $\mathcal{A} = \{a_1, \ldots, a_q\}$, which we refer to as an alphabet. A \emph{word} in this alphabet is any finite-length sequence of elements of $\mathcal{A}$, for example $a_2a_3a_1a_4$. A \emph{Lyndon word} is any word which occurs lexicographically strictly earlier than any word obtained by cyclically rotating its elements. For example, $a_2 a_2 a_3 a_4$ is a Lyndon word as it occurs strictly earlier than $a_2 a_3 a_4 a_2$ or $a_3 a_4 a_2 a_2$ or $a_4 a_2 a_2 a_3$, whilst $a_1 a_1$ is not a Lyndon word as it does not occur strictly earlier than $a_1 a_1$ (which is a rotation).
\end{definition}

The logsignature transform is obtained by computing the signature transform, and throwing out redundant terms to produce some minimal collection. This choice of minimal collection is nonunique. One computationally efficient choice is to retain precisely those terms $S^{k_1, \ldots, k_m}_{a, b}(x)$ for which $k_1 \cdots k_m$ is a Lyndon word over the alphabet $\{1,\ldots, d_x\}$. This is introduced in \cite{signatory} (and is confusingly a distinct notion from the `Lyndon basis', which is one of the other nonunique choices).

By fixing such a procedure -- via Lyndon words or otherwise -- we obtain the \textit{depth-$M$ logsignature transform of $x$}, denoted $\logsig^M_{a, b}(x) \in \reals^{\beta(d_x, M)}$, with
\begin{equation*}
\beta(d, M) = \sum_{k=1}^M \frac{1}{k} \sum_{j|k} \mu\left(\frac{k}{j}\right) d^j
\end{equation*}
where $\mu$ is the M{\"o}bius function.

See \cite{jeremy-logsig-pdf, iisignature} for further introduction to the logsignature transform.

\paragraph{Geometric interpretation}
The first two $m=1,2$ levels of the logsignature have geometric interpretations. The depth 1 terms are simply the increments of the path. The depth 2 is the signed (L{\'e}vy) area between the path and the chord joining its endpoints; equivalently this corresponds to a notion of order. See Figure \ref{fig:geometric-logsignature}.

Higher terms in the logsignature correspond to `repeated areas' and are not so easily visualised.

\colorlet{lightgray}{gray!10}
\definecolor{colorp1}{HTML}{4878d0}
\definecolor{colorp2}{HTML}{ee854a}
\definecolor{colorp3}{HTML}{6acc64}
\definecolor{colord1}{HTML}{003f5c}
\definecolor{colord2}{HTML}{bc5090}
\definecolor{colord3}{HTML}{ff6361}
\definecolor{colorls1}{HTML}{66C2A5}
\definecolor{colorls2}{HTML}{8DA0CB}
\definecolor{colorls3}{HTML}{FC8D62}

\global\def\xs{{0.        , 0.13157895, 0.26315789, 0.39473684, 0.52631579,
       0.65789474, 0.78947368, 0.92105263, 1.05263158, 1.18421053,
       1.31578947, 1.44736842, 1.57894737, 1.71052632, 1.84210526,
       1.97368421, 2.10526316, 2.23684211, 2.36842105, 2.5}}
\global\def\ys{{0.        , 0.48565753, 0.84647908, 1.09613282, 1.24828692,
       1.31660956, 1.31476892, 1.25643315, 1.15527045, 1.02494897,
       0.8791369 , 0.73150241, 0.59571366, 0.48543884, 0.41434611,
       0.39610366, 0.44437965, 0.57284225, 0.79515964, 1.125}}
       
\global\def\plotrange{0, 1, 2, 4, 5, 8, 11, 12, 13, 15, 18, 19}

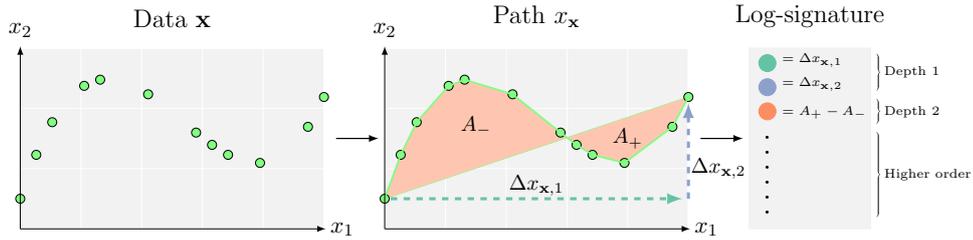
\begin{figure}
    \centering
    \resizebox{0.86\linewidth}{!}{
        \begin{tikzpicture}
            \draw[draw=white, fill=lightgray] (0, 0) grid (5, 3) rectangle (0, 0);

                \foreach \i in \plotrange {
                    \node at (\xs[\i] * 2, 0.5 + \ys[\i] * 1.5)[circle, draw=black, fill=green!50, inner sep=1.5pt] {};
                }
            \node at (2.5, 3) [above=2mm] {Data $\mathbf{x}$};
            \draw[black, ->] (0, 0) -- (0, 3.03);
            \draw[black, ->] (0, 0) -- (5.03, 0);
            \node at (5.3, 0) {\small $x_1$};
            \node at (0, 3.3) {\small $x_2$};
            
            \draw[->, line width=0.2mm] (5.2, 1.5) -- (5.9, 1.5);
            
            \draw[draw=white, fill=lightgray] (6, 0) grid (11, 3) rectangle (6, 0);
            \pgfmathsetmacro{\xprev}{6 + 0.05}
            \pgfmathsetmacro{\yprev}{0+0.5}
            \draw[colorls1, line width=0.5mm, dashed, ->] (\xprev + 0.05, \yprev) -- (\xs[19] * 2 + 6 - 0.1, \yprev);
            \draw[colorls2, line width=0.5mm, dashed, ->] (\xs[19] * 2 + 6, \yprev) -- (\xs[19] * 2 + 6, 0.5 + \ys[19] * 1.5 - 0.1);
            \node at (6 + 2.5, \yprev - 0.1) [above] {\footnotesize $\Delta x_{\mathbf{x}, 1}$};
            \node at (\xs[19] * 2 + 6 - 0.1, 1)  [right] {\footnotesize $\Delta x_{\mathbf{x}, 2}$};
            \draw[green!50, xshift=6cm, name path=one] plot coordinates {
                (0, 0.5 + 1.5 * 0) (2 * 0.13157894736842105, 0.5 + 1.5 * 0.48565753025222336) (2 * 0.2631578947368421, 0.5 + 1.5 * 0.846479078582884) (0.5263157894736842 * 2, 0.5 + 1.5 * 1.2482869222918795) (0.6578947368421052 * 2, 0.5 + 1.5 * 1.3166095640763957) (1.0526315789473684 * 2, 0.5 + 1.5 * 1.1552704475871118) (1.4473684210526314 * 2, 0.5 + 1.5 * 0.731502405598484) (1.5789473684210527 * 2, 0.5 + 1.5 * 0.5957136608835107) (1.7105263157894737 * 2, 0.5 + 1.5 * 0.4854388394809739) (1.9736842105263157 * 2, 0.5 + 1.5 * 0.39610365942557224) (2.3684210526315788 * 2, 0.5 + 1.5 * 0.7951596442630118) (2.5 * 2, 0.5 + 1.5 * 1.125)
            };
            \draw[green!50, xshift=6cm, name path=two] plot coordinates {
                (0, 0.5) (2.5 * 2, 0.5 + 1.5 * 1.125)
            };
            \tikzfillbetween[
                of=one and two, split
            ] {pattern=north west lines, colorls3!50};
            \node at (7.1, 1.7) [right=0.1mm] {\footnotesize $A_{-}$};
            \node at (9.6, 1.5) [right=0.1mm] {\footnotesize $A_{+}$};
            
            \foreach \i in \plotrange {
                \pgfmathsetmacro{\x}{\xs[\i] * 2 + 6}
                \pgfmathsetmacro{\y}{0.5 + \ys[\i] * 1.5}
                \node at (\x, \y)[circle, draw=black, fill=green!50, inner sep=1.5pt] {};
                \draw[green!50, thick, cap=round] (\xprev, \yprev) -- (\x, \y);
                \global\let\yprev=\y
                \global\let\xprev=\x
            }
            \node at (6 + 2.5, 3) [above=2mm] {Path $x_\mathbf{x}$};
            \draw[black, ->] (6, 0) -- (6, 3.03);
            \draw[black, ->] (6, 0) -- (6 + 5.03, 0);
            \node at (6 + 5.3, 0) {\small $x_1$};
            \node at (6, 3.3) {\small $x_2$};
            
            \draw[->, line width=0.2mm] (11.2, 1.5) -- (11.9, 1.5);
        
            \pgfmathsetmacro{\shift}{0}
            \pgfmathsetmacro{\xend}{14}
            \draw[fill=lightgray, lightgray] (12 + \shift, 0) rectangle (\xend, 3);
            \filldraw[colorls1] (12.3 + \shift, 2.75) circle (4pt);
            \filldraw[colorls2] (12.3 + \shift, 2.35) circle (4pt);
            \filldraw[colorls3] (12.3 + \shift, 1.95) circle (4pt);
            \node at (12.4 + \shift, 2.8) [right] {\tiny $= \Delta x_{\mathbf{x}, 1}$};
            \node at (12.4 + \shift, 2.4) [right] {\tiny $= \Delta x_{\mathbf{x}, 2}$};
            \node at (12.4 + \shift, 1.93) [right] {\tiny $= A_{+} - A_{-}$};
            \node[black, rotate=90] at (12.3 + \shift, 0.95) {\Large \ldots\ldots};
            \node at (12 + 1 + 0.5*\shift, 3) [above=2mm] {Log-signature};
            \draw [decorate, decoration={calligraphic brace, amplitude=1.5pt, mirror}] (\xend + 0.1, 2.3) -- node[midway, right, xshift=-0.5pt] {\tiny Depth 1} (\xend + 0.1, 2.9);
            \draw [decorate, decoration={calligraphic brace, amplitude=1.5pt, mirror}] (\xend + 0.1, 1.75) -- node[midway, right, xshift=-0.5pt] {\tiny Depth 2} (\xend + 0.1, 2.15);
            \draw [decorate, decoration={calligraphic brace, amplitude=1.5pt, mirror}] (\xend + 0.1, 0.2) -- node[midway, right, xshift=-0.5pt] {\tiny Higher order} (\xend + 0.1, 1.6);
        
        \end{tikzpicture}
    }
    \caption{Geometric intuition for the first two levels of the logsignature for a 2-dimensional path. The depth-1 terms correspond to the change in each of the coordinates over the interval. The depth-2 term corresponds to the \textit{L\'evy area} of the path, this being the signed area between the curve and the chord joining its start and endpoints.}
    \label{fig:geometric-logsignature}
\end{figure}

\paragraph{Order interpretation}
Consider just the region $A_-$ in Figure \ref{fig:geometric-logsignature}. Progressing from left to right, the green curve shown makes a large change in $x_2$, followed by a change in $x_1$, and correspondingly the initial part of this curve incurs a negative signed area $A_-$. If the order of these changes had been reversed then a positive signed area would have been accumulated instead. Likewise, applying the same procedure to clockwise and anticlockwise spirals would have produced areas of different signs.

As such (log)signatures -- including the higher order terms in (log)signatures -- encode information by capturing a notion of \textit{order of events}.

Now recall that the universal approximation theorem for CDEs (Theorem \ref{theorem-informal:informalunivapprox}) is proven by reducing CDEs to signatures, whilst Section \ref{section:cde:rnn-cde} reduces RNNs to CDEs. This brings us full circle, as RNNs are a model predicated upon the assumption that the order of inputs matter.

This appears to be the fundamental difference between RNNs and Transformers \cite{attention-is-all-you-need}. RNNs are predicated on assuming order is important to the data; Transformers are predicated on assuming that order is (mostly) unimportant. This is reflected in their typical use cases. RNNs are often preferred for time series, whilst Transformers are often preferred in natural language processing.

See also \cite{toth2021seqtens}, who make concrete some of these (relatively abstract) notions about order.

\subsection{The log-ODE method}\index{Log-ODE method}
Let $f \colon \reals^{d_y} \to \reals^{d_y \times d_x}$ be Lipschitz. Let $x \colon [0, T] \to \reals^{d_x}$ be of bounded variation. Let $y_0 \in \reals^{d_y}$. Consider $y \colon [0, T] \to \reals^{d_y}$ satisfying the CDE
\begin{equation*}
	y(0) = y_0, \qquad y(t) = y(0) + \int_0^t f(y(s)) \,\dd x(s)\quad\text{for $t \in (0, T]$}.
\end{equation*}

Then the log-ODE method states that for all $M \in \naturals$ there exists some $\widehat{f}_M \colon \reals^{d_y} \to \reals^{d_y \times \beta(d_x, M)}$ such that $\widehat{y}(T) \to y(T)$ as $M \to \infty$, where $\widehat{y}$ solves the ODE
\begin{equation}\label{eq:logode}
	\widehat{y}(0) = y_0,\qquad\widehat{y}(t) = \widehat{y}(0) + \int_0^t \widehat{f}_m(\widehat{y}(s))\frac{\logsig^M_{0, T}(x)}{T} \,\dd s.
\end{equation}
The right hand side denotes a matrix-vector product between $\widehat{f}_M(\widehat{y}(s)) \in \reals^{d_y \times \beta(d_x, M)}$ and $\logsig^M_{0, T}(x) \in \reals^{\beta(d_x, M)}$.

The exact form of $\widehat{f}_M$ is actually known, but is expensive to compute. For this reason we will soon circumvent the need for this computation.

\begin{example}\label{example:rde:depth-one}
	Consider when $M=1$. Here in fact $\widehat{f}_1 = f$. Suppose $T$ is small. Then equations \eqref{eq:subsignature} and \eqref{eq:signature} imply
	\begin{equation*}
		\frac{\logsig_{0, T}^1(x)}{T} = \frac{x(T) - x(0)}{T} \approx \frac{\dd x}{\dd t}(0).
	\end{equation*}
	As such \eqref{eq:logode} is an approximation to equation \eqref{eq:cde-to-ode}. If $x$ is piecewise linear then this approximation is exact, and the depth-1 log-ODE method is identical to equation \eqref{eq:cde-to-ode}.
\end{example}

\section{Neural vector fields}

We begin with the usual set-up for neural CDEs applied to potentially irregular time series, as in Sections \ref{section:cde:irregular} and \ref{section:cde:interpolation}.

We assume observations of a time series $\mathbf{x} = ((t_0, x_0), \ldots, (t_n, x_n))$ with $t_j \in \reals$ the timestamp for the observation $x_j \in (\reals \cup \{*\})^{d_x - 1}$, and $*$ denotes the possibility of missing data, and $t_0 < \cdots < t_n$.

Let $\mathbf{x} \mapsto x_\mathbf{x}$ be an interpolation scheme. We additionally require that each $x_\mathbf{x}$ be a continuous piecewise linear function. (So that either linear interpolation or rectilinear interpolation, see Section \ref{section:cde:interpolation}, would suffice.)

Let $f_\theta \colon \reals^{d_y} \to \reals^{d_y \times d_x}$ be any (Lipschitz) neural network depending on parameters $\theta$. The value $d_y \in \naturals$ is a hyperparameter describing the size of the hidden state. Let $\zeta_\theta \colon \reals^{d_x} \to \reals^{d_y}$ be any neural network depending on the parameters $\theta$.

A neural controlled differential equation was defined as the solution of the CDE
\begin{equation*}
	y(t) = y(0) + \int_{0}^t f_\theta(y(s)) \,\dd x(s)\quad\text{for $t \in (0, T]$},
\end{equation*}
where $y(0) = \zeta_\theta(x(0))$.

This was (typically) solved by reducing the CDE to an ODE, as in equations \eqref{eq:g-theta-x} and \eqref{eq:ncde-to-node}. We reproduce equations \eqref{eq:g-theta-x} and \eqref{eq:ncde-to-node} here: let
\begin{equation*}\tag{\ref{eq:g-theta-x} revisited}
	g_{\theta, x}(y, s) = f_\theta(y) \frac{\dd x}{\dd s}(s),
\end{equation*}
so that for $t \in (0, T]$,
\begin{align*}
	y(t) &= y(0) + \int_{0}^t f_\theta(y(s)) \,\dd x(s)\nonumber\\
	&= y(0) + \int_{0}^t f_\theta(y(s)) \frac{\dd x}{\dd s}(s) \,\dd s\nonumber\\
	&= y(0) + \int_{0}^t g_{\theta, x}(y(s), s) \,\dd s.\tag{\ref{eq:ncde-to-node} revisited}
\end{align*}

\subsection{Applying the log-ODE method}\index{Log-ODE method}

Pick points $r_j$ such that $t_0 = r_0 < r_1 < \cdots < r_m = t_n$. In principle these can be variably spaced but in practice we will typically space them equally far apart. The number of points $m$ should be chosen much smaller than $n$, that is to say $m \ll n$. The number and spacing of $r_j$ is a hyperparameter.

We also pick a logsignature depth hyperparameter $M \geq 1$.

We now replace \eqref{eq:g-theta-x} with the piecewise
\begin{equation}\label{eq:nrde}
	g_{\theta, x}(y, s) = f_\theta(y) \frac{\logsig^M_{r_j, r_{j + 1}}(x)}{r_{j + 1} - r_j}\quad\text{for $s\in [r_j, r_{j+1})$},
\end{equation}
where $f_\theta \colon \reals^{d_y} \to \reals^{d_y \times \beta(d_x, M)}$ is some neural network, $\logsig^M_{r_j, r_{j+1}}(x) \in \reals^{\beta(d_x, M)}$, and the right hand side denotes a matrix-vector product.

Equation \eqref{eq:ncde-to-node} remains unchanged:
\begin{equation*}
	y(t) = y(0) + \int_0^t g_{\theta, x}(y(s), s)\,\dd s.
\end{equation*}

This is an alternate method by which a CDE may be reduced to an ODE. This may now be solved as a (neural) ODE using standard ODE solvers.

An overview of this process is shown in Figure \ref{fig:nrde}.

\definecolor{bittersweet}{rgb}{1.0, 0.44, 0.37}

\pgfmathsetmacro{\pathyshift}{0.8}
\global\def\rs{{0, 1, 2.1, 3.3, 4.5, 5.5, 6.5}}
\global\def\lsone{{2.15, 2.3, 2.1, 2.4, 2.1, 2.4}}
\global\def\lstwo{{2.1, 2.4, 2.2, 2.5, 2.2, 2.1}}
\global\def\lsthree{{2.3, 2.1, 2.15, 2.35, 2.5, 2.2}}
\global\def\lspathraise{0.7}
\global\def\data{{0.4, 0.7, 0.9, 0.5, 0.6, 0.9, 0.8, 0.5, 0.3, 0.3, 0.8, 1.0, 0.2, 0.4, 0.7, 0.5, 1.0, 1.0, 0.8, 0.4, 0.2, 0.6, 0.8, 0.5, 0.3, 0.5, 0.2, 0.1, 0.5, 0.2, 0.4, 0.6, 0.3, 0.2, 0.1, 0.7, 0.6, 0.2, 0.5, 0.4, 0.4, 0.1}}

\definecolor{colorls1}{HTML}{66C2A5}
\definecolor{colorls2}{HTML}{8DA0CB}
\definecolor{colorls3}{HTML}{FC8D62}

\newcommand\ncdediagram[1]{
	\begin{tikzpicture}
		\pgfmathsetmacro{\logsigversion}{#1}
		
		\ifthenelse{\logsigversion=0}
		{\pgfmathsetmacro{\hreduce}{1.8}}
		{\pgfmathsetmacro{\hreduce}{0}}
		
		\draw[black, thick, ->] (0, 0) -- (7, 0);
		
		\ifthenelse{\logsigversion=0}
		{
			\foreach \i in {0, 6} {
				\node at (\rs[\i], 0)[circle,fill, inner sep=1.5pt] {};
			}
			\node at (\rs[0], -0.25) {\footnotesize $t_0$};
			\node at (\rs[6], -0.25) {\footnotesize $t_n$};
		}
		{
			\foreach \i in {0, 1, 2, 3, 4, 5, 6} {
				\node at (\rs[\i], 0)[circle,fill, inner sep=1.5pt] {};
			}
			\node at (\rs[0], -0.25) {\footnotesize $r_0$};
			\node at (\rs[1], -0.25) {\footnotesize $r_1$};
			\node at (\rs[2], -0.25) {\footnotesize $r_2$};
			\node at (\rs[3], -0.25) {\footnotesize $r_3$};
			\node at (\rs[4], -0.25) {\footnotesize $r_{m-2}$};
			\node at (\rs[5], -0.25) {\footnotesize $r_{m-1}$};
			\node at (\rs[6], -0.25) {\footnotesize $r_m$};
		}
		\node at (3.85, -0.25) {\footnotesize $\cdots$};

		\pgfmathsetmacro{\dataix}{0}
		\pgfmathsetmacro{\midwayprev}{0}
		\pgfmathsetmacro{\xprev}{0}
		\pgfmathsetmacro{\yprev}{0}
		\foreach \i in {0, 1, 2, 3, 4, 5} {
			\pgfmathsetmacro{\frac}{\rs[\i + 1] - \rs[\i]}
			\pgfmathsetmacro{\midway}{\rs[\i] + (\rs[\i + 1] - \rs[\i]) * 0.5}
			
			\foreach \d in {0, 1, 2, 3, 4, 5, 6} {
				\ifthenelse{\i>0 \AND \d=0}{\pgfmathsetmacro{\y}{\data[\dataix - 1] * 0.5 + 0.1}}{\pgfmathsetmacro{\y}{\data[\dataix] * 0.5 + 0.1}}
				\pgfmathsetmacro{\x}{\rs[\i] + \d * \frac * 0.1667}
				\pgfmathsetmacro{\dx}{\frac * 0.1667}
				\ifthenelse{\i=3}
				{
					\node at (\x, \y)[circle, draw=black, fill=green!50, inner sep=1.5pt, opacity=0.3] {};
					\node at (\x, \y + \pathyshift)[circle, draw=black, fill=green!50, inner sep=0.5pt, densely dashed] {};
					\draw[green!50, thick, cap=round, dashed] (\xprev, \yprev + \pathyshift) -- (\x, \y + \pathyshift);
				}
				{
					\ifthenelse{\logsigversion=0}{\draw (\x, -0.05) -- (\x, 0.05);}{}
					\node at (\x, \y)[circle, draw=black, fill=green!50, inner sep=1.5pt] {};
					\node at (\x, \y + \pathyshift)[circle, draw=black, fill=green!50, inner sep=0.5pt] {};
					\ifthenelse{\i=0 \AND \d=0}{}{\draw[green!50, thick, cap=round] (\xprev, \yprev + \pathyshift) -- (\x, \y + \pathyshift);}
					\ifthenelse{\logsigversion=0}
					{
						\ifthenelse{\d=6 \AND \i=5}
						{}
						{
							\node at (\x, 3.8 - \hreduce)[circle, draw=black, fill=orange!50, inner sep=1.5pt] {};
						}
					}
					{}
				}
				\pgfmathsetmacro{\dataixnew}{\dataix + 1}
				\global\let\dataix=\dataixnew
				\global\let\yprev=\y
				\global\let\xprev=\x
				
			}
			
			\ifthenelse{\logsigversion=1}{
				\ifthenelse{\i=3}
				{
					\node at (\rs[\i], 3.8)[circle, draw=black, fill=orange!50, inner sep=1.5pt] {};
					\draw[->, dashed] (\rs[\i] + 0.08, 3.8) -- (\rs[\i+1] - 0.08, 3.8);
				}
				{
					\draw[decoration={calligraphic brace, amplitude=3pt}, decorate, line width=1pt] (\rs[\i] + 0.05, 1.65) node {} -- (\rs[\i+1] - 0.05, 1.65);
					\draw[->, dashed] (\midway, 1.8) -- (\midway, 2.0);
					\node at (\midway, \lsone[\i])[circle, draw=black, fill=colorls1, inner sep=1pt] {};
					\node at (\midway, \lstwo[\i])[circle, draw=black, fill=colorls2, inner sep=1pt] {};
					\node at (\midway, \lsthree[\i])[circle, draw=black, fill=colorls3, inner sep=1pt] {};
					\node at (\rs[\i], 3.8)[circle, draw=black, fill=orange!50, inner sep=1.5pt] {};
					\draw[->] (\rs[\i] + 0.08, 3.8) -- (\rs[\i+1] - 0.08, 3.8);
				}
				\ifthenelse{\i=3}
				{
					
				}
				{
					\node at (\midway, \lsone[\i] + \lspathraise)[circle, draw=black, fill=colorls1, inner sep=.5pt] {};
					\node at (\midway, \lstwo[\i] + \lspathraise)[circle, draw=black, fill=colorls2, inner sep=.5pt] {};
					\node at (\midway, \lsthree[\i] + \lspathraise)[circle, draw=black, fill=colorls3, inner sep=.5pt] {};
				}
				\ifthenelse{\i>0}
				{   
					\ifthenelse{\i=3 \OR \i=4}
					{  
						\ifthenelse{\i=3}
						{
							\pgfmathsetmacro{\midmidway}{\midwayprev + (\midway - \midwayprev) * 0.5}
							\draw[colorls1, thick, cap=round] (\midwayprev, \lsone[\i-1] + \lspathraise) -- (\midmidway, \lsone[\i-1]*0.5 + \lsone[\i]*0.5 + \lspathraise);
							\draw[colorls1, thick, cap=round, densely dashed] (\midmidway, \lsone[\i-1]*0.5 + \lsone[\i]*0.5 + \lspathraise) -- (\midway, \lsone[\i] + \lspathraise);
							\draw[colorls2, thick, cap=round] (\midwayprev, \lstwo[\i-1] + \lspathraise) -- (\midmidway, \lstwo[\i-1]*0.5 + \lstwo[\i]*0.5 + \lspathraise);
							\draw[colorls2, thick, cap=round, densely dashed] (\midmidway, \lstwo[\i-1]*0.5 + \lstwo[\i]*0.5 + \lspathraise) -- (\midway, \lstwo[\i] + \lspathraise);
							\draw[colorls3, thick, cap=round] (\midwayprev, \lsthree[\i-1] + \lspathraise) -- (\midmidway, \lsthree[\i-1]*0.5 + \lsthree[\i]*0.5 + \lspathraise);
							\draw[colorls3, thick, cap=round, densely dashed] (\midmidway, \lsthree[\i-1]*0.5 + \lsthree[\i]*0.5 + \lspathraise) -- (\midway, \lsthree[\i] + \lspathraise);
						}
						{
							\pgfmathsetmacro{\midmidway}{\midwayprev + (\midway - \midwayprev) * 0.5}
							\draw[colorls1, thick, cap=round, densely dashed] (\midwayprev, \lsone[\i-1] + \lspathraise) -- (\midmidway, \lsone[\i-1]*0.5 + \lsone[\i]*0.5 + \lspathraise);
							\draw[colorls1, thick, cap=round] (\midmidway, \lsone[\i-1]*0.5 + \lsone[\i]*0.5 + \lspathraise) -- (\midway, \lsone[\i] + \lspathraise);
							\draw[colorls2, thick, cap=round, densely dashed] (\midwayprev, \lstwo[\i-1] + \lspathraise) -- (\midmidway, \lstwo[\i-1]*0.5 + \lstwo[\i]*0.5 + \lspathraise);
							\draw[colorls2, thick, cap=round] (\midmidway, \lstwo[\i-1]*0.5 + \lstwo[\i]*0.5 + \lspathraise) -- (\midway, \lstwo[\i] + \lspathraise);
							\draw[colorls3, thick, cap=round, densely dashed] (\midwayprev, \lsthree[\i-1] + \lspathraise) -- (\midmidway, \lsthree[\i-1]*0.5 + \lsthree[\i]*0.5 + \lspathraise);
							\draw[colorls3, thick, cap=round] (\midmidway, \lsthree[\i-1]*0.5 + \lsthree[\i]*0.5 + \lspathraise) -- (\midway, \lsthree[\i] + \lspathraise);
						}
					}
					{
						\draw[colorls1, thick, cap=round] (\midwayprev, \lsone[\i-1] + \lspathraise) -- (\midway, \lsone[\i] + \lspathraise);
						\draw[colorls2, thick, cap=round] (\midwayprev, \lstwo[\i-1] + \lspathraise) -- (\midway, \lstwo[\i] + \lspathraise);
						\draw[colorls3, thick, cap=round] (\midwayprev, \lsthree[\i-1] + \lspathraise) -- (\midway, \lsthree[\i] + \lspathraise);
					}
				}
				{
					\draw[colorls1, thick, cap=round] (\midwayprev, \lsone[\i] + \lspathraise) -- (\midway, \lsone[\i] + \lspathraise);
					\draw[colorls2, thick, cap=round] (\midwayprev, \lstwo[\i] + \lspathraise) -- (\midway, \lstwo[\i] + \lspathraise);
					\draw[colorls3, thick, cap=round] (\midwayprev, \lsthree[\i] + \lspathraise) -- (\midway, \lsthree[\i] + \lspathraise);
				}
				\ifthenelse{\i=5} 
				{
					\draw[colorls1, thick, cap=round] (\midway, \lsone[\i] + \lspathraise) -- (\rs[\i+1], \lsone[\i] + \lspathraise);
					\draw[colorls2, thick, cap=round] (\midway, \lstwo[\i] + \lspathraise) -- (\rs[\i+1], \lstwo[\i] + \lspathraise);
					\draw[colorls3, thick, cap=round] (\midway, \lsthree[\i] + \lspathraise) -- (\rs[\i+1], \lsthree[\i] + \lspathraise);
				}
				{}
			}
			{
			}
			
			\global\let\midwayprev=\midway
		}
		
		\ifthenelse{\logsigversion=1}{}{\draw[->, dashed] (\rs[3] + 0.08, 3.8 - \hreduce) -- (\rs[4] - 0.08, 3.8 - \hreduce);}
		
		\draw[blue!50, thick, cap=round] (\rs[0], 4.8 - \hreduce) .. controls (0.5, 3.5 - \hreduce) and (2, 5.3 - \hreduce) .. (\rs[3], 4.8 - \hreduce);
		\draw[blue!50, thick, cap=round, dashed] (\rs[3], 4.8 - \hreduce) .. controls (\rs[3] + 0.2, 4.7 - \hreduce) and (\rs[4] - 0.2, 4.3 - \hreduce) .. (\rs[4], 4.3 - \hreduce);
		\draw[blue!50, thick, cap=round] (\rs[4], 4.3 - \hreduce) .. controls (\rs[4] + 0.3, 4.2 - \hreduce) and (\rs[6] - 0.4, 4.4 - \hreduce) .. (\rs[6], 4.7 - \hreduce);
		
		\ifthenelse{\logsigversion=1}
		{\node at (4, 1.7) {$\cdots$};}
		{}
		
		\node at (8, 0) {\footnotesize Time};
		\node (data) at (8, 0.4) {\footnotesize Data $\mathbf{x}$};
		\node (path) at (8, 1.15) {\footnotesize Path $x_\mathbf{x}$};
		\ifthenelse{\logsigversion=1}{
			\node (logsig) at (8, 2.1) {\footnotesize $\logsig^M_{r_i, r_{i+1}}(x_\mathbf{x})$}; 
			\node (logsig_path) at (8, 2.95) {\footnotesize Logsignature path}; 
		}{}
		\node (hidden) at (8, 4.7 - \hreduce) {\footnotesize Hidden state $y(t)$}; 
		
		\draw[->] (data) -- (path);
		\ifthenelse{\logsigversion=1}
		{
			\draw[->] (path) -- (logsig);
			\draw[->] (logsig) -- (logsig_path);
			\draw [decorate, decoration={calligraphic brace,amplitude=1.5pt, mirror}] (6.5,3.6) -- node[midway, right, xshift=-5.5pt] {\scriptsize \begin{tabular}{l}Integration\\steps\end{tabular}} (6.5,4 - \hreduce);
			\draw[->] (8.3, 3.1) -- (8.3, 4.55 - \hreduce);
		}
		{
			\draw [decorate, decoration={calligraphic brace,amplitude=1.5pt, mirror}] (6.5, 3.6 - \hreduce) -- node[midway, right, xshift=-5.5pt] {\scriptsize \begin{tabular}{l}Integration\\steps\end{tabular}} (6.5, 4 - \hreduce);
			\draw[->] (8.3, 1.3) -- (8.3, 4.55 - \hreduce);
		}
		
	\end{tikzpicture}
}

\begin{figure}
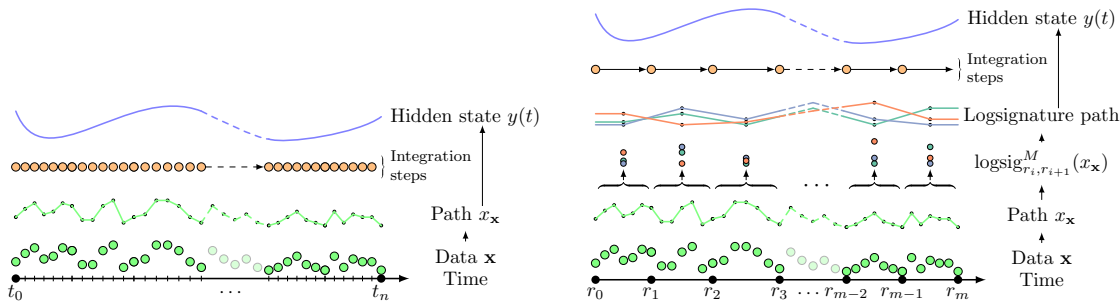

    \begin{subfigure}[t]{.5\linewidth}
        \centering
            \resizebox{\linewidth}{!}{
                \ncdediagram{0}
            }
    \end{subfigure}%
    \begin{subfigure}[t]{.5\linewidth}
        \centering
        \resizebox{\linewidth}{!}{
            \ncdediagram{1}
        }
    \end{subfigure}
    \caption{An overview of the log-ODE method. \textbf{Left:} The CDE approach as introduced in Chapter \ref{chapter:neural-cde}. The path $x$ is quickly varying, so that many integration steps may be needed to resolve it. \textbf{Right:} The log-ODE method with integration steps larger than the discretisation of the data. The path of logsignatures is more slowly varying (in a higher dimensional space), and needs fewer integration steps to resolve.}
    \label{fig:nrde}
\end{figure}

\begin{remark}
	These two approaches are intuitively similar. The quotient
	\begin{equation*}
		\frac{\logsig^M_{r_j, r_{j + 1}}(x)}{r_{j + 1} - r_j}
	\end{equation*}
	is roughly equivalent to the difference quotient
	\begin{equation*}
		\frac{x_\mathbf{x}(r_{j+1}) - x_\mathbf{x}(r_j)}{r_{j+1} - r_j} \approx \frac{\dd x_\mathbf{x}}{\dd t}(r_j).
	\end{equation*}
\end{remark}

\begin{remark}
	A continuous piecewise linear interpolation $x_\mathbf{x}$ is chosen as these are the only paths for which efficient algorithms for computing the (log)signature are known \cite{signatory}.
\end{remark}

\subsection{Discussion}
\paragraph{Length/channel trade-off}
The sequence of logsignatures is now of length $m \ll n$. As such, it is much more slowly varying over the interval $[t_0, t_n]$ than the original data, which was of length n. Correspondingly the differential equation \eqref{eq:nrde} is better-behaved than the original \eqref{eq:ncde-to-node}, and so larger integration steps may be used in the numerical solver. This is the source of the speed-ups of this method; we observe typical speed-ups by a factor of about ten.

\paragraph{Ease of implementation}
Note that \eqref{eq:nrde} is of precisely the same form as \eqref{eq:ncde-to-node}, with the driving path taken to be piecewise linear in logsignature space.

Correspondingly the log-ODE method may be implemented by preprocessing the data into logsignatures, calculated over each window $[r_j, r_{j + 1})$, interpolating the sequence of logsignatures into a piecewise linear path, and then solving a neural CDE as normal.

Every step in this procedure already exists as a software library, see \cite{torchcde}.

\paragraph{The log-ODE method as a binning procedure}
The interpretation of the previous heading draws an important connection to machine learning: the logsignature may be treated as a carefully-selected binning method, to reduce the amount of data considered whilst retaining the information most important for solving a CDE.

\paragraph{Modelling the vector field}
We avoided modelling some $f_\theta \colon \reals^{d_y} \to \reals^{d_y \times d_x}$ and then computing some `$\widehat{f_\theta}_M$', and instead modelled an $f_\theta \colon \reals^{d_y} \to \reals^{d_y \times \beta(d_x, N)}$ directly. Doing so avoids the computational expensive of computing some `$\widehat{f_\theta}_M$'.

\paragraph{Depth and step hyperparameters}
To solve a neural RDE accurately via the log-ODE method, we should be prepared to take the depth $M$ suitably large, or the intervals $r_{j+1} - r_j$ suitably small. In practice accomplishing this would require that these are taken infeasibly large or small, respectively. Instead, we treat these as hyperparameters. This makes the use of the log-ODE method a modelling choice rather than an implementation detail. This is a baked-in discretisation as in Section \ref{section:numerical:baked-in-discretisation}.

Increasing step size will lead to faster (but less informative) training by reducing the number of operations. Increasing depth will lead to slower (but more informative) training, as more information about each local interval is used in each update.

\subsection{Efficacy on long time series}
In principle the log-ODE method may be applied when applying neural CDEs in any context. However it is particularly helpful when applied to long time series.

\paragraph{Improved gradients}
It is a classical fact that it is relatively difficult to train RNNs (and, as they are of essentially the same character, neural CDEs) directly on long time series. During training RNNs may suffer from vanishing/exploding gradients, reducing overall model performance. See also Remark \ref{remark:cde:gru-ode}, which discusses the exponential decay of hidden state over time.

Reducing the length of the time series, as with the log-ODE method, is a simple and effective way to combat this.

\paragraph{Computational efficiency}
The sheer number of operations required to process a long time series implies a long computation time.

This is aggravated by the fact that the inherently serial nature of a neural CDE or RNN, working its way along the time series, is almost impossible to parallelise. (Section \ref{section:misc:multiple-shooting} notwithstanding.) In contrast a large number of channels in the time series is much less of an issue, due to the availability of parallelism.

Once again, reducing the length of the time series helps combat this. That this is performed via preprocessing is particularly beneficial for training: the computation of logsignatures need only be done once prior to training.

\paragraph{Memory efficiency}
Long time series consume substantial memory in order to perform backpropagation-through-time. As discussed in Section \ref{section:cde:training}, CDEs offer an attractive way to handle this through the availability of continuous adjoint methods, which consume only $\bigO{H + n}$ memory in the network size $H$ and the input length $n$. The log-ODE method further improves upon this by reducing the memory cost to $\bigO{H + m}$ with $m \ll n$.

\subsection{Limitations}
\paragraph{Number of hyperparameters} Two new hyperparameters -- truncation depth and step size -- with substantial effects on training time and memory usage must now also be tuned.

\paragraph{Number of input channels} The log-ODE method is most feasible with few input channels, as the number of log-signature channels $\beta(d, M)$ grows exponentially in $d$. For larger $d$ then the available parallelism may become saturated.

\section{Examples}\index{Examples!Neural RDEs}\label{appendix:rde:example}

\paragraph{Datasets}
We apply neural RDEs to three real-world datasets from the TSR archive \cite{MonashTSRegressionArchive}, coming originally from the Beth Israel Deaconess Medical Centre (BIDMC).

The goal is to predict a person's respiratory rate (RR), heart rate (HR), or oxygen saturation (SpO2) at the end of the sample, having observed photoplethysmography (PPG) and electrocardiogram (ECG) data over the length of the sample. The data is regularly sampled at 125Hz and each series has length 4\,000. There are 3 channels (including time). Performance is evaluated using the $L^2$ loss.

Every problem was chosen for its long length. The lengths are sufficiently long that optimise-then-discretise backpropagation (Section \ref{section:numerical:adjoint-cde-sde}) was needed simply to avoid running out of memory at any reasonable batch size.\footnote{Reversible solvers should/could also have been employed.}

\paragraph{Models}
The logsignature depth $M$ is varied over $\{2, 3\}$. ($M=1$ is identical to the standard neural CDE as per Example \ref{example:rde:depth-one}.) Likewise the number of observations within each interval $[r_j, r_{j+1}]$ were varied over $\{8, 128, 512\}$, which we refer to as the step size. In practice both depth and step size should be chosen as hyperparameters.

Two baseline models are also included. The first is a neural CDE; as the model we are extending then comparisons to this are our primary concern. A baseline against the ODE-RNN of \cite{latent-odes} is also included. For the neural CDE model, increased step sizes correspond to na{\"i}ve subsampling of the data. For the ODE-RNN model, the time dimension is instead folded into the feature dimension, so that at each step the ODE-RNN model sees several adjacent time points; this is an alternate technique for dealing with long time series.

\paragraph{Results}

\begin{table}
	\tiny
	\begin{center}
		\begin{tabular}{cccccccccc}
			\toprule
			\multirow{2}{*}{\textbf{Model}} &
			&
			\multirow{2}{*}{\begin{tabular}{c}\textbf{Step}\\\textbf{size}\end{tabular}} & \multicolumn{3}{c}{$\mathbf{L^2}$} & \multicolumn{3}{c}{\textbf{Time (Hrs)}} & \multirow{2}{*}{\textbf{Memory (Mb)}} \\
			\cmidrule(lr){4-6} \cmidrule(lr){7-9}
			& & & RR & HR & SpO$_2$ & RR & HR & SpO$_2$ & \\
			
			\midrule
			
			& & 1   &  -- & 13.06  $\pm$  0.0 & -- & -- & 10.5 & -- & 3653.0 \\
			ODE-RNN (folded) & \multirow{2}{*}{} & 8   & 2.47 $\pm$ 0.35 & 13.06 $\pm$ 0.00 & 3.3 $\pm$ 0.00 & 1.5 & 1.2 & 0.9 & 917.2 \\
			& & 128  &  1.62 $\pm$ 0.07 &  13.06 $\pm$ 0.00 &    3.3 $\pm$ 0.00 &           0.2 &           0.1 &           0.1 &               81.9 \\
			& & 512  &  1.66 $\pm$ 0.06 &   6.75 $\pm$ 0.9 &  1.98 $\pm$ 0.31 &           0.0 &           0.1 &           0.1 &               40.4 \\
			
			\hdashline\noalign{\vskip 0.5ex}
			
			& & 1   &  2.79 $\pm$ 0.04 &   9.82 $\pm$ 0.34 &  2.83 $\pm$ 0.27 &          23.8 &          22.1 &          28.1 &               56.5 \\
			\multirow{2}{*}{NCDE} & \multirow{2}{*}{} & 8   &   2.80 $\pm$ 0.06 &  10.72 $\pm$ 0.24 &  3.43 $\pm$ 0.17 &           3.0 &           2.6 &           4.8 &               14.3 \\
			&  & 128 &  2.64 $\pm$ 0.18 &  11.98 $\pm$ 0.37 &  2.86 $\pm$ 0.04 &           0.2 &           0.2 &           0.3 &                8.7 \\
			&  & 512 &  2.53 $\pm$ 0.03 &  12.22 $\pm$ 0.11 &  2.98 $\pm$ 0.04 &           0.1 &           0.0 &           0.1 &                8.4 \\
			
			\midrule
			& & 8   &  2.63 $\pm$ 0.12 &   8.63 $\pm$ 0.24 &  2.88 $\pm$ 0.15 &           2.1 &           3.4 &           3.3 &               21.8 \\
			NRDE (depth 2) &  & 128  &  1.86 $\pm$ 0.03 &   6.77 $\pm$ 0.42 &  1.95 $\pm$ 0.18 &           0.3 &           0.4 &           0.7 &               10.9 \\
			&  & 512 &  1.81 $\pm$ 0.02 &   5.05 $\pm$ 0.23 &  2.17 $\pm$ 0.18 &           0.1 &           0.2 &           0.4 &               10.3 \\
			\hdashline\noalign{\vskip 0.5ex}
			& & 8   &  \textbf{2.42 $\pm$ 0.19} &    \textbf{7.67 $\pm$ 0.40} &  \textbf{2.55 $\pm$ 0.13} &           2.9 &           3.2 &           3.1 &               43.3 \\
			NRDE (depth 3) &   & 128  &  \textbf{1.51 $\pm$ 0.08} &   \textbf{$\;$2.97 $\pm$ 0.45}$^*$ &  \textbf{1.37 $\pm$ 0.22} &           0.5 &           1.7 &           1.7 &               17.3 \\
			&  & 512  &  \textbf{$\;$1.49 $\pm$ 0.08}$^*$ &   \textbf{3.46 $\pm$ 0.13} &  \textbf{$\;$1.29 $\pm$ 0.15}$^*$ &           0.3 &           0.4 &           0.4 &               15.4 \\
			\bottomrule
		\end{tabular}
	\end{center}
	\caption{Experiments on the three BIDMC datasets. Mean $\pm$ standard deviation of test set $L^2$ loss, measured over three repeats, over each of three different vital signs prediction tasks (RR, HR, SpO$_2$). Also reported are memory usage and training time. Only mean times are shown for space. `--' denotes that the model could not be run within GPU memory. Bold denotes the best model score for a given step size, and $^*$ denotes that the score was the best achieved over all models and step sizes.}
	\label{table:rde:bidmc}
\end{table}

The results are shown in table \ref{table:rde:bidmc}.

We find that the depth-$3$ neural RDE is the top performer for every task at every step size, reducing test loss by $30$--$59\%$ compared to the corresponding neural CDE. Moreover, it does so with roughly an order of magnitude less training time. The ODE-RNN baseline produces poor results whilst requiring significantly more memory.

We attribute the improved test loss to the neural RDE model being better able to learn long-term dependencies due to the reduced sequence length: the performance of the rough models actually improves as the step size is increased.

Details of hyperparameter selection, optimisers, normalisation, and so on can be found in Appendix \ref{appendix:rde:experiment-details}. Additional results can be found in \cite[Appendix D]{morrill2021neuralrough}.

\section{Comments}
The log-ODE method is a classic approach in CDEs and rough path theory; see for example \cite[Section 7]{lyons2004}. A great many standard numerical SDE solvers -- such as the Euler--Maruyama method or Heun's method -- are obtained as ODE solvers applied to the depth-1 log-ODE method. Higher-order numerical SDE solvers -- such as Milstein's method -- are almost (but not exactly) equivalent to applying an ODE solver to the depth-2 log-ODE method.

An excellent brief introduction to signatures and logsignatures are provided by \cite{jeremy-logsig-pdf, iisignature}. An efficient computational implementation of signatures and logsignatures is provided by \cite{signatory}.

Neural rough differential equations were introduced in \cite{morrill2021neuralrough}, which is also where their application to long time series was considered. In particular see the appendices of \cite{morrill2021neuralrough} (and references therein) for further mathematical details beyond those presented here, such as the convergence of the log-ODE method.
\chapter{Proofs and Algorithms}

\section{Augmented neural ODEs are universal approximators even when their vector fields are not universal approximators}\label{appendix:ode-universal-augmented}\index{Universal approximation!ODEs}

Recall Theorem \ref{theorem:ode:univ-approx-wide}.

\myunivapprox*

\begin{proof}	
	Given some $x \in \reals^d$, consider the system of ODEs $y_0, \ldots, y_M$ solving
	\begingroup
	\allowdisplaybreaks
	\begin{align}
		y_0(0) = x \in \reals^d \quad& \frac{\dd y_0}{\dd t}(t) = 0,\nonumber\\
		y_1(0) = 0 \in \reals^{d \times d} \quad& \frac{\dd y_1}{\dd t}(t) = y_0(t) \otimes y_0(t),\nonumber\\
		y_2(0) = 0 \in \reals^{d \times d \times d} \quad& \frac{\dd y_2}{\dd t}(t) = y_1(t) \otimes y_0(t),\nonumber\\
		y_3(0) = 0 \in \reals^{d \times d \times d \times d} \quad& \frac{\dd y_3}{\dd t}(t) = y_2(t) \otimes y_0(t),\nonumber\\
		\cdots&\nonumber\\
		y_M(0) = 0 \in \reals^{d \times \cdots \times d} \quad& \frac{\dd y_M}{\dd t}(t) = y_{M-1}(t) \otimes y_0(t).\label{eq:proofs:_univ-approx-wide}
	\end{align}
	\endgroup

	The solution may be written down immediately:
	\begin{align*}
		y_0(t) &= x,\\
		y_1(t) &= \int_0^t x \otimes x \,\dd s = t x \otimes x,\\
		y_2(t) &= \int_0^t s x \otimes x \,\dd s = \frac{1}{2} t^2 x^{\otimes 3},\\
		y_3(t) &= \int_0^t \frac{1}{2} s^2 x^{\otimes 3} \otimes x \,\dd s = \frac{1}{6} t^3 x^{\otimes 4},\\
		\cdots&\\
		y_M(t) &= \frac{1}{M!} t^M x^{\otimes (M + 1)}.
	\end{align*}

	Evaluating at $t=1$ we obtain the collection of all (scaled) monomials in $x \in \reals^d$ up to degree $M + 1$, namely
	\begin{equation*}
		\left(x, x^{\otimes x}, \frac{1}{2!} x^{\otimes 3}, \frac{1}{3!} x^{\otimes 4}, \ldots, \frac{1}{M!} x^{\otimes (M+1)}\right).
	\end{equation*}

	The Stone--Weierstrass theorem states that polynomials are dense in the space of continuous functions $C(K; \reals^{d_o})$. Thus for any target $F \in C(K; \reals^{d_o})$ and $\varepsilon > 0$, there exists some $M \in \naturals$ large enough, and some affine map combining these monomials to form a polynomial $P$, such that
	\begin{equation*}
		\norm{F - P}_{\infty, K} < \infty.
	\end{equation*}

	The result is now proved. For each $d_l = \sum_{k=1}^{M+1} d^k$ let $f_{d_l}$ be the vector field specified in equation \eqref{eq:proofs:_univ-approx-wide}. Let each $\ell_1$ be the affine map augmenting $x$ with sufficient zeros for the initial condition. Let $\ell_2$ be the affine map transforming the monomials to form any given polynomial.
\end{proof}

\subsection{Comments}
It is perhaps a little questionable whether the construction shown here is truly a `neural ODE'. The only learnt parameters are in the final affine $\ell_2$. More subtly, the equation of \eqref{eq:proofs:_univ-approx-wide} are questionably ODEs: the vector field for each $y_k$ does not depend on $y_k$ (only $y_{k-1}$ and $y_0$), and is thus `only' an integral.
	
On the other hand, this is still essentially the same argument for universal approximation as for wide neural networks (\cite{pinkus1999}) or a Fourier series -- that is, a linear combination of enough terms -- so perhaps we should not complain.

This result is actually a special case of the universal approximation theorem for CDEs (Appendix \ref{appendix:cde:universal-approximation}). Given the input $x \in \reals^d$, define the continuous path $z \colon [0, T] \to \reals^d$ by $z(t) = tx$. Then the proof here is just a simplification and particular application of that result.

\section{Theoretical properties of neural CDEs}
\subsection{Neural CDEs are universal approximators}\label{appendix:cde:universal-approximation}\index{Universal approximation!CDEs}

We begin with universal approximation of CDEs with respect to continuous paths $X$. We then show how to extend this to universal approximation with respect to time series, in a generic way independent of the choice of interpolation, by requiring that the interpolation satisfy certain conditions.

\subsubsection{Universal approximation with respect to paths}
\begin{definition}
	Let $T > 0$ and let $d \in \naturals$. Let	
		\begin{equation*}
		\mathcal{V}^1([0, T]; \reals^d) = C([0, T]; \reals^d) \cap \BV([0, T]; \reals^d)
		\end{equation*}
	represent the space of continuous functions of bounded variation. Equip this space with the norm
	\begin{equation*}
		x \mapsto \norm{x}_{\mathcal{V}} = \norm{x}_{\infty} + \abs{X}_{\BV}.
	\end{equation*}
\end{definition}

\begin{remark}
This is a somewhat unusual norm to use, as bounded variation seminorms are more closely aligned with $L^1$ norms than $L^\infty$ norms.
\end{remark}

\begin{definition}
	Let $\mathcal{V}^1_0([0, T]; \reals^{d}) = \set{x \in \mathcal{V}^1([0, T]; \reals^{d})}{x(0) = 0}$.
\end{definition}

\begin{definition}\label{definition:cde-solution}
	For $f \in \Lip(\reals^{d_y}; \reals^{d_y \times d_x})$, $\zeta \in C(\reals^{d_x}; \reals^{d_y})$, $x \in \mathcal{V}^1([0, T]; \reals^{d_x})$, let $y_{f, \zeta, x} \colon [0, T] \to \reals^{d_y}$ denote the unique solution to the CDE
	\begin{equation*}
		y_{f, \zeta, x}(t) = y_{f, \zeta, x}(0) + \int_0^t f(y_{f, \zeta, x}(s)) \,\dd x(s)\quad\text{for $t \in (0, T]$,}
	\end{equation*}
	with $y_{f, \zeta, x}(0) = \zeta(x(0))$.
\end{definition}

\begin{definition}
	For any $d, M \in \naturals$, let $\kappa(d, M) = \sum_{i = 0}^M d^i$.
\end{definition}

\begin{definition}[Signature transform]\label{definition:signature}\index{Signatures}
	Let $x \in \mathcal{V}^1([0, T]; \reals^{d})$. Define the iterated Riemann--Stieltjes integrals
	\begin{equation*}
		S^{k_1, \ldots, k_m}_{a, b}(x) = \underset{a < t_1 < \cdots < t_m < b}{\int \cdots \int} \dd x_{k_1}(t_1) \cdots \dd x_{k_m}(t_m) \in \reals.
	\end{equation*}

	Let $M \in \naturals$. Put all such integrals, up to maximal index $M$, together into a single object:
	\begin{equation}
		\sig^M_{a, b}(x) = \left(1,
		\{
		S^{k}_{a, b}(x)
		\}_{k = 1}^d,
		\{
		S^{k_1, k_2}_{a, b}(x)
		\}_{k_1, k_2 = 1}^d,
		\ldots,
		\{
		S^{k_1, \ldots, k_M}_{a, b}(x)
		\}_{k_1, \ldots, k_M = 1}^d
		\right).\label{eq:proofs:_signature}
	\end{equation}
	By convention $1 \in \reals$ is also included at the start. Then $\sig^M_{a, b}(x)$ is known as the depth-$M$ signature transform of $x$.
	
	It is immediate from the definition that each term in the signature satisfies
	\begin{equation*}
		S^{k_1, \ldots, k_m}_{a, t}(x) = \int_0^t S^{k_1, \ldots, k_{m-1}}_{a, s}(x) \,\dd x_{k_m}(s).
	\end{equation*}
	By stacking all such equations together it is clear that there exists some
	\begin{equation*}
		f \in L(\reals^{\kappa(d, M)}; \reals^{\kappa(d, M) \times d_x})
	\end{equation*}
	such that $\sig^M_{a, \cdot}(x)$ satisfies the CDE
	\begin{equation*}
		\sig^M_{a, t}(x) = (1, 0, \ldots, 0) + \int_a^t f(\sig^M_{a, t}(x)) \,\dd x(t).
	\end{equation*}
\end{definition}

\begin{definition}\label{definition:uniqueness-of-signatures}
	Let $K \subseteq \mathcal{V}^1([0, T]; \reals^{d})$. We say that $K$ has \textit{uniqueness of signatures} if for all $x, \widehat{x} \in K$ with $x(0) = \widehat{x}(0)$, there exists $M \in \naturals$ such that $\sig^M_{0, T}(x) \neq \sig^M_{0, T}(\widehat{x})$.
\end{definition}

Practically speaking uniqueness of signatures is most easily obtained through the following lemma.

\begin{lemma}\label{lemma:cde:monotone-sig}
	For any $K \subseteq \mathcal{V}^1([0, T]; \reals^{d - 1})$, then
	\begin{equation*}
		K' = \set{(t \mapsto (t, x(t))) \in \mathcal{V}^1([0, T]; \reals^{d})}{x \in K}
	\end{equation*}
	has uniqueness of signatures.
\end{lemma}
\begin{proof}
	Without loss of generality assume $d=2$, as we will treat each channel of $K$ separately.
	
	Fix $x \in K$ with corresponding element $x' \in K'$. The (arbitrary depth) signature of $x'$ over $[0, T]$ contains all terms of the form
	\begin{equation}\label{eq:proofs:_sig-monomial}
		\int_0^T \underset{0 < t_1 < \cdots < t_m < t}{\int \cdots \int} \dd t_1 \cdots \dd t_m \dd x(t) = \int_0^T \frac{1}{m!}t^m \,\dd x(t).
	\end{equation}

	Fix $g \in C([0, T])$. Let $p_n$ be some sequence of polynomials for which $p_n \to g$ uniformly over $[0, T]$, which exist by the Weierstrass Approximation Theorem. Then $\int_0^T p_n(t) \,\dd x(t) \to \int_0^T g(t)\,\dd x(t)$ \cite[Proposition 2.8]{FrizVictoir10}. By \eqref{eq:proofs:_sig-monomial} all $\int_0^T p_n(t) \,\dd x(t)$ are determined by the signature of $x'$, and so for all $g \in C([0, T])$ we have that $\int_0^T g(t)\,\dd x(t)$ is determined by the signature of $x'$.
	
	Fix $\varepsilon > 0$ and $s \in [0, T - \varepsilon]$ and consider specifically $g_{s, \varepsilon} \in C([0, T])$ defined by
	\begin{equation*}
		g_{s, \varepsilon}(t) = \begin{cases}1 & t \in [0, s), \\ \frac{1}{\varepsilon}(s + \varepsilon - t) & t \in [s, s + \varepsilon), \\ 0 & t \in [s + \varepsilon, T]. \end{cases}
	\end{equation*}
	Then $\int_0^T g_{s, \varepsilon}(t)\,\dd x(t) = x(s) - x(0) + \bigO{\varepsilon}$. Letting $\varepsilon \to 0$ we have that every increment $x(s) - x(0)$ is determined by the signature of $x'$.
\end{proof}

\begin{remark}
	This fact is the fundamental reason that time is included as a channel in Section \ref{section:cde:inclusion-of-time}.
\end{remark}

\begin{remark}
\cite{hambly2010uniqueness} give a precise characterisation of this property, which is that all $x, \widehat{x}$ must lie in different equivalence classes with respect to `tree-like equivalence'.
\end{remark}

With these definitions out of the way, we are ready to state the famous universal nonlinearity property of the signature transform. We think \cite[Theorem 4.2]{perezarribas2018} gives the most straightforward proof of this result. This essentially states that the signature gives a basis for the space of functions on compact path space.
\begin{theorem}[Universal nonlinearity]\label{theorem:universal-nonlinearity}
	Let $T > 0$ and let $d_x, d_o \in \naturals$. Let $K \subseteq \mathcal{V}^1_0([0, T]; \reals^{d_x})$ be compact and have uniqueness of signatures.
	
	Then
	\begin{equation*}
		\bigcup_{M \in \naturals} \set{x \mapsto \ell(\sig^M_{0, T}(x))}{\ell \in L(\reals^{\kappa(d, M)}; \reals^{d_o})}
	\end{equation*}
	is dense in $C(K; \reals^{d_o})$.
\end{theorem}

With the universal nonlinearity property, we can now prove universal approximation of CDEs with respect to controlling paths $x$.
\begin{theorem}[Universal approximation with CDEs]\label{theorem:pathwise-universal-approximation}
	Let $T > 0$ and let $d_x, d_o \in \naturals$. Let $K \subseteq \mathcal{V}^1([0, T]; \reals^{d_x})$ be compact and have uniqueness of signatures.
	
	Then
	\begin{equation*}
		\bigcup_{d_y \in \naturals} \set{x \mapsto \ell(z_{f, \zeta, x}(T))}{f \in \Lip(\reals^{d_y}; \reals^{d_y \times d_x}), \zeta \in C(\reals^{d_x}; \reals^{d_y}), \ell \in L(\reals^{d_y}; \reals^{d_o})}
	\end{equation*}
	is dense in $C(K; \reals^{d_o})$, where $z_{f, \zeta, x}$ is as defined in Definition \ref{definition:cde-solution}.
\end{theorem}
\begin{proof}
	We begin by prepending a straight line segment to every element of $K$. For every $x \in K$, define $x^* \colon [-1, T] \to \reals^{d_x}$ by
	\begin{equation*}
		x^*(t) = \begin{cases}
			(t + 1)x(0) & t \in [-1 , 0),\\
			x(t) & t \in [0, T].
		\end{cases}
	\end{equation*}
	
	Then $K^* = \set{x^*}{x \in K} \subseteq \mathcal{V}^1_0([-1, T]; \reals^{d_x})$ is compact. By Theorem \ref{theorem:universal-nonlinearity},
	\begin{equation*}
		\bigcup_{M \in \naturals} \set{x^* \mapsto \ell(\sig^m_{-1, T}(x^*))}{\ell \in L(\reals^{\kappa(d, M)}; \reals^{d_o})}
	\end{equation*}
	is dense in $C(K^*; \reals^{d_o})$.
	
	So let $\alpha \in C(K; \reals^{d_o})$ and $\varepsilon > 0$. The map $x \mapsto x^*$ is a homeomorphism, so we may find $\beta \in C(K^*; \reals^{d_o})$ such that $\beta(x^*) = \alpha(x)$ for all $x \in K$. We have just established there exists some $M \in \naturals$ and $\ell \in L(\reals^{\kappa(d, M)}; \reals^{d_o})$ such that $\gamma$ defined by $\gamma \colon x^* \mapsto \ell(\sig^M_{-1, T}(x^*))$ is $\varepsilon$-close to $\beta$.
	
	By Definition \ref{definition:signature} there exists $f \in \Lip(\reals^{\kappa(d, M)}; \reals^{\kappa(d, M) \times d_x})$ so that $\sig^M_{-1, t}(x^*)$ is the unique solution of the CDE
	\begin{equation*}
		\sig^M_{-1, t}(x^*) = \sig^M_{-1, -1}(x^*) + \int_{-1}^t f(\sig^M_{-1, t}(x^*)) \,\dd x^*(s)\quad\text{for $t \in (-1, T]$,}
	\end{equation*}
	with $\sig^M_{-1, -1}(x^*) = (1, 0, \ldots, 0)$.
	
	Now let $\zeta \in C(\reals^{d_x}; \reals^{d_y})$ be defined by $\zeta(x(0)) = \sig^M_{-1, 0}(x^*)$, which we note is well defined because for $t \in [-1, 0]$ the value of $\sig^M_{-1, t}(x^*)$ only depends on $x(0)$. Then (by uniqueness of solution) we have that $\sig^M_{-1, t}(x^*) = z_{f, \zeta, x^*}(t)$ for $t \in [0, T]$.
	
	For all $x \in K$,
	\begin{equation*}
		\ell(z_{f, \zeta, x^*}(T)) = \ell(\sig^M_{-1, T}(x^*)) = \gamma(x^*)
	\end{equation*}
	is $\varepsilon$-close to $\beta(x^*) = \alpha(x)$. Thus density has been established.
\end{proof}

\begin{remark}
	For the reader familiar with rough path theory, the above proof is essentially just premultiplying the signature of $x$ by the signature of the straight line increment from $0$ to $x(0)$ so as to remove translational invariance.
\end{remark}

\subsubsection{Universal approximation with respect to time series}
Of course, the input to a neural CDE will often not be a continuous path. Very often it will instead be a discretised time series, which we interpolate. We need to extend our universal approximation result to this case. Our approach here will be agnostic to the choice of interpolation, and will instead impose conditions that the interpolation scheme must satisfy in order to provide universal approximation.

\begin{definition}[Space of time series]\label{definition:cde:time-series}
	Let $d \in \naturals$. We define the set of irregularly sampled time series over $\reals^d$ as
	\begin{align*}
		\ts{\reals^d} = \set{((t_0, x_0), \ldots, (t_n, x_n))}{n \in \naturals, t_j \in \reals, x_j \in (\reals \cup \{*\})^d, t_0 = 0, t_j < t_{j + 1}}.
	\end{align*}
	where $*$ denotes the possibility of missing data.
\end{definition}

\begin{definition}
	For each $\mathbf{x} = ((t_0, x_0), \ldots, (t_n, x_n)) \in \ts{\reals^d}$, let $x_j = (x_{j, 1}, \ldots, x_{j, d}) \in \reals^d$ and for notational convenience let $x_{j, 0} = t_j$. Then we define $\tssize(\mathbf{x})$ by
	\begin{equation*}
	\tssize(\mathbf{x}) = \max\left\{n, \max_{j = 0, \ldots, n}\max_{k = 0, \ldots, d}\abs{x_{j, k}}, \max_{j = 0, \ldots, n-1}\max_{k=0,\ldots,d}\frac{x_{j+1, k} - x_{j, k}}{t_{j+1} - t_j}, \max_{j = 0,\ldots,n-1}\max_{k=0,\ldots,d}\frac{x_{j+1, k} - x_{j, k}}{(t_{j+1} - t_j)^2}\right\}.
	\end{equation*}
	For simplicity the above definition ignores the presence of missing data. If necessary replace each $*$ with a $0$ to make the above well-defined.
\end{definition}

\begin{definition}
	For all $\mathbf{x} = ((t_0, x_0), \ldots, (t_n, x_n)) \in \ts{\reals^d}$, decompose $x_j = (x_{j,1}, \ldots, x_{j, d}) \in \reals^d$, and then define $c_j(\mathbf{x}) = (c_{j,1}(\mathbf{x}), \ldots, c_{j, d}(\mathbf{x})) \in \reals^{d}$, where $c_{j, k}(\mathbf{x}) = \sum_{m=0}^j \indicator{x_{m, k} \neq *}$ counts the number of observations in the $k$th channel by time $t_j$.
\end{definition}

\begin{definition}[Interpolation]\label{definition:cde:interpolation}\index{Interpolation}
	Let $T > 0$. Let $\mathcal{X} \subseteq \ts{\reals^d}$. We define an \emph{interpolation} as a map
	\begin{align*}
		\mathcal{X} &\to \mathcal{V}^1([0, T]; \reals^{2d + 1}),\\
		\mathbf{x} &\mapsto x_\mathbf{x},
	\end{align*}
	together with a collection of $0 = s_0 < \cdots < s_n = T$, such that
	\begin{equation}\label{eq:definition-interpolation}
		x_\mathbf{x}(s_j) = (t_j, x_j, c_j(\mathbf{x})) \in \reals^{2d + 1}
	\end{equation}
	for all $\mathbf{x} = ((t_0, x_0), \ldots, (t_n, x_n)) \in \mathcal{X}$ and $j \in \{0, \ldots, n\}$. (And any missing values $*$ are ignored for the purposes of determining equality in equation \eqref{eq:definition-interpolation}.) The values of $s_j$ may depend upon $\mathbf{x}$.
\end{definition}

\begin{remark}
	If the full dataset of time series $\mathbf{x}$ has no missing values then we may need only a single $c_{j}$ channel to capture the rate of observations. If every time series is additionally regularly sampled then these channels may be omitted altogether, as not carrying any information.
\end{remark}

\begin{definition}[Bounded interpolation]\label{definition:cde:bounded-interpolation}\index{Interpolation!Bounded}
	Let $T > 0$. Let $\mathcal{X} \subseteq \ts{\reals^d}$. Consider the interpolation
	\begin{align*}
		\mathcal{X} &\to \mathcal{V}^1([0, T]; \reals^{2d + 1})\\
		\mathbf{x} &\mapsto x_\mathbf{x}.
	\end{align*}
	We call this a \emph{bounded interpolation} if there exists $C > 0$ so that for all $\mathbf{x} \in \mathcal{X}$,
	\begin{equation*}
		\norm{x_\mathbf{x}}_\infty + \norm{\frac{\dd x_\mathbf{x}}{\dd t}}_\infty + \abs{\frac{\dd x_\mathbf{x}}{\dd t}}_{\BV} < C \tssize(\mathbf{x}).
	\end{equation*}
\end{definition}

\begin{remark}
	It is really for ease of this definition that we restrict an interpolation to being defined on only some $\mathcal{X} \subseteq \ts{\reals^d}$. If an interpolation was defined on all of $\ts{\reals^d}$, and we wished to define a bounded interpolation, then the codomain would need to be all of $\cup_{T > 0}\mathcal{V}^1([0, T]; \reals^{2d + 1})$.
	
	(Otherwise what must happen as the length of a time series increases? The points $s_j$ must be packed closer and closer together, and correspondingly the derivative of the interpolation may tend towards infinity, violating boundedness. Given that we would often like to take $s_j = t_j$ in practice, then the `natural' resolution is to allow the resulting interpolation to be defined over any $[0, T]$.)
	
	Allowing arbitrary domains would complicate the presentation somewhat, so we stick to the simple case. (The general case is mathematically doable, but tedious.)
\end{remark}

\begin{definition}[Signature-unique interpolation]\label{definition:cde:signature-unique-interpolation}\index{Interpolation!Signature-unique}
	Let $\mathcal{X} \subseteq \ts{\reals^d}$. Consider the interpolation
	\begin{align*}
		\mathcal{X} &\to \mathcal{V}^1([0, S]; \reals^{2d + 1})\\
		\mathbf{x} &\mapsto x_\mathbf{x}.
	\end{align*}
	We call this a \emph{signature-unique interpolation} if $\mathbf{x} \mapsto x_\mathbf{x}$ is injective, and if $\set{x_\mathbf{x}}{\mathbf{x} \in \mathcal{X}}$ has uniqueness of signatures in the sense of Definition \ref{definition:uniqueness-of-signatures}.
\end{definition}

\begin{remark}
	Injectivity is included in the above definition only for emphasis -- it is automatically true for any interpolation scheme. In the case of missing data, injectivity holds because of the extra $c_j(\mathbf{x})$ channels of Definition \ref{definition:cde:interpolation}.
	
	For example, $((t_0, x_0), (t_2, x_2))$ and $((t_0, x_0), (t_1, *), (t_2, x_2))$ might otherwise both be interpolated to produce the same result (perhaps a linear interpolation over $[t_0, t_2]$), and injectivity would have been lost.
\end{remark}

\begin{definition}[Time series topologies]
	Given any particular interpolation, we will equip $\mathcal{X} \subseteq \ts{\reals^d}$ with the weakest topology for which that interpolation is continuous.
\end{definition}

\begin{lemma}\label{lemma:bounded-interpolation}
	Let $\mathcal{X} \subseteq \ts{\reals^d}$, and let
	\begin{align*}
		\mathcal{X} &\to \mathcal{V}^1([0, T]; \reals^{2d + 1})\\
		\mathbf{x} &\mapsto x_\mathbf{x}.
	\end{align*}
	be a bounded interpolation. Suppose there exists $C > 0$ such that for all $\mathbf{x} \in \mathcal{X}$ that $\tssize(\mathbf{x}) < C$. Let $\mathfrak{X} = \set{x_\mathbf{x}}{\mathbf{x} \in \mathcal{X}}$. Then $\mathfrak{X}$ is relatively compact (that is, its closure is compact) in $\mathcal{V}^1([0, T]; \reals^{2d + 1})$.
\end{lemma}
\begin{proof}
	By boundedness of the interpolation then
	\begin{equation*}
		\sup_{x \in \mathfrak{X}} \left( \norm{x}_\infty + \norm{\frac{\dd x}{\dd t}}_\infty + \abs{\frac{\dd x}{\dd t}}_{\BV} \right) < \infty.
	\end{equation*}
	
	Now $\mathfrak{X}$ is bounded in $W^{1, \infty}([0, T]; \reals^{2d + 1})$ and so relatively compact in $L^\infty([0, T]; \reals^{2d + 1})$. Let $\mathfrak{X}' = \set{\nicefrac{\dd x}{\dd t}}{x \in \mathfrak{X}}$. Then $\mathfrak{X}'$ is bounded in $\BV([0, T]; \reals^{2d + 1})$ and so relatively compact in $L^1([0, T]; \reals^{2d + 1})$. Therefore $\mathfrak{X} \times \mathfrak{X}'$ is relatively compact in $L^\infty([0, T]; \reals^{2d + 1}) \times L^1([0, T]; \reals^{2d + 1})$.
	
	Let $\mathbb{X} = \set{(x, \nicefrac{\dd x}{\dd t})}{x \in \mathfrak{X}}$. Then $\mathbb{X} \subseteq \mathfrak{X} \times \mathfrak{X}'$ so $\mathbb{X}$ is also relatively compact in $L^\infty([0, T]; \reals^{2d + 1}) \times L^1([0, T]; \reals^{2d + 1})$. This implies that $\mathfrak{X}$ is relatively compact with respect to the topology generated by $x \mapsto \norm{x}_\infty + \norm{\nicefrac{\dd x}{\dd t}}_1$, and hence also with respect to the topology generated by $x \mapsto \norm{x}_\infty + \abs{x}_{\BV}$.
\end{proof}

\begin{theorem}[Universal approximation with neural CDEs on time series]\label{theorem:datawise-universal-approximation}
	Let $d_x, d_y, n \in \naturals$ and let $T>0$.
	
	For all $d_y \in \naturals$, let $F_{d_y} \subseteq \Lip(\reals^{d_y}; \reals^{d_y \times (2d_x + 1)})$ be dense in $C(\reals^{d_y}; \reals^{d_y \times (2d_x + 1)})$. Likewise let $\xi_{d_y} \subseteq C(\reals^{2d_x + 1}; \reals^{d_y})$ be dense in $C(\reals^{2d_x + 1}; \reals^{d_y})$. (Typically these will both be classes of neural networks).
	
	Let $\mathcal{X} \subseteq \ts{\reals^{d_x}}$ be such that there exists $C > 0$ such that
	\begin{equation}
		\tssize(\mathbf{x}) < C \label{eq:compact-require}
	\end{equation}
	for every $\mathbf{x} = ((t_0, x_0), \ldots, (t_n, x_n)) \in \mathcal{X}$. (With $C$ independent of $\mathbf{x}$.)
	
	Let
	\begin{align*}
		\mathcal{X} &\to \mathcal{V}^1([0, T]; \reals^{2d_x + 1})\\
		\mathbf{x} &\mapsto x_\mathbf{x}
	\end{align*}
	be a bounded signature-unique interpolation.
	
	Then
	\begin{equation*}
		\bigcup_{d_y \in \naturals} \set{\mathbf{x} \mapsto \ell(z_{f, \zeta, x_\mathbf{x}}(T))}{f \in F_{d_y}, \zeta \in \xi_{d_y}, \ell \in L(\reals^{d_y}; \reals^{d_o})}
	\end{equation*}
	is dense in $C(\mathcal{X}; \reals^{d_o})$, where $z_{f, \zeta, x}$ is as defined in Definition \ref{definition:cde-solution}.
\end{theorem}
\begin{proof}
	By equation \eqref{eq:compact-require} and boundedness of the interpolation, Lemma \ref{lemma:bounded-interpolation} implies that $\mathfrak{X} = \set{x_\mathbf{x}}{\mathbf{x} \in \mathcal{X}}$ is relatively compact in $\mathcal{V}^1([0, T]; \reals^{2d_x + 1})$.
	
	By Theorem \ref{theorem:pathwise-universal-approximation} and signature-uniqueness of the interpolation,
	\begin{equation*}
		\bigcup_{d_y \in \naturals} \set{x \mapsto \ell(z_{f, \zeta, x}(T))}{f \in \Lip(\reals^{d_y}; \reals^{d_y \times (2d_x + 1)}),\, \zeta \in C(\reals^{2d_x + 1}, \reals^{d_y}),\, \ell \in L(\reals^{d_y}; \reals^{d_o})}
	\end{equation*}
	is dense in $C(\overline{\mathfrak{X}}, \reals^{d_o})$, where the overline denotes a closure. 
	
	For any $f \in F_{d_y}$, any $\zeta \in \xi_{d_y}$, any $f' \in \Lip(\reals^{d_y}; \reals^{d_y \times (2d_x + 1)})$ and any $\zeta' \in C(\reals^{2d_x + 1}, \reals^{d_y})$, the terminal values $z_{f, \zeta, x}(T)$ and $z_{f', \zeta', x}(T)$ may be compared by standard estimates, for example as commonly used in the proof of Picard's theorem. Classical universal approximation results for neural networks \cite{pinkus1999, kidger2020deep} then yield that
	\begin{equation*}
		\bigcup_{d_y \in \naturals} \set{x \mapsto \ell(z_{f, \zeta, x}(T))}{f \in F_{d_y}, \zeta \in \xi_{d_y}, \ell \in L(\reals^{d_y}; \reals^{d_o})}
	\end{equation*}
	is dense in $C(\overline{\mathfrak{X}}, \reals^{d_o})$.
	
	By the definition of the topology on $\ts{\reals^{d_x}}$, then
	\begin{equation*}
		\bigcup_{d_y \in \naturals} \set{\mathbf{x} \mapsto \ell(z_{f, \zeta, x_\mathbf{x}}(T))}{f \in F_{d_y},\, \zeta \in \xi_{d_y},\, \ell \in L(\reals^{d_y}; \reals^{d_o})}
	\end{equation*}
	is dense in $C(\mathcal{X}, \reals^{d_o})$.
\end{proof}

It is now a relatively straightforward matter to determine boundedness and signature-uniqueness for any individual problem. Boundedness is typically obtained by demanding that $\mathcal{X}$ consist of time series of at most some length, of at most some value, and so on. Signature uniqueness is typically obtained via Lemma \ref{lemma:cde:monotone-sig}, and the fact that time is included as a channel.

\begin{remark}
For example, both boundedness and signature-uniqueness are immediately true of linear interpolation.

Likewise, \cite[Appendix B]{kidger2020neuralcde} demonstrates that these properties hold for natural cubic splines. There we fix $\tau < T$, consider $s_j = t_j$, and take $\mathcal{X} \subseteq \ts{\reals^d}$ to be those time series for which $t_0 = \tau$ and $t_n = T$.
\end{remark}

\subsection{Neural CDEs compared to alternative ODE models}\label{appendix:cde:nonlinear-g}
Suppose if instead of equation \eqref{eq:g-theta-x}, we replace $g_{\theta, x}(y, s)$ by $h_{\theta}(y, x(s))$ for some other vector field $h_\theta$. This might seem more natural. Instead of having $g_{\theta, x}$ linear in $\nicefrac{\dd x}{\dd s}$, then $h_\theta$ is potentially nonlinear in the control $x(s)$.

Have anything been gained by doing so? It turns out no, and in fact something has been lost. The neural CDE setup directly subsumes anything depending directly on $x$.

\begin{theorem}\label{theorem:continuous-rnn}
	Let $T > 0$. Let $d_x, d_y \in \naturals$ with $d_x < d_y$. Let
	\begin{equation*}
		\mathcal{V}^1([0, T]; \reals^{d_x - 1}) = C([0, T]; \reals^{d_x - 1}) \cap \BV([0, T]; \reals^{d_x - 1}).
	\end{equation*}
	
	For all $x \in \mathcal{V}^1([0, T]; \reals^{d_x - 1})$, let $\widehat{x}(t) = (t, x(t))$.
	
	Let $\pi \colon \reals^{d_y} \to \reals^{d_y - d_x}$ be the orthogonal projection onto the first $d_y - d_x$ coordinates.
	
	Let 
	\begin{align*}
		\mathcal{Y} &= \set{x \mapsto y_{h, \xi, x}}{h \in \Lip(\reals^{d_y - d_x} \times \reals^{d_x}; \reals^{d_y - d_x}),\, \xi \in C(\reals^{d_x}; \reals^{d_y - d_x})},\\
		\mathcal{Z} &= \set{x \mapsto \pi \circ z_{f, \zeta, x}}{f \in \Lip(\reals^{d_y}; \reals^{d_y \times d_x}),\, \zeta \in C(\reals^{d_x}; \reals^{d_y})},
	\end{align*}
	where $y_{h, \xi, x} \colon [0, T] \to \reals^{d_y - d_x}$ is the unique solution to
	\begin{equation*}
		y_{h, \xi, x}(t) = y_{h, \xi, x}(0) + \int_0^t h(y_{h, \xi, x}(s), \widehat{x}(s)) \,\dd s\quad\text{for $t \in (0, T]$,}
	\end{equation*}
	with $y_{h, \xi, x}(0) = \xi(\widehat{x}(0))$, and $z_{f, \zeta, x} \colon [0, T] \to \reals^{d_y}$ is the unique solution to
	\begin{equation*}
		z_{f, \zeta, x}(t) = z_{f, \zeta, x}(0) + \int_0^t f(z_{f, \zeta, x}(s)) \,\dd \widehat{x}(s)\quad\text{for $t \in (0, T]$,}
	\end{equation*}
	with $z_{f, \zeta, x}(0) = \zeta(\widehat{x}(0))$.
	
	Then $\mathcal{Y} \subsetneq \mathcal{Z}$.
\end{theorem}

In the above statement, a practical choice of $f \in \Lip(\reals^{d_y}; \reals^{d_y \times d_x})$ or $h \in \Lip(\reals^{d_y - d_x} \times \reals^{d_x}; \reals^{d_y - d_x})$ will be some trained neural network.

Note the inclusion of time via the augmentation $x \mapsto \widehat{x}$. Without it, the reparameterisation invariance property of CDEs (Section \ref{section:reparam-invariance}) would restrict the possible functions that CDEs can represent. This hypothesis is not necessary for the $\mathcal{Y} \neq \mathcal{Z}$ part of the conclusion.

Note also how the CDE uses a larger state space of $d_y$, compared to $d_y - d_x$ for the alternative ODE. The reason for this is that whilst $f$ has no explicit nonlinear dependence on $x$, we may construct it to have such a dependence implicitly, by recording $d$ into $d_x$ of its $d_y$ hidden channels, whereupon $x$ is hidden state and may be treated nonlinearly. This hypothesis is also not necessary to demonstrate the $\mathcal{Y} \neq \mathcal{Z}$ part of the conclusion.

\begin{proof}\quad\\	
	\textbf{That $\mathcal{Y} \neq \mathcal{Z}$:}\newline
	Let $\zeta \in C(\reals^{d_x}; \reals^{d_y})$ be arbitrary and let
	\begin{center}\begin{tikzpicture}
			\matrix (vec) [matrix of math nodes, left delimiter = {[}, right delimiter = {]}] {
				0 & 1 & 0 & \cdots & 0\\
				0 & 0 & 0 & \cdots & 0\\
				\vdots & \vdots & \vdots & & \vdots\\
				0 & 0 & 0 & \cdots & 0\\
			};
			
			\node (a) at (vec-1-5.north) [right=15pt]{};
			\node (d) at (vec-4-5.south) [right=15pt]{};
			\draw [decorate, decoration={brace, amplitude=10pt}] (a) -- (d) node[midway, right=10pt] {\footnotesize $d_y$};
			
			\node (e) at (vec-4-1.center) [left=0pt]{};
			\draw [] (e) node[midway, left=50pt] {$f(z) = $};
			
			\node (f) at (vec-4-1.west) [below=10pt]{};
			\node (i) at (vec-4-5.east) [below=10pt]{};
			\draw [decorate, decoration={brace, amplitude=10pt}] (i) -- (f) node[midway, below=10pt] {\footnotesize $d_x$};
	\end{tikzpicture}\end{center}
	
	Then for any $x \in \mathcal{V}^1([0, T]; \reals^{d_x - 1})$, the corresponding CDE solution $z_{f, \zeta, x} \in \mathcal{Z}$ satisfies
	\begin{equation*}
		z_{f, \zeta, x}(t) = z_{f, \zeta, x}(0) + \int_{0}^t f(z_{f, \zeta, x}(s)) \,\dd \widehat{x}(s),
	\end{equation*}
	and so the first component of its solution is
	\begin{equation*}
		z_{f, \zeta, x, 1}(t) = x_{1}(t) - x_{1}(0) + \zeta_{1}(\widehat{x}(0)),
	\end{equation*}
	whilst the other components are constant
	\begin{equation*}
		z_{f, \zeta, x, i}(t) = \zeta_i(\widehat{x}(0))
	\end{equation*}
	for $i \in \{2, \ldots, d_y\}$.
	
	Now suppose for contradiction that there exists $\xi \in C(\reals^{d_x}; \reals^{d_y - d_x})$ and $h \in \Lip(\reals^{d_y - d_x} \times \reals^{d_x}; \reals^{d_y - d_x})$ with a corresponding $y_{h, \xi, \cdot} \in \mathcal{Y}$, such that $y_{h, \xi, x} = \pi \circ z_{f, \zeta, x}$ for all $x \in \mathcal{V}^1([0, T]; \reals^{d_x - 1})$. Now $y_{h, \xi, x}$ must satisfy
	\begin{equation*}
		y_{h, \xi, x}(t) = y_{h, \xi, x}(0) + \int_0^t h(y_{h, \xi, x}(s), \widehat{x}(s)) \,\dd s,
	\end{equation*}
	and so
	\begin{align*}
		&\begin{bmatrix}
			x_{1}(t) - x_{1}(0) + \zeta_{1}(\widehat{x}(0))\\
			\zeta_2(\widehat{x}(0))\\
			\ldots\\
			\zeta_{d_y - d_x}(\widehat{x}(0))
		\end{bmatrix}\\
		&\hspace{8em}= y_{h, \xi, x}(0) + \int_{0}^t h(\begin{bmatrix}x_1(s) - x_1(0) + \zeta_1(\widehat{x}(0))\\ \zeta_2(\widehat{x}(0))\\ \ldots\\ \zeta_{d_y - d_x}(\widehat{x}(0))\end{bmatrix}, \widehat{x}(s)) \,\dd s.
	\end{align*}
	Consider those $x$ which are differentiable. Differentiating with respect to $t$ and considering the first component now gives
	\begin{equation}\label{eq:impossible}
		\frac{\dd x_1}{\dd t}(t) = h_1(\begin{bmatrix}x_1(s) - x_1(0) + \zeta_1(\widehat{x}(0))\\ \zeta_2(\widehat{x}(0))\\ \ldots\\ \zeta_{d_y - d_x}(\widehat{x}(0))\end{bmatrix}, \widehat{x}(t)).
	\end{equation}
	
	That is, $h_1$ satisfies equation \eqref{eq:impossible} for all differentiable $x$. This is clearly impossible: the right hand side is a function of $t$,  $x(t)$ and $x(0)$ only, which is insufficient to determine $\nicefrac{\dd x_1}{\dd t}(t)$.
	
	\textbf{That $\mathcal{Y} \subseteq \mathcal{Z}$:}\newline
	Let $y_{h, \xi, x} \in \mathcal{Y}$ for some $\xi \in C(\reals^{d_x}; \reals^{d_y - d_x})$ and $h \in \Lip(\reals^{d_y - d_x} \times \reals^{d_x}; \reals^{d_y - d_x})$. Let $\sigma \colon \reals^{d_y} \to \reals^{d_x}$ be the orthogonal projection onto the last $d_x$ coordinates. Let $\zeta \in C(\reals^{d_x}; \reals^{d_y})$ be such that $\pi \circ \zeta = \pi \circ \xi$ and $\sigma(\zeta(\widehat{x}(0))) = \widehat{x}(0)$. Then let $f \in \Lip(\reals^{d_y}; \reals^{d_y \times d_x})$ be defined by
	\begin{center}\begin{tikzpicture}
			\matrix (vec) [matrix of math nodes, left delimiter = {[}, right delimiter = {]}] {
				\null &h_1(\pi(z), \sigma(z)) & 0 & 0 & \cdots & 0 \\
				\null &\vdots & \vdots & \vdots & & \vdots\\
				\null &h_{d_y - d_x}(\pi(z), \sigma(z)) & 0 & 0 & \cdots & 0\\
				\null &1 & 0 & 0 & \cdots & 0\\
				\null &0 & 1 & 0 & \cdots & 0\\
				\null &0 & 0 & 1 & \cdots & 0\\
				\null &\vdots & \vdots & \vdots & \ddots & \vdots\\
				\null &0 & 0 & 0 & \cdots & 1\\
			};
			
			\node (a) at (vec-1-6.north) [right=20pt]{};
			\node (b) at (vec-3-6.south) [right=20pt]{};
			\node (c) at (vec-4-6.north) [right=20pt]{};
			\node (d) at (vec-8-6.south) [right=20pt]{};
			\draw [decorate, decoration={brace, amplitude=10pt}] (a) -- (b) node[midway, right=10pt] {\footnotesize $d_y - d_x$};
			\draw [decorate, decoration={brace, amplitude=10pt}] (c) -- (d) node[midway, right=10pt] {\footnotesize $d_x$};
			
			\node (e) at (vec-4-1.center) [left=0pt]{};
			\draw [] (e) node[midway, left=90pt] {$f(z) = $};
			
			\node (f) at (vec-8-1.east) [below=10pt]{};
			\node (g) at (vec-8-3.west) [below=10pt]{};
			\node (h) at (vec-8-3.west) [below=10pt]{};
			\node (i) at (vec-8-6.east) [below=10pt]{};
			\draw [decorate, decoration={brace, amplitude=10pt}] (g) -- (f) node[midway, below=10pt] {\footnotesize $1$};
			\draw [decorate, decoration={brace, amplitude=10pt}] (i) -- (h) node[midway, below=10pt] {\footnotesize $d_x - 1$};	
	\end{tikzpicture}\end{center}
	
	Then for $t \in (0, T]$,
	\begin{align*}
		z_{f, \zeta, x}(t) &= \zeta(\widehat{x}(0)) + \int_0^t f(z_{f, \zeta, x}(s)) \,\dd \widehat{x}(s)\\
		&=\zeta(\widehat{x}(0)) + \int_0^t \begin{bmatrix}
			h_1(\pi(z_{f, \zeta, x}(s)), \sigma(z_{f, \zeta, x}(s))) & 0 & 0 & \cdots & 0 \\
			\vdots & \vdots & \vdots & & \vdots\\
			h_{d_y - d_x}(\pi(z_{f, \zeta, x}(s)), \sigma(z_{f, \zeta, x}(s))) & 0 & 0 & \cdots & 0 \\
			1 & 0 & 0 & \cdots & 0\\
			0 & 1 & 0 & \cdots & 0\\
			0 & 0 & 1 & \cdots & 0\\
			\vdots & \vdots & \vdots & \ddots & \vdots\\
			0 & 0 & 0 & \cdots & 1\\
		\end{bmatrix}
		\begin{bmatrix}
			\dd s\\
			\dd x_1(s)\\
			\vdots\\
			\dd x_{d_x - 1}(s)
		\end{bmatrix}\\
		&= \zeta(\widehat{x}(0)) + \int_0^t \begin{bmatrix}
			h_1(\pi(z_{f, \zeta, x}(s)), \sigma(z_{f, \zeta, x}(s))) \,\dd s\\
			\vdots\\
			h_{d_y - d_x}(\pi(z_{f, \zeta, x}(s)), \sigma(z_{f, \zeta, x}(s))) \,\dd s\\
			\dd s\\
			\dd x_1(s)\\
			\vdots\\
			\dd x_{d_x - 1}(s)\\
		\end{bmatrix}\\
		&= \zeta(\widehat{x}(0)) + \int_0^t \begin{bmatrix}
			h(\pi(z_{f, \zeta, x}(s)), \sigma(z_{f, \zeta, x}(s))) \,\dd s\\
			\dd \widehat{x}(s)
		\end{bmatrix}.
	\end{align*}
	
	Thus
	\begin{equation*}
		\sigma(z_{f, \zeta, x}(t)) = \sigma(\zeta(\widehat{x}(0))) + \int_0^t \,\dd \widehat{x}(s) = \sigma(\zeta(\widehat{x}(0))) - \widehat{x}(0) + \widehat{x}(t) = \widehat{x}(t),
	\end{equation*}
	and so
	\begin{align*}
		\pi(z_{f, \zeta, x}(t)) &= \pi(\zeta(\widehat{x}(0))) + \int_0^t h(\pi(z_{f, \zeta, x}(s)), \sigma(z_{f, \zeta, x}(s))) \,\dd s  \\
		&= \pi(\xi(\widehat{x}(0))) + \int_0^t h(\pi(z_{f, \zeta, x}(s)), \widehat{x}(s)) \,\dd s.
	\end{align*}
	
	We see that $\pi \circ z_{f, \zeta, x}$ satisfies the same differential equation as $y_{h, \xi, x}$. So by uniqueness of solution \cite[Theorem 1.3]{levy-lyons}, $y_{h, \xi, x} = \pi \circ z_{f, \zeta, x} \in \mathcal{Z}$.
\end{proof}

\subsection{Reparameterisation invariance of CDEs}\label{appendix:cde:reparam-invariance}\index{Invariances!Reparameterisation}
\reparaminvariance*
\begin{proof}
	The proof is straightforward change of variables. For expository purposes we consider only differentiable paths; equivalent change-of-variable formulae may be used for bounded variation paths.
	
	Let $\widetilde{t} \in [0, \widetilde{T}]$ and let $t = \psi(\widetilde{t})$. Let $\widetilde{x} = x \circ \psi$ and $\widetilde{y} = y \circ \psi$.
	
	Then make the change of variables $s = \psi(\widetilde{s})$,
	\begin{align*}
		\widetilde{y}(\widetilde{t}) &= y(t)\\
		&= y(0) + \int_0^t f(y(s))\,\dd x(s)\\
		&= y(0) + \int_0^t f(y(s))\,\frac{\dd x}{\dd s}(s)\,\dd s\\
		&= y(\psi(0)) + \int_0^{\psi^{-1}(t)} f(y(\psi(\widetilde{s})))\,\frac{\dd x}{\dd s}(\psi(\widetilde{s}))\,\frac{\dd \psi}{\dd \widetilde{s}}(\widetilde{s})\,\dd \widetilde{s}\\
		&= (y\circ\psi)(0) + \int_0^{\psi^{-1}(t)} f((y\circ\psi)(\widetilde{s}))\,\frac{\dd(x\circ \psi)}{\dd \widetilde{s}}(\widetilde{s})\,\dd \widetilde{s}\\
		&= (y\circ\psi)(0) + \int_0^{\psi^{-1}(t)} f((y\circ \psi)(\widetilde{s}))\,\dd(x\circ \psi)(\widetilde{s})\\
		&= \widetilde{y}(0) + \int_0^{\widetilde{t}} f(\widetilde{y}(\widetilde{s})) \,\dd\widetilde{x}(\widetilde{s}).
	\end{align*}
\end{proof}

\subsection{Comments}
Surprisingly -- despite it being a well-known part of the folklore for signatures -- we could note find a direct statement of Lemma \ref{lemma:cde:monotone-sig} anywhere in the literature. (It is easy to prove, at least.)

To the best of our knowledge all of the discussion on interpolation schemes is new here. We find this a little surprising as the use of differential equations to control dynamical systems is well-studied, as is discrete-time control via for example reinforcement learning. Despite this we have encountered almost nothing written about the formalities of embedding discrete observations into continuous time.

The proof for the comparison of neural CDEs against alternative ODE models is a variation on a standard trick in rough path theory, in which the control is `recorded' into some additional state.

\section{Backpropagation via optimise-then-discretise}
We will now prove how to backpropagate via optimise-then-discretise for ODEs, CDEs, and SDEs.

In principle these may essentially all be thought of as special cases of the same general result (the one shown for SDEs), but in the interests of pedagogy each case is proved separately.

\subsection{Optimise-then-discretise for ODEs}\label{appendix:ode-adjoint}\index{Optimise-then-discretise!ODEs}

Recall backpropagation through ODEs via optimise-then-discretise.

\odeadjoint*

The following proof is both simpler and more precise than those we have typically seen in the literature.

\begin{proof}
Without loss of generality we will prove the equation for $a_y$ only. The equation for $a_\theta$ may be derived  by replacing $y$ with the $[y, \theta]$ and $f_\theta$ with $[f_\cdot, 0]$.

Now $y$ is continuous, and $f_\theta$ is continuously differentiable in $y$, so $t \mapsto \frac{\partial f_\theta}{\partial y}(t, y(t))$ is a continuous function on the compact set $[0, T]$, so it is bounded by some $C > 0$. Correspondingly for $a\in\reals^d$ then $(t, a) \mapsto -a^\top \frac{\partial f_\theta}{\partial y}(t, y(t))$ is Lipschitz in $a$ with Lipschitz constant $C$ and this constant is independent of $t$. Therefore by Picard's existence theorem (Theorem \ref{theorem:picard-ode}) the solution $a_y$ to equation \eqref{eq:ode-adjoint} exists and is unique.

We still need to show that $a_y(t) = \nicefrac{\dd L}{\dd y(t)}$.

For $s, t \in [0, T]$ with $s < t$ then
\begin{equation*}
	y(t) = y(s) + \int_s^t f_\theta(u, y(u)) \,\dd u,
\end{equation*}
so
\begin{equation}\label{eq:proofs:_forward-ode-sensitivity}
	\frac{\dd y(t)}{\dd y(s)} = \eye{d} + \int_s^t \frac{\partial f_\theta}{\partial y}(u, y(u))\frac{\dd y(u)}{\dd y(s)} \,\dd u,
\end{equation}
interchanging limits (Leibniz integral rule or dominated convergence theorem) as $\nicefrac{\partial f_\theta}{\partial y}$ was assumed to be bounded. This is the forward sensitivity equation (Theorem \ref{theorem:numerical:ode-forward-sensitivity}), which is an ODE for the Jacobian $\nicefrac{\dd y(t)}{\dd y(s)}$, the solution of which exists by Picard's existence theorem (Theorem \ref{theorem:picard-ode}).

For $s, t \in [0, T]$ with $s < t$ then
\begin{align}
	\frac{\dd}{\dd t} \left(a_y(t)^\top \frac{\dd y(t)}{\dd y(s)}\right) &= \frac{\dd a_y}{\dd t}(t)^\top\frac{\dd y(t)}{\dd y(s)} + a_y(t)^\top \frac{\dd}{\dd t} \left(\frac{\dd y(t)}{\dd y(s)}\right)\nonumber\\
	&= \frac{\dd a_y}{\dd t}(t)^\top\frac{\dd y(t)}{\dd y(s)} + a_y(t)^\top \frac{\partial f_\theta}{\partial y}(t, y(t)) \frac{\dd y(t)}{\dd y(s)}\nonumber\\
	&= \left(\frac{\dd a_y}{\dd t}(t) + a_y(t)^\top \frac{\partial f_\theta}{\partial y}(t, y(t))\right)\frac{\dd y(t)}{\dd y(s)}\nonumber\\
	&=0\label{eq:proofs:_ode-adjoint-diff}
\end{align}
where the second line is obtained by differentiating \eqref{eq:proofs:_forward-ode-sensitivity} directly.

Therefore
\begin{equation*}
	a_y(t) = a_y(t)^\top \frac{\dd y(t)}{\dd y(t)} = a_y(T)^\top \frac{\dd y(T)}{\dd y(t)} = \frac{\dd L}{\dd y(t)}
\end{equation*}
and in particular $\nicefrac{\dd L}{\dd y_0} = \nicefrac{\dd L}{\dd y(0)} = a_y(0)$.
\end{proof}

\subsection{Optimise-then-discretise for CDEs}\label{appendix:cde-adjoint}\index{Optimise-then-discretise!CDEs}
\cdeadjoint*

The following proof is precisely analogous to the one presented for ODEs in the previous section. The only difference is that the product rule is substituted for its integral equivalent, namely integration by parts.
\begin{proof}
	First we will demonstrate existence and uniqueness of the adjoint process $a$. Analogous to the ODE case, we may wish to consider the vector field as a map $(s, a) \mapsto -a^\top\frac{\partial f}{\partial y}(y(s))$. However in the CDE setting we have restricted ourselves to vector fields that are a function of the state (in this case $a$) only.
	
	The quickest resolution to this is to incorporate the time dependence into the control. That is, we reformulate the solution to \eqref{eq:numerical:_cde-adjoint} as the solution to
	\begin{equation}\label{eq:proofs:_cde-adjoint-true}
		a(t) = a(T) + \int_T^t -a(s)^\top \,\dd M(s)
	\end{equation}
	where $M \colon [0, T] \to \reals^{d_y \times d_y}$ is itself the value of the integral
	\begin{equation}\label{eq:proofs:_cde-adjoint-jacobian}
		M(t) = M(T) + \int_T^t \frac{\partial f}{\partial y}(y(s))\,\dd x(s),
	\end{equation}
	which we note is merely an integral and not a differential equation.
	
	As $f$ was assumed to have continuous derivative then $s \mapsto \frac{\partial f}{\partial y}(y(s))$ is continuous and so \eqref{eq:proofs:_cde-adjoint-jacobian} exists and is of bounded variation as a Riemann--Stieltjes integral. Then by Picard's existence theorem (Theorem \ref{theorem:cde:picard}), \eqref{eq:proofs:_cde-adjoint-true} exists and is unique as the vector field $a \mapsto -a^\top$ is Lipschitz.
	
	Next, let $J_s(t) = \nicefrac{\dd y(t)}{\dd y(s)} \in \reals^{d_y \times d_y}$, which by \cite[Theorem 4.4]{FrizVictoir10} exists and satisfies the CDE
	\begin{equation*}
		J_s(t) = J_s(0) + \int_0^t \frac{\partial f}{\partial y}(y(u)) J_s(u)\,\dd x(u).
	\end{equation*}
	(This CDE is the one we would expect, in analogy to the ODE case.)
	
	For $s, t, \tau \in [0, T]$ with $s < t < \tau$, and using Einstein notation over indices $k_1, k_2, k_3$,
	\begin{align*}
		&a_{k_1}(\tau) J_{s, k_1, k_2}(\tau) - a_{k_1}(t) J_{s, k_1, k_2}(t)\\
		&= \int_t^{\tau} a_{k_1}(u) \,\dd J_{k_1, k_2}(u) + \int_t^{\tau} J_{s, k_1, k_2}(u) \,\dd a_{k_1}(u)\\
		&= \int_t^{\tau} a_{k_1}(u) \frac{\partial f_{k_1}}{\partial y_{k_2}}(u, y(u)) J_{s, k_2, k_3}(u) \,\dd x(u) - \int_t^{\tau} a_{k_1}(u) \frac{\partial f_{k_1}}{\partial y_{k_2}}(u, y(u)) J_{s, k_2, k_3}(u) \,\dd x(u)\\
		&= 0.
	\end{align*}
	where the first equality is integration by parts for Riemann--Stieltjes integrals, and the second equality follows from substituting in the differential equations defining $a$ and $z$.
	
	Therefore
	\begin{equation*}
		a(t) = a(t)^\top \frac{\dd y(t)}{\dd y(t)} = a(t)^\top J_t(t) = a(T)^\top J_t(T) = a(T)^\top \frac{\dd y(T)}{\dd y(t)} = \frac{\dd L}{\dd y(t)}.
	\end{equation*}
\end{proof}

\subsection{Optimise-then-discretise for SDEs}\label{appendix:sde-adjoint}\index{Rough!Path theory}\index{Optimise-then-discretise!SDEs}
We now provide a precise statement for optimise-then-discretise backpropagation through SDEs (originally stated informally in Theorem \ref{theorem-informal:numerical:sde-adjoint}).

Classical SDE theory struggles to make sense of the backward-in-time SDE. This motivates our use of rough path theory.

We begin by outlining the rough path approach to SDEs. We assume familiarity with bounded variation paths, Riemann--Stieltjes integration, and the definition of Brownian motion. We will \textit{not} assume familiarity with classical SDE theory -- for such readers the following presentation should provide an introduction to SDEs that is (in this author's opinion) substantially more elegant than the classical approach.

\subsubsection{Fundamentals}
We begin by setting up a few abstract notions.

\begin{notation*}
For any $d \in \naturals$ and $k \in \{0, 1, 2\}$, let $\pi_k$ denote the projection $\reals \times \reals^d \times \reals^{d \times d} \to \reals^{d^k}$.

We will use $\norm{\,\cdot\,}$ to denote any choice of norm on $\reals$, $\reals^d$, $\reals^{d \times d}$; in finite dimensions all are equivalent so the choice of norm will not be important to us.

CDEs will appear several times. As such and for consistency with the usual way of writing down SDEs, we will switch from denoting solutions of CDEs by
\begin{equation*}
	y(t) = y(0) + \int_0^t f(y(s))\,\dd x(s)
\end{equation*}
to denoting them by
\begin{equation*}
	\dd y(t) = f(y(t))\,\dd x(t).
\end{equation*}

Finally, we recall the standard notation collected at the end of this thesis, including in particular the definition of the tensor product $\otimes$.
\end{notation*}

\begin{definition}[Depth-2 signature]\index{Signatures}
	Let $x \colon [0, T] \to \reals^d$ be continuous and of bounded variation. Then the \textit{depth-2 signature of $x$} is defined as
	\begin{align}
	\sig^2(x) &\colon [0, T] \to \reals \times \reals^d \times \reals^{d \times d},\nonumber\\
	\sig^2(x) &\colon t \mapsto \sig^2_{0, t}(x) = \left(1, x(t) - x(0), \int_0^t (x(s) - x(0)) \otimes \,\dd x(s)\right).\label{eq:proofs:_signature_again}
	\end{align}
	where the final term is defined via Riemann--Stieltjes integration. The constant term $1$ is included by convention.
\end{definition}
Note the use of the tensor (outer) product $\otimes$. This is a bilinear operator so by appropriately manipulating dimensions then the integral of equation \eqref{eq:proofs:_signature_again} may be interpreted as a matrix-vector product as already introduced for controlled differential equations \cite[Definition B.4]{kidger2020neuralcde}.

Note that the signature may be defined for arbitrary depths -- indeed this was used elsewhere in this Appendix, see Definition \ref{definition:signature} -- and the above is simply the special case of interest to us here.

\begin{definition}[Partition]
	A \textit{partition} of $[0, T]$ is some finite sequence $\mathcal{D} = (t_0, \ldots t_n)$ with $0 = t_0 < \cdots < t_n = T$.
\end{definition}

\begin{definition}[{Inhomogeneous $p$-variation, \cite[Section 3.2.1]{levy-lyons}, \cite[Definition 8.6.(i)]{FrizVictoir10}}]
	Let $X_1, X_2 \colon [0, T] \to \reals \times \reals^d \times \reals^{d \times d}$.
	
	For $p \in [2, 3)$, define
	\begin{equation*}
		\rho_p(X_1, X_2) = \max_{k = 0, 1, 2} \sup_{\mathcal{D}} \left(\sum_{t_i \in \mathcal{D}} \norm{
			\pi_k\Big( X_1(t_{i+1}) - X_1(t_{i}) - X_2(t_{i+1}) + X_2(t_i) \Big)
		}^{p/k} \right)^{k/p},
	\end{equation*}
	where the supremum is taken over all partitions $\mathcal{D}$ of $[0, T]$.
	
	Then define the \emph{$p$-variation metric} between $X_1$ and $X_2$ as
	\begin{equation*}
		d_p(X_1, X_2) + \max_{k=0,1,2} \norm{\pi_k(X_1 - X_2)}_\infty + \rho_p(X_1, X_2).
	\end{equation*}
\end{definition}

\begin{remark}
	Notions of $p$-variation are crucial to rough path theory. Correspondingly several remarks are in order.
	\begin{itemize}
		\item If we were to take $p=1$, $X_2 \equiv 0$, and consider only $k=1$, then $\rho_p(X_1, X_2)$ would recover the definition of the bounded/total variation seminorm of $X_1$. Indeed $p$-variation should be thought of as a generalisation of total variation.
		\item It is immediate from the definition that $d_p$ convergence implies uniform convergence. However if $\norm{\pi_k(X_1 - X_2)}_\infty$ is replaced with just $\abs{\pi_k(X_1(0) - X_2(0))}$, then in fact $d_p$ convergence still implies uniform convergence \cite[Definition 3.12]{levy-lyons}.
		\item There are several quantities related to $p$-variation, often going by similar names \cite[Chapter 8]{FrizVictoir10}. Take care not to trip up when reading the literature.
		\item $p$-variation is a subtly different notion to that of quadratic variation used in classical SDE theory. Where $p$-variation takes a supremum over all partitions, quadratic variation instead takes a limit. The quadratic variation of a path may be smaller than its $2$-variation, and in particular almost all samples of Brownian motion have finite quadratic variation but infinite $2$-variation.
		\item A path which is H{\"o}lder continuous with exponent $\alpha \in (\frac{1}{3}, \frac{1}{2}]$ has finite $\frac{1}{\alpha}$-variation. For example Brownian motion is H{\"o}lder continuous with exponent $\alpha$ for all $\alpha \in (0, \frac{1}{2})$, and correspondingly Brownian motion has finite $p$-variation for all $p \in (2, 3)$.
	\end{itemize}
\end{remark}

\begin{definition}
	We say that a sequence of continuous and bounded variation paths $x_n \colon [0, T] \to \reals^d$ converge in $p$-variation to a continuous $X \colon [0, T] \to \reals \times \reals^d \times \reals^{d \times d}$ if
	\begin{equation*}
		\text{$d_p(\sig^2(x_n), X) \to 0$ as $n \to \infty$}.
	\end{equation*}
	Whenever such a limit exists, we refer to $X$ as a \textit{geometric $p$-rough path}.
\end{definition}

\begin{theorem}[{Brownian motion as a geometric rough path, \cite[Corollaries 13.20, 13.22]{FrizVictoir10}}]\label{theorem:proofs:stratonovich-brownian-motion}\index{Stratonovich}
	Let $w \colon [0, T] \to \reals^d$ be a Brownian motion. Let $D_n = (t_0, \ldots, t_n)$ be a uniform partition of $[0, T]$. Let $w_n$ be the unique continuous piecewise linear function with knots $t_j$ such that $w_n(t_j) = w(t_j)$.
	
	Let $W \colon [0, T] \to \reals \times \reals^d \times \reals^{d \times d}$ be defined by
	\begin{equation*}
		W(t) = \left(1, w(t) - w(0), \int_0^t (w(s) - w(0)) \otimes \,\circ\dd w(s)\right)
	\end{equation*}
	with $\circ$ denoting that the integral is defined in the Stratonovich sense.
	
	Let $p \in (2, 3)$ (but not $p=2$). Then $w_n$ converges to $W$ in $p$-variation almost surely. $W$ is called Stratonovich Brownian motion, and it is almost surely a geometric $p$-rough path.
\end{theorem}

\paragraph{Summary}
Let us take stock of what has been introduced.

We have seen that for any continuous bounded variation path $x \colon [0, T] \to \reals^d$, we may consider `enhancing' it with $\int_0^t (x(s) - x(0)) \otimes \,\dd x(s)$. This extra term is completely determined by the base path $x$.

Meanwhile for a Brownian path $w \colon [0, T] \to \reals^d$, we may consider enhancing it with $\int_0^t (w(s) - w(0)) \otimes \,\circ\dd w(s)$. This time the extra term is not completely determined by the base path, and we had to make a choice: what notion of integration to use. We chose Stratonovich integration, but could equally have chosen another form of integration, such as It{\^o} integration. Because $w$ is not of bounded variation, then there is not a single unique notion of integration.

One way or the other, we have lifted ourselves into the larger dimensional space $\reals \times \reals^d \times \reals^{d \times d}$. In this lifted space we have defined the notion of convergence we are interested in, namely $p$-variation. In performing this lift, it transpires that we have \textit{completely defined} what it means to integrate against a path $[0, T] \to \reals \times \reals^d \times \reals^{d \times d}$: in particular we have already made the choice of Stratonovich over It{\^o}. As such we will sometimes think of \textit{the lift ($\sig^2(x)$ or $W$) as the fundamental object}, and reverse what is defined by what, so that $x$ or $w$ is defined as the projection of $\sig^2(x)$ or $W$ by $\pi_1$.

\subsubsection{Rough differential equations and the universal limit theorem}\index{Rough!Differential equations}
We are now ready to define what is meant by a rough differential equation.

\begin{definition}[$\Lip(\gamma)$ functions]
	Let $\gamma > 1$. A function $f \colon \reals^d \to \reals^d$ is said to be $\Lip(\gamma)$ if it is bounded, $\floor{\gamma}$-times differentiable, all derivative are bounded, and the highest derivative is $(\gamma - \floor{\gamma})$-H{\"o}lder continuous.\footnote{This notation is conventional in rough path theory; when other fields have needed this concept it is sometimes denoted in other ways, such as `$f \in C^{k + \alpha}_b$' with $k = \floor{\gamma}$, $\alpha = \gamma - \floor{\gamma}$, and $b$ denoting boundedness.}
\end{definition}

\begin{theorem}[{Universal limit theorem, \cite[Theorem 5.3]{levy-lyons}, \cite[Theorems 10.29, 10.50, 10.57]{FrizVictoir10}}]\label{theorem:proofs:universal-limit-theorem}\index{Universal limit theorem}
	Let $d_x, d_y \in \naturals$. Let $p \in (2, 3)$ and let $\gamma > p$. Let $f \colon \reals^{d_y} \to \reals^{d_y \times d_x}$ be either linear or $\Lip(\gamma)$.
	
	Let $\zeta_n \in \reals^{d_y}$ be a sequence converging to $\zeta \in \reals^{d_y}$.
	
	Let $x_n \colon [0, T] \to \reals^{d_x}$ be a sequence of continuous bounded variation paths, which converge in $p$-variation to a geometric $p$-rough path $X \colon [0, T] \to \reals \times \reals^{d_x} \times \reals^{d_x \times d_x}$.
	
	Let $y_n \colon [0, T] \to \reals^{d_y}$ solve the CDEs
	\begin{equation*}
		y_n(0) = \zeta_n,\qquad \dd y_n(t) = f(y_n(t))\,\dd x_n(t).
	\end{equation*}

	Then there exists a unique geometric $p$-rough path $Y \colon [0, T] \to \reals \times \reals^{d_y} \times \reals^{d_y \times d_y}$ such that $Y(0) = (1, \zeta, 0)$ and $y_n$ converges to $Y$ in $p$-variation.
	
	Moreover, the limit $Y$ depends only on $X$, $f$ and $\zeta$, and in particular not on the sequence $x_n$. As such it is referred to as the `universal limit', and is said to solve the `rough differential equation'
	\begin{equation*}
		\dd Y(t) = f(\pi_1(Y(t))) \,\dd X(t).
	\end{equation*}
\end{theorem}

Given drift $\mu$, diffusion $\sigma$, and Brownian motion $w$, then we may now immediately deduce a corollary specifically for SDEs, by taking $f = \begin{bmatrix}\mu & 0 \\ 0 & \sigma \end{bmatrix}$ and $x_n(t) = [t, w_n(t)]$ in the above result.

\begin{notation*}
Let $\Tau(t) = \sig^2_{0, t}(\mathrm{id}) \in \reals \times \reals \times \reals$, where $\mathrm{id}$ is the identity function. Where a stochastic integral $\int \cdots \circ\dd w(t)$ will become $\int \cdots \dd W(t)$ when lifting to the rough setting, a deterministic integral $\int \cdots \dd t$ will become $\int \cdots \dd \Tau(t)$.
\end{notation*}

\begin{corollary}[Universal limit theorem for Stratonovich SDEs]\label{theorem:proofs:universal-limit-stratonovich}\index{Stratonovich}\index{Universal limit theorem}
	Let $d_w, d_y \in \naturals$. Let $p \in (2, 3)$ and let $\gamma > p$. Let $\mu \colon \reals \times \reals^{d_y} \to \reals^{d_y}$ and $\sigma \colon \reals \times \reals^{d_y} \to \reals^{d_y \times d_w}$ be either linear or $\Lip(\gamma)$.
	
	Let $\zeta_n \in \reals^{d_y}$ be a sequence converging to $\zeta \in \reals^{d_y}$.
	
	Let $w_n \colon [0, T] \to \reals^{d_w}$ be as defined in Theorem \ref{theorem:proofs:stratonovich-brownian-motion}, converging to the Stratonovich Brownian motion $W \colon [0, T] \to \reals \times \reals^{d_w} \times \reals^{d_w \times d_w}$.
	
	Let $y_n \colon [0, T] \to \reals^{d_y}$ solve the (random) CDEs
	\begin{equation*}
		y_n(0) = \zeta_n,\qquad \dd y_n(t) = \mu(t, y_n(t))\,\dd t + \sigma(t, y_n(t))\,\dd w_n(t).
	\end{equation*}

	Then $y_n$ converge in $p$-variation almost surely to a unique geometric $p$-rough path $Y$ solving the rough differential equation
	\begin{equation}\label{eq:proofs:_univeral-limit-stratonovich-rough}
		Y(0) = (1, \zeta, 0),\qquad \dd Y(t) = \mu(s, \pi_1(Y(s)))\,\dd \Tau(t) + \sigma(s, \pi_1(Y(s))) \,\dd W(s),
	\end{equation}
	and moreover the process $y(t) = \pi_1(Y(t))$ satisfies the Stratonovich SDE
	\begin{equation*}
		\dd y(t) = \mu(t, y(t))\,\dd t + \sigma(t, y(t))\,\circ\dd w(t)
	\end{equation*}
	defined in the classical sense.
\end{corollary}

See Section \ref{appendix:proofs:rough-sde-subtle} for an appendix on some technical points associated with Theorems \ref{theorem:proofs:universal-limit-theorem} and \ref{theorem:proofs:universal-limit-stratonovich}.

\paragraph{Summary}
The key point of the universal limit theorem is that instead of defining a differential equation driven by some continuous path $[0, T] \to \reals^d$ (for example as was done with CDEs), we have defined a differential equation driven by an enhanced path $[0, T] \to \reals \times \reals^d \times \reals^{d \times d}$.

In particular we have defined integration against Stratonovich Brownian motion. Note the terminology of `Stratonovich Brownian motion' rather than `Stratonovich SDE': the rough path approach has entirely contained both the `Stratonovich-ness' and the stochasticity to within the enhanced Brownian motion $W$. After that we simply sample $W$, and deterministically solve the RDE driven by this Brownian sample.

Rough objects are easily dealt with via the universal limit theorem: any time we encounter an RDE we may simply approximate it with a sequence of CDEs, perform the appropriate manipulations, and then take a limit.

Overall, we see that the notions of stochasticity, roughness, and control have been factored apart. This is in contrast to the classical approach to SDEs, which muddles together these three separate ideas.

In passing, note the relatively high regularity $\Lip(\gamma)$ assumed of the vector fields. This is needed to `offset' the roughness of the driving signal.

\subsubsection{Rough adjoints}

Having established what is meant by a rough differential equation, and how it may be used as a notion of solution to a stochastic differential equation, it is now straightforward to derive our main result. This is the precise statement corresponding to the informal Theorem \ref{theorem-informal:numerical:sde-adjoint}.

\begin{theorem}[Optimise-then-discretise for SDEs]\label{theorem:proofs:sde-adjoint}
	Fix $d_y, d_w \in \naturals$ and $\gamma > 2$. Let $\mu \colon \reals \times \reals^{d_y} \to \reals^{d_y}$ and $\sigma \colon \reals \times \reals^{d_y} \to \reals^{d_y \times d_w}$ be linear or $\Lip(\gamma + 1)$.
	
	Let $W \colon \reals \times \reals^{d_w} \times \reals^{d_w \times d_w}$ denote a Stratonovich Brownian motion as in Theorem \ref{theorem:proofs:stratonovich-brownian-motion}.
	
	Let $L \colon \reals^{d_y} \to \reals$ be continuously differentiable (and scalar just for simplicity). Let $y_0 \in \reals^{d_y}$ and let $Y \colon [0, T] \to \reals \times \reals^{d_y} \times \reals^{d_y \times d_y}$ solve the rough differential equation
	\begin{equation}\label{eq:proofs:_sde-adjoint-forward}
		Y(0)=(1, y_0, 0),\qquad\dd Y(t) = \mu(t, y(t))\,\dd \Tau(t) + \sigma(t, y(t)) \,\dd W(t),
	\end{equation}
	where $y(t) = \pi_1(Y(t))$.

	Consider the adjoint process $A$ solving the backwards-in-time linear rough differential equation
	\begin{align}
		A(T) &= \left(1, \frac{\dd L(y(T))}{\dd y(T)}, 0\right),\nonumber\\
		\dd A(t) &= -a(t)^\top\frac{\partial \mu}{\partial y}(t, y(t))\,\dd \Tau(t) - a(t)^\top \frac{\partial \sigma}{\partial y}(t, y(t)) \,\dd W(t),\label{eq:proofs:_sde-adjoint-backward}
	\end{align}
	where $a(t) = \pi_1(A(t))$.
	
	Then the solution $A$ exists and is unique, and for almost all sample paths $W$, we have $a(t) = \nicefrac{\dd L(y(T))}{\dd y(t)} \in \reals^{d_y}$.
\end{theorem}

For completeness we note that the non-rough (classical SDE) equivalent to \eqref{eq:proofs:_sde-adjoint-forward} is
\begin{equation*}
	y(0)= y_0,\qquad\dd y(t) = \mu(t, y(t))\,\dd t + \sigma(t, y(t)) \circ\dd w(t),
\end{equation*}
whilst the non-rough equivalent to \eqref{eq:proofs:_sde-adjoint-backward} is
\begin{equation*}
	\dd a_{k_1}(t) = -a_{k_2}(t)\frac{\partial \mu_{k_2}}{\partial y_{k_1}}(t, y(t))\,\dd t - a_{k_2}(t) \frac{\partial \sigma_{k_2, k_3}}{\partial y_{k_1}}(t, y(t)) \,\circ\dd w_{k_3}(t),
\end{equation*}
in Einstein notation is over the indices $k_1, k_2, k_3$. This latter equation is not technically defined as a Stratonovich SDE ($a$ is not measurable with respect to the natural filtration of $w$), and so is best interpreted as the projection under $\pi_1$ of the rough differential equation.

\begin{proof}
	Let $p \in (2, \gamma)$.
	
	Let $w_n \colon [0, T] \to \reals^{d_w}$ be as defined in Theorem \ref{theorem:proofs:stratonovich-brownian-motion}, converging in $p$-variation to the Stratonovich Brownian motion $W$.
	
	Let $y_n \colon [0, T] \to \reals^{d_y}$ solve the (random) CDEs
	\begin{equation*}
		y_n(0) = y_0,\qquad \dd y_n(t) = \mu(s, y_n(s))\,\dd s + \sigma(s, y_n(s))\,\dd w_n(s),
	\end{equation*}
	which by the universal limit theorem (Theorem \ref{theorem:proofs:universal-limit-stratonovich}) converge to $Y$ in $p$-variation almost surely.
	
	By optimise-then-discretise for CDEs (Theorem \ref{theorem:numerical:cde-adjoint}), each adjoint process
	\begin{equation}\label{eq:proof:_adjoint-sde-limiting-solution}
	a_n(t) = \frac{\dd L(y_n(T))}{\dd y_n(t)}
	\end{equation}
	satisfies
	\begin{equation}\label{eq:proof:_adjoint-sde-limiting}
		\dd a_{n, k_1}(t) = -a_{n, k_2}(t)\frac{\partial \mu_{k_2}}{\partial y_{k_1}}(t, y_n(t))\,\dd t - a_{n, k_2}(t) \frac{\partial \sigma_{k_2, k_3}}{\partial y_{k_1}}(t, y_n(t)) \,\dd w_{n, k_3}(t)
	\end{equation}
	starting from the terminal condition $a_n(T) = \nicefrac{\dd L(y_n(T))}{\dd y_n(T)}$, and Einstein notation is used over the indices $k_1, k_2, k_3$.
	
	As in the proof of Theorem \ref{theorem:numerical:cde-adjoint}, we interpret the solution \eqref{eq:proof:_adjoint-sde-limiting} as the solution to the CDE
	\begin{equation*}
		\dd a_{n, k_1}(t) = -a_{n, k_2}(t) \,\dd M_{n, k_2, k_1}(t),
	\end{equation*}
	where $M_n \colon [0, T] \to \reals^{d_y \times d_y}$ is a bounded variation path satisfying
	\begin{equation}\label{eq:proof:_adjoint-sde-jacobian}
		\dd M_n(t) = \frac{\partial \mu}{\partial y}(t, y_n(s))\,\dd s + \frac{\partial \sigma}{\partial y}(t, y_n(s)) \,\dd w_n(s).
	\end{equation}

	We would like to take $n \to \infty$ in equation \eqref{eq:proof:_adjoint-sde-jacobian} via the universal limit theorem. The version we have stated here does not allow for $n$-dependent vector fields (note that the vector fields depend on $y_n$). This is resolved by the standard trick of replacing $M_n$ with $[M_n, z_n]$ and $[s, w_n(s)]$ with $[s, w_n(s), y_n(s)]$, where
	\begin{align}
		\dd M_n(t) &= \frac{\partial \mu}{\partial y}(t, z_n(s))\,\dd s + \frac{\partial \sigma}{\partial y}(t, z_n(s)) \,\dd w_n(s),\nonumber\\
		\dd z_n(t) &= \dd y_n(t),\label{eq:proof:_adjoint-sde-jacobian-two}
	\end{align}
	so that the vector fields are a function of the state $[M_n, z_n]$ only.

	We may now take $n \to \infty$ in equation \eqref{eq:proof:_adjoint-sde-jacobian-two} by the universal limit theorem, as $\nicefrac{\partial \mu}{\partial y}$, $\nicefrac{\partial \sigma}{\partial y}$ are $\Lip(\gamma)$.
	
	As such $M_n$ converges in $p$-variation almost surely to a geometric $p$-rough path
	\begin{equation*}
	\mathbb{M} \colon [0, T] \to \reals \times \reals^{d_y \times d_y} \times \reals^{d_y \times d_y \times d_y \times d_y}
	\end{equation*}
	satisfying\footnote{Implying that the corresponding the non-rough $M(t) = \pi_1(\mathbb{M})$ satisfies
		\begin{equation*}
			\dd M(t) = \frac{\partial \mu}{\partial y}(s, y(s))\,\dd s + \frac{\partial \sigma}{\partial y}(s, y(s)) \,\circ \dd w(s).
		\end{equation*}}
	\begin{equation*}
		\dd \mathbb{M}(t) = \frac{\partial \mu}{\partial y}(s, y(s))\,\dd \Tau(s) + \frac{\partial \sigma}{\partial y}(s, y(s)) \,\dd W(s).
	\end{equation*}
	
	By the universal limit theorem (with linear vector field), then $a_n$ now converges in $p$-variation almost surely to a geometric $p$-rough path $A$ satisfying
	\begin{equation*}
		\dd A(t) = -\pi_1(A(t)) \,\dd \mathbb{M}(t),
	\end{equation*}
	which we may rewrite as
	\begin{equation*}
		\dd A(t) = -a(t)^\top\frac{\partial \mu}{\partial y}(t, y(t))\,\dd \Tau(t) - a(t)^\top \frac{\partial \sigma}{\partial y}(t, y(t)) \,\dd W(t),
	\end{equation*}
	with $a(t) = \pi_1(A(t))$. We are now halfway through the proof, and have derived our desired RDE.
	
	\textit{Overall what we have done is very simple: just take the limit $n \to \infty$ in \eqref{eq:proof:_adjoint-sde-limiting}. The argument until now has just been to shuffle things around so that the appropriate theorems may be applied.}
	
	It remains to show that $a(t) = \nicefrac{\dd L(y(T))}{\dd y(t)}$. (Implying in particular the terminal condition $A(T) = \left(1, \frac{\dd L(y(T))}{\dd y(T)}, 0\right)$.)
	
	Fix $s, t \in [0, T]$ with $s < t$ and let $J_n(t) = \frac{\dd y_n(t)}{\dd y_n(s)}$. \cite[Theorem 4.4]{FrizVictoir10} gives the `forward sensitivity' result for CDEs (the forward-mode autodifferentiation counterpart to the reverse-mode autodifferentiation version we are currently deriving), and states that the Jacobian $J_n$ evolves according to the CDE
	\begin{equation*}
		\dd J_n(t) = \frac{\partial \mu}{\partial y}(t, y_n(t)) J_n(t) \,\dd t + \frac{\partial \sigma}{\partial y}(t, y_n(t)) J_n(t)\,\dd w_n(t).
	\end{equation*}
	By the universal limit theorem (and applying the same trick as with $\mathbb{M}$, moving $y_n$ into the state and control), this sequence converges in $p$-variation almost surely to a geometric $p$-rough path
	\begin{equation*}
		\mathbb{J} \colon [0, T] \to \reals \times \reals^{d_y \times d_y} \times \reals^{d_y \times d_y \times d_y \times d_y}
	\end{equation*}
	satisfying
	\begin{equation*}
		\dd \mathbb{J}(t) = \frac{\partial \mu}{\partial y}(t, y(t)) \pi_1(\mathbb{J}(t)) \,\dd \Tau(t) + \frac{\partial \sigma}{\partial y}(t, y(t)) \pi_1(\mathbb{J}(t))\,\dd W(t).
	\end{equation*}
	We appeal to our final theorem. \cite[Theorem 11.3]{FrizVictoir10} gives the `forward sensitivity' result for RDEs, satisfied by the the lift of $\nicefrac{\dd y(t)}{\dd y(s)}$. And unsurprisingly, this is the same equation we have just derived. So by uniqueness of solution $\pi_1(\mathbb{J}(t)) = \nicefrac{\dd y(t)}{\dd y(s)}$.
	
	In summary: we have shown that the Jacobian flow $t \mapsto \frac{\dd y_n(t)}{\dd y_n(s)}$ converges, and moreover it converges to (the lift of) $t \mapsto \frac{\dd y(t)}{\dd y(s)}$. (In each case with respect to $p$-variation almost surely, and therefore also uniformly almost surely and therefore also pointwise almost surely.)
	
%
%
	
	Consequently and by continuous differentiability of $L$,
	\begin{equation*}
		a_n(t) = \frac{\dd L(y_n(T))}{\dd y_n(t)} = \frac{\dd L}{\dd y}(y_n(T)) \frac{\dd y_n(T)}{\dd y_n(t)} \to \frac{\dd L}{\dd y}(y(T)) \frac{\dd y(T)}{\dd y(t)}
	\end{equation*}
	pointwise over $t$. As also $a_n(t) \to a(t)$ pointwise (as $a_n \to A$ in $p$-variation), then by uniqueness of limits $a(t) = \frac{\dd L(y(T))}{\dd y(t)}$.
\end{proof}

\begin{remark}
	The above method of proof may be trivially extended to any RDE driven by a geometric $p$-rough path.
\end{remark}

\subsection{Comments}
\subsubsection{On ODEs}
Optimise-then-discretise for ODEs is also referred to as Pontryagin's Maximum Principle (PMP)\index{Pontryagin's maximum principle}. Often only special cases of the result shown here are presented; frequently only the behaviour at an optimum is considered.

Our proof is new here -- whilst combining flavours of various previous proofs -- and much simpler than most versions found in the literature. The fact that it lacks any meaningful reliance on the differentiability of the forward or reverse sensitivities is what allows the later generalisation to the CDE case.

The basic idea of finding two processes $(a, z)$ for which $\nicefrac{\dd}{\dd t}(a(t)z(t)) = 0$ is the same notion used in duality of stochastic processes \cite[Proposition 4.1.(ii)]{stochastic-duality}, and was inspired by the fact that PMP may be proved in a similar way \cite[Proof 2.7]{pmp-proof}. The discrete analogue is \cite[Equation (3.4)]{griewank2008evaluating}. That the proof proceeds by considering the interaction between the forward and reverse sensitivities is vaguely reminiscent of \cite{jax-autodiff}, who derive reverse sensitivities by combining forward sensitivities and transposition rules.

\subsubsection{On CDEs}
Optimise-then-discretise for CDEs was first shown in \cite[Appendix A.1]{kidger2020sdeunpublished}, but this was never formally published. The proof we present here is substantially simpler than that of \cite{kidger2020sdeunpublished}, and is new here.

\subsubsection{On SDEs}\label{appendix:proofs:rough-sde-subtle}
Optimise-then-discretise for CDEs was first shown in \cite[Appendix A.2, Appendix A.3]{kidger2020sdeunpublished}, but this was never formally published. The presentation shown here follows essentially the same lines, whilst being a bit simpler and fixing some technical holes.

The following are technical notes (mostly for the expert) on the theorems given here.

\paragraph{The universal limit theorem (Theorem \ref{theorem:proofs:universal-limit-theorem})}\index{Universal limit theorem}
The statement given here is a slightly custom mish-mash of the different ways in which this theorem is sometimes expressed. The bulk of the statement comes from \cite[Theorem 5.3]{levy-lyons}, although we have simplified the statements about $p$-variation from the general (potentially non-geometric) form given there to the geometric-only form considered here.

Surprisingly, we could not find a form of this theorem which explicitly included the convergence of the initial points $\zeta_n$, as is stated here. This may be recovered from Davie's Lemma \cite[Theorem 10.29]{FrizVictoir10}.

For simplicity of presentation the statement given here has elided the usual notion by which integrator and integrand are coupled together into a single rough path.

\paragraph{The universal limit theorem for Stratonovich SDEs (Theorem \ref{theorem:proofs:universal-limit-stratonovich})}\index{Universal limit theorem}
It is possible to admit lower regularity on the drift $\mu$ than the $\Lip(\gamma)$ assumed here. This is because the drift is not integrated against Brownian motion, but is only integrated against time -- for which classical ODE theory would demand only that $\mu$ be Lipschitz. See \cite[Chapter 12]{FrizVictoir10}.

Note the dependence on time in the vector fields $(\mu, \sigma)$, despite this not appearing in Theorem \ref{theorem:proofs:universal-limit-theorem}. Technically speaking we have accomplished this by concatenating $\dd t = 1 \dd t + 0 \,\dd w_n(t)$ to each $y_n$, so that time becomes part of the state. This is only possible because time is also part of the control.

Having taken the control to be the limit of $(t, w_n(t)) \in \reals^{1 + d_w}$, its rough lift will actually be in $\reals \times \reals^{1 + d_w} \times \reals^{(1 + d_w) \times (1 + d_w)}$, which is larger than the two separate $\Tau(t) \in \reals \times \reals \times \reals$, $W(t) \in \reals \times \reals^{d_w} \times \reals^{d_w \times d_w}$. We have elided the reduction to two separate terms, and \eqref{eq:proofs:_univeral-limit-stratonovich-rough} may just be considered a formal notation for the `true' RDE if the reader so prefers.

\paragraph{Optimise-then-discretise for Stratonovich SDEs (Theorem \ref{theorem:proofs:sde-adjoint})}
For all of the theorems we have seen in this section -- ODEs, CDEs, and SDEs -- we have for simplicity only considered derivatives of the solution with respect to the initial condition $y_0$. In general we may wish to also consider derivatives with respect to either the driving signal $x$ or the vector field $f$.
	
To the best of our knowledge, no complete account of every case (both forward and backward sensitivities; derivatives with respect to all of $y_0$, $x$, $f$; ODEs, CDEs, SDEs or the general RDE case) yet exists in the literature. (Although the forward sensitivity with respect to $x$ may be found in \cite[Theorem 4.4, Theorem 11.3, Exercise 11.10]{FrizVictoir10}.)
	
In practice the result we have shown is almost always the most important, from which important special cases of the other sensitivities may be derived. For example, derivatives with respect to $\theta$ for $f = f_\theta$ may be derived by replacing $y$ with $[y, \theta]$, $y_0$ with $[y_0, \theta]$ and $f_\theta$ with $f(y, \theta) = [f_\theta, 0]$.

The assumed regularity of $\Lip(\gamma + 1)$ is for simplicity of presentation and is much higher than is probably necessary. As with the universal limit theorem for Stratonovich SDEs, we expect to require only minimal regularity on the drift. Moreover substituting the universal limit theorem for \cite[Theorem 17.1]{FrizVictoir10} when obtaining $\mathbb{M}$ would allow for only $\Lip(\gamma)$ regularity on the diffusion. Something similar could like be arranged when obtaining $\mathbb{J}$.

Likewise for simplicity of presentation, the proof leaves a few things implicit (including the inclusion of $y_n$ in the control when taking the limit in $\Phi_n$; that the Jacobian $\nicefrac{\dd y_n(T)}{\dd y_n(t)}$ should be thought of as a solution map from $t$ to $T$).

\section{Convergence and stability of the reversible Heun method}\label{appendix:numerical:ode-reversible-heun}\index{Reversible Heun}
Recall the definition of the reversible Heun method, as applied to ODEs.
\begin{align*}
	t_{j+1} &= t_j + \Delta t,\\
	\widehat{y}_{j+1} &= 2 y_j - \widehat{y}_j + \mu_j \Delta t,\\
	\mu_{j+1} &= \mu(t_{j+1}, \widehat{y}_{j+1}),\\
	y_{j+1} &= y_j + \frac{1}{2}(\mu_j + \mu_{j+1})\Delta t.
\end{align*}

(See Section \ref{section:numerical:reversible-heun} and/or \cite[Appendix D]{kidger2021sde2} for the full SDE case.)

\subsection{Convergence}
\reversibleheunconvergence*
\begin{proof}
	Consider a two-step update over the $\widehat{y}$ component. Then
	\begin{align*}
		\widehat{y}_{j+1} &= \widehat{y}_{j-1} + 2 \mu(t_{j}, \widehat{y}_j) \Delta t,
	\end{align*}
	which is precisely the equation for the leapfrog/midpoint method \cite{shampine-leapfrog}; see also Section \ref{section:numerical:leapfrog-midpoint}. This is a second-order method.
	
	Therefore
	\begin{align*}
		y_n &= y_0 + \sum_{j=0}^{n-1} \frac{1}{2}\left(\mu_j + \mu_{j+1}\right)\Delta t\\
		&= y_0 + \sum_{j=0}^{n-1} \frac{1}{2}\left(\mu(t_j, y(t_j)) + \mu(t_{j+1}, y(t_{j+1})\right) + \bigO{\Delta t^2})\Delta t\\
		&= y_0 + \sum_{j=0}^{n-1} \left(\frac{1}{2}\left(\mu(t_j, y(t_j)) + \mu(t_{j+1}, y(t_{j+1}))\right)\Delta t  + \bigO{\Delta t^3}\right)\\
		&= y_0 + \sum_{j=0}^{n-1} \left(\int_{t_j}^{t_{j+1}} \mu(t, y(t)) \,\dd t + \bigO{\Delta t^3}  + \bigO{\Delta t^3}\right)\\
		&= y(t_n) + \bigO{\Delta t^2}.
	\end{align*}
\end{proof}

\subsection{Stability}
\begin{definition}
	Fix some numerical differential equation solver (we will consider just the reversible Heun method). Let $\{y_{j, \lambda, \Delta t}\}_{j\in\naturals}$ be the numerical approximation to the linear (Dahlquist) test equation
	\begin{equation*}
		y(0) \in \reals,\qquad\frac{\dd y}{\dd t}(t) = \lambda y(t)\qquad\text{for $t \in [0, \infty)$}
	\end{equation*}
	with $\lambda \in \complexes$, numerical step size $\Delta t > 0$ and $y_{j, \lambda, \Delta t} \approx y(j \Delta t)$. We define the \emph{region of stability} as
	\begin{equation*}
		\set{\lambda \Delta t \in \complexes}{\text{$\{y_{j, \lambda, \Delta t}\}_{j\in\naturals}$ is uniformly bounded over $j$}}.
	\end{equation*}
	That is, there exists a constant $C$ depending on $\lambda$ and $\Delta t$ but independent of $j$ for which $\abs{y_{j, \lambda, \Delta t}} < C$.
\end{definition}

\begin{remark}
	We have chosen to define stability in terms of the boundedness of the numerical solution. (Which is the behaviour of the analytical solution for $\realpart(\lambda) \leq 0$.) Some authors define the region of stability in terms of the slightly stronger condition that the numerical solution converges to zero. (Which is the behaviour of the analytical solution for $\realpart(\lambda) < 0$.)
\end{remark}

\reversibleheunstability*
\begin{proof}
	Consider a two-step update over the $\widehat{y}$ component. Then
	\begin{align*}
		\widehat{y}_{j+1} &= \widehat{y}_{j-1} + 2 \mu(t_{j}, \widehat{y}_j) \Delta t,
	\end{align*}
	which is precisely the equation for the leapfrog/midpoint method \cite{shampine-leapfrog}; see also Section \ref{section:numerical:leapfrog-midpoint}.
	
	This is a difference equation for $\widehat{y}$, which may be solved explicitly to obtain
	\begin{equation*}
		\widehat{y}_j = \alpha_1 \eta^j + \beta \kappa^j,
	\end{equation*}
	where
	\begin{align*}
		\alpha &= \frac{1}{2}y_0\left(1 + \frac{1}{\sqrt{1 + \lambda^2\Delta t^2}}\right),\\
		\beta &= \frac{1}{2}y_0\left(1 - \frac{1}{\sqrt{1 + \lambda^2\Delta t^2}}\right),\\
		\eta &= \lambda \Delta t + \sqrt{1 + \lambda^2\Delta t^2},\\
		\kappa &= \lambda \Delta t - \sqrt{1 + \lambda^2\Delta t^2}.
	\end{align*}
	(With $z \mapsto \sqrt{1 + z^2}$ putting branch cuts down $(-i\infty, -i)$ and $(i, i \infty)$.)

	Therefore
	\begin{align}
		y_n &= y_0 + \sum_{j=0}^{n-1} \frac{1}{2}\left(\mu_j + \mu_{j+1}\right)\Delta t\nonumber\\
		&= y_0 + \frac{1}{2}\lambda \Delta t \alpha (1 + \eta) \sum_{j=0}^{n-1} \eta^j + \frac{1}{2}\lambda \Delta t \beta (1 + \kappa) \sum_{j=0}^{n-1} \kappa^j\nonumber\\
		&= y_0 + \frac{1}{2}\lambda \Delta t \alpha \frac{1 + \eta}{1 - \eta}(1 - \eta^n) + \frac{1}{2} \Delta t \beta \frac{1 + \kappa}{1 - \kappa}(1 - \kappa^n).\label{eq:proofs:_reversible-heun-stability}
	\end{align}

	Consider when $\lambda \Delta t \in [-i, i]$. Then $\eta = \lambda \Delta t + \sqrt{1 - \abs{\lambda \Delta t}^2}$ and $\kappa = \lambda \Delta t - \sqrt{1 - \abs{\lambda \Delta t}^2}$, so that
	\begin{align*}
		\abs{\eta}^2 &= \abs{\lambda \Delta t}^2 + (1 - \abs{\lambda \Delta t}^2) = 1,\\	\abs{\kappa}^2 &= \abs{\lambda \Delta t}^2 + (1 - \abs{\lambda \Delta t}^2) = 1,
	\end{align*}
	and therefore by \eqref{eq:proofs:_reversible-heun-stability} $\abs{y_n}$ is bounded independent of $n$.
	
	Conversely consider when $\lambda \Delta t \notin [-i, i]$. Then $\abs{\eta} \neq 1$. (A fact most easily verified via the usual `proof by dodgy diagram'\footnote{A term which we must thank Hilary Priestley for introducing to our lexicon.} typically used for determining the image of a composition of conformal functions.) Now $\eta \kappa = -1$ so one term in \ref{eq:proofs:_reversible-heun-stability} will decay and the other will blow up as $n \to \infty$; consequently $\abs{y_n}$ is not bounded over $n$.
\end{proof}

\begin{remark}
	This is the same region of stability as both the asynchronous leapfrog method \cite[Appendix A.4]{zhuang2021mali} and the leapfrog/midpoint method \cite[Section 2]{shampine-leapfrog}.
\end{remark}

\section{Brownian Interval}\label{appendix:proofs:binterval}\index{Brownian!Interval}

This Appendix continues the discussion and definition of the Brownian Interval of Section \ref{section:numerical:brownian-interval}.

\subsection{Algorithmic definitions}
See Algorithms \ref{alg:numerical:binterval-sampling}--\ref{alg:numerical:sample} for the full description of how the binary tree is traversed, modified, and $w(s, t)$ subsequently sampled.

Let \textit{List} be an ordered data structure that can be appended to, and iterated over sequentially. For example a linked list would suffice. Let \textit{Node} denote a 5-tuple consisting of an interval, a seed, and three optional \textit{Node}s, corresponding to the parent node, and two child nodes, respectively. (Optional as the root has no parent and leaves have no children.)

We let $\texttt{split\_seed}$ denote a splittable PRNG as above and \texttt{bridge} to denote equation \eqref{eq:numerical:binterval-bbridge}. We use $*$ to denote an unfilled part of the data structure, equivalent to \texttt{None} in Python or a null pointer in C/C++; in particular this is used as a placeholder for the (nonexistent) children of leaf nodes and the (nonexistent) parent of the root node.

We use $=$ to denote the creation of a new local variable, and $\leftarrow$ to denote in-place modification of a variable.

We use $x : T$ to denote that $x$ is a value with type $T$.

\begin{algorithm}
	\SetKwInput{kwState}{State}
	\SetKwInput{kwInput}{Input}
	\SetKwProg{Def}{def}{:}{}
	\SetAlgoVlined
	\kwInput{Interval $[s, t] \subseteq [0, T]$}
	\kwState{Binary tree with elements of type \textit{Node}, with root $\widehat{I} = ([0, T], \widehat{\rho}, *, \widehat{I}_\texttt{left}, \widehat{I}_\texttt{right})$. A \textit{Node} $\widehat{J}$, which provides a hint for where to start the traversal from (for efficiency).}
	\KwResult{Sample increment $W_{s, t}$}
	\quad\\
	
	\# The returned `nodes' is a list of \textit{Node}s whose intervals partition $[s, t]$.\\
	\# Practically speaking this will usually have only one or two elements.\\
	nodes = \texttt{traverse}($\widehat{J}, [s, t]$)\\
	$\widehat{J} \leftarrow \text{nodes}[-1]$ \qquad\# last element of `nodes'\\
	return $\sum_{I \in \text{nodes}} \texttt{sample}(I, \widehat{I})$
	\caption{Sampling the Brownian Interval}\label{alg:numerical:binterval-sampling}
\end{algorithm}

\begin{algorithm}
	\SetKwProg{Def}{def}{:}{}
	\SetAlgoLined
	
	\SetKwProg{Def}{def}{:}{}
	\SetAlgoLined
	
	\Def{\emph{\texttt{traverse}}($I$ : \textit{Node},\, $[c, d]$ : \textit{Interval})}{
		Let nodes be an empty \textit{List}.\\
		\texttt{traverse\_impl}($I$, $[c, d]$, nodes)\\
		return nodes\\
	}
	\quad\\
	
	\Def{\emph{\texttt{traverse\_impl}}($I$ : \textit{Node},\, $[c, d]$ : \textit{Interval}, \emph{nodes} : \textit{List[Node]})}{
		Decompose $([a, b], \rho, I_{\texttt{parent}}, I_{\texttt{left}}, I_{\texttt{right}}) = I$\\
		\quad\\
		\# Outside our jurisdiction - pass to our parent\\
		\If{$c < a$ \normalfont{or} $d > b$}{
			\texttt{traverse\_impl}($I_\texttt{parent}, [c, d]$, nodes)\\
			return\\
		}
		\quad\\
		
		\# It's $I$ that is sought. Add $I$ to the list and return.\\
		\If{$c = a$ \normalfont{and} $d = b$}{
			nodes.append($I$)\\
			return\\
		}
		\quad\\
		
		\# Check if $I$ is a leaf or not.\\
		\eIf{$I_\texttt{left}$ \normalfont{is} $*$}{
			\# $I$ is a leaf\\
			\If{$a = c$}{
				\# Create children and add on the left child.\\
				\texttt{bisect}($I, d$)\qquad\# $I_\texttt{left}$ is created.\\
				nodes.append($I_\texttt{left}$)\\
				return\\[2pt]
			}
			\# Otherwise create children and pass on to our right child.\\
			\texttt{bisect}($I, c$)\qquad\# $I_\texttt{right}$ is created.\\
			\texttt{traverse\_impl}($I_\texttt{right}, [c, d]$, nodes)\\
			return\\
		}{
			\# $I$ is not a leaf.\\
			Decompose $([a, m], \rho_\texttt{left}, I, I_{l\,l}, I_{lr}) = I_\texttt{left}$\\
			\If{$d \leq m$}{
				\# Strictly our left child's problem.\\
				\texttt{traverse\_impl}($I_\texttt{left}, [c, d]$, nodes)\\
				return \\
			}
			\If{$c \geq m$}{
				\# Strictly our right child's problem.\\
				\texttt{traverse\_impl}($I_\texttt{right}, [c, d]$, nodes)\\
				return \\[2pt]
			}
			\# A problem for both of our children.\\
			\texttt{traverse\_impl}($I_\texttt{left}, [c, m]$, nodes)\\
			\texttt{traverse\_impl}($I_\texttt{right}, [m, d]$, nodes)\\
			return\\
		}
	}
	\caption{Definitions of \texttt{traverse} and \texttt{traverse\_impl}. The argument nodes is passed by reference, that is to say it is mutated and these mutations are visible to the calling function.}
\end{algorithm}

\begin{algorithm}
	\SetKwProg{Def}{def}{:}{}
	\SetAlgoLined
	
	\Def{\emph{\texttt{bisect}}($I$ : \textit{Node}, x : $\reals$)}{
		\# Only called on leaf nodes\\
		Decompose $([a, b], \rho, I_{\texttt{parent}}, *, *) = I$\\
		$\rho_{\texttt{left}}, \rho_{\texttt{right}} = \texttt{split\_seed}(\rho)$\\
		$I_{\texttt{left}} = ([a, x], \rho_{\texttt{left}}, I, *, *)$\\
		$I_{\texttt{right}} = ([x, b], \rho_{\texttt{right}}, I, *, *)$\\
		$I \leftarrow ([a, b], \sigma, I_{\texttt{parent}}, I_{\texttt{left}}, I_{\texttt{right}})$\\
		return\\
	}
	\caption{Definition of \texttt{bisect}}\label{alg:numerical:bisect}
\end{algorithm}

\begin{algorithm}
	\SetKwProg{Def}{def}{:}{}
	\SetAlgoVlined
	
	\Def{\emph{\texttt{sample}($I$ : \textit{Node}, $\widehat{I}$ : \textit{Node})}}{
		\If{$I$ {\normalfont{is}} $\widehat{I}$}{
			Decompose $([a, b], \widehat{\rho}, *, \widehat{I}_{\mathrm{left}}, \widehat{I}_{\mathrm{left}}) = \widehat{I}$\\
			return $\normal{0}{T}$ sampled with seed $\widehat{\rho}$.
		}
		Decompose $([a, b], \rho, I_{\texttt{parent}}, I_{\texttt{left}}, I_{\texttt{right}}) = I$\\
		Decompose $([a_p, b_p], \rho_p, I_{pp}, I_{lp}, I_{rp}) = I_{\texttt{parent}}$\\
		$w_{\texttt{parent}}$ = \texttt{sample}($I_\texttt{parent}$)\\
		\eIf{$I$ \normalfont{is} $I_{rp}$}{
			$w_\texttt{left} \sim \texttt{bridge}(a_p, b_p, a, w_{\texttt{parent}})$ sampled with seed $\rho$\\
			return $w_{\texttt{parent}} - w_\texttt{left}$
		}{
			\# $I$ \normalfont{is} $I_{lp}$\\
			return $\texttt{bridge}(a_p, b_p, b, w_{\texttt{parent}})$ sampled with seed $\rho$\\
		}
	}
	\quad\\
	
	\texttt{sample} = \texttt{LRUCache}(\texttt{sample})
	
	\caption{Definition of \texttt{sample}}\label{alg:numerical:sample}
\end{algorithm}

\subsection{Discussion}

There are some further technical considerations worth mentioning. Recall that the context we are explicitly considering is when sampling Brownian motion to solve an SDE forwards in time, then the adjoint backwards in time, and then discarding the Brownian motion. This motivates several of the choices here.

\paragraph{Small intervals} First, the access patterns of SDE solvers are quite specific. Queries will be over relatively small intervals: the step that the solver is making. This means that the list of nodes populated by \texttt{traverse} is typically small: usually only consisting of a single element; occasionally two.

In contrast if the Brownian Interval has built up a reasonable tree of previous queries, and was then queried over $[0, s]$ for $s \gg 0$, then a long (inefficient) list would be returned. It is the fact that SDE solvers do not make such queries that means this is acceptable.

\paragraph{Search hints: starting from $\widehat{J}$} Moreover, the queries are either just ahead (fixed-step solvers; accepted steps of adaptive-step solvers) or just before (rejected steps of adaptive-step solvers) previous queries. Thus in Algorithm \ref{alg:numerical:binterval-sampling}, we keep track of the most recent node $\widehat{J}$, so that we begin \texttt{traverse} near to the correct location. This ensures the modal time complexity of the search procedure is only $\bigO{1}$, and not $\bigO{\log(1/h)}$ in the average step size $h$, which for example would be the case if searching commenced from the root on every query.

\paragraph{LRU cache} The fact that queries are often close to one another is also what makes the strategy of using an LRU (least recently used) cache work. Most queries will correspond to a node that have a recently-computed parent in the cache.

\paragraph{Backward pass} The queries are broadly made left-to-right (on the forward pass), and then right-to-left (on the backward pass). (Other than the occasional rejected adaptive step.)

Left to its own devices, the forward pass will thus build up a highly imbalanced binary tree. At any one time, the LRU cache will contain only nodes whose intervals are a subset of some contiguous subinterval $[s, t]$ of the query space $[0, T]$. Letting $n$ be the number of queries on the forward pass, then this means that the backward pass will consume $\bigO{n^2}$ time -- each time the backward pass moves past $s$, then queries will miss the LRU cache, and a full recomputation to the root  will be triggered, costing $\bigO{n}$. This will then hold only nodes whose intervals are subsets of some contiguous subinterval $[u, s]$: once we move past $u$ then this $\bigO{n}$ procedure is repeated, $\bigO{n}$ times. This is clearly undesirable.

This is precisely analogous to the classical problem of optimal recomputation for performing backpropagation, whereby a dependency graph is constructed, certain values are checkpointed, and a minimal amount of recomputation is desired; see \cite{treeverse}.

In principle the same solution may be applied: apply a snapshotting procedure in which specific extra nodes are held in the cache. This is a perfectly acceptable solution, but implementing it requires some additional engineering effort, carefully determining which nodes to augment the cache with.

Fortunately, we have an advantage that \cite{treeverse} does not: we have some control over the dependency structure between the nodes, as we are free to prespecify any dependency structure we like. That is, we do not have to start the binary tree as just a stump. We may exploit this to produce an easier solution.

Let the size of the LRU cache be $L$ and let $\nu$ be some estimate of the average step size of the SDE solver (which may be fixed and known if using a fixed step size solver, or estimated from the first few steps if using an adaptive step size solver). Then \emph{before a user makes any further queries}, we simply make some queries of our own. These queries correspond to the intervals $[0, T/2], [T/2, T], [0, T/4], [T/4, T/2], \ldots$, so as to create a dyadic tree, such that the smallest intervals (the final ones in this sequence) are of size not more than $\nu L$. (In practice we use $\frac{4}{5}\nu L$ as an additional safety factor.)

Letting $[s, t]$ be some interval at the bottom of this dyadic tree, where $t \approx s + \frac{4}{5} \nu L$, then we are capable of holding every node within this interval in the LRU cache. Once we move past $s$ on the backward pass, then we may in turn hold the entire previous subinterval $[u, s]$ in the LRU cache, and in particular the values of the nodes whose intervals lie within $[u, s]$ may be computed in only logarithmic time, due to the dyadic tree structure.

This is now analogous to the Virtual Brownian Tree of \cite{gaines, scalable-sde}. (Up to the use of intervals rather than points.) If desired, this approach may be loosely interpreted as placing a Brownian Interval on every leaf of a shallow Virtual Brownian Tree.

\paragraph{Recursion errors} We find that for some problems, the recursive computations of \texttt{traverse} (and in principle also \texttt{sample}, but this is less of an issue due to the LRU cache) can occasionally grow very deep. In particular this occurs when crossing the midpoint of the pre-specified tree: for this particular query, the traversal must ascend the tree to the root, and then descend all the way down again. As such \texttt{traverse} should be implemented with trampolining and/or tail recursion to avoid maximum depth recursion errors.

\paragraph{CPU vs GPU memory} We describe this algorithm as requiring only constant memory. To be more precise, the algorithm requires only constant GPU memory, corresponding to the fixed size of the LRU cache. As the Brownian Interval receives queries then its internal tree tracking dependencies will grow, and CPU memory will increase. For deep learning models, GPU memory is usually the limiting (and so more relevant) factor.
\chapter{Experimental Details}\label{appendix:experimental}
\section{Continuous normalising flows on images}\label{appendix:experimental:cnf}
This appendix provides the details for the example of Section \ref{section:ode:cnf-example}.

\textit{At time of writing, this experiment may be found implemented as an example in the Diffrax software package \cite{diffrax}.}

The experiment was implemented using the JAX, Equinox, Diffrax, and Optax software libraries \cite{jax2018github, equinox, diffrax, optax2020github}. (Providing autodifferentiation, neural networks, differential equation solvers, and optimisers respectively.)

The differential equation was solved from $t=0$ to $t=0.5$, with fixed timestep of size $0.05$, using the Tsitouras 5(4) solver \cite{tsit5}. Backpropagation was performed via discretise-then-optimise (Section \ref{section:numerical:adjoint-ode}).

Every operation was performed at 32-bit floating point precision.

The optimiser used was `AdamW' \cite{kingma2015, optax2020github}, with a batch size of 1000, a learning rate of $10^{-3}$, and a weight decay of $10^{-5}$. It was trained for 10\,000 steps. Each step took a couple of seconds on an A100 GPU.

The vector field was parameterised as an MLP acting on the state $y$, except that each affine transformation was replaced by the variant layer of Section \ref{section:ode:variant-layers}. (Which induces a time dependence.) The activation function was taken to be tanh.

For the `target' dataset, the width of each hidden layer was 128, and 3 hidden layers were used. 2 CNFs were stacked on top of each other to produce the overall transformation. (Equivalently, it was a single CNF with piecewise vector field, split into two pieces, as in Section \ref{section:ode:stacking}.)

For the `cat' dataset, the width of each hidden layer was 64, and 3 hidden layers were used. 2 CNFs were stacked on top of each other to produce the overall transformation.

For the `butterfly' dataset, the width of each hidden layer was 64, and 3 hidden layers were used. 3 CNFs were stacked on top of each other to produce the overall transformation.

Each dataset is normalised to have zero mean and unit variance.

Due to the low dimensionality, exact trace-Jacobian calculations were used (not the approximate scheme of Section \ref{section:ode:hutchinson}).

\section{Latent ODEs on decaying oscillators}\label{appendix:experimental:latent-odes}
This appendix provides the details for the example of Section \ref{section:ode:latent-ode-example}.

\textit{At time of writing, this experiment may be found implemented as an example in the Diffrax software package \cite{diffrax}.}

The experiment was implemented using the JAX, Equinox, Diffrax, and Optax software libraries \cite{jax2018github, equinox, diffrax, optax2020github}. (Providing autodifferentiation, neural networks, differential equation solvers, and optimisers respectively.)

A dataset of 10\,000 sample paths were produced, as solutions to the linear differential equation
\begin{equation*}
	\frac{\dd}{\dd t} \begin{bmatrix}y(t)\\z(t)\end{bmatrix} = \begin{bmatrix}-0.1 & 1.3\\ -1 & -0.1\end{bmatrix} \begin{bmatrix}y(t)\\z(t)\end{bmatrix}.
\end{equation*}
Correspondingly the data dimensionality is $d=2$.

The initial $y(0)$ was sampled from a two-dimensional standard normal distribution, independently for each sample path. The time interval solved over was taken to be $[0, T]$, where $T \sim \uniform{2}{3}$ independently for each sample path. Each sample path was observed at 20 time points independently sampled from $\uniform{0}{T}$.

The differential equation was solved over each $[0, T]$ at train time, and over the larger interval $[0, 12]$ at test time. The solver used was Dormand--Prince 5(4) \cite{dopri5}, with a fixed timestep of $0.4$. Backpropagation was performed via discretise-then-optimise (Section \ref{section:numerical:adjoint-ode}).

Every operation was performed at 32-bit floating point precision.

The optimiser used was Adam \cite{kingma2015}, with a batch size of 256, and a learning rate of $10^{-2}$. It was trained for 250 steps. Each step took about half a second on an A100 GPU.

The dimensionality of the evolving state $y$ is taken to be $d_l = 16$. The vector field $f_\theta$ was taken to be autonomous and of the form
\begin{equation*}
	y \mapsto \alpha \tanh(\mathrm{MLP}(y)),
\end{equation*}
where $\alpha \in \reals$ is a learnt parameter initialised at one, and $\mathrm{MLP}$ is an MLP of width 16, with 3 hidden layers using softplus activation functions.

The initial noise-to-$y(0)$ network $g_\theta$ is taken to be an MLP of width 16 and 3 hidden layers, using ReLU activation functions. The latent space is taken of dimension $d_m = 16$.

The probability $p_{\theta, y}$ is parameterised as $\normal{\ell_\theta(y)}{\eye{d_l}}$, where $\ell_\theta \colon \reals^{d_l} \to \reals^d$ is learnt and affine.

The encoder $\nu_\theta$ is parameterised as a single-layer GRU with hidden size 16. The final hidden state of size 16 is mapped into $\reals^{d_m} \times \reals^{d_m}$ by a learnt affine transformation, to produce the mean $\mu_{\theta, \mathbf{x}} \in \reals^{d_m}$ and the log-standard deviation $\log \sigma_{\theta, \mathbf{x}} \in \reals^{d_m}$.

\section{Neural CDEs on spirals}\label{appendix:cde:spiral-experiment}
This appendix provides the details for the example of Section \ref{section:cde:spiral-experiment}.

\textit{At time of writing, this experiment may be found implemented as an example in the Diffrax software package \cite{diffrax}.}

The experiment was implemented using the JAX, Equinox, Diffrax, and Optax software libraries \cite{jax2018github, equinox, diffrax, optax2020github}. (Providing autodifferentiation, neural networks, differential equation solvers, and optimisers respectively.)

The dataset is of size 256. Each time series consists of 100 regularly sampled points over the interval $[0, 4 \pi] \ni t$ of
\begin{equation*}
	\begin{bmatrix} y(t) \\ z(t) \end{bmatrix} = \exp\left(t \begin{bmatrix}-0.3 & 2\\ -2 & -0.3\end{bmatrix}\right) \begin{bmatrix} y_0 \\ z_0 \end{bmatrix},
\end{equation*}
where $(y_0, z_0) = (\cos \theta, \sin \theta)$ with $\theta \sim \uniform{0}{2 \pi}$. Half of the time series are then flipped in the $y$ axis so that the dataset consists of 128 clockwise and 128 counter-clockwise spirals.

The neural CDE was solved by reducing it to an ODE as in Section \ref{section:cde:evaluate}, and using the Tsitouras 5(4) solver \cite{tsit5}. The step size is selected adaptively, and the initial step size is selected automatically, as in \cite[Section II.4]{hairer}. Backpropagation was performed via discretise-then-optimise (Section \ref{section:numerical:adjoint-cde-sde}).

Every operation was performed at 32-bit floating point precision.

The optimiser used was Adam \cite{kingma2015}, with a batch size of 32, and a learning rate of $10^{-2}$. It was trained for just 20 steps. Each step took about 1.5 seconds on an A100 GPU.

The initial network $\zeta_\theta$ is parameterised as an MLP with a single hidden layer of width 128 and ReLU activation functions. The vector field $f_\theta$ is parameterised as an MLP with a single hidden layer of width 128 and softplus activation functions. The output of the MLP is passed through a tanh as discussed in Section \ref{section:cde:gating}. The evolving hidden state $y$ is taken to have $d_l = 8$ dimensions. The output of the model is given by applying a learnt affine transform $\ell_\theta \colon \reals^{d_l} \to \reals$, followed by a sigmoid to map the result into $(0, 1)$.

The interpolation scheme used is Hermite cubic splines with backward differences as discussed in Section \ref{section:cde:interpolation}.

The loss function used is binary cross-entropy.

The final model achieves 100\% (test) accuracy.

\section{Neural SDEs on time series}\label{appendix:experimental:sde}
This appendix provides details for the examples of Section \ref{section:sde:example}.

\subsection{Brownian motion}

\textit{At time of writing, this experiment may be found implemented as an example in the Diffrax software package \cite{diffrax}.}

The experiment was implemented using the JAX, Equinox, Diffrax, and Optax software libraries \cite{jax2018github, equinox, diffrax, optax2020github}. (Providing autodifferentiation, neural networks, differential equation solvers, and optimisers respectively.)

The dataset is of size 8192. Each sample is of $v + w(t)$, where $v \sim \uniform{-1}{1}$ and $w \colon [0, 10] \to \reals$ is a Brownian motion. Each time series consists of 11 regularly sampled points over the interval $[0, 10]$.

The neural SDE and neural CDE were both solved using the reversible Heun method (Section \ref{section:numerical:reversible-heun}), with unit step size. Backpropagation was performed via discretise-then-optimise\footnote{Which is in any case essentially equivalent to optimise-then-discretise when using a reversible solver.} (Section \ref{section:numerical:adjoint-cde-sde}).

Every operation was performed at 32-bit floating point precision.

The optimiser used -- for both generator and discriminator -- was RMSprop (which is similar to Adadelta, and used for simplicity as Optax does not provide a built-in Adadelta optimiser). The batch size was 1024. It was trained for 10\,000 steps. The generator and discriminator are trained via simultaneous gradient descent (rather than by alternating training steps for the generator and discriminator).

The initial network $\zeta_\phi$ in the generator used a learning rate of $2 \times 10^{-4}$. The other components of the generator ($\mu_\theta, \sigma_\theta, \alpha_\theta, \beta_\theta)$ used a learning rate of $2 \times 10^{-5}$. The initial network $\xi_\phi$ of the discriminator used a learning rate of $10^{-3}$. The other components of the generator ($f_\phi, g_\phi, m_\phi$) used a learning rate of $10^{-4}$.

All parameters (for both generator and discriminator) were initialised close to zero. In practice this was done by initialising them as per Equinox's default, and then multiplying every parameter by 0.01.

\begin{remark}
The discriminator uses a larger learning rate than the generator (by a factor of 5) as per \cite{two-timescale}. In brief: the discriminator must be able to `keep up' with the generator as it trains, so that it can always provide informative gradients. As such we may either train the discriminator for multiple steps for every step of the generator, or (since we use simultaneous gradient descent here) give the discriminator a larger learning rate.

The initial networks were taken to be use a larger learning rate as this was found to improve the speed at which they converged to the true distribution, without creating any instability.
\end{remark}

The initial networks $\zeta_\theta$ and $\xi_\phi$ were taken to be MLPs with a single hidden layer of width 16 and ReLU activation function. The vector fields $\mu_\theta$ and $\sigma_\theta$ were parameterised as
\begin{equation*}
(t, y) \mapsto \gamma \tanh(\mathrm{MLP}_\theta(t, y)),
\end{equation*}
where $\gamma$ is a learnt parameter randomly initialised from $\uniform{0.9}{1.1}$, and MLP$_\theta$ had a single hidden layer of width 16 with LipSwish activation function.\footnote{Which is not necessary for the generator -- just the discriminator, see Section \ref{section:sde:lipschitz} -- so this activation function was used just for simplicity.} The vector fields $f_\phi$ and $g_\phi$ were parameterised as
\begin{equation*}
	(t, y) \mapsto \tanh(\mathrm{MLP}_\theta(t, y)),
\end{equation*}
where MLP$_\theta$ had a single hidden layer of width 16 with LipSwish activation function.

The dimensionality of the initial noise was taken to be $d_v = 5$. The dimensionality of the Brownian motion was taken to be $d_w = 3$. The dimensionality of the evolving hidden state $y$ was taken to be $d_y = 16$.

Lipschitzness of the discriminator was maintained using careful clipping (Section \ref{section:sde:careful-clipping}). Both the real and the generated data were treated as time series and linearly interpolated before passing to the neural CDE (the discriminator).

The output of the discriminator was defined using the alternate formula $D = m_\phi \cdot h(0) + m_\phi \cdot h(T)$, to encourage better learning the initial distribution. (Contrast equation \eqref{eq:sde:discriminator} and see also Section \ref{section:sde:discriminator-parameterisation}.)

\subsection{Time-dependent Ornstein--Uhlenbeck process}
\textit{At time of writing, this experiment may be found implemented as an example in the Diffrax software package \cite{diffrax}.}

The dataset is taken to be samples from
\begin{equation*}
	y(0) \sim \uniform{-1}{1},\qquad \dd y(t) = (a t - b y) \,\dd t + ct \,\dd w(t),\qquad\text{for $t \in [0, 63]$},
\end{equation*}
with $a=0.02$, $b=0.1$, $c=0.013$. Samples were obtained from this equation by solving using the Euler--Maruyama method and a step size of 0.1. Each time series consisted of 64 regularly spaced points over the interval $[0, 63]$.

In all other respects this example is identical to the Brownian motion example discussed above.

\subsection{Damped harmonic oscillator}

This experiment was implemented using the PyTorch, torchdiffeq, and torchsde libraries \cite{pytorch, torchdiffeq, torchsde}. (Providing autodifferentiation, ordinary differential equation solvers, and stochastic differential equation solvers respectively.)

The dataset is of size 8192. Each sample is of
\begin{equation*}
	y_1(0), y_2(0) \sim \uniform{-1}{1},\qquad \dd \begin{bmatrix} y_1(t) \\ y_2(t) \end{bmatrix} = \begin{bmatrix} -0.01 & 0.13 \\ -0.1 & -0.01 \end{bmatrix} \begin{bmatrix} y_1(t) \\ y_2(t) \end{bmatrix} \, \dd t,
\end{equation*}
for $t \in [0, 100]$. Samples were obtained from this equation by solving using the Dormand--Prince 5(4) method with an adaptively chosen step size. Each time series consists of 101 regularly sampled points over the interval $[0, 100]$. Prior to training the dataset was normalised to have zero mean and unit standard deviation.

The auxiliary neural SDE was trained, and the neural SDE sampled, with the reversible Heun method (Section \ref{section:numerical:reversible-heun}) with unit step size. Backpropagation was performed via discretise-then-optimise (Section \ref{section:numerical:adjoint-cde-sde}).

Every operation was performed at 32-bit floating point precision.

The optimiser used was Adam \cite{kingma2015}, with a batch size of 1024, and trained over 20\,000 steps. The initial networks $\zeta_\theta$ and $\xi_\phi$ used a learning rate of $10^{-3}$. Every other component used a learning rate of $2 \times 10^{-4}$.

All parameters were initialised relatively small. In practice this was done by initialising them as per PyTorch's default, and then multiplying every parameter by 0.5. The parameters of the initial network $\zeta_\theta$ were instead multiplied by 0.25.

The auxiliary network $\nu_\phi$ was parameterised as $\nu_\phi(t, y, x) = \nu_{\phi, 1}(t, y, \nu_{\phi, 2}(\restr{x}{[t, T]}))$, where $\nu_{\phi, 1}$ was an MLP and $\nu_{\phi, 2}$ is the evaluation function $\nu_{\phi, 2}(\restr{x}{[t, T]}) = x(t)$. (Rounded to the nearest discrete timestamp.)

Every neural network was parameterised as an MLP with a single hidden layer of width 32 and LipSwish activation function. The vector fields $\mu_\theta$ and $\nu_{\phi, 1}$ additionally had a final tanh nonlinearity. The vector field $\sigma_\theta$ produced a diagonal matrix as its output, and additionally had a $z \mapsto \sigmoid(z) + 10^{-4}$ nonlinearity. (A simple way to avoid the numerical issues with the diffusion, as discussed in Section \ref{section:sde:diffusion-param}.)

The standard deviation outputted by the encoder $\xi_\phi$ (nominally in $(0, \infty)$) was performed by outputting a log-standard deviation (valued in $\reals$), and then clipping the log-standard deviation to the range $(-10, 10)$, to promote better numerical stability.

The dimensionality of the initial noise was taken to be $d_v = 2$. The dimensionality of the Brownian motion was taken to be $d_w = 2$. The dimensionality of the evolving hidden state $y$ was taken to be $d_y = 64$.

\subsection{Lorenz attractor}
This example considered samples from the Lorenz attractor
\begin{align*}
	&y \sim \normal{0}{\eye{3}},\\
	\dd y_1(t) &= a_1 (y_2(t) - y_1(t)) \,\dd t + b_1 y_1(t) \dd w(t),\\
	\dd y_2(t) &= (a_2 y_1(t) - y_1(t)y_3(t) )\,\dd t + b_2 y_2(t) \dd w(t),\\
	\dd y_3(t) &= (y_1(t) y_2(t) - a_3 y_3(t)) \,\dd t + b_3 y_3(t) \dd w(t),
\end{align*}
for $t \in [0, 2]$. We take specifically $a_1 = 10$, $a_2 = 28$, $a_3 = \frac{8}{3}$, $b_1 = 0.1$, $b_2 = 0.28$, $b_3 = 0.3$. Samples were obtained from this equation using Milstein's method and a step size of 0.1. Each time series consisted of 100 regularly spaced points over the interval $[0, 2]$.

Except as now otherwise stated, this was otherwise identical to the damped harmonic oscillator just discussed.

The latent SDE component continued to use the Adam optimiser. When training as an SDE-GAN, the Adadelta optimiser was used.\footnote{So that the components of the auxiliary neural SDE used in the latent SDE were associated only with an Adam optimiser, the discriminator was associated only with an Adadelta optimiser, and the generator was associated with both an Adam and an Adadelta optimiser independently of each other. This is arguably a little questionable -- the statistics tracked in the Adam and the Adadelta optimisers no longer do exactly what is expected -- but we found that training a latent SDE with Adadelta or training an SDE-GAN with Adam seemed to fail.}

The initial network $\xi_\phi$ of the discriminator (not to be confused with the same notation $\xi_\phi$ also being used in the latent SDE) was parameterised as an MLP with a single hidden layer of width 32 and LipSwish activation function. The vector fields $f_\phi$ and $g_\phi$ were parameterised as
\begin{equation*}
	(t, y) \mapsto \tanh(\mathrm{MLP}_\theta(t, y)),
\end{equation*}
where MLP$_\theta$ had a single hidden layer of width 32 with LipSwish activation function.

The dimensionality of the discriminator hidden state was taken to be $d_h = 64$.

Lipschitzness of the discriminator was maintained using careful clipping (Section \ref{section:sde:careful-clipping}). Both the real and the generated data were treated as time series and linearly interpolated before passing to the neural CDE (the discriminator).

The neural CDE of the discriminator was solved using the reversible Heun method (Section \ref{section:numerical:reversible-heun}).

The output of the discriminator was defined using the alternate formula $D = \kappa_\phi(x(0)) + m_\phi \cdot h(T)$, to encourage better learning the initial distribution, where $\kappa_\phi$ is an MLP mapping $\reals^{d_x} \to \reals$, parameterised with a single hidden layer of width 32 and LipSwish activation function. (Contrast equation \eqref{eq:sde:discriminator} and see also Section \ref{section:sde:discriminator-parameterisation}.)

\section{Symbolic regression on a nonlinear oscillator}\label{appendix:experimental:symbolic}
This appendix provides details for the example of Section \ref{section:misc:symbolic-example}.

\textit{At time of writing, this experiment may be found implemented as an example in the Diffrax software package \cite{diffrax}.}

The experiment was implemented using the JAX, Equinox, Diffrax, Optax, and PySR software libraries \cite{jax2018github, equinox, diffrax, optax2020github, pysr}. (Providing autodifferentiation, neural networks, differential equation solvers, gradient-based optimisers and regularised evolution algorithms respectively.)

The dataset is of size 256. Each time series consists of 100 regularly sampled points over the interval $[0, 10] \ni t$ of
\begin{equation*}
	\frac{\dd}{\dd t}\begin{bmatrix}x\\y\end{bmatrix}(t) = \begin{bmatrix}\frac{y(t)}{1 + y(t)}\\\frac{-x(t)}{1 + x(t)}\end{bmatrix} \quad\text{for $t \in [0, T]$},
\end{equation*}
where $x(0), y(0) \sim \uniform{-0.6}{1}$.

The neural ODE was solved using the Tsitouras 5(4) solver \cite{tsit5} with a fixed step size of 0.1.

Every operation was performed at 32-bit floating point precision.

The gradient-based optimiser used was AdaBelief \cite{zhuang2020adabelief}, with a batch size of 32, and a learning rate of $3 \times 10^{-3}$. It was trained for 5000 steps. Each step took about 0.9 seconds on an A100 GPU.

The first 500 steps of gradient-based optimisation were performed on only the first 10 sample points of each time series (so that approximately the interval $[0, 1]$ was considered instead). This helps to avoid local minima during training.

The neural vector field was parameterised as an MLP with two hidden layers, each of width 64.

The loss function used was $L^2$. The final loss was of order $\bigO{10^{-5}}$.
	
Symbolic regression was performed by flattening all observations together into a single dataset of size 25600 and randomly selecting some 2000 samples. We then performed regularised evolution on this dataset. Regularised evolution comes with numerous hyperparameters; unless otherwise specified the PySR version 0.6.13 defaults were used. We used 100 populations each of size 20. 10 rounds of optimisation were performed; 100 mutations were performed in each round and between each round equations migrated between populations. Constants in each expression were optimised using 50 steps of BFGS.

Symbolic regression produces a Pareto front of equations, trading off loss against complexity. Each equation on the Pareto front was fine-tuned using full-batch gradient descent (possibly superfluously, given the earlier use of BFGS) and a learning rate of $3 \times 10^{-4}$. The best equation was then selected as being the one minimising
\begin{equation*}
	\log_2(\text{loss}) + \text{complexity},
\end{equation*}
where `loss' is the $L^2$ loss when regressing $f_\theta(x(t), y(t))$ against $(x(t), y(t))$, and `complexity' is the number of symbols in the symbolic expression: for example $c$ is of size one, $c + x$ is of size three, and $a \times x + b \times y$ is of size seven. The use of base two in the logarithm serves as a quantitative measure of trading off loss against complexity: the use of an extra symbol must halve the loss if is to produce a `better' expression.

The constants of the symbolic expression are then optimised by gradient descent, by plugging it back into the original (neural) optimisation problem. The gradient-based optimiser used was full-batch Adam, with a learning rate of $3 \times 10^{-3}$. It was trained for 500 steps.

Finally, the constants are rounded to the nearest multiple of 0.01.

\section{Neural RDEs on BIDMC}\label{appendix:rde:experiment-details}
This appendix provides details for the example of Appendix \ref{appendix:rde:example}.

This experiment was implemented using the PyTorch, torchdiffeq, torchcde, and Signatory libraries \cite{pytorch, torchdiffeq, torchcde, signatory}. (Providing autodifferentiation, differential equation solvers, and logsignature computations respectively.)

(This experiment predates the creation of the Diffrax software library, and in any case and at time of writing there does not exist a JAX library for computing logsignatures.)

The dataset is split into a 70\%/15\%/15\% train/validation/test split. Each time series consists of 4\,000 points sampled at 125 Hertz. Each channel is normalised to have zero mean and unit variance.

The RDE was solved by reducing it to an ODE and using the RK4 with 3/8 rule solver. The step size was fixed, and was equal to the `step' hyperparameter (over either 8, 128 or 512 data points at once).

Every operation was performed at 32-bit floating point precision.

A batch size of 512 and a learning rate of $6.25\times 10^{-5}$ was used. If the validation loss failed to decrease over 15 epochs then the learning rate was reduce by a factor of 10. If the validation loss failed to improve over 60 epochs then training was terminated, and the model rolled back to the point at which it achieved the best validation loss.

Every neural network is parameterised as an MLP with three hidden layers of width 192 and ReLU activation functions. The evolving hidden state $y$ is taken to have $d_l = 64$ dimensions. These were selected as a result of hyperparameter optimisation for the baseline (competing) neural CDE model.

The interpolation scheme used in the data space (necessarily) linear interpolation. This is necessary as the only choice for which the logsignature can be computed efficiently. In addition linear interpolation was used to interpolate the sequence of logsignatures.

The loss function used is $L^2$ loss.

See also \cite[Appendix C]{morrill2021neuralrough} for details of this experiment.
\pagestyle{plain}
\printbibliography[heading=bibintoc, title={Bibliography}]
\chapter*{Notation}\addcontentsline{toc}{chapter}{Notation}
\renewcommand{\arraystretch}{1.3}
\begin{tabularx}{\textwidth}{lX}
	$\reals^{d_1 \times d_2}$ & The space of real matrices with $d_1$ rows and $d_2$ columns.\\
	$\reals^{d_1 \times \cdots \times d_k}$ & The space of tensors of shape $(d_1, \ldots, d_k)$.\\
	$\eye{d}$ & The $d \times d$ identity matrix.\\
	$\trace(A)$ & The trace of a square matrix $A \in \reals^{d \times d}$.\\
	$\diag(v)$ & The diagonal matrix in $\reals^{d \times d}$ whose diagonal entries are given by the vector $v \in \reals^d$.\\
	$\realpart(\lambda)$ & The real part of a complex number $\lambda \in \complexes$.\\
	$L(X; Y)$ & The space of linear (not affine) functions $X \to Y$. If $Y$ is omitted then $Y = \reals$.\\
	$\affine(X; Y)$ & The space of affine (not linear) functions $X \to Y$. If $Y$ is omitted then $Y = \reals$.\\
	$C(X; Y)$ & The space of continuous functions $X \to Y$. (With respect to some topologies on $X$ and $Y$.) If $Y$ is omitted then $Y = \reals$.\\
	$L^p(X; Y)$ & The space of $p$-integrable functions $X \to Y$. If $Y$ is omitted then $Y = \reals$.\\
	$W^{1, p}(X; Y)$ & The space of $p$-integrable functions $X \to Y$ with $p$-integrable first derivative. If $Y$ is omitted then $Y = \reals$.\\
	$\BV(X; Y)$ & The space of (possibly discontinuous) functions $X \to Y$ with bounded variation. If $Y$ is omitted then $Y = \reals$.\\
	$\Lip(X; Y)$ & The space of Lipschitz functions $X \to Y$. If $Y$ is omitted then $Y = \reals$.\\
	$a \mapsto b$ & The function mapping $a$ to $b$. That is, $a \mapsto b$ denotes the function $f$ such that $f(a) = b$. (Sometimes referred to as an `anonymous' or `lambda' function.)\\
	$\norm{\,\cdot\,}_p$ & The $L^p$ norm.\\
	$\abs{\,\cdot\,}_{\BV}$ & The bounded variation seminorm.\\
	$\indicator{A}$ & The indicator function with value 1 when the condition $A$ is true, and value 0 when $A$ is false.\\
	$\sim$ & Denotes sampling from a probability distribution: $x \sim \mu$ denotes a sample $x$ from probability distribution $\mu$.\\
	$\normal{\mu}{\sigma^2}$ & Normal distribution with mean $\mu$ and variance $\sigma^2$.\\
	$\uniform{a}{b}$ & Uniform distribution over $[a, b] \subseteq \reals$.\\
	$\prob$ & A probability measure; the probability of an event occurring or a statement being true.\\
	$\kl{\mathbb{P}}{\mathbb{Q}}$ & The Kullback--Leibler divergence between two probability measures $\mathbb{P}$ and $\mathbb{Q}$. Will also be written $\kl{X}{Y}$ to denote the KL divergence between the laws of two random variables $X$,$Y$, or $\kl{p}{q}$ to denote the KL divergence between the two probability measures corresponding to the densities $p$, $q$.\\
	$\int f(z(s)) \,\dd X(s)$ & A Riemann--Stieltjes integral driven by $X$. `$f\,\dd X$' refers to a matrix-vector product.\\
	$\mathop\circ \dd w(t)$ & Used to denote integration as a Stratonovich SDEs. (As opposed to just `$\dd w(t)$' denoting an It{\^o} SDE.)\\
	$\otimes$ &\begin{minipage}[t]{\linewidth}The tensor product, defined by
	\begin{equation*}
		\,\otimes \colon \reals^{d_1 \times \cdots \times d_k} \times \reals^{d_{k+1} \times \cdots \times d_m} \to \reals^{d_1 \times \cdots \times d_m},
	\end{equation*}
	\begin{equation*}
		\otimes \colon (a_{i_1, \ldots, i_k}, b_{i_{k+1}, \ldots, i_m}) \to a_{i_1, \ldots, i_k} b_{i_{k+1}, \ldots, i_m}.
	\end{equation*}
	When applied to two vectors, this is the outer product.\end{minipage}
\end{tabularx}

In addition, following the convention of the machine learning literature (and with apologies to the mathematicians), we sometimes use `$\max$' and `$\min$' where `$\sup$' and `$\inf$' would be technically correct.

\chapter*{Abbreviations}\addcontentsline{toc}{chapter}{Abbreviations}
Ordered alphabetically:

\begin{tabularx}{\textwidth}{ll}
	API & Application programming interface\\
	CDE & Controlled differential equation\\
	CNF & Continuous normalising flows\\
	CNN & Convolutional neural network\\
	DAG & Directed acyclic graph\\
	DEQ & Deep equilibrium model\\
	GAN & Generative adversarial network\\
	GRU & Gated recurrent unit\\
	JAX & JAX (not an abbreviation)\\
	KL & Kullback--Leibler (divergence)\\
	LASSO & Least absolute shrinkage and selection operator\\
	LRU & Least recently used (a form of caching)\\
	LSTM & Long-short term memory\\
	MLP & Multi-layer perceptron; feedforward neural network\\
	MMD & Maximum mean discrepancy\\
	NDE & Neural differential equation\\
	NCDE & Neural controlled differential equation\\
	NODE & Neural ordinary differential equation\\
	NRDE & Neural rough differential equation\\
	NSDE & Neural stochastic differential equation\\
	ODE & Ordinary differential equation\\
	PRNG & Pseudo-random number generator\\
	RDE & Rough differential equation\\
	RL & Reinforcement learning\\
	RMS & Root-mean-squared\\
	RNN & Recurrent neural network\\
	SDE & Stochastic differential equation\\
	SGD & Stochastic gradient descent\\
	SINDy & Sparse identification of nonlinear dynamics\\
	UDE & Universal differential equation\\
	&\hspace{\textwidth}
\end{tabularx}
\renewcommand{\arraystretch}{1}
\printindex

\end{document}